\documentclass{article}

\usepackage[final]{neurips_2023}

\usepackage{microtype}
\usepackage{graphicx}
\usepackage{booktabs} 

\usepackage{hyperref}

\usepackage{amsmath}
\usepackage{amssymb}
\usepackage{mathtools}
\usepackage{amsthm}

\usepackage{multirow}
\usepackage{layout}

\usepackage{amsmath}
\usepackage{amsthm}
\usepackage{amssymb}
\usepackage{amsfonts}
\usepackage{mathtools}
\usepackage{siunitx}

\usepackage{bbm}

\usepackage{url}
\usepackage{color}
\usepackage{graphicx}
\usepackage{algorithmic}
\usepackage{algorithm}
\usepackage{verbatim}
\usepackage{subcaption}

\usepackage{apptools}
\usepackage{thmtools}

\newcommand{\norm}[1]{\left\| #1 \right\|}

\newcommand{\inp}[2]{\left\langle#1,#2\right\rangle} 

\newcommand{\cC}{\mathcal{C}}

\newcommand{\cO}{\mathcal{O}}

\newcommand{\del}[1]{}

\newcommand{\R}{\mathbb{R}} 

\newcommand{\eqdef}{:=} 

\newcommand{\Prob}{\mathbf{Prob}} 

\newcommand{\Exp}[1]{{\rm E}\left[#1\right]}

\newcommand{\ExpSub}[2]{{\rm E}_{#1}\left[#2\right]}

\usepackage{tcolorbox}
\usepackage{pifont}
\definecolor{mydarkgreen}{RGB}{39,130,67}
\definecolor{mydarkred}{RGB}{192,47,25}
\newcommand{\green}{\color{mydarkgreen}}
\newcommand{\red}{\color{mydarkred}}
\newcommand{\cmark}{\green\ding{51}}%
\newcommand{\xmark}{\red\ding{55}}%

\newcommand{\algname}[1]{{\green\small \sf #1}}
\newcommand{\algnamesmall}[1]{{\green\scriptsize \sf #1}}

\newtheorem{assumption}{Assumption}
\newtheorem{lemma}{Lemma}

\newtheorem{theorem}{Theorem}

\theoremstyle{plain}

\theoremstyle{definition}

\newtheorem{definition}[theorem]{Definition}

\newcommand{\alglinelabel}{%
  \addtocounter{ALC@line}{-1}
  \refstepcounter{ALC@line}
  \label
}
\usepackage{footnote}
\usepackage{mathtools}
\usepackage[flushleft]{threeparttable}
\usepackage{etoolbox}
\usepackage{nicefrac}       
\usepackage{microtype}      
\usepackage{xcolor}         
\usepackage{enumitem}
\usepackage{tabularx}

\newcommand{\algorithmname}{DASHA-PP}

\newcommand*{\probavailable}{\ensuremath{p_{\textnormal{a}}}}
\newcommand*{\probpairaa}{\ensuremath{p_{\textnormal{aa}}}}
\newcommand*{\probpairan}{\ensuremath{p_{\textnormal{an}}}}
\newcommand*{\probpairnn}{\ensuremath{p_{\textnormal{nn}}}}

\newcommand*{\probpage}{\ensuremath{p_{\text{page}}}}
\newcommand*{\probmega}{\ensuremath{p_{\text{mega}}}}

\allowdisplaybreaks
\makeatletter
\newcommand{\vast}{\bBigg@{4}}
\newcommand{\Vast}{\bBigg@{5}}
\makeatother

\hypersetup{
  colorlinks   = true, 
  urlcolor     = blue, 
  linkcolor    = {red!75!black}, 
  citecolor   = {blue!50!black} 
}

\usepackage[capitalize,noabbrev]{cleveref}

\usepackage[textsize=tiny]{todonotes}

\title{A Computation and Communication Efficient Method for Distributed Nonconvex Problems \\ in the Partial Participation Setting}

\author{%
  Alexander Tyurin\\
  KAUST\\
  Saudi Arabia\\
  \texttt{alexandertiurin@gmail.com} \\
  \And
  Peter Richt\'{a}rik \\
  KAUST\\
  Saudi Arabia\\
  \texttt{richtarik@gmail.com} \\
}

\begin{document}

\maketitle

\begin{abstract}
  We present a new method that includes three key components of distributed optimization and federated learning: variance reduction of stochastic gradients, partial participation, and compressed communication. We prove that the new method has optimal oracle complexity and state-of-the-art communication complexity in the partial participation setting. Regardless of the communication compression feature, our method successfully combines variance reduction and partial participation: we get the optimal oracle complexity, never need the participation of all nodes, and do not require the bounded gradients (dissimilarity) assumption.
\end{abstract}

\section{Introduction}
Federated and distributed learning have become very popular in recent years \citep{konevcny2016federated, mcmahan2017communication}. 
The current optimization tasks require much computational resources and machines. Such requirements emerge in machine learning, where massive datasets and computations are distributed between cluster nodes \citep{lin2017deep, ramesh2021zero}. In federated learning, nodes, represented by mobile phones, laptops, and desktops, do not send their data to a server due to privacy and their huge number \citep{ramaswamy2019federated}, and the server remotely orchestrates the nodes and communicates with them to solve an optimization problem.

As in classical optimization tasks, one of the main current challenges is to find \textbf{computationally efficient} optimization algorithms. However, the nature of distributed problems induces many other \citep{kairouz2021advances}, including i) \textbf{partial participation} of nodes in algorithm steps: due to stragglers \citep{li2020federated} or communication delays \citep{vogels2021relaysum}, ii) \textbf{communication bottleneck}: even if a node participates, it can be costly to transmit information to a server or other nodes \citep{alistarh2017qsgd, ramesh2021zero,kairouz2021advances, sapio2019scaling, narayanan2019pipedream}. It is necessary to develop a method that considers these problems.

\section{Optimization Problem}
\label{sec:opt_problem}
Let us consider the nonconvex distributed optimization problem
\begin{align}
    \label{eq:main_problem} 
    \min \limits_{x \in \R^d}\left\{f(x) \eqdef \frac{1}{n}\sum \limits_{i=1}^n f_i(x)\right\},
\end{align}
where $f_i\,:\,\R^d \rightarrow \R$ is a smooth nonconvex function for all $i \in [n] \eqdef \{1, \dots, n\}.$ The full information about function $f_i$ is stored on $i$\textsuperscript{th} node. The communication between nodes is maintained in the parameters server fashion \citep{kairouz2021advances}: we have a server that receives compressed information from nodes, updates a state, and broadcasts an updated model.\footnote{Note that this strategy can be used in peer-to-peer communication, assuming that the server is an abstraction and all its algorithmic steps are performed on each node.} Since we work in the nonconvex world, our goal is to find an $\varepsilon$-solution ($\varepsilon$-stationary point) of \eqref{eq:main_problem}: a (possibly random) point $\widehat{x}\in \R^d$, such that ${\rm E}\big[\norm{\nabla f(\widehat{x})}^2\big] \leq \varepsilon.$

We consider three settings: \\
1. \textbf{Gradient Setting.}
The $i$\textsuperscript{th} node has only access to the gradient $\nabla f_i \,:\,\R^d \rightarrow \R^d$ of function $f_i$. Moreover, the following assumptions for the functions $f_i$ hold.
\begin{assumption}
    \label{ass:lower_bound}
    There exists $f^* \in \R$ such that $f(x) \geq f^*$ for all $x \in \R$.
\end{assumption}
\begin{assumption}
    \label{ass:lipschitz_constant}
    The function $f$ is $L$--smooth, i.e., $\norm{\nabla f(x) - \nabla f(y)} \leq L \norm{x - y}$ for all $x, y \in \R^d.$
\end{assumption}
\begin{assumption} \leavevmode
    \label{ass:nodes_lipschitz_constant}
    The functions $f_i$ are $L_i$--smooth for all $i \in [n]$. Let us define $\widehat{L}^2 \eqdef \frac{1}{n} \sum_{i=1}^{n} L_i^2.$\footnote{Note that $L \leq \widehat{L},$ $\widehat{L} \leq L_{\max},$ and $\widehat{L} \leq L_{\sigma}.$}
\end{assumption}
2. \textbf{Finite-Sum Setting.}
The functions $\{f_i\}_{i=1}^n$ have the finite-sum form
\begin{align}
    \label{eq:task_minibatch} f_i(x) = \frac{1}{m}\sum \limits_{j=1}^m f_{ij}(x), \qquad \forall i \in [n],
\end{align}
where $f_{ij} :\R^d  \rightarrow \R$ is a smooth nonconvex  function for all $j \in [m].$

\begin{table*}
  \caption{\small Summary of methods that solve the problem \eqref{eq:main_problem} in the stochastic setting \eqref{eq:task_staochastic}. Abbr.: \emph{VR} (Variance Reduction) = Does a method have the optimal oracle complexity $\cO\left(\frac{\sigma^2}{\varepsilon} + \frac{\sigma}{\varepsilon^{\nicefrac{3}{2}}}\right)$?
  \emph{PP} (Partial Participation) = Does a method support partial participation from Section~\ref{sec:partial_participation}? \emph{CC} = Does a method have the communication complexity equals to $\cO\left(\frac{\omega}{\sqrt{n} \varepsilon}\right)$ \iffalse w.r.t. the dimension $d$, the number of nodes $n,$ and accuracy $\varepsilon$ \fi ?}
  \label{table:comparison}
  \centering 
  \scriptsize
  \begin{threeparttable}
    \begin{tabular}{ccccc}
      \midrule
    Method & \emph{VR} & 
    \emph{PP} & \emph{CC} & Limitations \\
     \midrule\midrule
     \begin{tabular}{@{}c@{}}\algnamesmall{SPIDER}, \algnamesmall{SARAH}, \algnamesmall{PAGE}, \algnamesmall{STORM} \\ \citep{SPIDER, SARAH} \\ \citep{PAGE, cutkosky2019momentum}\end{tabular} & \cmark & \xmark & \xmark & ---\\
     \midrule
     \begin{tabular}{@{}c@{}}\algnamesmall{MARINA} \\ \citep{MARINA} \end{tabular} & \cmark & \xmark\textsuperscript{{\color{blue}(a)}} & \cmark\textsuperscript{{\color{blue}(b)}} & \begin{tabular}{@{}c@{}}Suboptimal convergence rate \\ (see \citep{tyurin2022dasha}). \end{tabular}\\
     \midrule
     \begin{tabular}{@{}c@{}}\algnamesmall{FedPAGE} \\ \citep{zhao2021fedpage} \end{tabular} & \xmark & \xmark\textsuperscript{{\color{blue}(a)}} & \xmark & Suboptimal oracle complexity $\cO\left(\frac{\sigma^2}{\varepsilon^2}\right)$. \\
     \midrule
     \begin{tabular}{@{}c@{}}\algnamesmall{FRECON} \\ \citep{zhao2021faster} \end{tabular} & \xmark & \cmark & \cmark & --- \\
     \midrule
     \begin{tabular}{@{}c@{}}\algnamesmall{FedAvg} \\ \citep{mcmahan2017communication, karimireddy2020scaffold} \end{tabular} & \xmark & \cmark & \xmark & Bounded gradients (dissimilarity) assumption of $f_i$.\\
     \midrule
     \begin{tabular}{@{}c@{}}\algnamesmall{SCAFFOLD} \\ \citep{karimireddy2020scaffold} \end{tabular} & \xmark & \cmark & \xmark & Suboptimal convergence rate\textsuperscript{{\color{blue}(e)}}. \\
     \midrule
     \begin{tabular}{@{}c@{}}\algnamesmall{MIME}\textsuperscript{{\color{blue}(c)}} \\ \citep{karimireddy2020mime} \end{tabular} & \xmark\textsuperscript{{\color{blue}(d)}} & \cmark & \xmark & \begin{tabular}{@{}c@{}}Calculates full gradient. \\ Bounded gradients (dissimilarity) assumption of $f_i$. \\ Suboptimal oracle compl. $\cO\left(\nicefrac{1}{\varepsilon^{3/2}}\right)$ in the setting \eqref{eq:task_minibatch}.\end{tabular} \\
     \midrule
     \begin{tabular}{@{}c@{}}\algnamesmall{CE-LSGD} (for Partial Participation)\textsuperscript{{\color{blue}(c)}} \\ \citep{pateltowards} (concurrent work) \end{tabular} & \cmark & \cmark & \xmark & \begin{tabular}{@{}c@{}}Bounded gradients (dissimilarity) assumption of $f_i$. \\ Suboptimal oracle compl. $\cO\left(\nicefrac{1}{\varepsilon^{3/2}}\right)$ in the setting \eqref{eq:task_minibatch}.\end{tabular} \\
     \midrule
     \begin{tabular}{@{}c@{}}\algnamesmall{DASHA} \\ \citep{tyurin2022dasha} \end{tabular} & \begin{tabular}{@{}c@{}}\cmark \\ \phantom{or} \\ \xmark \end{tabular} & \begin{tabular}{@{}c@{}}\xmark \\ or\\ \cmark \end{tabular} & \begin{tabular}{@{}c@{}}\cmark \\ \phantom{or}\\ \cmark \end{tabular} & --- \\
     \midrule
     \begin{tabular}{@{}c@{}}\algnamesmall{DASHA-PP} \\ (new) \end{tabular} & \cmark & \cmark & \cmark & --- \\
    \midrule\midrule
    \end{tabular}
    \begin{tablenotes}
      \item [{\color{blue}(a)}] \algnamesmall{MARINA} and \algnamesmall{FedPAGE}, with a small probability, require the participation of all nodes so that they can not support partial participation from Section~\ref{sec:partial_participation}. Moreover, these methods provide suboptimal oracle complexities.
      \item [{\color{blue}(b)}] On average, \algnamesmall{MARINA} provides the compressed communication mechanism with complexity $\cO\left(\frac{\omega}{\sqrt{n} \varepsilon}\right).$ However, with a small probability, this method sends non-compressed vectors.
      \item [{\color{blue}(c)}] Note that \algnamesmall{MIME} and \algnamesmall{CE-LSGD} can not be directly compared with \algnamesmall{DASHA-PP} because \algnamesmall{MIME} and \algnamesmall{CE-LSGD} consider the online version of the problem \eqref{eq:main_problem}, and require more strict assumptions.
      \item [{\color{blue}(d)}] Although \algnamesmall{MIME} obtains the convergence rate $\cO\left(\frac{1}{\varepsilon^{3/2}}\right)$ of a variance reduced method, it requires the calculation of the full (exact) gradients.
      \item [{\color{blue}(e)}] It can be seen when $\sigma^2 = 0.$ Consider the $s$-nice sampling of the nodes, then \algnamesmall{SCAFFOLD} requires $\cO\left(\nicefrac{n^{3/2}}{\varepsilon s^{3/2}}\right)$ communication rounds to get $\varepsilon$-solution, while \algnamesmall{DASHA-PP} requires $\cO\left(\nicefrac{\sqrt{n}}{\varepsilon s}\right)$ communication rounds (see Theorem~\ref{theorem:stochastic} with $\omega = 0,$ $b = \nicefrac{\probavailable}{2 - \probavailable},$ and $\probavailable = \frac{s}{n}$).
    \end{tablenotes}
\end{threeparttable}
\end{table*}

We assume that Assumptions~\ref{ass:lower_bound}, \ref{ass:lipschitz_constant} and \ref{ass:nodes_lipschitz_constant} hold and the following assumption.
\begin{assumption}
  \label{ass:max_lipschitz_constant}
  The function $f_{ij}$ is $L_{ij}$-smooth for all $i \in [n], j \in [m].$ Let $L_{\max} \eqdef \max_{i \in [n], j \in [m]} L_{ij}.$
\end{assumption}
3. \textbf{Stochastic Setting.}
The function $f_i$ is an expectation of a stochastic function, 
\begin{align}
    \label{eq:task_staochastic}
    f_i(x) = \ExpSub{\xi}{f_i(x;\xi)}, \qquad \forall i \in [n],
\end{align}
where $f_i :\R^d \times \Omega_{\xi} \rightarrow \R.$ For a fixed $x \in \R,$ $f_i(x;\xi)$ is a random variable over some distribution $\mathcal{D}_i$,
and, for a fixed $\xi \in \Omega_{\xi},$ $f_i(x;\xi)$ is a smooth nonconvex function.
The $i$\textsuperscript{th} node has only access to a stochastic gradients $\nabla f_i(\cdot; \xi_{ij})$ 
of the function $f_i$ through the distribution $\mathcal{D}_i,$ where $\xi_{ij}$ is a sample from $\mathcal{D}_i.$
We assume that Assumptions~\ref{ass:lower_bound}, \ref{ass:lipschitz_constant} and \ref{ass:nodes_lipschitz_constant} hold and the following assumptions.
\begin{assumption}
  \label{ass:stochastic_unbiased_and_variance_bounded}
  For all $i \in [n]$ and for all $x \in \R^d,$ the stochastic gradient $\nabla f_i(x;\xi)$ is unbiased and has bounded variance, i.e., $\ExpSub{\xi}{\nabla f_i(x;\xi)} = \nabla f_i(x),$ and $\ExpSub{\xi}{\norm{\nabla f_i(x;\xi) - \nabla f_i(x)}^2} \leq \sigma^2,$ where $\sigma^2 \geq 0.$
\end{assumption}
\begin{assumption}
  \label{ass:mean_square_smoothness}
  For all $i \in [n]$ and for all $x, y \in \R,$ the stochastic gradient $\nabla f_i(x;\xi)$ satisfies the mean-squared smoothness property, i.e., $\ExpSub{\xi}{\norm{\nabla f_i(x;\xi) - \nabla f_i(y;\xi)}^2} \leq L_{\sigma}^2 \norm{x - y}^2.$
\end{assumption}
We compare algorithms using \textit{the oracle complexity}, i.e., the number of (stochastic) gradients that each node has to calculate to get $\varepsilon$-solution, and \textit{the communication complexity}, i.e., the number of bits that each node has to send to the server to get $\varepsilon$-solution.

\subsection{Unbiased Compressors}
We use the concept of unbiased compressors to alleviate the communication bottleneck. The unbiased compressors quantize and/or sparsify vectors that the nodes send to the server.
\begin{definition}
    \label{def:unbiased_compression}
    A stochastic mapping $\cC\,:\,\R^d \rightarrow \R^d$ is an \textit{unbiased compressor} if
    there exists $\omega \in \R$
    \begin{align}
        \label{eq:compressor}
        \textnormal{such that }\quad \Exp{\cC(x)} = x \quad \textnormal{ and } \quad \Exp{\norm{\cC(x) - x}^2} \leq \omega \norm{x}^2 \quad \forall x \in \R^d.
    \end{align}
\end{definition}
We denote a set of stochastic mappings that satisfy Definition~\ref{def:unbiased_compression} as $\mathbb{U}(\omega).$
In our methods, the nodes make use of unbiased compressors $\{\cC_i\}_{i=1}^n.$ 
The community developed a large number of unbiassed compressors, including Rand$K$ (see Definition~\ref{def:rand_k}) \citep{beznosikov2020biased, stich2018sparsified}, Adaptive sparsification \citep{wangni2018gradient} and Natural compression and dithering \citep{horvath2019natural}. We are aware of correlated compressors by \cite{szlendak2021permutation} and quantizers by \cite{suresh2022correlated} that help in the homogeneous regimes, but in this work, we are mainly concentrated on generic heterogeneous regimes, though, for simplicity, assume the independence of the compressors.
\begin{assumption}
\label{ass:compressors}
 $\cC_i \in \mathbb{U}(\omega)$ for all $i\in [n]$, and the compressors are \textit{statistically independent}.
\end{assumption}

\subsection{Nodes Partial Participation Assumptions}
\label{sec:partial_participation}

\newcommand\Item[1][]{%
  \ifx\relax#1\relax  \item \else \item[#1] \fi
  \abovedisplayskip=0pt\abovedisplayshortskip=0pt~\vspace*{-\baselineskip}}

We now try to formalize the notion of partial participation. Let us assume that we have $n$ events $\{i^{\textnormal{th}} \textnormal{ node is \textit{participating}}\}$ with the following properties.
\begin{assumption}
  \label{ass:partial_participation}
  The partial participation of nodes has the following distribution: exists constants $\probavailable \in (0, 1]$ and $\probpairaa \in [0, 1],$ such that
  \begin{enumerate}
    \Item \begin{align*}\Prob\left(i^{\textnormal{th}} \textnormal{ node is \textit{participating}}\right) = \probavailable \quad \forall i \in [n],\end{align*}
    \Item \begin{align*}\Prob\left(i^{\textnormal{th}} \textnormal{ and } j^{\textnormal{th}} \textnormal{ nodes are \textit{participating}}\right) = \probpairaa \quad \forall i \neq j \in [n].\end{align*}
    \Item \begin{align} \label{eq:partial_participation_constraint}
      \probpairaa \leq \probavailable^2,
    \end{align}
  \end{enumerate}
  and these events from different communication rounds are independent.
\end{assumption}

\begin{table*}
  \caption{\small Summary of methods that solve the problem \eqref{eq:main_problem} in the finite-sum setting \eqref{eq:task_minibatch}. Abbr.: \emph{VR} (Variance Reduction) = Does a method have the optimal oracle complexity $\cO\left(m + \frac{\sqrt{m}}{\varepsilon}\right)$? \emph{PP} and \emph{CC} are defined in Table~\ref{table:comparison}.}
  \label{table:comparison_finite}
  \centering 
  \scriptsize
  \begin{threeparttable}
    \begin{tabular}{ccccc}
      \midrule
    Method & \emph{VR} & 
    \emph{PP} & \emph{CC} & Limitations \\
     \midrule\midrule
     \begin{tabular}{@{}c@{}}\algnamesmall{SPIDER}, \algnamesmall{PAGE} \\ \citep{SPIDER, PAGE} \end{tabular} & \cmark & \xmark & \xmark & ---\\
     \midrule
     \begin{tabular}{@{}c@{}}\algnamesmall{MARINA} \\ \citep{MARINA} \end{tabular} & \cmark & \xmark\textsuperscript{{\color{blue}(a)}} & \cmark\textsuperscript{{\color{blue}(b)}} & \begin{tabular}{@{}c@{}}Suboptimal convergence rate \\ (see \citep{tyurin2022dasha}). \end{tabular}\\
     \midrule
     \begin{tabular}{@{}c@{}}\algnamesmall{ZeroSARAH} \\ \citep{li2021zerosarah} \end{tabular} & \cmark & \cmark & \xmark & Only homogeneous regime, i.e., the functions $f_i$ are equal. \\
     \midrule
     \begin{tabular}{@{}c@{}}\algnamesmall{FedPAGE} \\ \citep{zhao2021fedpage} \end{tabular} & \xmark & \xmark\textsuperscript{{\color{blue}(a)}} & \xmark & Suboptimal oracle complexity $\cO\left(\frac{m}{\varepsilon}\right)$. \\
     \midrule
     \begin{tabular}{@{}c@{}}\algnamesmall{DASHA} \\ \citep{tyurin2022dasha} \end{tabular} & \cmark & \xmark & \cmark & --- \\
     \midrule
     \begin{tabular}{@{}c@{}}\algnamesmall{DASHA-PP} \\ (new) \end{tabular} & \cmark & \cmark & \cmark & --- \\
    \midrule\midrule
    \end{tabular}
    \begin{tablenotes}
      \item [{\color{blue}(a)}, {\color{blue}(b)}]: see Table~\ref{table:comparison}.
    \end{tablenotes}
\end{threeparttable}
\end{table*}

We are not fighting for the full generality and believe that more complex sampling strategies can be considered in the analysis. For simplicity, we settle upon Assumption~\ref{ass:partial_participation}. Standard partial participation strategies, including $s$--nice sampling, where the server chooses uniformly $s$ nodes without replacement ($\probavailable = \nicefrac{s}{n}$ and $\probpairaa~=~\nicefrac{s(s-1)}{n(n-1)}$),
and independent participation, where each node independently participates with probability $\probavailable$ (due to independence, we have $\probpairaa = \probavailable^2$), satisfy Assumption~\ref{ass:partial_participation}. In the literature, $s$--nice sampling is one of the most popular strategies \citep{zhao2021faster, richtarik2021ef21, reddi2020adaptive, konevcny2016federated}. 

\section{Motivation and Related Work}
The main goal of our paper is to develop a method for the nonconvex distributed optimization that will include three key features: variance reduction of stochastic gradients, compressed communication, and partial participation. We now provide an overview of the literature (see also Table~\ref{table:comparison} and Table~\ref{table:comparison_finite}).

\textbf{1. Variance reduction of stochastic gradients}\\
It is important to consider finite-sum \eqref{eq:task_minibatch} and stochastic \eqref{eq:task_staochastic} settings because, in machine learning tasks, either the number of local functions $m$ is huge or the functions $f_i$ is an expectation of a stochastic function due to the batch normalization \citep{ioffe2015batch} or random augmentation \citep{goodfellow2016deep}, and it is infeasible to calculate the full gradients analytically. Let us recall the results from the nondistributed optimization. In the gradient setting, the optimal oracle complexity is $\cO\left(\nicefrac{1}{\varepsilon}\right)$, achieved by the vanilla gradient descent (\algname{GD}) \citep{carmon2020lower, nesterov2018lectures}. In the finite-sum setting and stochastic settings, the optimal oracle complexities are $\cO\left(m + \frac{\sqrt{m}}{\varepsilon}\right)$ and $\cO\left(\frac{\sigma^2}{\varepsilon} + \frac{\sigma}{\varepsilon^{\nicefrac{3}{2}}}\right)$ \citep{SPIDER, PAGE, arjevani2019lower}, accordingly, achieved by methods \algname{SPIDER}, \algname{SARAH}, \algname{PAGE}, and \algname{STORM} from \citep{SPIDER,SARAH,PAGE,cutkosky2019momentum}.

\textbf{2. Compressed communication} \\
In distributed optimization \citep{ramesh2021zero, xu2021grace}, lossy communication compression can be a powerful tool to increase the communication speed between the nodes and the server. Different types of compressors are considered in the literature, including unbiased compressors \citep{alistarh2017qsgd,beznosikov2020biased,szlendak2021permutation}, contractive (biased) compressors \citep{richtarik2021ef21}, 3PC compressors \citep{richtarik20223pc}. We will focus on unbiased compressors because methods \algname{DASHA} and \algname{MARINA} \citep{tyurin2022dasha, szlendak2021permutation, MARINA} that employ unbiased compressors provide the current theoretical state-of-the-art (SOTA) communication complexities.

Many methods analyzed optimization methods with the unbiased compressors \citep{alistarh2017qsgd, DIANA, horvath2019stochastic, MARINA, tyurin2022dasha}. In the gradient setting, the methods \algname{MARINA} and \algname{DASHA} by \cite{MARINA} and \cite{tyurin2022dasha} establish the current SOTA communication complexity, each method needs $\frac{1 + \nicefrac{\omega}{\sqrt{n}}}{\varepsilon}$ communication rounds to get an $\varepsilon$--solution. In the finite-sum and stochastic settings, the current SOTA communication complexity is attained by the \algname{DASHA} method, while maintaining the optimal oracle complexities $\cO\left(m + \frac{\sqrt{m}}{\varepsilon \sqrt{n}}\right)$ and $\cO\left(\frac{\sigma^2}{\varepsilon n} + \frac{\sigma}{\varepsilon^{\nicefrac{3}{2}} n}\right)$ per node.

\textbf{3. Partial participation} \\
From the beginning of federated learning era, the partial participation has been considered to be the essential feature of distributed optimization methods \citep{mcmahan2017communication, konevcny2016federated, kairouz2021advances}. 
However, previously proposed methods have limitations: i) methods \algname{MARINA} and \algname{FedPAGE} from \citep{MARINA, zhao2021fedpage} still require synchronization of all nodes with a small probability. ii) in the stochastic settings, methods \algname{FedAvg}, \algname{SCAFFOLD}, and \algname{FRECON} with the partial participation mechanism \citep{mcmahan2017communication, karimireddy2020scaffold, zhao2021faster} provide results without variance reduction techniques from \citep{SPIDER, PAGE, cutkosky2019momentum} and, therefore, get suboptimal oracle complexities. Note that \algname{FRECON} and \algname{DASHA} reduce the variance \textit{only from compressors} (in the partial participation and stochastic setting). iii) in the finite-sum setting, the \algname{ZeroSARAH} method by \cite{li2021zerosarah} focuses on the homogeneous regime only (the functions $f_i$ are equal). iv) The \algname{MIME} method by \cite{karimireddy2020mime} and the \algname{CE-LSGD} method (for Partial Participation) by the concurrent paper \citep{pateltowards} consider the online version of the problem \eqref{eq:main_problem}. Therefore, \algname{MIME} and \algname{CE-LSGD} (for Partial Participation) require stricter assumptions, including the bounded inter-client gradient variance assumption. In the finite-sum setting \eqref{eq:task_minibatch}, \algname{MIME} and \algname{CE-LSGD} obtain a suboptimal oracle complexity $\cO\left(\nicefrac{1}{\varepsilon^{3/2}}\right)$ while, in the full participation setting, it is possible to get the complexity $\cO\left(\nicefrac{1}{\varepsilon}\right)$.

\section{Contributions}
We propose a new method \algname{\algorithmname} for the nonconvex distributed optimization.

$\bullet$ As far as we know, this is the first method that includes three key ingredients of federated learning methods: \textit{variance reduction of stochastic gradients, compressed communication, and partial participation.} 

$\bullet$ Moreover, this is the first method that combines \textit{variance reduction of stochastic gradients and partial participation} flawlessly: i) it gets the optimal oracle complexity ii) does not require the participation of all nodes iii) does not require the bounded gradients assumption of the functions $f_i$.

$\bullet$ We prove convergence rates and show that this method has \textit{the optimal oracle complexity and the state-of-the-art communication complexity in the partial participation setting.} Moreover, in our work, we observe a nontrivial side-effect from mixing the variance reduction of stochastic gradients and partial participation. It is a general problem not related to our methods or analysis that we discuss in Section~\ref{sec:partial_participation_sampling}.

$\bullet$ In Section~\ref{sec:experiments}, we present experiments where we validate our theory and compare our new methods to previous ones.

\begin{algorithm*}
  \caption{\algname{\algorithmname}}
  \label{alg:main_algorithm}
  \begin{algorithmic}[1]
  \STATE \textbf{Input:} starting point $x^0 \in \R^d$, stepsize $\gamma > 0$, momentum $a \in (0, 1]$, 
  momentum $b \in (0, 1]$, 
  probability $\probpage \in (0, 1]$ (only in \algname{\algorithmname-PAGE}),
  batch size $B$ (only in \algname{\algorithmname-PAGE}, \algname{\algorithmname-FINITE-MVR} and \algname{\algorithmname-MVR}),
  probability $\probavailable \in (0, 1]$ that a node is \textit{participating}\textsuperscript{\red (a)},
  number of iterations~$T \geq 1$
  \STATE Initialize $g^0_i\in \R^d$, $h^0_i\in \R^d$ on the nodes and  $g^0 = \frac{1}{n}\sum_{i=1}^n g^0_i$ on the server
  \STATE Initialize $h^0_{ij}\in \R^d$ on the nodes and take $h^0_i = \frac{1}{m}\sum_{j=1}^m h^0_{ij}$ (only in \algname{\algorithmname-FINITE-MVR})
  \FOR{$t = 0, 1, \dots, T - 1$}
  \STATE $x^{t+1} = x^t - \gamma g^t$ \alglinelabel{alg:main_algorithm:x_update} 
  \STATE Broadcast $x^{t+1}, x^{t}$ to all \textit{participating}\textsuperscript{\red (a)} nodes
  \FOR{$i = 1, \dots, n$ in parallel}
  \IF{$i^{\textnormal{th}} \textnormal{ node is \textit{participating}}$\textsuperscript{\red (a)}}
      \STATE Calculate $k^{t+1}_i$ \alglinelabel{alg:calculate_k} using Algorithm~\ref{alg:main_algorithm:dasha_pp}, \ref{alg:main_algorithm:dasha_pp_page}, \ref{alg:main_algorithm:dasha_pp_finite_mvr} or \ref{alg:main_algorithm:dasha_pp_mvr}
      \STATE $h^{t+1}_i = h^t_i + \frac{1}{\probavailable}k^{t+1}_i$ 
      \STATE $m^{t+1}_i = \cC_i\left(\frac{1}{\probavailable}k^{t+1}_i - \frac{a}{\probavailable} \left(g^{t}_i - h^{t}_i\right)\right)$
      \STATE $g^{t+1}_i = g^{t}_i + m^{t+1}_i$
      \STATE Send $m^{t+1}_i$ to the server
  \ELSE
      \STATE $h^{t+1}_{ij} = h^{t}_{ij}$ (only in \algname{\algorithmname-FINITE-MVR})
      \STATE $h^{t+1}_i = h^{t}_i, \quad g^{t+1}_i = g^{t}_i, \quad m^{t+1}_i = 0$
  \ENDIF
  \ENDFOR
  \STATE $g^{t+1} = g^t + \frac{1}{n} \sum_{i=1}^n m^{t+1}_i$
  \ENDFOR
  \STATE \textbf{Output:} $\hat{x}^T$ chosen uniformly at random from $\{x^t\}_{k=0}^{T-1}$ 

  {\red (a)}: For the formal description see Section~\ref{sec:partial_participation}.
  \end{algorithmic}
\end{algorithm*}

\begin{algorithm*}
  \caption{Calculate $k^{t+1}_i$ for \algname{\algorithmname} in the gradient setting. See line~\ref{alg:calculate_k} in Alg.~\ref{alg:main_algorithm}}
  \label{alg:main_algorithm:dasha_pp}
  \begin{algorithmic}[1]
    \STATE $k^{t+1}_i = \nabla f_i(x^{t+1}) - \nabla f_i(x^{t}) - b \left(h^t_i - \nabla f_i(x^{t})\right)$ 
  \end{algorithmic}
\end{algorithm*}

\begin{algorithm*}
  \caption{Calculate $k^{t+1}_i$ for \algname{\algorithmname-PAGE} in the finite-sum setting. See line~\ref{alg:calculate_k} in Alg.~\ref{alg:main_algorithm}}
  \label{alg:main_algorithm:dasha_pp_page}
  \begin{algorithmic}[1]
    \STATE Generate a random set $I^t_i$ of size $B$ from $[m]$ \textit{with replacement}
    \STATE $k^{t+1}_i = 
        \begin{cases}
          \nabla f_i(x^{t+1}) - \nabla f_i(x^{t}) - \frac{b}{\probpage} \left(h^t_i - \nabla f_i(x^{t})\right),& \\ \qquad \textnormal{with probability $\probpage$ on all \textit{participating} nodes,} \\
          \frac{1}{B}\sum_{j \in I^t_i}\left(\nabla f_{ij}(x^{t+1}) - \nabla f_{ij}(x^{t})\right), & \\
          \qquad \textnormal{with probability $1 - \probpage$ on all \textit{participating} nodes} \end{cases}$
  \end{algorithmic}
\end{algorithm*}

\begin{algorithm*}
  \caption{Calc. $k^{t+1}_i$ for \algname{\algorithmname-FINITE-MVR} in the finite-sum setting. See line~\ref{alg:calculate_k} in Alg.~\ref{alg:main_algorithm}}
  \label{alg:main_algorithm:dasha_pp_finite_mvr}
  \begin{algorithmic}[1]
    \STATE Generate a random set $I^t_i$ of size $B$ from $[m]$ \textit{without replacement}
    \STATE $k^{t+1}_{ij} = 
          \begin{cases}
            \frac{m}{B}\left(\nabla f_{ij}(x^{t+1}) - \nabla f_{ij}(x^{t}) - b \left(h^t_{ij} - \nabla f_{ij}(x^{t})\right)\right), & j \in I^t_i,  \\
            0, & j \not\in I^t_i
          \end{cases}$
    \STATE $h^{t+1}_{ij} = h^t_{ij} + \frac{1}{\probavailable} k^{t+1}_{ij}$
    \STATE $k^{t+1}_i = \frac{1}{m}\sum_{j=1}^m k^{t+1}_{ij}$
  \end{algorithmic}
\end{algorithm*}

\begin{algorithm*}
  \caption{Calculate $k^{t+1}_i$ for \algname{\algorithmname-MVR} in the stochastic setting. See line~\ref{alg:calculate_k} in Alg.~\ref{alg:main_algorithm}}
  \label{alg:main_algorithm:dasha_pp_mvr}
  \begin{algorithmic}[1]
    \STATE Generate i.i.d.\,samples $\{\xi^{t+1}_{ij}\}_{j=1}^B$ of size $B$ from $\mathcal{D}_i.$
    \STATE $k^{t+1}_i = \frac{1}{B}\sum_{j=1}^B \nabla f_i(x^{t+1};\xi^{t+1}_{ij}) - \frac{1}{B}\sum_{j=1}^B \nabla f_i(x^{t};\xi^{t+1}_{ij}) - b \left(h^t_i - \frac{1}{B}\sum_{j=1}^B \nabla f_i(x^{t};\xi^{t+1}_{ij})\right)$
  \end{algorithmic}
\end{algorithm*}

\section{Algorithm Description and Main Challenges Towards Partial Participation}

We now present \algname{\algorithmname} (see Algorithm~\ref{alg:main_algorithm}), a family of methods to solve the optimization problem \eqref{eq:main_problem}. When we started investigating the problem, we took \algname{DASHA} as a baseline method for two reasons: the family of algorithms \algname{DASHA} provides the current state-of-the-art communication complexities in the \emph{non-partial participation} setting, and, unlike \algname{MARINA}, it does not send non-compressed gradients and does not synchronize all nodes. Let us briefly discuss the main idea of \algname{DASHA}, its problem in the \emph{partial participation} setting, and why the refinement of \algname{DASHA} is not an exercise.

In fact, the original \algname{DASHA} method supports the partial participation of nodes \emph{in the gradient setting}. Since the nodes only do the following steps (see full algorithm in Algorithm~\ref{alg:main_algorithm_dasha}):
\begin{align}
  \label{eq:explain_g_grad}
  g^{t+1}_i = g^{t}_i + \cC_i\left(\nabla f_i(x^{t+1}) - (1 - a)\nabla f_i(x^{t}) - a g^{t}_i\right).
\end{align}
The partial participation mechanism (independent participation from Section~\ref{sec:partial_participation}) can be easily implemented here if we temporally redefine the compressor and use another one\footnote{If $\cC_{i} \in \mathbb{U}\left(\omega\right),$ then $\cC_{i}^{\probavailable} \in \mathbb{U}\left(\nicefrac{\omega + 1}{\probavailable} - 1\right).$} instead:
\begin{align*}
  \cC_{i}^{\probavailable} \eqdef \begin{cases}
    \frac{1}{\probavailable}\cC_{i},\textnormal{w.p. } \probavailable,\\
    0, \hspace{0.2cm} \textnormal{w.p. } 1 - \probavailable.
  \end{cases} \overset{\eqref{eq:explain_g_grad}}{\Rightarrow} g^{t+1}_i = \begin{cases}
    g^{t}_i + \frac{1}{\probavailable}\cC_i\left(\nabla f_i(x^{t+1}) - (1 - a)\nabla f_i(x^{t}) - a g^{t}_i\right), \textnormal{w.p. } \probavailable\\
    g^{t}_i, \hspace{6.0cm}\textnormal{w.p. } 1 - \probavailable.
  \end{cases}
\end{align*}
With probability $1 - \probavailable,$ a node does not update $g_i^{t}$ and does not send anything to the server. The main observation is that we can do this trick since $g^{t+1}_i$ depends only on the vectors $x^{t+1},$ $x^{t},$ and $g^{t}_i$. The points $x^{t+1}$ and $x^{t}$ are only available in a node only during its participation.

However, we focus our attention on partial participation \emph{in the finite-sum and stochastic settings}. Consider the nodes' steps in \algname{DASHA-MVR} \citep{tyurin2022dasha} (see Algorithm~\ref{alg:main_algorithm_dasha_mvr}) that is designed for the stochastic setting:
\begin{align}
  &h^{t+1}_i = \nabla f_i(x^{t+1};\xi^{t+1}_{i}) + (1 - b) (h^t_i - \nabla f_i(x^{t};\xi^{t+1}_{i})), \label{eq:explain_h}\\
  &g^{t+1}_i = g^{t}_i + \cC_i\left(h^{t+1}_i - h^{t}_i - a \left(g^{t}_i - h^{t}_i\right)\right). \label{eq:explain_g}
\end{align}
Now we have two sequences $h^{t}_i$ and $g^{t}_i.$ Even if we use the same trick for \eqref{eq:explain_g}, we still have to update \eqref{eq:explain_h} in every iteration of the algorithm since $g^{t+1}_i$ additionally depends on $h^{t+1}_i$ and $h^{t}_i.$ In other words, if a node does not update $g_i^{t}$ and does not send anything to the server, it still has to update $h^{t}_i,$ what is impossible without the points $x^{t+1}$ and $x^{t}.$
One of the main challenges was to ``guess'' how to generalize \eqref{eq:explain_h} and \eqref{eq:explain_g} to the partial participation setting. We now provide a solution (\algname{\algorithmname-MVR} with the batch size $B = 1$):
\begin{eqnarray}
\begin{aligned}
  & h^{t+1}_i = h^t_i + \frac{1}{\probavailable}k^{t+1}_i, \,k^{t+1}_i = \nabla f_i(x^{t+1};\xi^{t+1}_{i}) - \nabla f_i(x^{t};\xi^{t+1}_{i}) - b \left(h^t_i - \nabla f_i(x^{t};\xi^{t+1}_{i})\right), \label{eq:explain_h_new}\\
  & g^{t+1}_i = g^{t}_i + \cC_i\left(\frac{1}{\probavailable}k^{t+1}_i - \frac{a}{\probavailable} \left(g^{t}_i - h^{t}_i\right)\right) \textnormal{with probability } \probavailable, \\
  & \textnormal{and } h^{t+1}_i = h^{t}_i,\, g^{t+1}_i = g^{t}_i \textnormal{ with probability } 1 - \probavailable.
\end{aligned}
\end{eqnarray}
Now both control variables $g_i^{t}$ and $h_i^{t}$ do not change with the probability $1 - \probavailable.$ When the $i$\textsuperscript{th} node participates, the update rules of $g_i^{t+1}$ and $h_i^{t+1}$ in \eqref{eq:explain_h_new} were adapted to make the proof work. When $\probavailable = 1$ (no partial participation), the update rules from \eqref{eq:explain_h_new} reduce to \eqref{eq:explain_h} and \eqref{eq:explain_g}.

The theoretical analysis of the new algorithm became more complicated: unlike \eqref{eq:explain_h} and \eqref{eq:explain_g}, the control variables $h^{t+1}_i$ and $g^{t+1}_i$ in \eqref{eq:explain_h_new} (see also main Algorithm~\ref{alg:main_algorithm}) are coupled by the randomness from the partial participation. 
Going deeper into details, for instance, one can compare Lemma~I.2 from \citep{tyurin2022dasha} and Lemma~\ref{lemma:g_h}, which both bound $\norm{g^{t+1}_i - h^{t+1}_i}^2.$ The former lemma does not use the knowledge about the update rules of $h^{t+1}_i,$ works with one expectation $\ExpSub{\cC}{\cdot},$ uses only \eqref{eq:compressor}, \eqref{auxiliary:jensen_inequality}, and \eqref{auxiliary:variance_decomposition}. The latter lemma additionally requires and uses the structure of the update rule of $h^{t+1}_i$ 
(the structure is very important in the lemma since the control variables $h^{t+1}_i$ and $g^{t+1}_i$ are coupled), 
surgically copes with the expectations $\ExpSub{\cC}{\cdot}$ and $\ExpSub{\probavailable}{\cdot}$ (for instance, it is not trivial in each order one should apply the expectations), and uses the sampling lemma (Lemma~\ref{lemma:sampling}). The same reasoning applies to other parts of the analysis and the finite-sum setting: the generalization of the previous algorithm and the additional randomness from the partial participation required us to rethink the previous proofs.

At the first reading of the proofs, we suggest the reader follow the proof of Theorem~\ref{theorem:gradient_oracle} in the gradient setting (\algname{\algorithmname}), which takes a small part of the paper. Although the appendix seems to be dense and large, the size is justified by the fact that we consider four different sub-algorithms, \algname{\algorithmname}, \algname{\algorithmname-PAGE}, \algname{\algorithmname-FINITE-MVR}, and \algname{\algorithmname-MVR}, and also P\L-condition (The theory is designed so that the proofs do not repeat steps of each other and use one framework).

\section{Theorems}

\label{sec:theorems}

We now present the convergence rates theorems of \algname{\algorithmname} in different settings. We will compare the theorems with the results of the current state-of-the-art methods, \algname{MARINA} and \algname{DASHA}, that work in the full participation setting. Suppose that \algname{MARINA} or \algname{DASHA} converges to $\varepsilon$-solution after $T$ communication rounds. Then, ideally, we would expect the convergence of the new algorithms to $\varepsilon$-solution after up to $\nicefrac{T}{\probavailable}$ communication rounds due to the partial participation constraints\footnote{We check this numerically in Section~\ref{sec:experiments}.}. The detailed analysis of the algorithms under Polyak-\L ojasiewicz condition we provide in Section~\ref{sec:pl_condition}. Let us define $\Delta_0 \eqdef f(x^0) - f^*.$

\subsection{Gradient Setting}

\label{sec:gradien_setting}

\begin{restatable}{theorem}{CONVERGENCE}
  \label{theorem:gradient_oracle}
  Suppose that Assumptions \ref{ass:lower_bound}, \ref{ass:lipschitz_constant}, \ref{ass:nodes_lipschitz_constant}, \ref{ass:compressors} and \ref{ass:partial_participation} hold. Let us take $a = \frac{\probavailable}{2 \omega + 1} ,$ $b = \frac{\probavailable}{2 - \probavailable},$ {\scriptsize \begin{align*}\gamma \leq \left(L + \left[\frac{48 \omega \left(2 \omega + 1\right)}{n \probavailable^2} + \frac{16}{n \probavailable^2}\left(1 - \frac{\probpairaa}{\probavailable}\right)\right]^{\nicefrac{1}{2}}\widehat{L}\right)^{-1},\end{align*}} and $g^{0}_i = h^{0}_i = \nabla f_i(x^0)$ for all $i \in [n]$
  in Algorithm~\ref{alg:main_algorithm} \algname{(\algorithmname)}, then $\Exp{\norm{\nabla f(\widehat{x}^T)}^2} \leq \frac{2 \Delta_0}{\gamma T}.$
\end{restatable}

Let us recall the convergence rate of \algname{MARINA} or \algname{DASHA}, the number of communication rounds to get $\varepsilon$-solution equals
$\cO\left(\frac{\Delta_0}{\varepsilon}\left[L + \frac{\omega}{\sqrt{n}}\widehat{L}\right]\right),$ while the rate of \algname{\algorithmname} equals $\cO\left(\frac{\Delta_0}{\varepsilon}\left[L + \frac{\omega + 1}{\probavailable \sqrt{n}}\widehat{L}\right]\right)$. Up to Lipschitz constants factors, we get the degeneration up to $\nicefrac{1}{\probavailable}$ factor due to the partial participation. This is the expected result since each worker sends useful information only with the probability $\probavailable.$

\subsection{Finite-Sum Setting}

\label{sec:finite_sum_setting}

\begin{restatable}{theorem}{CONVERGENCEPAGE}
  \label{theorem:page}
  Suppose that Assumptions \ref{ass:lower_bound}, \ref{ass:lipschitz_constant}, \ref{ass:nodes_lipschitz_constant}, \ref{ass:max_lipschitz_constant}, \ref{ass:compressors}, and \ref{ass:partial_participation} hold. Let us take $a = \frac{\probavailable}{2 \omega + 1} ,$ $b = \frac{\probpage \probavailable}{2 - \probavailable},$ probability $\probpage \in (0, 1]$,
  {\scriptsize \begin{align*}&\gamma \leq \Bigg(L + \left[\frac{48 \omega (2 \omega + 1)}{n \probavailable^2} \left(\widehat{L}^2 + \frac{(1 - \probpage)L_{\max}^2}{B}\right) + \frac{16}{n \probavailable^2 \probpage} \left(\left(1 - \frac{\probpairaa}{\probavailable}\right)\widehat{L}^2 + \frac{(1 - \probpage)L_{\max}^2}{B}\right)\right]^{\nicefrac{1}{2}}\Bigg)^{-1}\end{align*}}
  and $g^{0}_i = h^{0}_i = \nabla f_i(x^0)$ for all $i \in [n]$ in Algorithm~\ref{alg:main_algorithm} \algname{(\algorithmname-PAGE)}
  then $\Exp{\norm{\nabla f(\widehat{x}^T)}^2} \leq \frac{2 \Delta_0}{\gamma T}.$
\end{restatable}

We now choose $\probpage$ to balance heavy full gradient and light mini-batch calculations. Let us define $\mathbbm{1}_{\probavailable} \eqdef \sqrt{1 - \frac{\probpairaa}{\probavailable}} \in [0, 1].$ Note that if $\probavailable = 1$ then $\probpairaa = 1$ and $\mathbbm{1}_{\probavailable} = 0.$
\begin{restatable}{corollary}{COROLLARYPAGE}
    \label{cor:mini_batch_oracle}
Let the assumptions from Theorem~\ref{theorem:page} hold and $\probpage = \nicefrac{B}{(m + B)}.$ 
Then \algname{\algorithmname-PAGE}
    needs 
    \begin{align}
      \label{eq:rate_dasha_pp_page}
      T &\eqdef \cO\Bigg(\frac{\Delta_0}{\varepsilon}\Bigg[L + \frac{\omega}{\probavailable\sqrt{n}}\left(\widehat{L} + \frac{L_{\max}}{\sqrt{B}}\right) + \frac{1}{\probavailable}\sqrt{\frac{m}{n}}\left(\frac{\mathbbm{1}_{\probavailable}\widehat{L}}{\sqrt{B}} + \frac{L_{\max}}{B}\right)\Bigg]\Bigg)
    \end{align}
    communication rounds to get an $\varepsilon$-solution and the expected number of gradient calculations per node equals $\cO\left(m + B T\right).$
\end{restatable}

The convergence rate the rate of the current state-of-the-art method \algname{DASHA-PAGE} without partial participation equals
$\cO\left(\frac{\Delta_0}{\varepsilon}\left[L + \frac{\omega}{\sqrt{n}}\left(\widehat{L} + \frac{L_{\max}}{\sqrt{B}}\right) + \sqrt{\frac{m}{n}}\frac{L_{\max}}{B}\right]\right).$ Let us closer compare it with \eqref{eq:rate_dasha_pp_page}. As expected, we see that the second term w.r.t. $\omega$ degenerates up to $\nicefrac{1}{\probavailable}$. Surprisingly, the third term w.r.t. $\sqrt{\nicefrac{m}{n}}$ can degenerate up to $\nicefrac{\sqrt{B}}{\probavailable}$ when $\widehat{L} \approx L_{\max}.$ Hence, in order to keep degeneration up to $\nicefrac{1}{\probavailable},$ one should take the batch size $B = \cO\left(\nicefrac{L_{\max}^2}{\widehat{L}^2}\right).$ This interesting effect we analyze separately in Section~\ref{sec:partial_participation_sampling}. The fact that the degeneration is up to $\nicefrac{1}{\probavailable}$ we check numerically in Section~\ref{sec:experiments}.

In the following corollary, we consider Rand$K$ compressors\footnote{The choice of the compressor is driven by simplicity, and the following analysis can be used for other unbiased compressors.} (see Definition~\ref{def:rand_k}) and show that with the particular choice of parameters, up to the Lipschitz constants factors, \algname{\algorithmname-PAGE} gets the optimal oracle complexity and SOTA communication complexity. Indeed, comparing the following result with \citep[Corollary 6.6]{tyurin2022dasha}, one can see that we get the degeneration up to $\nicefrac{1}{\probavailable}$ factor, which is expected in the partial participation setting. 
Note that the complexities improve with the number of workers $n.$

\begin{restatable}{corollary}{COROLLARYPAGERANDK}
  \label{cor:mini_batch_oracle:randk}
  Suppose that assumptions of Corollary~\ref{cor:mini_batch_oracle} hold, $B \leq \min\left\{\frac{1}{\probavailable}\sqrt{\frac{m}{n}},\frac{L_{\max}^2}{\mathbbm{1}_{\probavailable}^2 \widehat{L}^2}\right\}$\footnote{If $\mathbbm{1}_{\probavailable} = 0,$ then $\frac{L_{\sigma}^2}{\mathbbm{1}_{\probavailable}^2 \widehat{L}^2} = +\infty$}, and we use the unbiased compressor Rand$K$ with $K = \Theta\left(\nicefrac{B d}{\sqrt{m}}\right).$ Then
  the communication complexity of Algorithm~\ref{alg:main_algorithm} is
  $
      \cO\left(d + \frac{L_{\max} \Delta_0 d}{\probavailable \varepsilon \sqrt{n}}\right),
  $
and the expected number of gradient calculations per node equals
  $
      \cO\left(m + \frac{L_{\max} \Delta_0 \sqrt{m}}{\probavailable \varepsilon \sqrt{n}} \right).
  $
\end{restatable}
The convergence rate of \algname{\algorithmname-FINITE-MVR} is provided in Section~\ref{sec:proof_finite_mvr}. 

\subsection{Stochastic Setting}

\label{sec:stochastic_setting}

We define $h^t \eqdef \frac{1}{n}\sum_{i=1}^n h^t_i$.

\begin{restatable}{theorem}{CONVERGENCEMVR}
  \label{theorem:stochastic}
  Suppose that Assumptions \ref{ass:lower_bound}, \ref{ass:lipschitz_constant}, \ref{ass:nodes_lipschitz_constant}, \ref{ass:stochastic_unbiased_and_variance_bounded}, \ref{ass:mean_square_smoothness}, \ref{ass:compressors} and \ref{ass:partial_participation} hold. Let us take $a = \frac{\probavailable}{2\omega + 1}$, $b~\in~\left(0, \frac{\probavailable}{2 - \probavailable}\right]$, 
  {\scriptsize $\gamma \leq \Bigg(L + \left[\frac{48 \omega (2 \omega + 1)}{n \probavailable^2} \left(\widehat{L}^2 + \frac{(1 - b)^2 L_{\sigma}^2}{B}\right) + \frac{12}{n \probavailable b}\left(\left(1 - \frac{\probpairaa}{\probavailable}\right) \widehat{L}^2 + \frac{(1 - b)^2 L_{\sigma}^2}{B}\right)\right]^{\nicefrac{1}{2}}\Bigg)^{-1},$}
  and $g^{0}_i = h^{0}_i$ for all $i \in [n]$
  in Algorithm~\ref{alg:main_algorithm} \algname{(\algorithmname-MVR)}.
  Then 
{\scriptsize\begin{align*}
  &\Exp{\norm{\nabla f(\widehat{x}^T)}^2} \leq \frac{1}{T}\vast[\frac{2 \Delta_0}{\gamma} + \frac{2}{b} \norm{h^{0} - \nabla f(x^{0})}^2 + \left(\frac{32 b \omega (2 \omega + 1)}{n \probavailable^2} + \frac{4 \left(1 - \frac{\probpairaa}{\probavailable}\right)}{n \probavailable}\right)\left(\frac{1}{n}\sum_{i=1}^n\norm{h^{0}_i - \nabla f_i(x^{0})}^2\right)\vast] \\
  &+ \left(\frac{48 b^2 \omega (2 \omega + 1)}{\probavailable^2} + \frac{12 b}{\probavailable}\right) \frac{\sigma^2}{n B}.
\end{align*}}
\end{restatable}

In the next corollary, we choose momentum $b$ and initialize vectors $h^{0}_i$ to get $\varepsilon$-solution. 

\begin{restatable}{corollary}{COROLLARYSTOCHASTIC}
  \label{cor:stochastic}
  Suppose that assumptions from Theorem~\ref{theorem:stochastic} hold, momentum $b = \Theta\left(\min \left\{\ \frac{\probavailable}{\omega}\sqrt{\frac{n \varepsilon B}{\sigma^2}}, \frac{\probavailable n \varepsilon B}{\sigma^2}\right\}\right),$ $\frac{\sigma^2}{n \varepsilon B} \geq 1,$
  and $h^{0}_i = \frac{1}{B_{\textnormal{init}}} \sum_{k = 1}^{B_{\textnormal{init}}} \nabla f_i(x^0; \xi^0_{ik})$ for all $i \in [n],$ and batch size $B_{\textnormal{init}} = \Theta\left(\frac{\sqrt{\probavailable} B}{b}\right),$ then Algorithm~\ref{alg:main_algorithm} \algname{(\algorithmname-MVR)} needs
  \begin{align*}&T \eqdef \cO\Bigg(\frac{\Delta_0}{\varepsilon}\Bigg[L + \frac{\omega}{\probavailable\sqrt{n}}\left(\widehat{L} + \frac{L_{\sigma}}{\sqrt{B}}\right) + \frac{\sigma}{\probavailable \sqrt{\varepsilon} n}\left(\frac{\mathbbm{1}_{\probavailable}\widehat{L}}{\sqrt{B}} + \frac{L_\sigma}{B}\right)\Bigg] + \frac{\sigma^2}{\sqrt{\probavailable} n \varepsilon B}\Bigg)\end{align*}
  communication rounds to get an $\varepsilon$-solution and the number of stochastic gradient calculations per node equals $\cO(B_{\textnormal{init}} + BT).$
\end{restatable}

The convergence rate of the \algname{DASHA-SYNC-MVR}, the state-of-the-art method without partial participation, equals $\cO\left(\frac{\Delta_0}{\varepsilon}\left[L + \frac{\omega}{\sqrt{n}}\left(\widehat{L} + \frac{L_{\sigma}}{\sqrt{B}}\right) + \frac{\sigma}{\sqrt{\varepsilon} n}\frac{L_\sigma}{B}\right] + \frac{\sigma^2}{n \varepsilon B}\right).$ Similar to Section~\ref{sec:finite_sum_setting}, we see that in the regimes when $\widehat{L} \approx L_{\sigma}$ the third term w.r.t. $\nicefrac{1}{\varepsilon^{3/2}}$ can degenerate up to $\nicefrac{\sqrt{B}}{\probavailable}.$ However, if we take $B = \cO\left(\nicefrac{L_{\sigma}^2}{\widehat{L}^2}\right),$ then the degeneration of the third term will be up to $\nicefrac{1}{\probavailable}.$ This effect we analyze in Section~\ref{sec:partial_participation_sampling}. The fact that the degeneration is up to $\nicefrac{1}{\probavailable}$ we check numerically in Section~\ref{sec:experiments}.

In the following corollary, we consider Rand$K$ compressors (see Definition~\ref{def:rand_k}) and show that with the particular choice of parameters, up to the Lipschitz constants factors, \algname{\algorithmname-MVR} gets the optimal oracle complexity and SOTA communication complexity of \algname{DASHA-SYNC-MVR} method. Indeed, comparing the following result with \citep[Corollary 6.9]{tyurin2022dasha}, one can see that we get the degeneration up to $\nicefrac{1}{\probavailable}$ factor, which is expected in the partial participation setting. Note that the complexities improve with the number of workers $n.$

\begin{restatable}{corollary}{COROLLARYSTOCHASTICRANDK}
  \label{cor:stochastic:randk}
  Suppose that assumptions of Corollary~\ref{cor:stochastic} hold, batch size $B \leq \min\left\{\frac{\sigma}{\probavailable\sqrt{\varepsilon} n}, \frac{L_{\sigma}^2}{\mathbbm{1}_{\probavailable}^2 \widehat{L}^2}\right\},$ we take Rand$K$ compressors with $K = \Theta\left(\frac{B d \sqrt{\varepsilon n}}{\sigma}\right).$ Then
  the communication complexity equals 
  $
      \cO\left(\frac{d \sigma}{\sqrt{\probavailable} \sqrt{n \varepsilon}} + \frac{L_{\sigma} \Delta_0 d}{\probavailable \sqrt{n} \varepsilon}\right),
  $
  and the expected number of stochastic gradient calculations per node equals
  $
      \cO\left(\frac{\sigma^2}{\sqrt{\probavailable} n \varepsilon} + \frac{L_{\sigma} \Delta_0 \sigma}{\probavailable \varepsilon^{\nicefrac{3}{2}} n}\right).
  $
\end{restatable}

We are aware that the initial batch size $B_{\textnormal{init}}$ can be suboptimal w.r.t.\,$\omega$ in \algname{\algorithmname-MVR} in some regimes (see also \citep{tyurin2022dasha}). This is a side effect of mixing the variance reduction of stochastic gradients and compression. However, Corollary~\ref{cor:stochastic:randk} reveals that we can escape these regimes by choosing the parameter $K$ of Rand$K$ compressors in a particular way. To get the complete picture, we analyze the same phenomenon under P\L\, condition (see Section~\ref{sec:pl_condition}) and provide a new method \algname{\algorithmname-SYNC-MVR} (see Section~\ref{sec:main_algorithm_mvr_sync}).

\subsection*{Acknowledgements}
This work of P. Richt\'{a}rik and A. Tyurin was supported by the KAUST Baseline Research Scheme (KAUST BRF) and the KAUST Extreme Computing Research Center (KAUST ECRC), and the work of P. Richt\'{a}rik was supported by the SDAIA-KAUST Center of Excellence in Data Science and Artificial Intelligence (SDAIA-KAUST AI).

\bibliography{neurips_2023}

\begin{thebibliography}{}

\bibitem[Alistarh et~al., 2017]{alistarh2017qsgd}
Alistarh, D., Grubic, D., Li, J., Tomioka, R., and Vojnovic, M. (2017).
\newblock {QSGD}: {C}ommunication-efficient {SGD} via gradient quantization and
  encoding.
\newblock In {\em Advances in Neural Information Processing Systems (NIPS)},
  pages 1709--1720.

\bibitem[Arjevani et~al., 2019]{arjevani2019lower}
Arjevani, Y., Carmon, Y., Duchi, J.~C., Foster, D.~J., Srebro, N., and
  Woodworth, B. (2019).
\newblock Lower bounds for non-convex stochastic optimization.
\newblock {\em arXiv preprint arXiv:1912.02365}.

\bibitem[Beznosikov et~al., 2020]{beznosikov2020biased}
Beznosikov, A., Horv{\'a}th, S., Richt{\'a}rik, P., and Safaryan, M. (2020).
\newblock On biased compression for distributed learning.
\newblock {\em arXiv preprint arXiv:2002.12410}.

\bibitem[Carmon et~al., 2020]{carmon2020lower}
Carmon, Y., Duchi, J.~C., Hinder, O., and Sidford, A. (2020).
\newblock Lower bounds for finding stationary points i.
\newblock {\em Mathematical Programming}, 184(1):71--120.

\bibitem[Chang and Lin, 2011]{chang2011libsvm}
Chang, C.-C. and Lin, C.-J. (2011).
\newblock {LIBSVM}: a library for support vector machines.
\newblock {\em {ACM} {T}ransactions on {I}ntelligent {S}ystems and {T}echnology
  (TIST)}, 2(3):1--27.

\bibitem[Cutkosky and Orabona, 2019]{cutkosky2019momentum}
Cutkosky, A. and Orabona, F. (2019).
\newblock Momentum-based variance reduction in non-convex {SGD}.
\newblock {\em arXiv preprint arXiv:1905.10018}.

\bibitem[Fang et~al., 2018]{SPIDER}
Fang, C., Li, C.~J., Lin, Z., and Zhang, T. (2018).
\newblock {SPIDER}: Near-optimal non-convex optimization via stochastic path
  integrated differential estimator.
\newblock In {\em NeurIPS Information Processing Systems}.

\bibitem[Goodfellow et~al., 2016]{goodfellow2016deep}
Goodfellow, I., Bengio, Y., Courville, A., and Bengio, Y. (2016).
\newblock {\em Deep learning}, volume~1.
\newblock MIT Press.

\bibitem[Gorbunov et~al., 2021]{MARINA}
Gorbunov, E., Burlachenko, K., Li, Z., and Richt\'{a}rik, P. (2021).
\newblock {MARINA}: {F}aster non-convex distributed learning with compression.
\newblock In {\em 38th International Conference on Machine Learning}.

\bibitem[Horv{\'a}th et~al., 2019a]{horvath2019natural}
Horv{\'a}th, S., Ho, C.-Y., Horvath, L., Sahu, A.~N., Canini, M., and
  Richt{\'a}rik, P. (2019a).
\newblock Natural compression for distributed deep learning.
\newblock {\em arXiv preprint arXiv:1905.10988}.

\bibitem[Horv{\'a}th et~al., 2019b]{horvath2019stochastic}
Horv{\'a}th, S., Kovalev, D., Mishchenko, K., Stich, S., and Richt{\'a}rik, P.
  (2019b).
\newblock Stochastic distributed learning with gradient quantization and
  variance reduction.
\newblock {\em arXiv preprint arXiv:1904.05115}.

\bibitem[Ioffe and Szegedy, 2015]{ioffe2015batch}
Ioffe, S. and Szegedy, C. (2015).
\newblock Batch normalization: Accelerating deep network training by reducing
  internal covariate shift.
\newblock In {\em International Conference on Machine Learning}, pages
  448--456. PMLR.

\bibitem[Kairouz et~al., 2021]{kairouz2021advances}
Kairouz, P., McMahan, H.~B., Avent, B., Bellet, A., Bennis, M., Bhagoji, A.~N.,
  Bonawitz, K., Charles, Z., Cormode, G., Cummings, R., et~al. (2021).
\newblock Advances and open problems in federated learning.
\newblock {\em Foundations and Trends{\textregistered} in Machine Learning},
  14(1--2):1--210.

\bibitem[Karimireddy et~al., 2020a]{karimireddy2020mime}
Karimireddy, S.~P., Jaggi, M., Kale, S., Mohri, M., Reddi, S.~J., Stich, S.~U.,
  and Suresh, A.~T. (2020a).
\newblock Mime: Mimicking centralized stochastic algorithms in federated
  learning.
\newblock {\em arXiv preprint arXiv:2008.03606}.

\bibitem[Karimireddy et~al., 2020b]{karimireddy2020scaffold}
Karimireddy, S.~P., Kale, S., Mohri, M., Reddi, S., Stich, S., and Suresh,
  A.~T. (2020b).
\newblock Scaffold: Stochastic controlled averaging for federated learning.
\newblock In {\em International Conference on Machine Learning}, pages
  5132--5143. PMLR.

\bibitem[Kone{\v{c}}n{\'y} et~al., 2016]{konevcny2016federated}
Kone{\v{c}}n{\'y}, J., McMahan, H.~B., Yu, F.~X., Richt{\'a}rik, P., Suresh,
  A.~T., and Bacon, D. (2016).
\newblock Federated learning: Strategies for improving communication
  efficiency.
\newblock {\em arXiv preprint arXiv:1610.05492}.

\bibitem[Li et~al., 2020]{li2020federated}
Li, T., Sahu, A.~K., Zaheer, M., Sanjabi, M., Talwalkar, A., and Smith, V.
  (2020).
\newblock Federated optimization in heterogeneous networks.
\newblock {\em Proceedings of Machine Learning and Systems}, 2:429--450.

\bibitem[Li et~al., 2021a]{PAGE}
Li, Z., Bao, H., Zhang, X., and Richt{\'a}rik, P. (2021a).
\newblock {PAGE}: A simple and optimal probabilistic gradient estimator for
  nonconvex optimization.
\newblock In {\em International Conference on Machine Learning}, pages
  6286--6295. PMLR.

\bibitem[Li et~al., 2021b]{li2021zerosarah}
Li, Z., Hanzely, S., and Richt{\'a}rik, P. (2021b).
\newblock {ZeroSARAH}: Efficient nonconvex finite-sum optimization with zero
  full gradient computation.
\newblock {\em arXiv preprint arXiv:2103.01447}.

\bibitem[Lin et~al., 2017]{lin2017deep}
Lin, Y., Han, S., Mao, H., Wang, Y., and Dally, W.~J. (2017).
\newblock Deep gradient compression: Reducing the communication bandwidth for
  distributed training.
\newblock {\em arXiv preprint arXiv:1712.01887}.

\bibitem[McMahan et~al., 2017]{mcmahan2017communication}
McMahan, B., Moore, E., Ramage, D., Hampson, S., and y~Arcas, B.~A. (2017).
\newblock Communication-efficient learning of deep networks from decentralized
  data.
\newblock In {\em Artificial intelligence and statistics}, pages 1273--1282.
  PMLR.

\bibitem[Mishchenko et~al., 2019]{DIANA}
Mishchenko, K., Gorbunov, E., Tak{\'a}{\v{c}}, M., and Richt{\'a}rik, P.
  (2019).
\newblock Distributed learning with compressed gradient differences.
\newblock {\em arXiv preprint arXiv:1901.09269}.

\bibitem[Narayanan et~al., 2019]{narayanan2019pipedream}
Narayanan, D., Harlap, A., Phanishayee, A., Seshadri, V., Devanur, N.~R.,
  Ganger, G.~R., Gibbons, P.~B., and Zaharia, M. (2019).
\newblock {PipeDream}: generalized pipeline parallelism for dnn training.
\newblock In {\em Proceedings of the 27th ACM Symposium on Operating Systems
  Principles}, pages 1--15.

\bibitem[Nesterov, 2018]{nesterov2018lectures}
Nesterov, Y. (2018).
\newblock {\em Lectures on convex optimization}, volume 137.
\newblock Springer.

\bibitem[Nguyen et~al., 2017]{SARAH}
Nguyen, L., Liu, J., Scheinberg, K., and Tak{\'a}{\v{c}}, M. (2017).
\newblock {SARAH}: A novel method for machine learning problems using
  stochastic recursive gradient.
\newblock In {\em The 34th International Conference on Machine Learning}.

\bibitem[Paszke et~al., 2019]{paszke2019pytorch}
Paszke, A., Gross, S., Massa, F., Lerer, A., Bradbury, J., Chanan, G., Killeen,
  T., Lin, Z., Gimelshein, N., Antiga, L., et~al. (2019).
\newblock Pytorch: An imperative style, high-performance deep learning library.
\newblock In {\em Advances in Neural Information Processing Systems (NeurIPS)}.

\bibitem[Patel et~al., 2022]{pateltowards}
Patel, K.~K., Wang, L., Woodworth, B., Bullins, B., and Srebro, N. (2022).
\newblock Towards optimal communication complexity in distributed non-convex
  optimization.
\newblock In {\em Advances in Neural Information Processing Systems}.

\bibitem[Ramaswamy et~al., 2019]{ramaswamy2019federated}
Ramaswamy, S., Mathews, R., Rao, K., and Beaufays, F. (2019).
\newblock Federated learning for emoji prediction in a mobile keyboard.
\newblock {\em arXiv preprint arXiv:1906.04329}.

\bibitem[Ramesh et~al., 2021]{ramesh2021zero}
Ramesh, A., Pavlov, M., Goh, G., Gray, S., Voss, C., Radford, A., Chen, M., and
  Sutskever, I. (2021).
\newblock Zero-shot text-to-image generation.
\newblock {\em arXiv preprint arXiv:2102.12092}.

\bibitem[Reddi et~al., 2020]{reddi2020adaptive}
Reddi, S., Charles, Z., Zaheer, M., Garrett, Z., Rush, K., Kone{\v{c}}n{\`y},
  J., Kumar, S., and McMahan, H.~B. (2020).
\newblock Adaptive federated optimization.
\newblock {\em arXiv preprint arXiv:2003.00295}.

\bibitem[Richt{\'a}rik et~al., 2021]{richtarik2021ef21}
Richt{\'a}rik, P., Sokolov, I., and Fatkhullin, I. (2021).
\newblock {EF21}: A new, simpler, theoretically better, and practically faster
  error feedback.
\newblock {\em In Neural Information Processing Systems, 2021.}

\bibitem[Richt{\'a}rik et~al., 2022]{richtarik20223pc}
Richt{\'a}rik, P., Sokolov, I., Fatkhullin, I., Gasanov, E., Li, Z., and
  Gorbunov, E. (2022).
\newblock {3PC}: Three point compressors for communication-efficient
  distributed training and a better theory for lazy aggregation.
\newblock {\em arXiv preprint arXiv:2202.00998}.

\bibitem[Sapio et~al., 2019]{sapio2019scaling}
Sapio, A., Canini, M., Ho, C.-Y., Nelson, J., Kalnis, P., Kim, C.,
  Krishnamurthy, A., Moshref, M., Ports, D.~R., and Richt{\'a}rik, P. (2019).
\newblock Scaling distributed machine learning with in-network aggregation.
\newblock {\em arXiv preprint arXiv:1903.06701}.

\bibitem[Stich et~al., 2018]{stich2018sparsified}
Stich, S.~U., Cordonnier, J.-B., and Jaggi, M. (2018).
\newblock Sparsified {SGD} with memory.
\newblock {\em Advances in Neural Information Processing Systems}, 31.

\bibitem[Suresh et~al., 2022]{suresh2022correlated}
Suresh, A.~T., Sun, Z., Ro, J.~H., and Yu, F. (2022).
\newblock Correlated quantization for distributed mean estimation and
  optimization.
\newblock {\em arXiv preprint arXiv:2203.04925}.

\bibitem[Szlendak et~al., 2021]{szlendak2021permutation}
Szlendak, R., Tyurin, A., and Richt{\'a}rik, P. (2021).
\newblock Permutation compressors for provably faster distributed nonconvex
  optimization.
\newblock {\em arXiv preprint arXiv:2110.03300}.

\bibitem[Tyurin and Richt{\'a}rik, 2023]{tyurin2022dasha}
Tyurin, A. and Richt{\'a}rik, P. (2023).
\newblock {DASHA}: Distributed nonconvex optimization with communication
  compression and optimal oracle complexity.
\newblock {\em International Conference on Learning Representations (ICLR)}.

\bibitem[Vogels et~al., 2021]{vogels2021relaysum}
Vogels, T., He, L., Koloskova, A., Karimireddy, S.~P., Lin, T., Stich, S.~U.,
  and Jaggi, M. (2021).
\newblock {RelaySum} for decentralized deep learning on heterogeneous data.
\newblock {\em Advances in Neural Information Processing Systems}, 34.

\bibitem[Wangni et~al., 2018]{wangni2018gradient}
Wangni, J., Wang, J., Liu, J., and Zhang, T. (2018).
\newblock Gradient sparsification for communication-efficient distributed
  optimization.
\newblock {\em Advances in Neural Information Processing Systems}, 31.

\bibitem[Xu et~al., 2021]{xu2021grace}
Xu, H., Ho, C.-Y., Abdelmoniem, A.~M., Dutta, A., Bergou, E.~H., Karatsenidis,
  K., Canini, M., and Kalnis, P. (2021).
\newblock Grace: A compressed communication framework for distributed machine
  learning.
\newblock In {\em 2021 IEEE 41st International Conference on Distributed
  Computing Systems (ICDCS)}, pages 561--572. IEEE.

\bibitem[Zhao et~al., 2021a]{zhao2021faster}
Zhao, H., Burlachenko, K., Li, Z., and Richt{\'a}rik, P. (2021a).
\newblock Faster rates for compressed federated learning with client-variance
  reduction.
\newblock {\em arXiv preprint arXiv:2112.13097}.

\bibitem[Zhao et~al., 2021b]{zhao2021fedpage}
Zhao, H., Li, Z., and Richt{\'a}rik, P. (2021b).
\newblock {FedPAGE}: A fast local stochastic gradient method for
  communication-efficient federated learning.
\newblock {\em arXiv preprint arXiv:2108.04755}.

\end{thebibliography}
\bibliographystyle{apalike}

\clearpage
\newpage
\appendix
\onecolumn
\tableofcontents
\newpage

\section{Numerical Verification of Theoretical Dependencies}
\label{sec:experiments}

\usetikzlibrary{decorations.pathreplacing}
\begin{figure}[h]
  \centering
  \begin{subfigure}{.5\textwidth}
    \centering
    \begin{tikzpicture}
      \node[anchor=south west,inner sep=0] at (0,0) {\includegraphics[width=\textwidth]{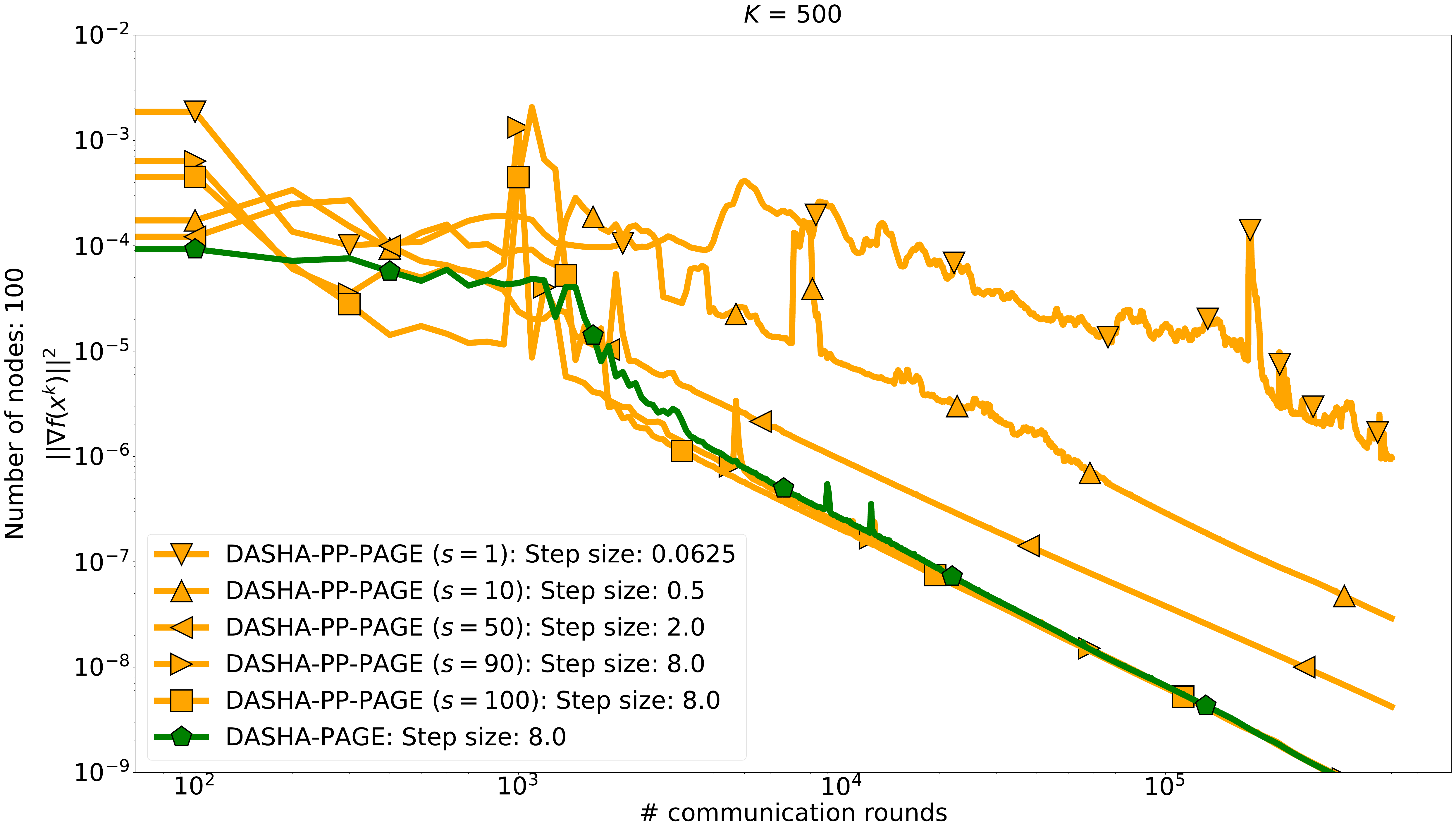}};
    \end{tikzpicture}
    \caption{Finite-sum setting, $K = 500$ in Rand$K$.}
    \label{fig:real_sim_finite_sum}
  \end{subfigure}\hfill
  \begin{subfigure}{.5\textwidth}
    \centering
    \vspace{0.3cm}
    \begin{tikzpicture}
      \node[anchor=south west,inner sep=0] at (0,0) {\includegraphics[width=\textwidth]{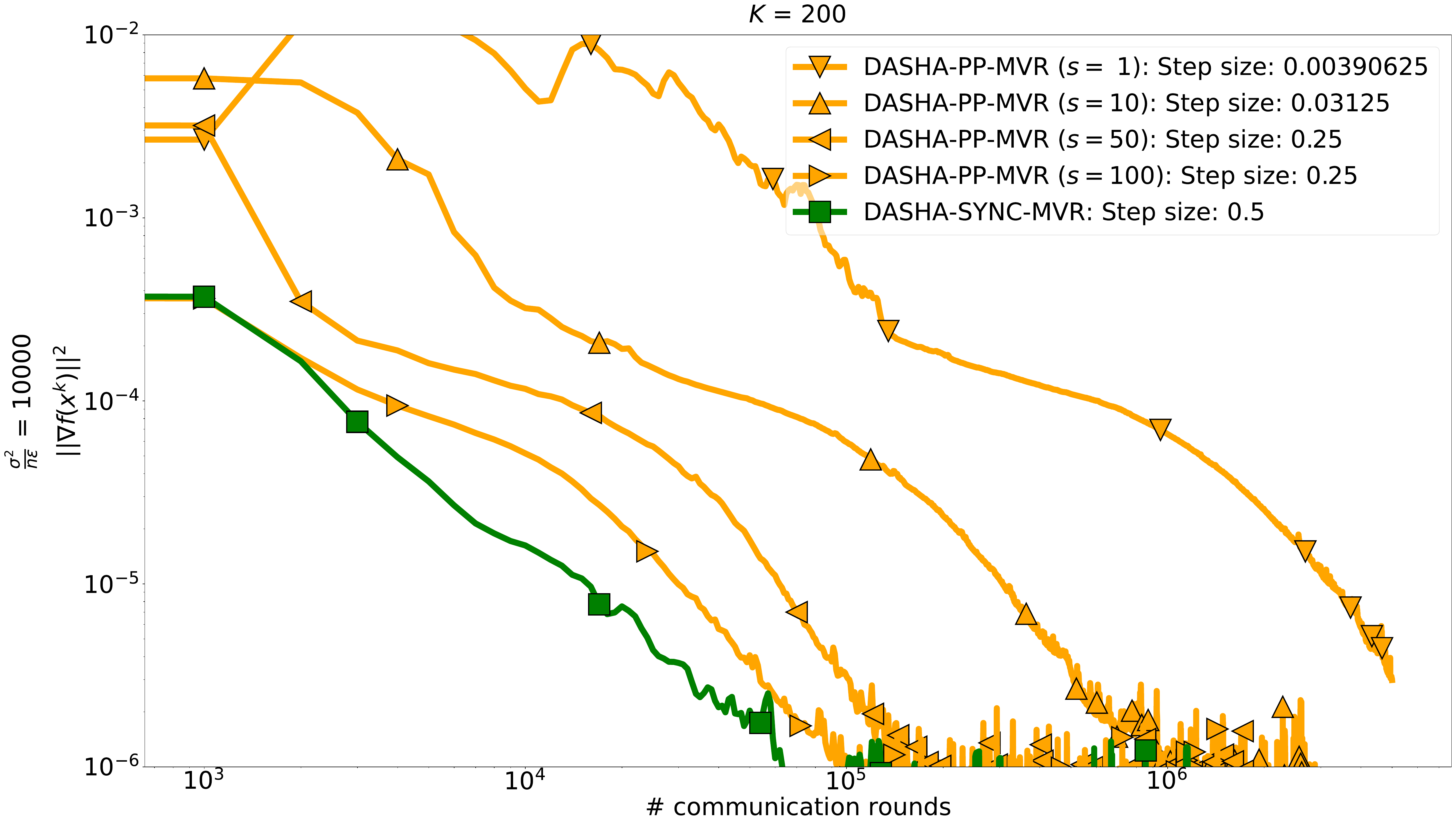}};
    \end{tikzpicture}
    \caption{Stochastic setting, $\nicefrac{\sigma^2}{n \varepsilon B} = 10000,$ and $K = 200$ in Rand$K$.}
    \label{fig:real_sim_stochastic}
  \end{subfigure}
  \caption{Classification task with the \textit{real-sim} dataset.}
  \label{fig:test}
  \end{figure}

Our main goal is to verify the dependeces from the theory. We compare \algname{\algorithmname} with \algname{DASHA}. Clearly, \algname{\algorithmname} can not generally perform better than \algname{DASHA}.
In different settings, we verify that the bigger $\probavailable$, the closer \algname{\algorithmname} is to \algname{DASHA}, i.e., \algname{\algorithmname} converges no slower than $\nicefrac{1}{\probavailable}$ times.

In all experiments, we take the \textit{real-sim} dataset with dimension $d = \num[group-separator={,}]{20958}$ and the number of samples equals $\num[group-separator={,}]{72309}$ from LIBSVM datasets \citep{chang2011libsvm} (under the 3-clause BSD license), and randomly split the dataset between $n = 100$ nodes equally, ignoring residual samples.
In the finite-sum setting, we solve a classification problem with functions
\begin{align}
  f_i(x) \eqdef \frac{1}{m}\sum_{j=1}^m \left(1 - \frac{1}{1 + \exp(y_{ij} a_{ij}^\top x)}\right)^2,
  \label{eq:func_finite}
\end{align}
where $a_{ij} \in \R^{d}$ is the feature vector of a sample on the $i$\textsuperscript{th} node, $y_{ij} \in \{-1, 1\}$ is the corresponding label,  and $m$ is the number of samples on the $i$\textsuperscript{th} node for all $i \in [n].$
In the stochastic setting, we consider functions
\begin{align}
  f_i(x_1, x_2) \eqdef {\rm E}_{j \sim [m]}\vast[-\log\left(\frac{\exp\left(a_{ij}^\top x_{y_{ij}}\right)}{\sum_{y \in \{1, 2\}} \exp\left(a_{ij}^\top x_{y}\right)}\right) + \lambda \sum_{y \in \{1, 2\}} \sum_{k = 1}^d \frac{\{x_{y}\}_k^2}{1 + \{x_{y}\}_k^2}\vast],
  \label{eq:func_stoch}
\end{align}
where $x_1, x_2 \in \R^{d}$, $\{\cdot\}_k$ is an indexing operation,
$a_{ij} \in \R^{d}$ is a feature of a sample on the $i$\textsuperscript{th} node, $y_{ij} \in \{1, 2\}$ is a corresponding label, $m$ is the number of samples located on the $i$\textsuperscript{th} node, constant $\lambda = 0.001$ for all $i \in [n].$

The code was written in Python 3.6.8 using PyTorch 1.9 \citep{paszke2019pytorch}. A distributed environment was emulated on a machine with Intel(R) Xeon(R) Gold 6226R CPU @ 2.90GHz and 64 cores.

We use the standard setting in experiments\footnote{Code: https://github.com/mysteryresearcher/dasha-partial-participation} where all parameters except step sizes are taken as suggested in theory. Step sizes are finetuned from a set~$\{2^{i}\,|\,i \in [-10, 10]\}.$ We emulate the partial participation setting using $s$-nice sampling with the number of nodes $n = 100$. We consider the Rand$K$ compressor and take the batch size $B = 1$. 
We plot the relation between communication rounds and values of the norm of gradients at each communication round.

In the finite-sum (Figure~\ref{fig:real_sim_finite_sum}) and in the stochastic setting (Figure~\ref{fig:real_sim_stochastic}), 
we see that the bigger probability $\probavailable = \nicefrac{s}{n}$ to $1$, the closer \algname{\algorithmname} to \algname{DASHA}. Moreover, \algname{\algorithmname} with $s = 10$ and $s = 1$ converges approximately $\times 10$ ($= \nicefrac{1}{\probavailable}$) and $\times 100$ ($= \nicefrac{1}{\probavailable}$) times slower, accordingly. Our theory predicts such behavior.

\subsection{Experiments in Partial Participation Setting}
In this experiments, we compare our new algorithm \algname{DASHA-PP} with previous baselines \algname{MARINA} and \algname{FRECON} in the partial participation setting. We consider \algname{MARINA} and \algname{FRECON} because they are the previous SOTA methods in the \emph{partial participation setting with compression}. We investigate the same optimization problem and setup as in Section~\ref{sec:experiments} of the paper. 
All methods use the Rand$K$ compressor in these experiments.

{1. \bf Finite-Sum Setting.} We now consider the function from \eqref{eq:func_finite}. In Figures~\ref{fig:finite-sum-real-sim} and \ref{fig:finite-sum-mnist}, we compare all three methods in the finite-sum setting on two different datasets: \textit{real-sim} and \textit{MNIST}. The parameter $s$ is the number of clients participating in each round that are selected randomly using the $s$-nice sampling (server chooses uniformly $s$ nodes without replacement).
We can see that \algname{DASHA-PP} converges faster than \algname{MARINA}. Since \algname{FRECON} does not support variance reduction of stochastic gradients, it converges to less accurate solutions.
\begin{figure}[H]
  \centering
  \begin{subfigure}{.28\textwidth}
      \includegraphics[width=\textwidth]{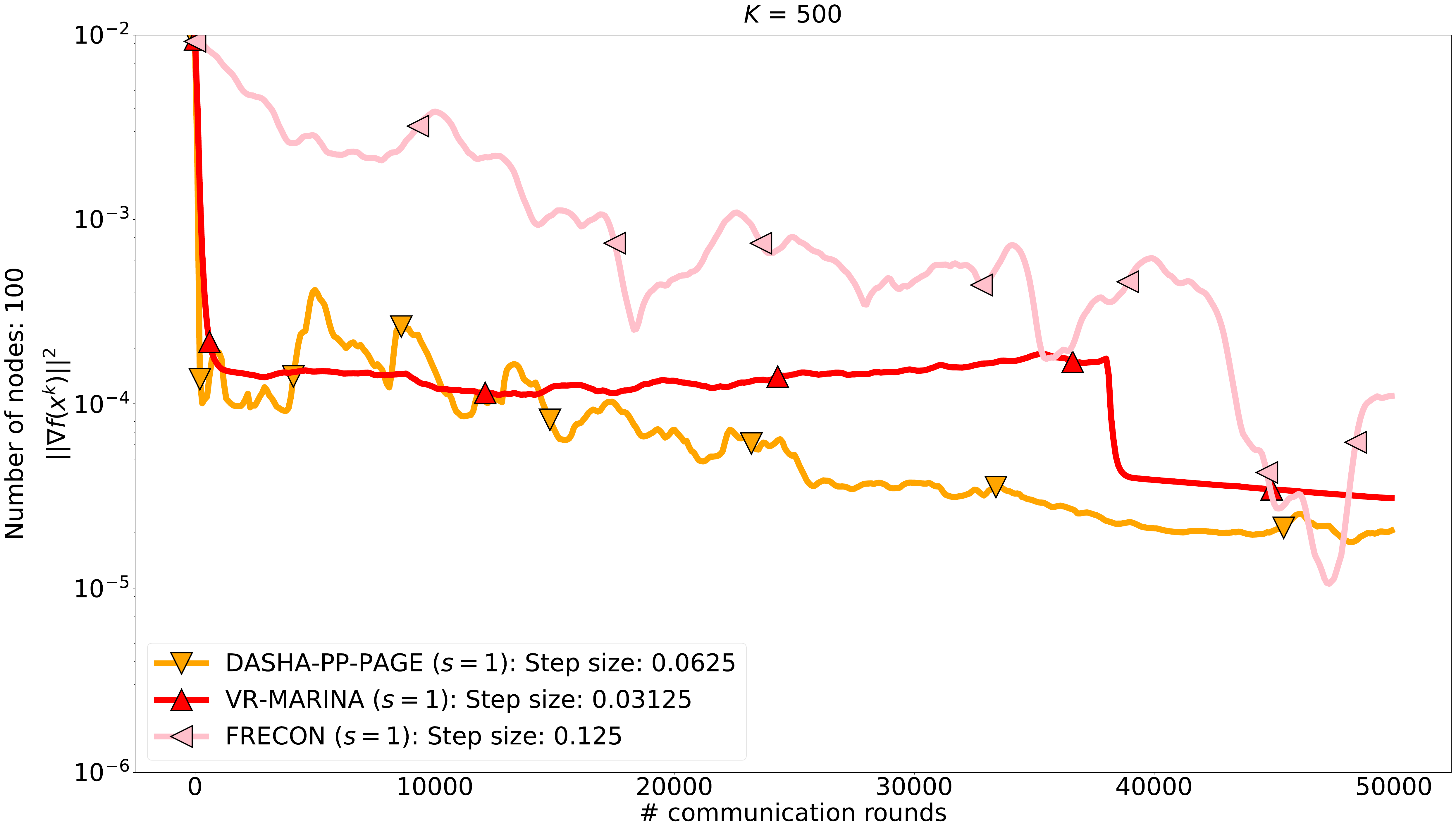}
      \caption{1 \% of nodes participating}
  \end{subfigure}
  \begin{subfigure}{.28\textwidth}
      \includegraphics[width=\textwidth]{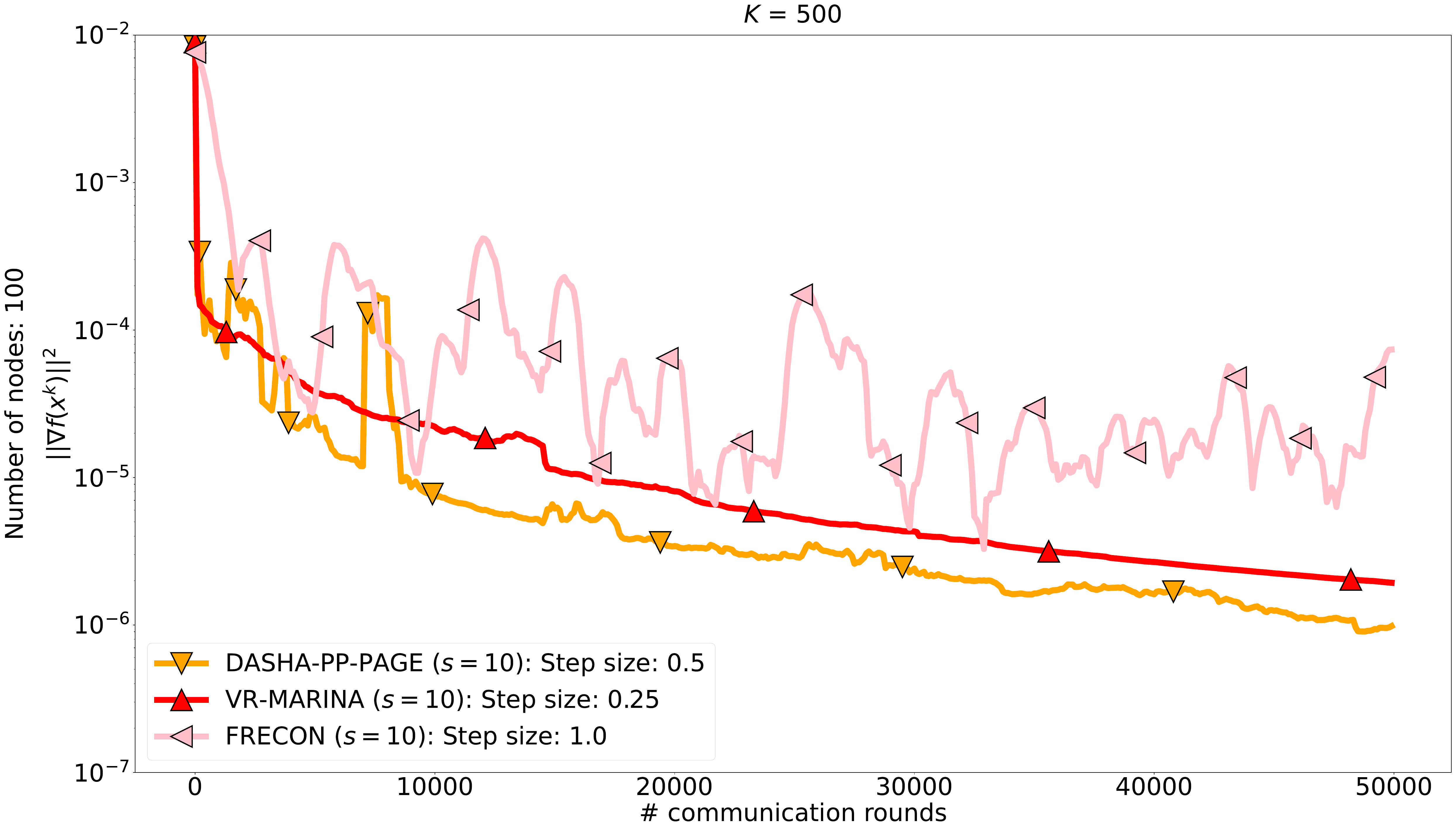}
      \caption{10 \% of nodes participating}
  \end{subfigure}
  \begin{subfigure}{.28\textwidth}
      \includegraphics[width=\textwidth]{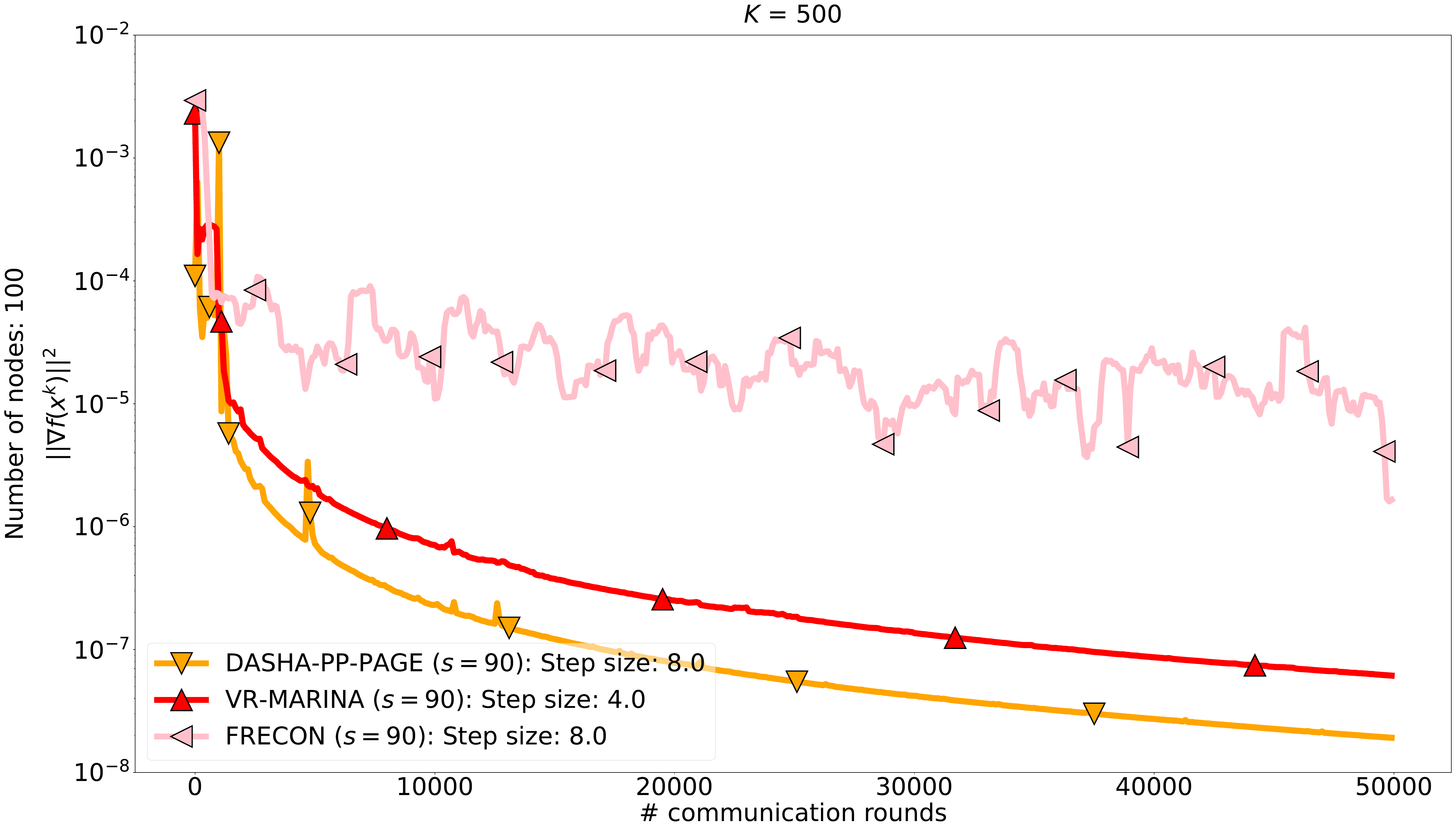}
      \caption{90 \% of nodes participating}
  \end{subfigure}
\caption{Classification task on \textit{real-sim}}
\label{fig:finite-sum-real-sim}
\end{figure}
\vspace{-0.5cm}
\begin{figure}[H]
  \centering
  \begin{subfigure}{.28\textwidth}
      \includegraphics[width=\textwidth]{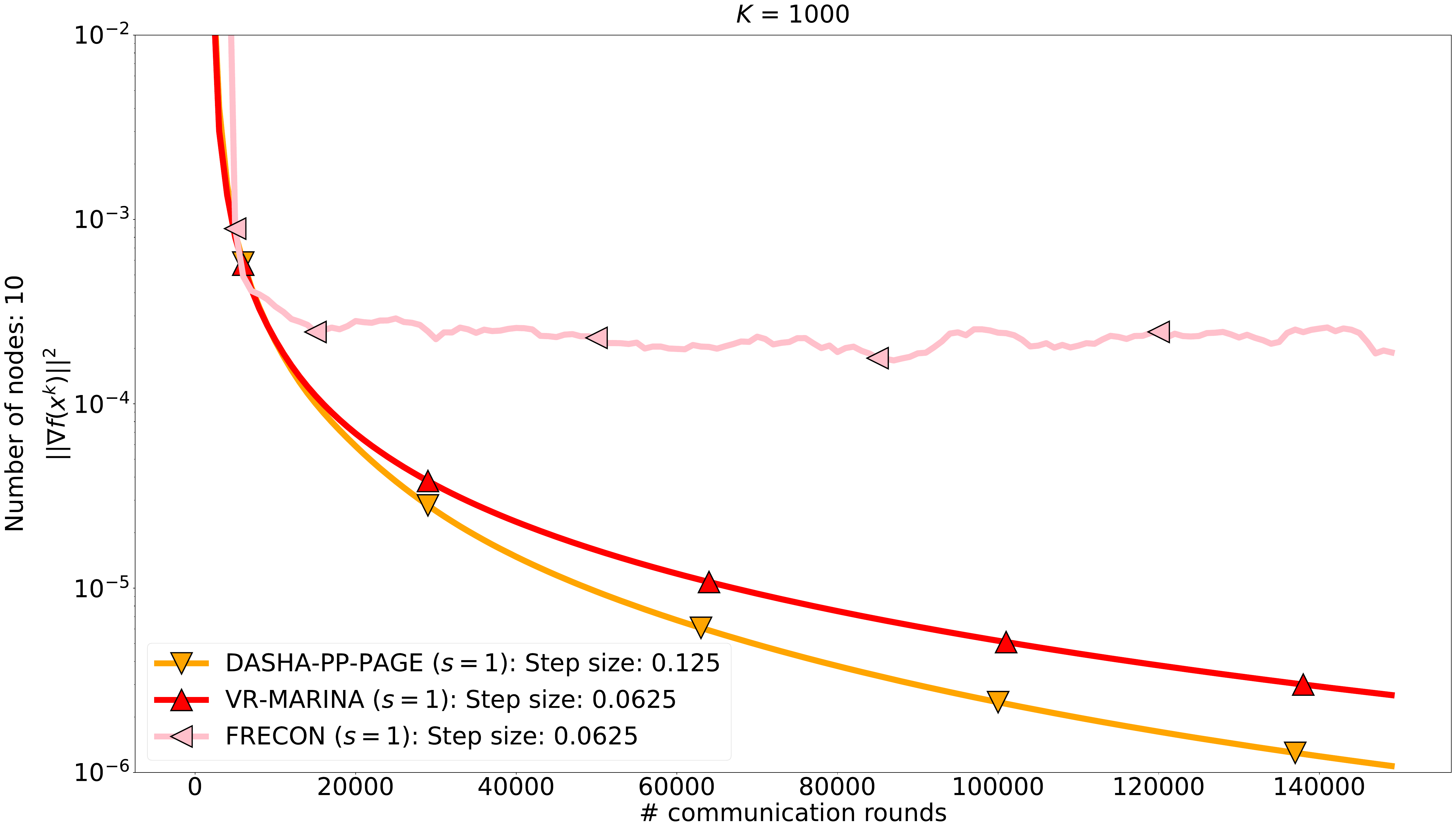}
      \caption{10 \% of nodes participating}
  \end{subfigure}
  \begin{subfigure}{.28\textwidth}
      \includegraphics[width=\textwidth]{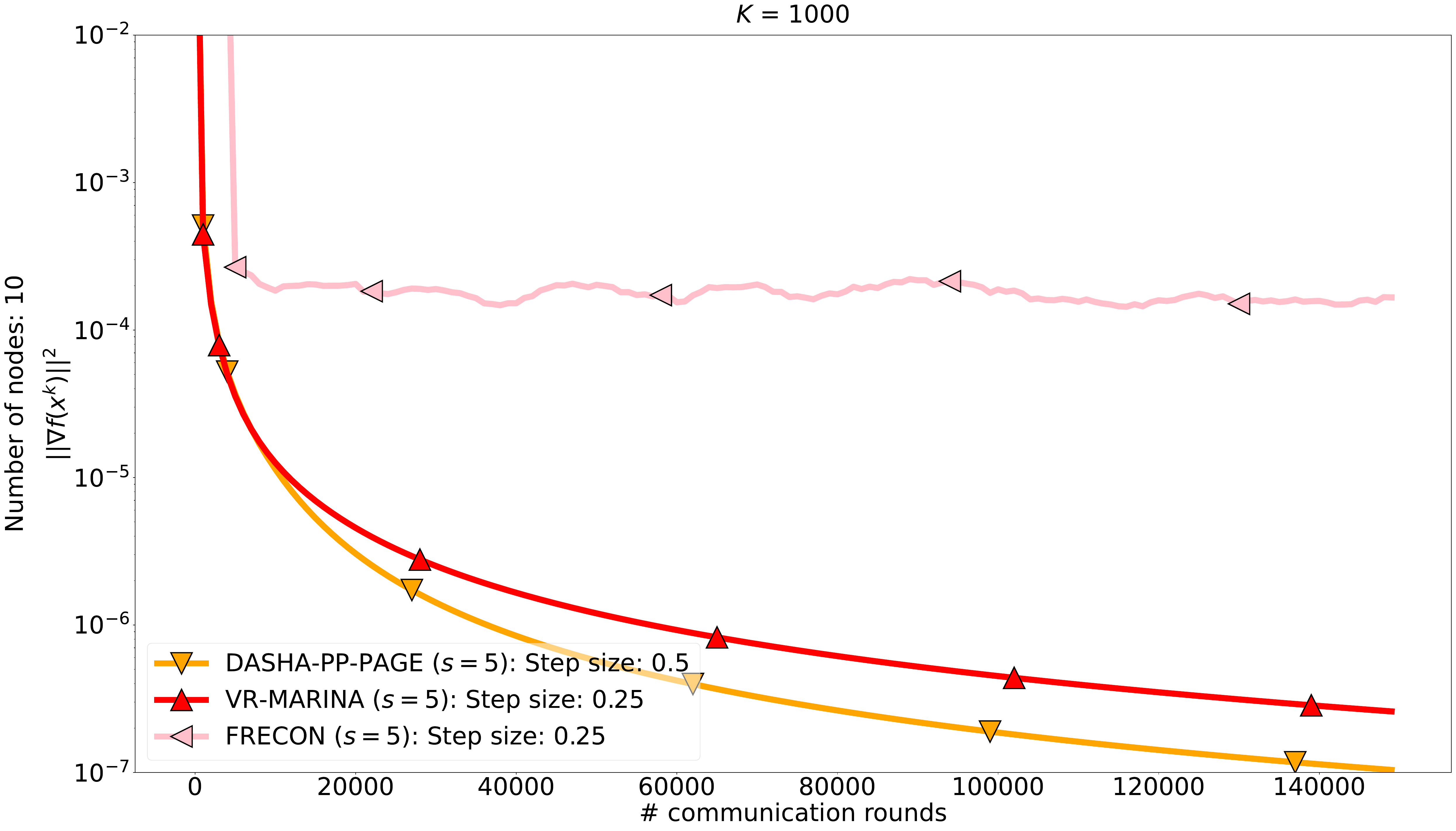}
      \caption{50 \% of nodes participating}
  \end{subfigure}
  \begin{subfigure}{.28\textwidth}
      \includegraphics[width=\textwidth]{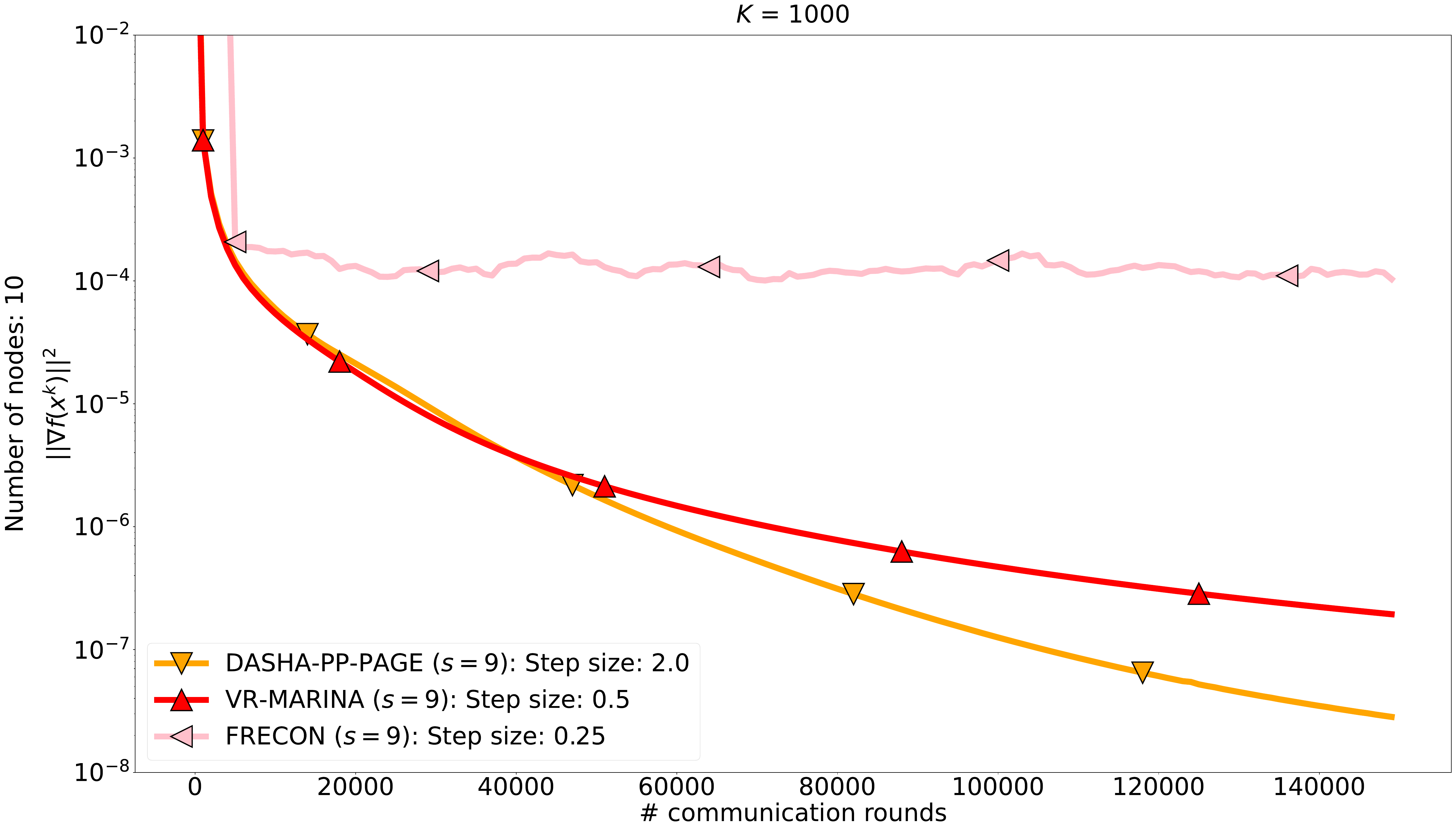}
      \caption{90 \% of nodes participating}
  \end{subfigure}
  \caption{Classification task on \textit{MNIST}}
\label{fig:finite-sum-mnist}
\end{figure}

{2. \bf Stochastic Setting.} In Figures~\ref{fig:stoch-real-sim} and \ref{fig:stoch-sum-mnist}, we consider the stochastic setting with the function from \eqref{eq:func_finite}. We can see that \algname{DASHA-PP} convergences to high accuracy solutions, unlike \algname{FRECON}. Moreover, \algname{DASHA-PP} improves the convergence rates of \algname{MARINA}.

\begin{figure}[H]
  \centering
  \begin{subfigure}{.28\textwidth}
      \includegraphics[width=\textwidth]{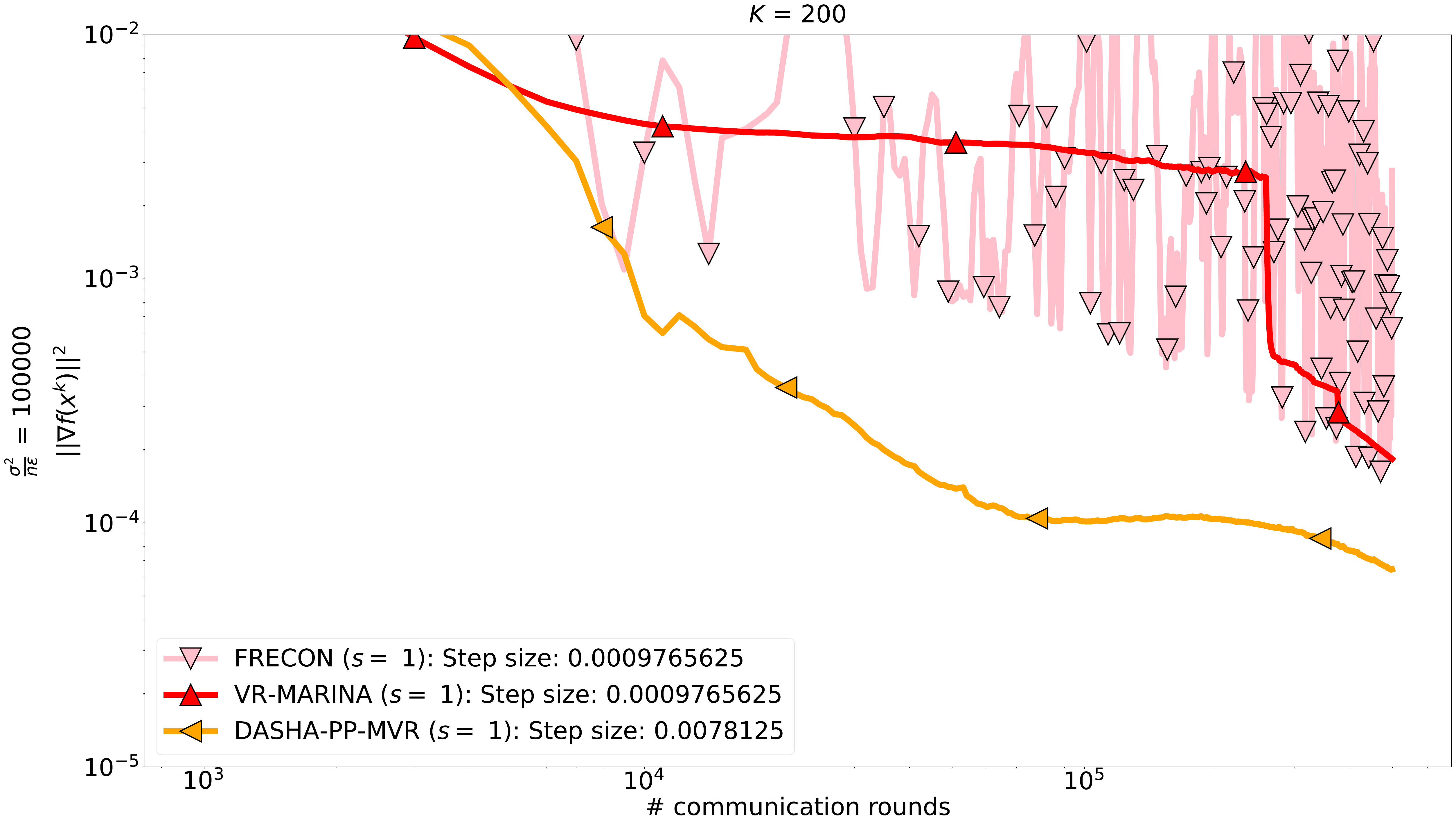}
      \caption{10 \% of nodes participating}
  \end{subfigure}
  \begin{subfigure}{.28\textwidth}
      \includegraphics[width=\textwidth]{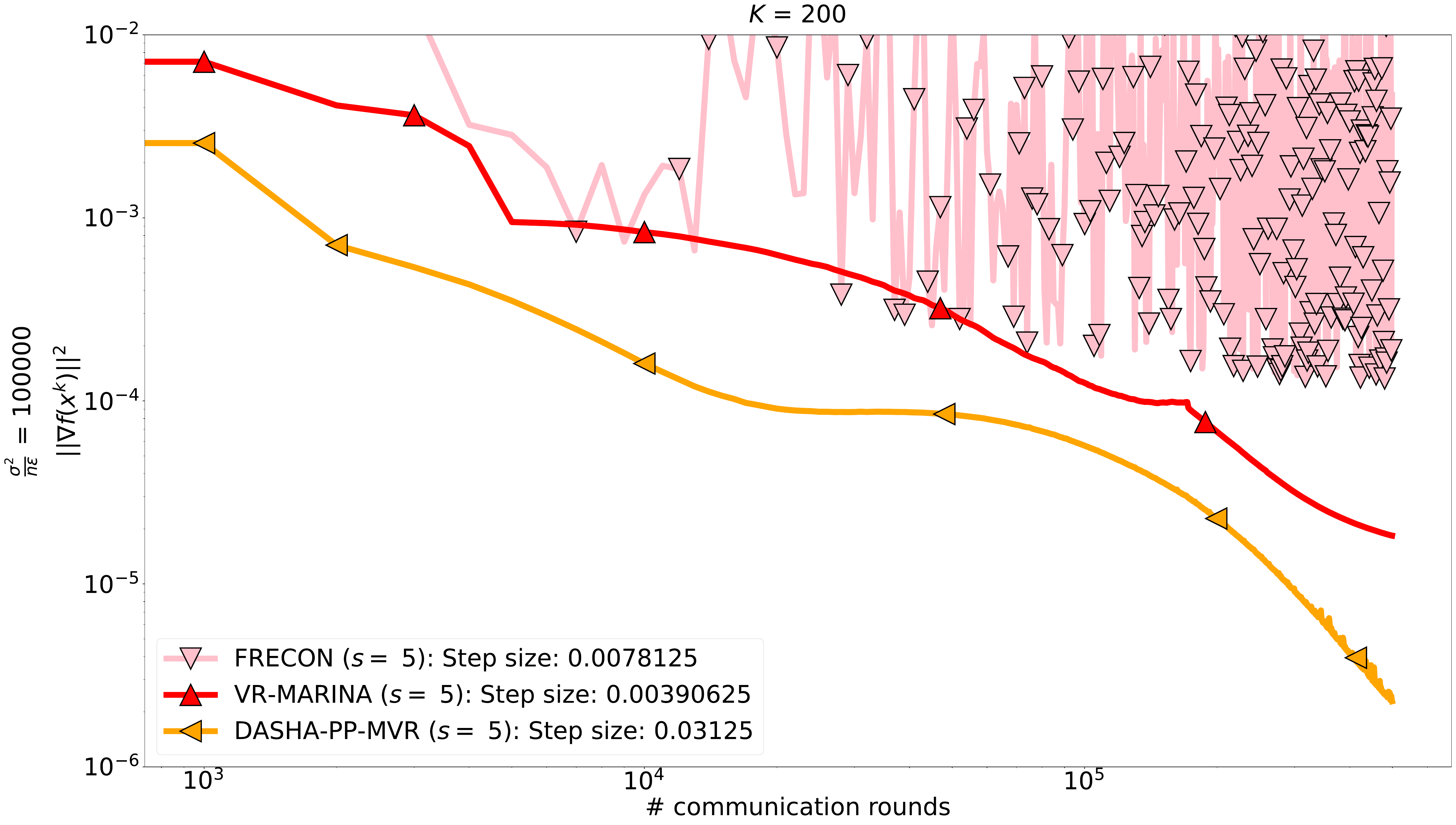}
      \caption{50 \% of nodes participating}
  \end{subfigure}
  \begin{subfigure}{.28\textwidth}
      \includegraphics[width=\textwidth]{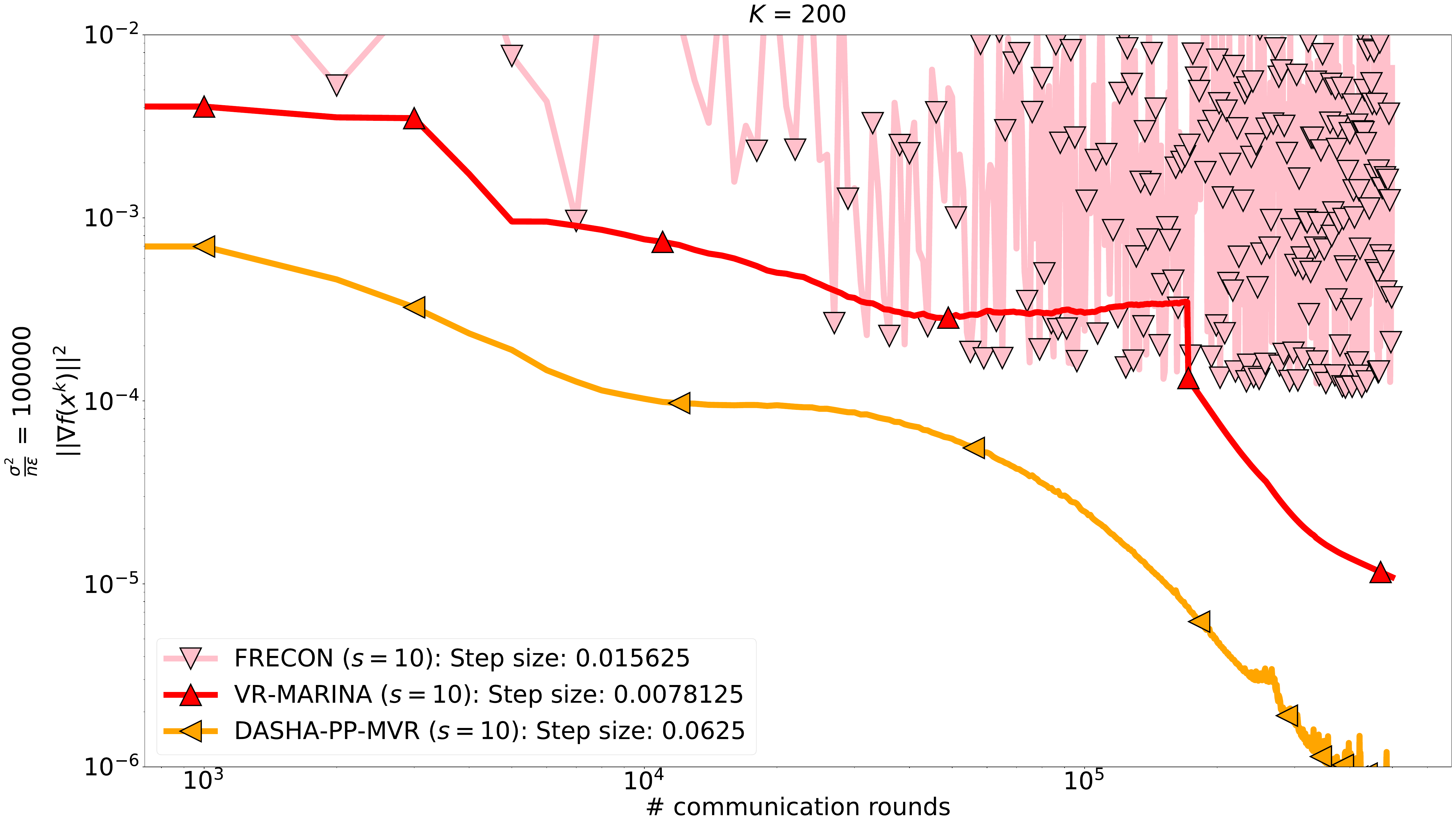}
      \caption{100 \% of nodes participating}
  \end{subfigure}
\caption{Classification task on \textit{real-sim}}
\label{fig:stoch-real-sim}
\end{figure}
\vspace{-0.5cm}
\begin{figure}[H]
  \centering
  \begin{subfigure}{.28\textwidth}
      \includegraphics[width=\textwidth]{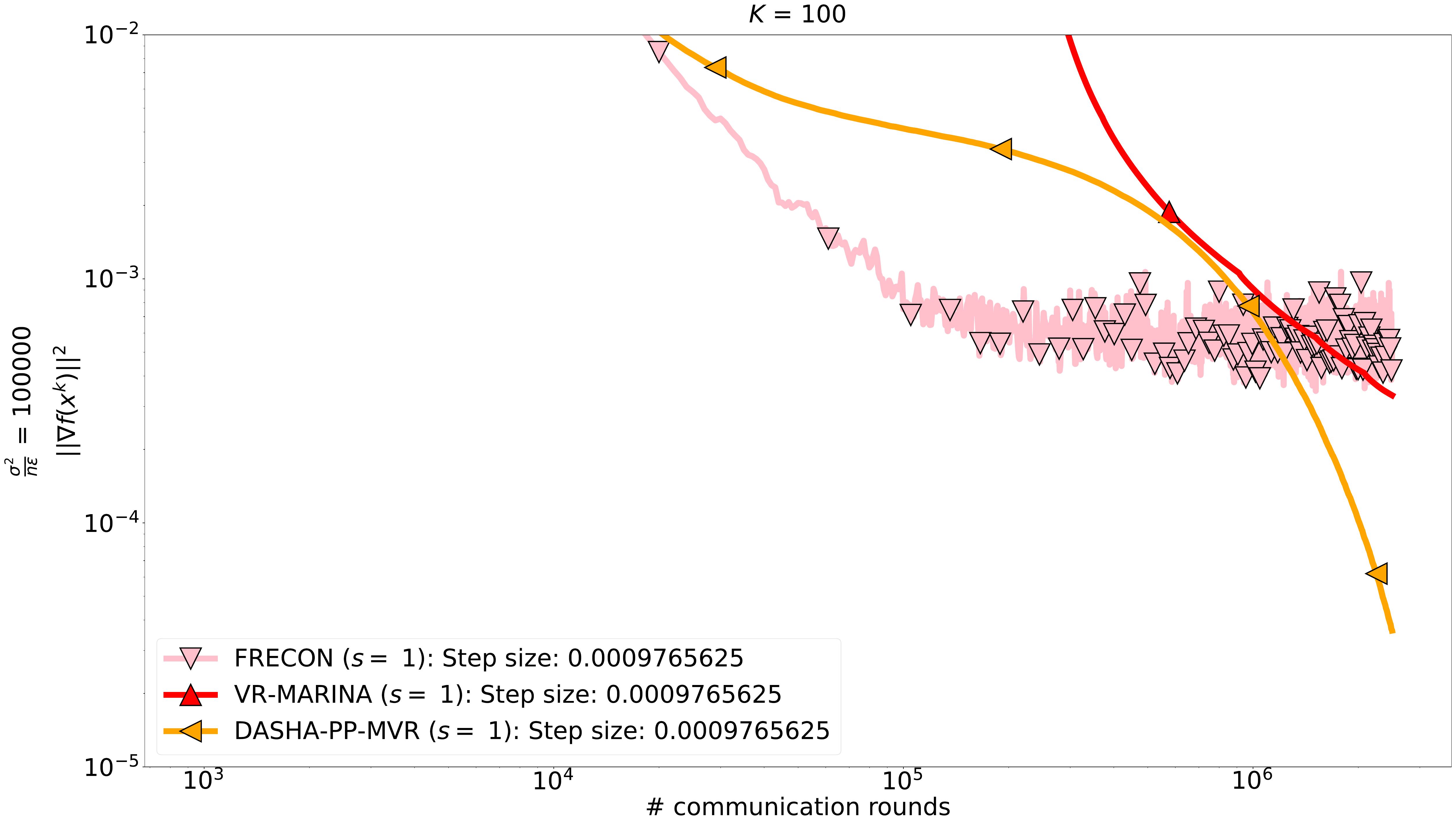}
      \caption{10 \% of nodes participating}
  \end{subfigure}
  \begin{subfigure}{.28\textwidth}
      \includegraphics[width=\textwidth]{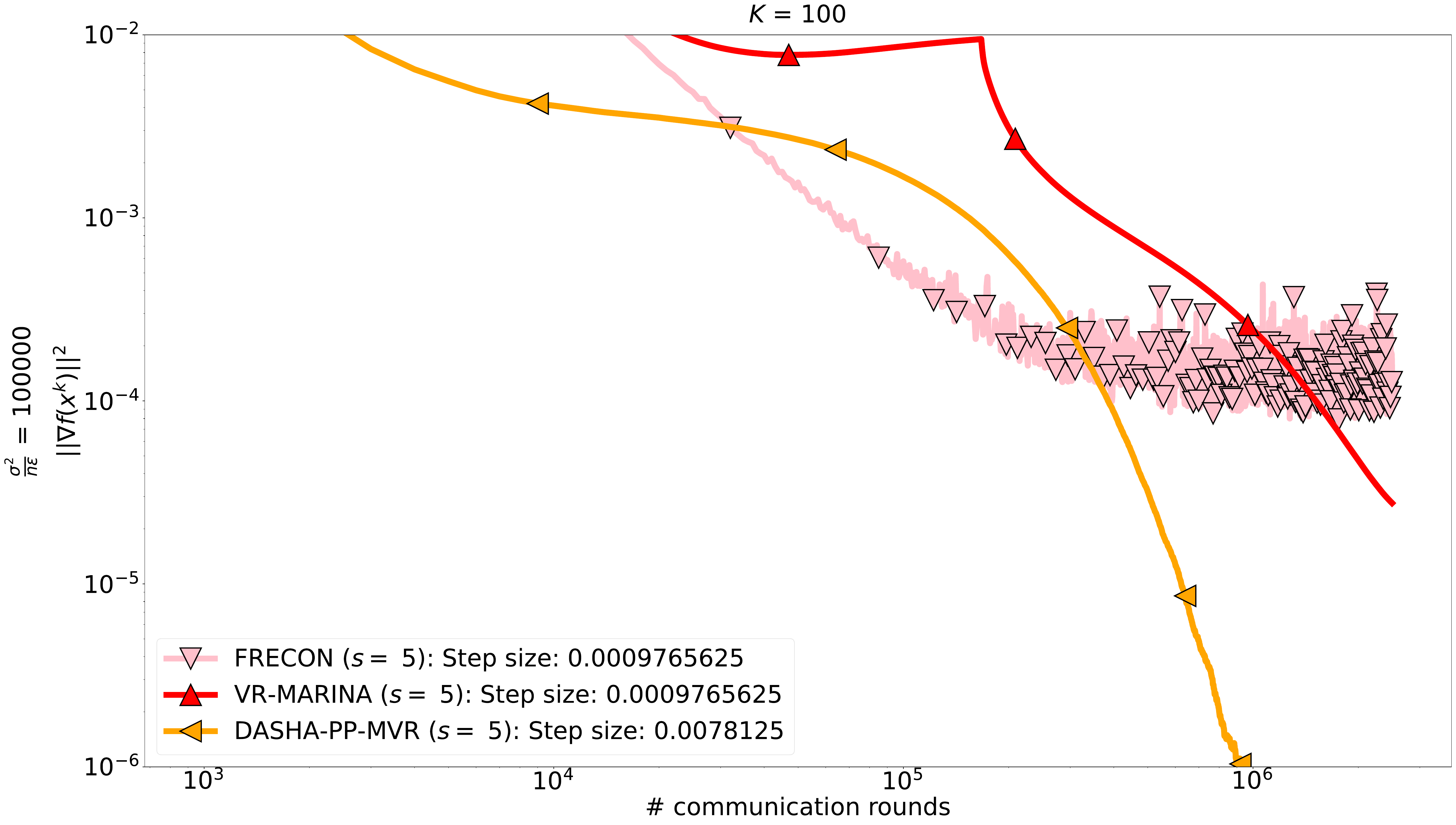}
      \caption{50 \% of nodes participating}
  \end{subfigure}
  \begin{subfigure}{.28\textwidth}
      \includegraphics[width=\textwidth]{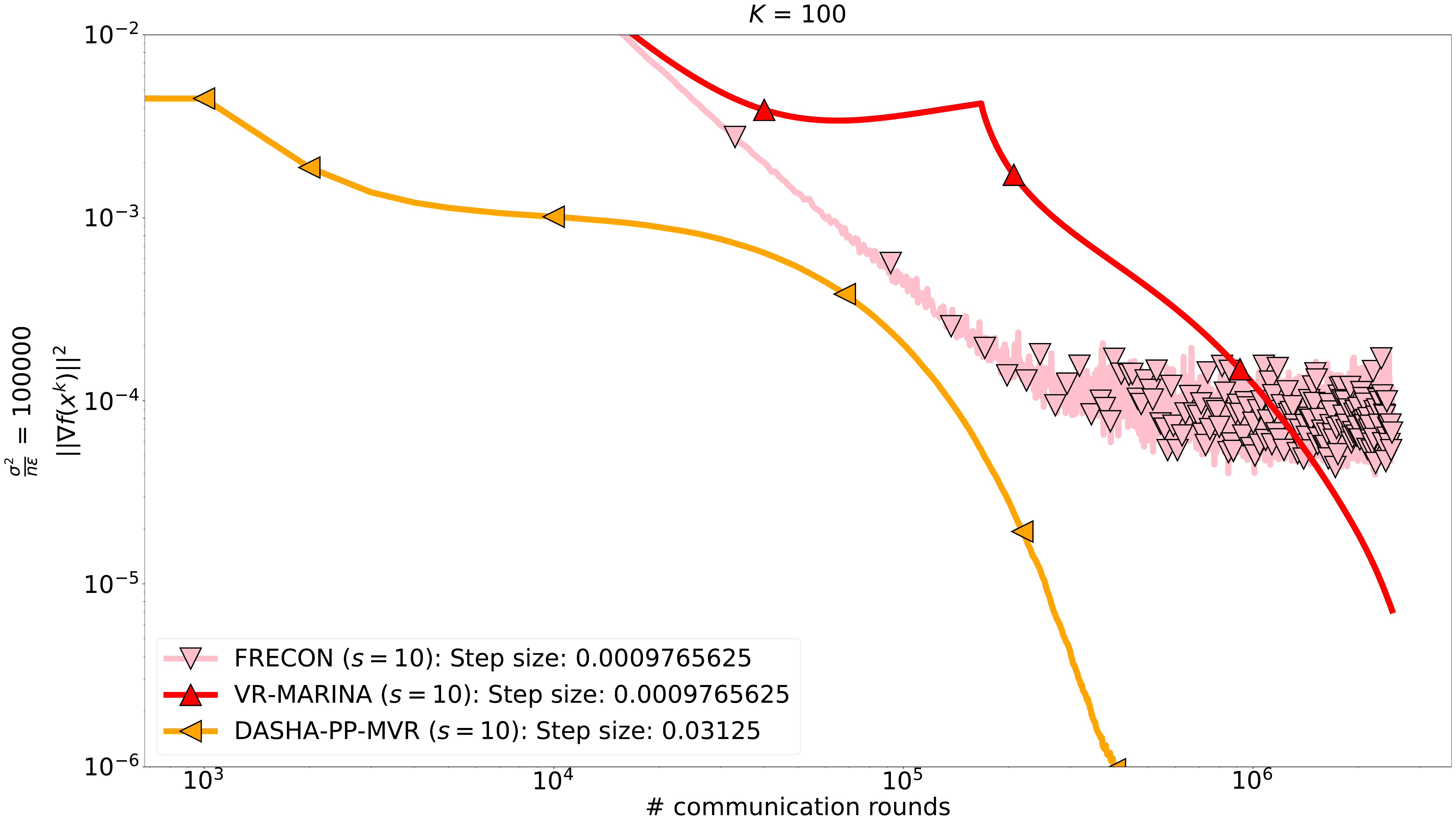}
      \caption{100 \% of nodes participating}
  \end{subfigure}
  \caption{Classification task on \textit{MNIST}}
\label{fig:stoch-sum-mnist}
\end{figure}

\newpage

\section{Original \algname{DASHA} and \algname{DASHA-MVR} Methods}
To simplify the discussion and explanation from the main part, we present the algorithms from \citep{tyurin2022dasha}

\begin{algorithm*}
  \caption{\algname{DASHA}}
  \label{alg:main_algorithm_dasha}
  \begin{algorithmic}[1]
  \STATE \textbf{Input:} starting point $x^0 \in \R^d$, stepsize $\gamma > 0$, momentum $a \in (0, 1]$, 
  number of iterations~$T \geq 1$
  \STATE Initialize $g^0_i\in \R^d$ on the nodes and  $g^0 = \frac{1}{n}\sum_{i=1}^n g^0_i$ on the server
  \FOR{$t = 0, 1, \dots, T - 1$}
  \STATE $x^{t+1} = x^t - \gamma g^t$
  \STATE Broadcast $x^{t+1}$ and $x^{t}$
  \FOR{$i = 1, \dots, n$ in parallel}
  \STATE {\green $m^{t+1}_i = \cC_i\left(\nabla f_i(x^{t+1}) - \nabla f_i(x^{t}) - a \left(g^{t}_i - \nabla f_i(x^{t})\right)\right)$}
  \STATE $g^{t+1}_i = g^{t}_i + m^{t+1}_i$
  \STATE Send $m^{t+1}_i$ to the server
  \ENDFOR
  \STATE $g^{t+1} = g^t + \frac{1}{n} \sum_{i=1}^n m^{t+1}_i$
  \ENDFOR
  \STATE \textbf{Output:} $\hat{x}^T$ chosen uniformly at random from $\{x^t\}_{k=0}^{T-1}$ 
  \end{algorithmic}
\end{algorithm*}

\begin{algorithm*}
  \caption{\algname{DASHA-MVR} (with batch size $B = 1$)}
  \label{alg:main_algorithm_dasha_mvr}
  \begin{algorithmic}[1]
  \STATE \textbf{Input:} starting point $x^0 \in \R^d$, stepsize $\gamma > 0$, momentums $a, b \in (0, 1]$, 
  number of iterations~$T \geq 1$
  \STATE Initialize $g^0_i\in \R^d$ on the nodes and  $g^0 = \frac{1}{n}\sum_{i=1}^n g^0_i$ on the server
  \FOR{$t = 0, 1, \dots, T - 1$}
  \STATE $x^{t+1} = x^t - \gamma g^t$
  \STATE Broadcast $x^{t+1}$ and $x^{t}$
  \FOR{$i = 1, \dots, n$ in parallel}
  \STATE {\green $h^{t+1}_i = \nabla f_i(x^{t+1};\xi^{t+1}_{i}) + (1 - b) (h^t_i - \nabla f_i(x^{t};\xi^{t+1}_{i})), \quad \xi^{t+1}_{i} \sim \mathcal{D}_i$}
  \STATE {\green $m^{t+1}_i = \cC_i\left(h^{t+1}_i - h^{t}_i - a \left(g^{t}_i - h^{t}_i\right)\right)$}
  \STATE $g^{t+1}_i = g^{t}_i + m^{t+1}_i$
  \STATE Send $m^{t+1}_i$ to the server
  \ENDFOR
  \STATE $g^{t+1} = g^t + \frac{1}{n} \sum_{i=1}^n m^{t+1}_i$
  \ENDFOR
  \STATE \textbf{Output:} $\hat{x}^T$ chosen uniformly at random from $\{x^t\}_{k=0}^{T-1}$ 
  \end{algorithmic}
\end{algorithm*}

\newpage

\section{Problem of Estimating the Mean in the Partial Participation Setting}
\label{sec:partial_participation_sampling}
We now provide the example to explain why the only choice of $B = \cO\left(\min\left\{\frac{1}{\probavailable}\sqrt{\frac{m}{n}},\frac{L_{\max}^2}{\mathbbm{1}_{\probavailable}^2 \widehat{L}^2}\right\}\right)$ and $B = \cO\left(\min\left\{\frac{\sigma}{\probavailable \sqrt{\varepsilon} n}, \frac{L_{\sigma}^2}{\mathbbm{1}_{\probavailable}^2 \widehat{L}^2}\right\}\right)$ in \algname{\algorithmname-PAGE} and \algname{\algorithmname-MVR}, accordingly, guarantees the degeneration up to $\nicefrac{1}{\probavailable}.$ This is surprising, because in methods with the variance reduction of stochastic gradients \citep{PAGE, tyurin2022dasha} we can take the size of batch size $B = \cO\left(\sqrt{\frac{m}{n}}\right)$ and $B = \cO\left(\frac{\sigma}{\sqrt{\varepsilon} n}\right)$ and guarantee the optimality. Note that the smaller the batch size $B$, the more the server and the nodes have to communicate to get $\varepsilon$-solution.

Let us consider the task of estimating the mean of vectors in the distributed setting. Suppose that we have $n$ nodes, and each of them contains $m$ vectors $\{x_{ij}\}_{j=1}^m$, where $x_{ij} \in \R^d$ for all $i \in [n], j \in [m].$ First, let us consider that each node samples a mini-batch $I^i$ of size $B$ with replacement and sends it to the server. Then the server calculates the mean of the mini-batches from nodes. One can easily show that the variance of the estimator is
\begin{align}
  \label{eq:partial_participation_sampling:all_nodes}
  &\Exp{\norm{\frac{1}{n B}\sum_{i=1}^n \sum_{j \in I^i} x_{ij} - \frac{1}{n m}\sum_{i=1}^n \sum_{j=1}^m x_{ij}}^2} \\
  &= \frac{1}{n B} \frac{1}{n m} \sum_{i=1}^n \sum_{j=1}^m \norm{x_{ij} - \frac{1}{m}\sum_{j=1}^m x_{ij}}^2.\nonumber
\end{align}
Next, we consider the same task in the partial participation setting with $s$--nice sampling, i.e., we sample a random set $S \subset [n]$ of $s \in [n]$ nodes without replacement and receive the mini-batches only from the sampled nodes. Such sampling of nodes satisfy Assumption~\ref{ass:partial_participation} with $\probavailable = \nicefrac{s}{n}$ and $\probavailable = \nicefrac{s (s-1)}{n (n-1)}$. In this case, the variance of the estimator (See Lemma~\ref{lemma:sampling} with $r_i = 0$ and $s_i = \sum_{j \in I^i} x_{ij}$) is
\begin{align}
  \label{eq:partial_participation_sampling:s_nodes}
  &\Exp{\norm{\frac{1}{s B}\sum_{i \in S} \sum_{j \in I^i} x_{ij} - \frac{1}{n m}\sum_{i=1}^n \sum_{j=1}^m x_{ij}}^2} \\
  &= \frac{1}{s B}\underbrace{\frac{1}{n m} \sum_{i=1}^n \sum_{j=1}^m \norm{x_{ij} - \frac{1}{m}\sum_{j=1}^m x_{ij}}^2}_{\mathcal{L}_{\max}^2} \nonumber\\
  &\quad+\frac{n - s}{s (n - 1)} \underbrace{\frac{1}{n} \sum_{i=1}^n \norm{\frac{1}{m}\sum_{j=1}^m x_{ij} - \frac{1}{n m}\sum_{i=1}^n \sum_{j=1}^m x_{ij}}^2}_{\widehat{\mathcal{L}}^2} \nonumber.
\end{align}
Let us assume that $s \leq \nicefrac{n}{2}.$ Note that \eqref{eq:partial_participation_sampling:all_nodes} scales with any $B \geq 1,$ while \eqref{eq:partial_participation_sampling:s_nodes} only scales when $B = \cO\left(\nicefrac{\mathcal{L}_{\max}^2}{\widehat{\mathcal{L}}^2}\right).$ In other words, for large enough $B,$ the variance in \eqref{eq:partial_participation_sampling:s_nodes} does not significantly improves with the growth of $B$ due to the term $\widehat{\mathcal{L}}^2$. In our proof, due to partial participation, the variance from \eqref{eq:partial_participation_sampling:s_nodes} naturally appears, and we get the same effect. As was mentioned in Sections~\ref{sec:finite_sum_setting} and \ref{sec:stochastic_setting}, it can be seen in our convergence rate bounds.

\newpage

\section{Auxiliary facts}
We list auxiliary facts that we use in our proofs:
\begin{enumerate}
    \item 
        For all $x, y \in \R^d,$ we have
        \begin{align}
            \norm{x + y}^2 \leq 2\norm{x}^2 + 2\norm{y}^2
            \label{auxiliary:jensen_inequality}
        \end{align}
    \item
        Let us take a \textit{random vector} $\xi \in \R^d$, then
        \begin{align}
            \Exp{\norm{\xi}^2} = \Exp{\norm{\xi - \Exp{\xi}}^2} + \norm{\Exp{\xi}}^2.
            \label{auxiliary:variance_decomposition}
        \end{align}
\end{enumerate}

\subsection{Sampling Lemma}
This section provides a lemma that we regularly use in our proofs, and it is useful for samplings that satisfy Assumption~\ref{ass:partial_participation}.
\begin{lemma}
  \label{lemma:sampling}
  Suppose that a set $S$ is a random subset of a set $[n]$ such that
  \begin{enumerate}
    \Item \begin{align*}\Prob\left(i \in S\right) = \probavailable, \quad \forall i \in [n],\end{align*}
    \Item \begin{align*}\Prob\left(i \in S, j \in S\right) = \probpairaa, \quad \forall i \neq j \in [n],\end{align*}
    \Item \begin{align*}
      \probpairaa \leq \probavailable^2,
    \end{align*}
  \end{enumerate}
  where $\probavailable \in (0, 1]$ and $\probpairaa \in [0, 1].$ Let us take random \textit{independent} vectors $s_i \in \R^d$ for all $i \in [n]$, nonrandom vector $r_i \in \R^d$ for all $i \in [n],$ and random vectors
  \begin{align*}
    v_i = \begin{cases}
      r_i + \frac{1}{\probavailable} s_i, i \in S, \\
      r_i, i \not\in S,
    \end{cases}
  \end{align*}
  then
  \begin{align*}
    &\Exp{\norm{\frac{1}{n}\sum_{i=1}^n v_i - \Exp{\frac{1}{n}\sum_{i=1}^n v_i}}^2} \\
    &= \frac{1}{n^2 \probavailable}\sum_{i=1}^n\Exp{\norm{s_i - \Exp{s_i}}^2} +\frac{\probavailable - \probpairaa}{n^2 \probavailable^2}\sum_{i=1}^n\norm{\Exp{s_i}}^2 + \frac{\probpairaa - \probavailable^2}{\probavailable^2}\norm{\frac{1}{n}\sum_{i=1}^n\Exp{s_i}}\\
    &\leq \frac{1}{n^2 \probavailable}\sum_{i=1}^n\Exp{\norm{s_i - \Exp{s_i}}^2} +\frac{\probavailable - \probpairaa}{n^2 \probavailable^2}\sum_{i=1}^n\norm{\Exp{s_i}}^2.
  \end{align*}
\end{lemma}

\begin{proof}
  Let us define additional constants $\probpairan$ and $\probpairnn$, such that
  \begin{enumerate}
    \Item \begin{align*}\Prob\left(i \in S, j \not\in S\right) = \probpairan, \quad \forall i \neq j \in [n],\end{align*}
    \Item \begin{align*}\Prob\left(i \not\in S, j \not\in S\right) = \probpairnn, \quad \forall i \neq j \in [n].\end{align*}
  \end{enumerate}
  Note, that 
  \begin{align}
    \label{auxiliary_facts_partial_participation:one}
    \probavailable = \probpairaa + \probpairan
  \end{align}
  and 
  \begin{align}
    \label{auxiliary_facts_partial_participation:two}
    \probpairnn = 1 - \probpairaa - 2 \probpairan.
  \end{align}
  Using the law of total expectation and 
  \begin{align*}
    \Exp{v_i} = \probavailable \left(r_i + \Exp{\frac{1}{\probavailable} s_i}\right) + (1 - \probavailable) r_i = r_i + \Exp{s_i},
  \end{align*}
  we have
  \begin{align*}
    &\Exp{\norm{\frac{1}{n}\sum_{i=1}^n v_i - \Exp{\frac{1}{n}\sum_{i=1}^n v_i}}^2} \\
    &=\frac{1}{n^2}\sum_{i=1}^n\Exp{\norm{v_i - \left(r_i + \Exp{s_i}\right)}^2} \\
    &\quad +\frac{1}{n^2}\sum_{i\neq j}^n\Exp{\inp{v_i - \left(r_i + \Exp{s_i}\right)}{v_j - \left(r_j + \Exp{s_j}\right)}} \\
    &=\frac{\probavailable}{n^2}\sum_{i=1}^n\Exp{\norm{r_i + \frac{1}{\probavailable} s_i - \left(r_i + \Exp{s_i}\right)}^2} \\
    &\quad +\frac{1 - \probavailable}{n^2}\sum_{i=1}^n\norm{r_i - \left(r_i + \Exp{s_i}\right)}^2 \\
    &\quad +\frac{\probpairaa}{n^2}\sum_{i\neq j}^n\Exp{\inp{r_i + \frac{1}{\probavailable} s_i - \left(r_i + \Exp{s_i}\right)}{r_j + \frac{1}{\probavailable} s_j - \left(r_j + \Exp{s_j}\right)}} \\
    &\quad +\frac{2 \probpairan}{n^2}\sum_{i\neq j}^n\Exp{\inp{r_i + \frac{1}{\probavailable} s_i - \left(r_i + \Exp{s_i}\right)}{r_j - \left(r_j + \Exp{s_j}\right)}} \\
    &\quad +\frac{\probpairnn}{n^2}\sum_{i\neq j}^n\inp{r_i - \left(r_i + \Exp{s_i}\right)}{r_j - \left(r_j + \Exp{s_j}\right)}.
  \end{align*}
  From the independence of random vectors $s_i$, we obtain
  \begin{align*}
    &\Exp{\norm{\frac{1}{n}\sum_{i=1}^n v_i - \Exp{\frac{1}{n}\sum_{i=1}^n v_i}}^2} \\
    &=\frac{\probavailable}{n^2}\sum_{i=1}^n\Exp{\norm{\frac{1}{\probavailable} s_i - \Exp{s_i}}^2} \\
    &\quad +\frac{1 - \probavailable}{n^2}\sum_{i=1}^n\norm{\Exp{s_i}}^2 \\
    &\quad +\frac{\probpairaa (1 - \probavailable)^2}{n^2 \probavailable^2}\sum_{i\neq j}^n\inp{\Exp{s_i}}{\Exp{s_j}} \\
    &\quad +\frac{2 \probpairan (\probavailable - 1)}{n^2 \probavailable}\sum_{i\neq j}^n\inp{\Exp{s_i}}{\Exp{s_j}} \\
    &\quad +\frac{\probpairnn}{n^2}\sum_{i\neq j}^n\inp{\Exp{s_i}}{\Exp{s_j}}.
  \end{align*}
  Using \eqref{auxiliary_facts_partial_participation:one} and \eqref{auxiliary_facts_partial_participation:two}, we have
  \begin{align*}
    &\Exp{\norm{\frac{1}{n}\sum_{i=1}^n v_i - \Exp{\frac{1}{n}\sum_{i=1}^n v_i}}^2} \\
    &=\frac{\probavailable}{n^2}\sum_{i=1}^n\Exp{\norm{\frac{1}{\probavailable} s_i - \Exp{s_i}}^2} \\
    &\quad +\frac{1 - \probavailable}{n^2}\sum_{i=1}^n\norm{\Exp{s_i}}^2 \\
    &\quad +\frac{\probpairaa - \probavailable^2}{n^2 \probavailable^2}\sum_{i\neq j}^n\inp{\Exp{s_i}}{\Exp{s_j}} \\
    &\overset{\eqref{auxiliary:variance_decomposition}}{=}\frac{1}{n^2 \probavailable}\sum_{i=1}^n\Exp{\norm{s_i - \Exp{s_i}}^2} \\
    &\quad +\frac{1 - \probavailable}{n^2 \probavailable}\sum_{i=1}^n\norm{\Exp{s_i}}^2 \\
    &\quad +\frac{\probpairaa - \probavailable^2}{n^2 \probavailable^2}\sum_{i\neq j}^n\inp{\Exp{s_i}}{\Exp{s_j}} \\
    &=\frac{1}{n^2 \probavailable}\sum_{i=1}^n\Exp{\norm{s_i - \Exp{s_i}}^2} \\
    &\quad +\frac{\probavailable - \probpairaa}{n^2 \probavailable^2}\sum_{i=1}^n\norm{\Exp{s_i}}^2 \\
    &\quad +\frac{\probpairaa - \probavailable^2}{\probavailable^2}\norm{\frac{1}{n}\sum_{i=1}^n\Exp{s_i}}.
  \end{align*}
  Finally, using that $\probpairaa \leq \probavailable^2$, we have
  \begin{align*}
    &\Exp{\norm{\frac{1}{n}\sum_{i=1}^n v_i - \Exp{\frac{1}{n}\sum_{i=1}^n v_i}}^2} \\
    &\leq\frac{1}{n^2 \probavailable}\sum_{i=1}^n\Exp{\norm{s_i - \Exp{s_i}}^2} +\frac{\probavailable - \probpairaa}{n^2 \probavailable^2}\sum_{i=1}^n\norm{\Exp{s_i}}^2.
  \end{align*}
\end{proof}

\subsection{Compressors Facts}

We define the Rand$K$ compressor that chooses without replacement $K$ coordinates, scales them by a constant factor to preserve unbiasedness and zero-out other coordinates.
\begin{definition}
    \label{def:rand_k}
    Let us take a random subset $S$ from $[d],$ $|S| = K,$ $K \in [d].$ We say that a stochastic mapping $\cC\,:\, \R^d \rightarrow \R^d$ is Rand$K$ if
    $$\cC(x) = \frac{d}{K} \sum_{j \in S} x_j e_j,$$ where $\{e_i\}_{i=1}^d$ is the standard unit basis.
\end{definition}

\begin{theorem}
    \label{theorem:rand_k}
    If $\cC$ is Rand$K$, then $\cC \in \mathbb{U}\left(\frac{d}{k} - 1\right).$
\end{theorem}
See the proof in \citep{beznosikov2020biased}.

\section{Proofs of Theorems}
\label{sec:proof_of_theorems}

There are three different sources of randomness in Algorithm~\ref{alg:main_algorithm}: the first one from vectors $\{k_i^{t+1}\}_{i=1}^n$, the second one from compressors $\{\cC_i\}_{i=1}^n$, and the third one from availability of nodes. We define $\ExpSub{k}{\cdot}$, $\ExpSub{\cC}{\cdot}$ and $\ExpSub{\probavailable}{\cdot}$ to be conditional expectations w.r.t.\,$\{k_i^{t+1}\}_{i=1}^n$, $\{\cC_i\}_{i=1}^n, $ and availability, accordingly, conditioned on all previous randomness. Moreover, we define $\ExpSub{t+1}{\cdot}$ to be a conditional expectation w.r.t. all randomness in iteration $t+1$ conditioned on all previous randomness. Note, that $\ExpSub{t+1}{\cdot} = \ExpSub{k}{\ExpSub{\cC}{\ExpSub{\probavailable}{\cdot}}}.$

In the case of \algname{\algorithmname-PAGE}, there are two different sources of randomness from $\{k_i^{t+1}\}_{i=1}^n$. We define $\ExpSub{\probpage}{\cdot}$ and $\ExpSub{B}{\cdot}$ to be conditional expectations w.r.t.\, the probabilistic switching and mini-batch indices $I_{i}^t$, accordingly, conditioned on all previous randomness. Note, that $\ExpSub{t+1}{\cdot} = \ExpSub{B}{\ExpSub{\cC}{\ExpSub{\probavailable}{\ExpSub{\probpage}{\cdot}}}}$ and $\ExpSub{t+1}{\cdot} = \ExpSub{B}{\ExpSub{\probpage}{\ExpSub{\cC}{\ExpSub{\probavailable}{\cdot}}}}.$

\subsection{Standard Lemmas in the Nonconvex Setting}

We start the proof of theorems by providing standard lemmas from the nonconvex optimization.

\begin{lemma}
  \label{lemma:page_lemma}
  Suppose that Assumption~\ref{ass:lipschitz_constant} holds and let $x^{t+1} = x^{t} - \gamma g^{t}$. Then for any $g^{t} \in \R^d$ and $\gamma > 0$, we have
  \begin{eqnarray}
    \label{eq:page_lemma}
    f(x^{t + 1}) \leq f(x^t) - \frac{\gamma}{2}\norm{\nabla f(x^t)}^2 - \left(\frac{1}{2\gamma} - \frac{L}{2}\right)
    \norm{x^{t+1} - x^t}^2 + \frac{\gamma}{2}\norm{g^{t} - \nabla f(x^t)}^2.
  \end{eqnarray}
\end{lemma}

\begin{proof}
  Using $L-$smoothness, we have 
  \begin{align*}
    f(x^{t+1}) &\leq f(x^t) + \inp{\nabla f(x^t)}{x^{t+1} - x^{t}} + \frac{L}{2} \norm{x^{t+1} - x^{t}}^2 \\
    &= f(x^t) - \gamma \inp{\nabla f(x^t)}{g^t} + \frac{L}{2} \norm{x^{t+1} - x^{t}}^2.
  \end{align*}
  Next, due to $-\inp{x}{y} = \frac{1}{2}\norm{x - y}^2 - \frac{1}{2}\norm{x}^2 - \frac{1}{2}\norm{y}^2,$ we obtain 
  \begin{align*}
    f(x^{t+1}) \leq f(x^t) -\frac{\gamma}{2} \norm{\nabla f(x^t)}^2 - \left(\frac{1}{2\gamma} - \frac{L}{2}\right) \norm{x^{t+1} - x^{t}}^2 + \frac{\gamma}{2}\norm{g^t - \nabla f(x^t)}^2.
  \end{align*}
\end{proof}

\begin{lemma}
  \label{lemma:good_recursion}
  Suppose that Assumption \ref{ass:lower_bound} holds and
  \begin{align}
      \label{eq:private:good_recursion}
      \Exp{f(x^{t+1})} + \gamma \Psi^{t+1} \leq \Exp{f(x^t)} - \frac{\gamma}{2}\Exp{\norm{\nabla f(x^t)}^2} + \gamma \Psi^{t} + \gamma C,
  \end{align}
  where $\Psi^{t}$ is a sequence of numbers, $\Psi^{t} \geq 0$ for all $t \in [T]$, constant $C \geq 0$, and constant $\gamma > 0.$ Then 
  \begin{align}
      \label{eq:good_recursion}
      \Exp{\norm{\nabla f(\widehat{x}^T)}^2} \leq \frac{2 \Delta_0}{\gamma T} + \frac{2\Psi^{0}}{T} + 2 C,
  \end{align}
  where a point $\widehat{x}^T$ is chosen uniformly from a set of points $\{x^t\}_{t=0}^{T-1}.$
\end{lemma}

\begin{proof}
  By unrolling \eqref{eq:private:good_recursion} for $t$ from $0$ to $T - 1$, we obtain
  \begin{align*}
      \frac{\gamma}{2}\sum_{t = 0}^{T - 1}\Exp{\norm{\nabla f(x^t)}^2} + \Exp{f(x^{T})} + \gamma \Psi^{T} \leq f(x^0) + \gamma \Psi^{0} + \gamma T C.
  \end{align*}
  We subtract $f^*$, divide inequality by $\frac{\gamma T}{2},$ and take into account that $f(x) \geq f^*$ for all $x \in \R$, and $\Psi^{t} \geq 0$ for all $t \in [T],$ to get the following inequality:
  \begin{align*}
      \frac{1}{T}\sum_{t = 0}^{T - 1}\Exp{\norm{\nabla f(x^t)}^2} \leq \frac{2 \Delta_0}{\gamma T} + \frac{2\Psi^{0}}{T} + 2 C.
  \end{align*}
  It is left to consider the choice of a point $\widehat{x}^T$ to complete the proof of the lemma.
\end{proof}

\begin{lemma}
  \label{lemma:gamma}
  If $0 < \gamma \leq (L + \sqrt{A})^{-1},$ $L > 0$, and $A \geq 0,$ then $$\frac{1}{2\gamma} - \frac{L}{2} - \frac{\gamma A}{2} \geq 0.$$
\end{lemma}
The lemma can be easily checked with the direct calculation.

\subsection{Generic Lemmas}
\begin{lemma}
  \label{lemma:g_h}
  Suppose that Assumptions~\ref{ass:compressors} and \ref{ass:partial_participation} hold and let us consider sequences $g^{t+1}_i$, $h^{t+1}_i,$ and $k^{t+1}_i$ from Algorithm~\ref{alg:main_algorithm}, then
  \begin{align}
      \label{eq:compressor_global_error}
      &\ExpSub{\cC}{\ExpSub{\probavailable}{\norm{g^{t+1} - h^{t+1}}^2}} \nonumber \\
      &\leq \frac{2 \omega}{n^2 \probavailable}\sum_{i=1}^n\norm{k^{t+1}_i}^2 +\frac{a^2 (\left(2 \omega + 1\right)\probavailable - \probpairaa)}{n^2 \probavailable^2}\sum_{i=1}^n\norm{g^{t}_i - h^{t}_i}^2 + (1 - a)^2\norm{g^{t} - h^{t}}^2,
  \end{align}
  and
  \begin{align}
      \label{eq:compressor_local_error}
      &\ExpSub{\cC}{\ExpSub{\probavailable}{\norm{g^{t+1}_i - h^{t+1}_i}^2}} \nonumber \\
      &\leq \frac{2 \omega}{\probavailable}\norm{k^{t+1}_i}^2 + \left(\frac{a^2(2\omega + 1 - \probavailable)}{\probavailable} + (1 - a)^2\right)\norm{g^{t}_i - h^{t}_i}^2\quad \forall i \in [n].
  \end{align}
\end{lemma}
\begin{proof}
  First, we estimate $\ExpSub{\cC}{\ExpSub{\probavailable}{\norm{g^{t+1} - h^{t+1}}^2}}$:
  \begin{align*}
      &\ExpSub{\cC}{\ExpSub{\probavailable}{\norm{g^{t+1} - h^{t+1}}^2}} \\
      &=\ExpSub{\cC}{\ExpSub{\probavailable}{\norm{g^{t+1} - h^{t+1} - \ExpSub{\cC}{\ExpSub{\probavailable}{g^{t+1} - h^{t+1}}}}^2}} + \norm{\ExpSub{\cC}{\ExpSub{\probavailable}{g^{t+1} - h^{t+1}}}}^2,
  \end{align*}
  where we used \eqref{auxiliary:variance_decomposition}.
  Due to Assumption~\ref{ass:partial_participation}, we have
  \begin{align*}
      &\ExpSub{\cC}{\ExpSub{\probavailable}{g^{t+1}_i}} \\
      &=\probavailable \ExpSub{\cC}{g^{t}_i + \cC_i\Bigg(\frac{1}{\probavailable}k^{t+1}_i - \frac{a}{\probavailable} \left(g^{t}_i - h^{t}_i\right)\Bigg)} + (1 - \probavailable) g^{t}_i \\
      &=g^{t}_i  + \probavailable\ExpSub{\cC}{\cC_i\Bigg(\frac{1}{\probavailable}k^{t+1}_i - \frac{a}{\probavailable} \left(g^{t}_i - h^{t}_i\right)\Bigg)} \\
      &=g^{t}_i  + k^{t+1}_i - a \left(g^{t}_i - h^{t}_i\right), \\
  \end{align*}
  and 
  \begin{align*}
      &\ExpSub{\cC}{\ExpSub{\probavailable}{h^{t+1}_i}} = \probavailable \ExpSub{\cC}{h^t_i + \frac{1}{\probavailable}k^{t+1}_i} + (1 - \probavailable) h^{t}_i = h^{t}_i + k^{t+1}_i. \\
  \end{align*}
  Thus, we can get
  \begin{align*}
      &\ExpSub{\cC}{\ExpSub{\probavailable}{\norm{g^{t+1} - h^{t+1}}^2}} \\
      &=\ExpSub{\cC}{\ExpSub{\probavailable}{\norm{g^{t+1} - h^{t+1} - \ExpSub{\cC}{\ExpSub{\probavailable}{g^{t+1} - h^{t+1}}}}^2}} + (1 - a)^2\norm{g^{t} - h^{t}}^2.
  \end{align*}
  Due to the independence of compressors, we can use Lemma~\ref{lemma:sampling} with $r_i = g^{t}_i - h^{t}_i$ and $s_i = \probavailable\cC_i\Bigg(\frac{1}{\probavailable}k^{t+1}_i - \frac{a}{\probavailable} \left(g^{t}_i - h^{t}_i\right)\Bigg) - k^{t+1}_i,$ and obtain
  \begin{align*}
    &\ExpSub{\cC}{\ExpSub{\probavailable}{\norm{g^{t+1} - h^{t+1}}^2}} \\
    &\leq \frac{1}{n^2 \probavailable}\sum_{i=1}^n\ExpSub{\cC}{\norm{\probavailable\cC_i\Bigg(\frac{1}{\probavailable}k^{t+1}_i - \frac{a}{\probavailable} \left(g^{t}_i - h^{t}_i\right)\Bigg) - k^{t+1}_i - \ExpSub{\cC}{\probavailable\cC_i\Bigg(\frac{1}{\probavailable}k^{t+1}_i - \frac{a}{\probavailable} \left(g^{t}_i - h^{t}_i\right)\Bigg) - k^{t+1}_i}}^2} \\
    &\quad +\frac{\probavailable - \probpairaa}{n^2 \probavailable^2}\sum_{i=1}^n\norm{\ExpSub{\cC}{\probavailable\cC_i\Bigg(\frac{1}{\probavailable}k^{t+1}_i - \frac{a}{\probavailable} \left(g^{t}_i - h^{t}_i\right)\Bigg) - k^{t+1}_i}}^2 \\
    &\quad + (1 - a)^2\norm{g^{t} - h^{t}}^2 \\
    &= \frac{\probavailable}{n^2}\sum_{i=1}^n\ExpSub{\cC}{\norm{\cC_i\Bigg(\frac{1}{\probavailable}k^{t+1}_i - \frac{a}{\probavailable} \left(g^{t}_i - h^{t}_i\right)\Bigg) - \Bigg(\frac{1}{\probavailable}k^{t+1}_i - \frac{a}{\probavailable} \left(g^{t}_i - h^{t}_i\right)\Bigg)}^2} \\
    &\quad +\frac{a^2 \left(\probavailable - \probpairaa\right)}{n^2 \probavailable^2}\sum_{i=1}^n\norm{g^{t}_i - h^{t}_i}^2 + (1 - a)^2\norm{g^{t} - h^{t}}^2.
  \end{align*}
  From Assumption~\ref{ass:compressors}, we have
  \begin{align*}
    &\ExpSub{\cC}{\ExpSub{\probavailable}{\norm{g^{t+1} - h^{t+1}}^2}} \\
    &\leq \frac{\omega \probavailable}{n^2}\sum_{i=1}^n\norm{\frac{1}{\probavailable}k^{t+1}_i - \frac{a}{\probavailable} \left(g^{t}_i - h^{t}_i\right)}^2 +\frac{a^2 \left(\probavailable - \probpairaa\right)}{n^2 \probavailable^2}\sum_{i=1}^n\norm{g^{t}_i - h^{t}_i}^2 + (1 - a)^2\norm{g^{t} - h^{t}}^2 \\
    &= \frac{\omega}{n^2 \probavailable}\sum_{i=1}^n\norm{k^{t+1}_i - a \left(g^{t}_i - h^{t}_i\right)}^2 +\frac{a^2 \left(\probavailable - \probpairaa\right)}{n^2 \probavailable^2}\sum_{i=1}^n\norm{g^{t}_i - h^{t}_i}^2 + (1 - a)^2\norm{g^{t} - h^{t}}^2 \\
    &\overset{\eqref{auxiliary:jensen_inequality}}{\leq} \frac{2 \omega}{n^2 \probavailable}\sum_{i=1}^n\norm{k^{t+1}_i}^2 +\frac{a^2 \left((2\omega + 1)\probavailable - \probpairaa\right)}{n^2 \probavailable^2}\sum_{i=1}^n\norm{g^{t}_i - h^{t}_i}^2 + (1 - a)^2\norm{g^{t} - h^{t}}^2.
  \end{align*}
  The second inequality can be proved almost in the same way:
  \begin{align*}
    &\ExpSub{\cC}{\ExpSub{\probavailable}{\norm{g^{t+1}_i - h^{t+1}_i}^2}} \\
    &=\ExpSub{\cC}{\ExpSub{\probavailable}{\norm{g^{t+1}_i - h^{t+1}_i - \ExpSub{\cC}{\ExpSub{\probavailable}{g^{t+1}_i - h^{t+1}_i}}}^2}} + \norm{\ExpSub{\cC}{\ExpSub{\probavailable}{g^{t+1}_i - h^{t+1}_i}}}^2 \\
    &=\ExpSub{\cC}{\ExpSub{\probavailable}{\norm{g^{t+1}_i - h^{t+1}_i - g^{t}_i + a \left(g^{t}_i - h^{t}_i\right) + h^{t}_i}^2}} + (1 - a)^2\norm{g^{t}_i - h^{t}_i}^2 \\
    &=\probavailable\ExpSub{\cC}{\norm{\cC_i\Bigg(\frac{1}{\probavailable}k^{t+1}_i - \frac{a}{\probavailable} \left(g^{t}_i - h^{t}_i\right)\Bigg) - \frac{1}{\probavailable}k^{t+1}_i + a \left(g^{t}_i - h^{t}_i\right)}^2} \\
    &\quad + a^2(1 - \probavailable)\norm{g^{t}_i - h^{t}_i}^2 + (1 - a)^2\norm{g^{t}_i - h^{t}_i}^2 \\
    &\overset{\eqref{auxiliary:variance_decomposition}}{=}\probavailable\ExpSub{\cC}{\norm{\cC_i\Bigg(\frac{1}{\probavailable}k^{t+1}_i - \frac{a}{\probavailable} \left(g^{t}_i - h^{t}_i\right)\Bigg) - \left(\frac{1}{\probavailable}k^{t+1}_i - \frac{a}{\probavailable} \left(g^{t}_i - h^{t}_i\right)\right)}^2} \\
    &\quad + a^2 \frac{(1 - \probavailable)^2}{\probavailable} \norm{g^{t}_i - h^{t}_i}^2 \\
    &\quad + a^2(1 - \probavailable)\norm{g^{t}_i - h^{t}_i}^2 + (1 - a)^2\norm{g^{t}_i - h^{t}_i}^2 \\
    &\leq \frac{\omega}{\probavailable}\norm{k^{t+1}_i - a \left(g^{t}_i - h^{t}_i\right)}^2 \\
    &\quad + \frac{a^2(1 - \probavailable)}{\probavailable} \norm{g^{t}_i - h^{t}_i}^2 + (1 - a)^2\norm{g^{t}_i - h^{t}_i}^2 \\
    &\overset{\eqref{auxiliary:jensen_inequality}}{\leq} \frac{2 \omega}{\probavailable}\norm{k^{t+1}_i}^2 + \frac{a^2(2\omega + 1 - \probavailable)}{\probavailable} \norm{g^{t}_i - h^{t}_i}^2 + (1 - a)^2\norm{g^{t}_i - h^{t}_i}^2.
  \end{align*}
\end{proof}

\begin{lemma}
  \label{lemma:main_lemma}
  Suppose that Assumptions \ref{ass:lipschitz_constant}, \ref{ass:compressors}, and \ref{ass:partial_participation} hold and let us take $a = \frac{\probavailable}{2 \omega + 1},$ then
  \begin{align*}
    &\Exp{f(x^{t + 1})} + \frac{\gamma (2 \omega + 1)}{\probavailable} \Exp{\norm{g^{t+1} - h^{t+1}}^2} + \frac{\gamma (\left(2 \omega + 1\right)\probavailable - \probpairaa)}{n \probavailable^2} \Exp{\frac{1}{n}\sum_{i=1}^n\norm{g^{t+1}_i - h^{t+1}_i}^2}\\
    &\leq \Exp{f(x^t) - \frac{\gamma}{2}\norm{\nabla f(x^t)}^2 - \left(\frac{1}{2\gamma} - \frac{L}{2}\right)
    \norm{x^{t+1} - x^t}^2 + \gamma \norm{h^{t} - \nabla f(x^t)}^2}\nonumber\\
    &\quad + \frac{\gamma (2 \omega + 1)}{\probavailable}\Exp{\norm{g^{t} - h^t}^2}+ \frac{\gamma (\left(2 \omega + 1\right)\probavailable - \probpairaa)}{n \probavailable^2}\Exp{\frac{1}{n} \sum_{i=1}^n\norm{g^t_i - h^{t}_i}^2} + \frac{4 \gamma \omega (2 \omega + 1)}{n \probavailable^2} \Exp{\frac{1}{n} \sum_{i=1}^n\norm{k^{t+1}_i}^2}.
  \end{align*}
\end{lemma}

\begin{proof}
  Due to Lemma \ref{lemma:page_lemma} and the update step from Line~\ref{alg:main_algorithm:x_update} in Algorithm~\ref{alg:main_algorithm}, we have
  \begin{align*}
    &\ExpSub{t+1}{f(x^{t + 1})} \nonumber\\
    &\leq \ExpSub{t+1}{f(x^t) - \frac{\gamma}{2}\norm{\nabla f(x^t)}^2 - \left(\frac{1}{2\gamma} - \frac{L}{2}\right)
      \norm{x^{t+1} - x^t}^2 + \frac{\gamma}{2}\norm{g^{t} - \nabla f(x^t)}^2} \nonumber \\
      &= \ExpSub{t+1}{f(x^t) - \frac{\gamma}{2}\norm{\nabla f(x^t)}^2 - \left(\frac{1}{2\gamma} - \frac{L}{2}\right)
      \norm{x^{t+1} - x^t}^2 + \frac{\gamma}{2}\norm{g^{t} - h^t + h^t - \nabla f(x^t)}^2} \nonumber \\
      &\overset{\eqref{auxiliary:variance_decomposition}}{\leq} \ExpSub{t+1}{f(x^t) - \frac{\gamma}{2}\norm{\nabla f(x^t)}^2 - \left(\frac{1}{2\gamma} - \frac{L}{2}\right)
      \norm{x^{t+1} - x^t}^2 + \gamma\left(\norm{g^{t} - h^t}^2 + \norm{h^t - \nabla f(x^t)}^2}\right). \nonumber \\
      \nonumber
  \end{align*}
  Let us fix some constants $\kappa, \eta \in [0,\infty)$ that we will define later. Combining the last inequality, bounds \eqref{eq:compressor_global_error}, \eqref{eq:compressor_local_error} and using the law of total expectation, we get
  \begin{align*}
      &\Exp{f(x^{t + 1})} \nonumber \\
      &\quad  + \kappa \Exp{\norm{g^{t+1} - h^{t+1}}^2} + \eta \Exp{\frac{1}{n}\sum_{i=1}^n\norm{g^{t+1}_i - h^{t+1}_i}^2} \nonumber\\
      &=\Exp{\ExpSub{t+1}{f(x^{t + 1})}} \nonumber \\
      &\quad  + \kappa \Exp{\ExpSub{\cC}{\ExpSub{\probavailable}{\norm{g^{t+1} - h^{t+1}}^2}}} + \eta \Exp{\ExpSub{\cC}{\ExpSub{\probavailable}{\frac{1}{n}\sum_{i=1}^n\norm{g^{t+1}_i - h^{t+1}_i}^2}}} \nonumber\\
      &\leq \Exp{f(x^t) - \frac{\gamma}{2}\norm{\nabla f(x^t)}^2 - \left(\frac{1}{2\gamma} - \frac{L}{2}\right)
      \norm{x^{t+1} - x^t}^2 + \gamma\left(\norm{g^{t} - h^t}^2 + \norm{h^{t} - \nabla f(x^t)}^2\right)} \nonumber\\
      &\quad +\kappa \Exp{\frac{2 \omega}{n^2 \probavailable}\sum_{i=1}^n\norm{k^{t+1}_i}^2 +\frac{a^2 (\left(2 \omega + 1\right)\probavailable - \probpairaa)}{n^2 \probavailable^2}\sum_{i=1}^n\norm{g^{t}_i - h^{t}_i}^2 + (1 - a)^2\norm{g^{t} - h^{t}}^2} \nonumber\\
      &\quad +\eta \Exp{\frac{2 \omega}{n \probavailable} \sum_{i=1}^n \norm{k^{t+1}_i}^2 + \left(\frac{a^2(2\omega + 1 - \probavailable)}{\probavailable} + (1 - a)^2\right)\frac{1}{n}\sum_{i=1}^n\norm{g^{t}_i - h^{t}_i}^2} \\
      &= \Exp{f(x^t) - \frac{\gamma}{2}\norm{\nabla f(x^t)}^2 - \left(\frac{1}{2\gamma} - \frac{L}{2}\right)
      \norm{x^{t+1} - x^t}^2 + \gamma \norm{h^{t} - \nabla f(x^t)}^2}\nonumber\\
      &\quad + \left(\gamma + \kappa \left(1 - a\right)^2\right)\Exp{\norm{g^{t} - h^t}^2}\nonumber\\
      &\quad + \left(\frac{\kappa a^2 (\left(2 \omega + 1\right)\probavailable - \probpairaa)}{n \probavailable^2} + \eta\left(\frac{a^2(2\omega + 1 - \probavailable)}{\probavailable} + (1 - a)^2\right)\right)\Exp{\frac{1}{n} \sum_{i=1}^n\norm{g^t_i - h^{t}_i}^2}\nonumber\\
      &\quad + \left(\frac{2 \kappa \omega}{n \probavailable} + \frac{2 \eta \omega}{\probavailable}\right)\Exp{\frac{1}{n} \sum_{i=1}^n\norm{k^{t+1}_i}^2}.
  \end{align*}
  Now, by taking $\kappa = \frac{\gamma}{a}$, we can see that $\gamma + \kappa \left(1 - a\right)^2 \leq \kappa, $ and thus
  \begin{align*}
      &\Exp{f(x^{t + 1})} \\
      &\quad  + \frac{\gamma}{a} \Exp{\norm{g^{t+1} - h^{t+1}}^2} + \eta \Exp{\frac{1}{n}\sum_{i=1}^n\norm{g^{t+1}_i - h^{t+1}_i}^2}\\
      &\leq \Exp{f(x^t) - \frac{\gamma}{2}\norm{\nabla f(x^t)}^2 - \left(\frac{1}{2\gamma} - \frac{L}{2}\right)
      \norm{x^{t+1} - x^t}^2 + \gamma \norm{h^{t} - \nabla f(x^t)}^2}\nonumber\\
      &\quad + \frac{\gamma}{a}\Exp{\norm{g^{t} - h^t}^2}\nonumber\\
      &\quad + \left(\frac{\gamma a (\left(2 \omega + 1\right)\probavailable - \probpairaa)}{n \probavailable^2} + \eta\left(\frac{a^2(2\omega + 1 - \probavailable)}{\probavailable} + (1 - a)^2\right)\right)\Exp{\frac{1}{n} \sum_{i=1}^n\norm{g^t_i - h^{t}_i}^2}\nonumber\\
      &\quad + \left(\frac{2 \gamma \omega}{a n \probavailable} + \frac{2 \eta \omega}{\probavailable}\right)\Exp{\frac{1}{n} \sum_{i=1}^n\norm{k^{t+1}_i}^2}.
  \end{align*}
  Next, by taking $\eta = \frac{\gamma (\left(2 \omega + 1\right)\probavailable - \probpairaa)}{n \probavailable^2}$ and considering the choice of $a$, one can show that $\left(\frac{\gamma a (\left(2 \omega + 1\right)\probavailable - \probpairaa)}{n \probavailable^2} + \eta\left(\frac{a^2(2\omega + 1 - \probavailable)}{\probavailable} + (1 - a)^2\right)\right) \leq \eta.$ Thus

  \begin{align*}
    &\Exp{f(x^{t + 1})} \\
    &\quad  + \frac{\gamma (2 \omega + 1)}{\probavailable} \Exp{\norm{g^{t+1} - h^{t+1}}^2} + \frac{\gamma (\left(2 \omega + 1\right)\probavailable - \probpairaa)}{n \probavailable^2} \Exp{\frac{1}{n}\sum_{i=1}^n\norm{g^{t+1}_i - h^{t+1}_i}^2}\\
    &\leq \Exp{f(x^t) - \frac{\gamma}{2}\norm{\nabla f(x^t)}^2 - \left(\frac{1}{2\gamma} - \frac{L}{2}\right)
    \norm{x^{t+1} - x^t}^2 + \gamma \norm{h^{t} - \nabla f(x^t)}^2}\nonumber\\
    &\quad + \frac{\gamma (2 \omega + 1)}{\probavailable}\Exp{\norm{g^{t} - h^t}^2}+ \frac{\gamma (\left(2 \omega + 1\right)\probavailable - \probpairaa)}{n \probavailable^2}\Exp{\frac{1}{n} \sum_{i=1}^n\norm{g^t_i - h^{t}_i}^2}\nonumber\\
    &\quad + \left(\frac{2 \gamma (2 \omega + 1)\omega}{n \probavailable^2} + \frac{2 \gamma (\left(2 \omega + 1\right)\probavailable - \probpairaa) \omega}{n \probavailable^3}\right)\Exp{\frac{1}{n} \sum_{i=1}^n\norm{k^{t+1}_i}^2}.
  \end{align*}
  Considering that $\probpairaa \geq 0,$ we can simplify the last term and get
  \begin{align*}
    &\Exp{f(x^{t + 1})} \\
    &\quad  + \frac{\gamma (2 \omega + 1)}{\probavailable} \Exp{\norm{g^{t+1} - h^{t+1}}^2} + \frac{\gamma (\left(2 \omega + 1\right)\probavailable - \probpairaa)}{n \probavailable^2} \Exp{\frac{1}{n}\sum_{i=1}^n\norm{g^{t+1}_i - h^{t+1}_i}^2}\\
    &\leq \Exp{f(x^t) - \frac{\gamma}{2}\norm{\nabla f(x^t)}^2 - \left(\frac{1}{2\gamma} - \frac{L}{2}\right)
    \norm{x^{t+1} - x^t}^2 + \gamma \norm{h^{t} - \nabla f(x^t)}^2}\nonumber\\
    &\quad + \frac{\gamma (2 \omega + 1)}{\probavailable}\Exp{\norm{g^{t} - h^t}^2}+ \frac{\gamma (\left(2 \omega + 1\right)\probavailable - \probpairaa)}{n \probavailable^2}\Exp{\frac{1}{n} \sum_{i=1}^n\norm{g^t_i - h^{t}_i}^2}\nonumber\\
    &\quad + \frac{4 \gamma (2 \omega + 1)\omega}{n \probavailable^2} \Exp{\frac{1}{n} \sum_{i=1}^n\norm{k^{t+1}_i}^2}.
  \end{align*}
\end{proof}

\subsection{Proof for \algname{\algorithmname}}

\begin{lemma}
  \label{lemma:gradient}
  Suppose that Assumptions~\ref{ass:nodes_lipschitz_constant} and \ref{ass:partial_participation} hold. For $h^{t+1}_i$ and $k^{t+1}_i$ from Algorithm~\ref{alg:main_algorithm} (\algname{\algorithmname}) we have
  \begin{enumerate}
  \item
      \begin{align*}
          &\ExpSub{\probavailable}{\norm{h^{t+1} - \nabla f(x^{t+1})}^2} \\
          & \leq \frac{2\left(\probavailable - \probpairaa\right)\widehat{L}^2}{n \probavailable^2} \norm{x^{t+1} - x^{t}}^2 + \frac{2 b^2 \left(\probavailable - \probpairaa\right)}{n^2 \probavailable^2} \sum_{i=1}^n \norm{h^t_i - \nabla f_i(x^{t})}^2 + \left(1 - b\right)^2 \norm{h^{t} - \nabla f(x^{t})}^2.
      \end{align*}
  \item
      \begin{align*}
          &\ExpSub{\probavailable}{\norm{h^{t+1}_i - \nabla f_i(x^{t+1})}^2} \\
          & \leq \frac{2(1 - \probavailable)}{\probavailable} L^2_i \norm{x^{t+1} - x^{t}}^2 + \left(\frac{2 b^2 (1 - \probavailable)}{\probavailable} + (1 - b)^2\right) \norm{h^{t}_i - \nabla f_i(x^{t})}^2, \quad \forall i \in [n].
      \end{align*}
  \item
      \begin{align*}
        &\norm{k^{t+1}_i}^2 \leq 2L^2_i\norm{x^{t+1} - x^{t}}^2 + 2b^2 \norm{h^t_i - \nabla f_i(x^{t})}^2, \quad \forall i \in [n].
      \end{align*}
  \end{enumerate}
\end{lemma}

\begin{proof}
  First, let us proof the bound for $\ExpSub{k}{\ExpSub{\probavailable}{\norm{h^{t+1} - \nabla f(x^{t+1})}^2}}$:
  \begin{align*}
      &\ExpSub{\probavailable}{\norm{h^{t+1} - \nabla f(x^{t+1})}^2} \\
      &=\ExpSub{\probavailable}{\norm{h^{t+1} - \ExpSub{\probavailable}{h^{t+1}}}^2} + \norm{\ExpSub{\probavailable}{h^{t+1}} - \nabla f(x^{t+1})}^2.
  \end{align*}
  Using
  \begin{align*}
      \ExpSub{\probavailable}{h^{t+1}_i} = h^{t}_i + \nabla f_i(x^{t+1}) - \nabla f_i(x^{t}) - b(h^t_i - \nabla f_i(x^{t}))
  \end{align*}
  and \eqref{auxiliary:variance_decomposition}, we have
  \begin{align*}
    &\ExpSub{\probavailable}{\norm{h^{t+1} - \nabla f(x^{t+1})}^2} \\
    &=\ExpSub{\probavailable}{\norm{h^{t+1} - \ExpSub{\probavailable}{h^{t+1}}}^2} + \left(1 - b\right)^2 \norm{h^{t} - \nabla f(x^{t})}^2.
  \end{align*}
  We can use Lemma~\ref{lemma:sampling} with $r_i = h^{t}_i$ and $s_i = k^{t+1}_i$ to obtain
  \begin{align*}
    &\ExpSub{\probavailable}{\norm{h^{t+1} - \nabla f(x^{t+1})}^2} \\
    &\leq \frac{1}{n^2 \probavailable}\sum_{i=1}^n\norm{k^{t+1}_i - k^{t+1}_i}^2 +\frac{\probavailable - \probpairaa}{n^2 \probavailable^2}\sum_{i=1}^n\norm{k^{t+1}_i}^2 + \left(1 - b\right)^2 \norm{h^{t} - \nabla f(x^{t})}^2\\
    &= \frac{\probavailable - \probpairaa}{n^2 \probavailable^2}\sum_{i=1}^n\norm{\nabla f_i(x^{t+1}) - \nabla f_i(x^{t}) - b \left(h^t_i - \nabla f_i(x^{t})\right)}^2 + \left(1 - b\right)^2 \norm{h^{t} - \nabla f(x^{t})}^2\\
    &\overset{\eqref{auxiliary:jensen_inequality}}{\leq} \frac{2\left(\probavailable - \probpairaa\right)}{n^2 \probavailable^2}\sum_{i=1}^n\norm{\nabla f_i(x^{t+1}) - \nabla f_i(x^{t})}^2 + \frac{2 b^2 \left(\probavailable - \probpairaa\right)}{n^2 \probavailable^2}\sum_{i=1}^n\norm{h^t_i - \nabla f_i(x^{t})}^2 + \left(1 - b\right)^2 \norm{h^{t} - \nabla f(x^{t})}^2\\
    &\leq \frac{2\left(\probavailable - \probpairaa\right)\widehat{L}^2}{n \probavailable^2} \norm{x^{t+1} - x^{t}}^2 + \frac{2 b^2 \left(\probavailable - \probpairaa\right)}{n^2 \probavailable^2}\sum_{i=1}^n\norm{h^t_i - \nabla f_i(x^{t})}^2 + \left(1 - b\right)^2 \norm{h^{t} - \nabla f(x^{t})}^2.
  \end{align*}
  In the last in inequality, we used Assumption~\ref{ass:nodes_lipschitz_constant}. Now, we prove the second inequality:
  \begin{align*}
    &\ExpSub{\probavailable}{\norm{h^{t+1}_i - \nabla f_i(x^{t+1})}^2} \\
    &=\ExpSub{\probavailable}{\norm{h^{t+1}_i - \ExpSub{\probavailable}{h^{t+1}_i}}^2} + \norm{\ExpSub{\probavailable}{h^{t+1}_i  } - \nabla f_i(x^{t+1})}^2 \\
    &=\ExpSub{\probavailable}{\norm{h^{t+1}_i - \left(h^{t}_i + \nabla f_i(x^{t+1}) - \nabla f_i(x^{t}) - b(h^t_i - \nabla f_i(x^{t}))\right)}^2} + (1 - b)^2 \norm{h^{t}_i - \nabla f_i(x^{t})}^2 \\
    &=\frac{(1 - \probavailable)^2}{\probavailable} \norm{\nabla f_i(x^{t+1}) - \nabla f_i(x^{t}) - b(h^t_i - \nabla f_i(x^{t}))}^2 \\
    &\quad + (1 - \probavailable)\norm{\nabla f_i(x^{t+1}) - \nabla f_i(x^{t}) - b(h^t_i - \nabla f_i(x^{t}))}^2 + (1 - b)^2 \norm{h^{t}_i - \nabla f_i(x^{t})}^2 \\
    &=\frac{(1 - \probavailable)}{\probavailable} \norm{\nabla f_i(x^{t+1}) - \nabla f_i(x^{t}) - b(h^t_i - \nabla f_i(x^{t}))}^2 + (1 - b)^2 \norm{h^{t}_i - \nabla f_i(x^{t})}^2 \\
    &\leq \frac{2(1 - \probavailable)}{\probavailable} L^2_i \norm{x^{t+1} - x^{t}}^2 + \left(\frac{2 b^2 (1 - \probavailable)}{\probavailable} + (1 - b)^2\right) \norm{h^{t}_i - \nabla f_i(x^{t})}^2.
  \end{align*}
  Finally, the third inequality of the theorem follows from \eqref{auxiliary:jensen_inequality} and Assumption~\ref{ass:nodes_lipschitz_constant}.
\end{proof}

\CONVERGENCE*

\begin{proof}
  Let us fix constants $\nu, \rho \in [0,\infty)$ that we will define later. Considering Lemma~\ref{lemma:main_lemma}, Lemma~\ref{lemma:gradient}, and the law of total expectation, we obtain
    \begin{align*}
      &\Exp{f(x^{t + 1})} + \frac{\gamma (2 \omega + 1)}{\probavailable} \Exp{\norm{g^{t+1} - h^{t+1}}^2} + \frac{\gamma (\left(2 \omega + 1\right)\probavailable - \probpairaa)}{n \probavailable^2} \Exp{\frac{1}{n}\sum_{i=1}^n\norm{g^{t+1}_i - h^{t+1}_i}^2}\\
      &\quad  + \nu \Exp{\norm{h^{t+1} - \nabla f(x^{t+1})}^2} + \rho \Exp{\frac{1}{n}\sum_{i=1}^n\norm{h^{t+1}_i - \nabla f_i(x^{t+1})}^2}\\
      &=\Exp{f(x^{t + 1})} + \frac{\gamma (2 \omega + 1)}{\probavailable} \Exp{\norm{g^{t+1} - h^{t+1}}^2} + \frac{\gamma (\left(2 \omega + 1\right)\probavailable - \probpairaa)}{n \probavailable^2} \Exp{\frac{1}{n}\sum_{i=1}^n\norm{g^{t+1}_i - h^{t+1}_i}^2}\\
      &\quad  + \nu \Exp{\ExpSub{\probavailable}{\norm{h^{t+1} - \nabla f(x^{t+1})}^2}} + \rho \Exp{\ExpSub{\probavailable}{\frac{1}{n}\sum_{i=1}^n\norm{h^{t+1}_i - \nabla f_i(x^{t+1})}^2}}\\
      &\leq \Exp{f(x^t) - \frac{\gamma}{2}\norm{\nabla f(x^t)}^2 - \left(\frac{1}{2\gamma} - \frac{L}{2}\right)
      \norm{x^{t+1} - x^t}^2 + \gamma \norm{h^{t} - \nabla f(x^t)}^2}\nonumber\\
      &\quad + \frac{\gamma (2 \omega + 1)}{\probavailable}\Exp{\norm{g^{t} - h^t}^2}+ \frac{\gamma (\left(2 \omega + 1\right)\probavailable - \probpairaa)}{n \probavailable^2}\Exp{\frac{1}{n} \sum_{i=1}^n\norm{g^t_i - h^{t}_i}^2} \\
      &\quad + \frac{4 \gamma \omega (2 \omega + 1)}{n \probavailable^2} \Exp{2\widehat{L}^2\norm{x^{t+1} - x^{t}}^2 + 2b^2 \frac{1}{n}\sum_{i=1}^n \norm{h^t_i - \nabla f_i(x^{t})}^2} \\
      &\quad + \nu \Exp{\frac{2\left(\probavailable - \probpairaa\right)\widehat{L}^2}{n \probavailable^2} \norm{x^{t+1} - x^{t}}^2 + \frac{2 b^2 \left(\probavailable - \probpairaa\right)}{n^2 \probavailable^2} \sum_{i=1}^n \norm{h^t_i - \nabla f_i(x^{t})}^2 + \left(1 - b\right)^2 \norm{h^{t} - \nabla f(x^{t})}^2} \\
      &\quad + \rho \Exp{\frac{2(1 - \probavailable)}{\probavailable} \widehat{L}^2 \norm{x^{t+1} - x^{t}}^2 + \left(\frac{2 b^2 (1 - \probavailable)}{\probavailable} + (1 - b)^2\right) \frac{1}{n}\sum_{i=1}^n \norm{h^{t}_i - \nabla f_i(x^{t})}^2}.
    \end{align*}
    After rearranging the terms, we get
    \begin{align*}
      &\Exp{f(x^{t + 1})} + \frac{\gamma (2 \omega + 1)}{\probavailable} \Exp{\norm{g^{t+1} - h^{t+1}}^2} + \frac{\gamma (\left(2 \omega + 1\right)\probavailable - \probpairaa)}{n \probavailable^2} \Exp{\frac{1}{n}\sum_{i=1}^n\norm{g^{t+1}_i - h^{t+1}_i}^2}\\
      &\quad  + \nu \Exp{\norm{h^{t+1} - \nabla f(x^{t+1})}^2} + \rho \Exp{\frac{1}{n}\sum_{i=1}^n\norm{h^{t+1}_i - \nabla f_i(x^{t+1})}^2}\\
      &\leq \Exp{f(x^t)} - \frac{\gamma}{2}\Exp{\norm{\nabla f(x^t)}^2} \\
      &\quad + \frac{\gamma (2 \omega + 1)}{\probavailable} \Exp{\norm{g^{t} - h^{t}}^2} + \frac{\gamma (\left(2 \omega + 1\right)\probavailable - \probpairaa)}{n \probavailable^2} \Exp{\frac{1}{n}\sum_{i=1}^n\norm{g^{t}_i - h^{t}_i}^2} \\
      &\quad - \left(\frac{1}{2\gamma} - \frac{L}{2} - \frac{8 \gamma \omega \left(2 \omega + 1\right) \widehat{L}^2}{n \probavailable^2} - \nu \frac{2\left(\probavailable - \probpairaa\right)\widehat{L}^2}{n \probavailable^2} - \rho \frac{2(1 - \probavailable) \widehat{L}^2}{\probavailable} \right) \Exp{\norm{x^{t+1} - x^t}^2} \\
      &\quad + \left(\gamma + \nu (1 - b)^2\right) \Exp{\norm{h^{t} - \nabla f(x^{t})}^2} \\
      &\quad + \left(\frac{8 b^2 \gamma \omega (2 \omega + 1)}{n \probavailable^2} + \nu \frac{2 b^2 \left(\probavailable - \probpairaa\right)}{n \probavailable^2} + \rho \left(\frac{2 b^2 (1 - \probavailable)}{\probavailable} + (1 - b)^2\right) \right)\Exp{\frac{1}{n}\sum_{i=1}^n\norm{h^{t}_i - \nabla f_i(x^{t})}^2}.
    \end{align*}
    By taking $\nu = \frac{\gamma}{b},$ one can show that $\left(\gamma + \nu (1 - b)^2\right) \leq \nu,$ and
    \begin{align*}
      &\Exp{f(x^{t + 1})} + \frac{\gamma (2 \omega + 1)}{\probavailable} \Exp{\norm{g^{t+1} - h^{t+1}}^2} + \frac{\gamma (\left(2 \omega + 1\right)\probavailable - \probpairaa)}{n \probavailable^2} \Exp{\frac{1}{n}\sum_{i=1}^n\norm{g^{t+1}_i - h^{t+1}_i}^2}\\
      &\quad  + \frac{\gamma}{b} \Exp{\norm{h^{t+1} - \nabla f(x^{t+1})}^2} + \rho \Exp{\frac{1}{n}\sum_{i=1}^n\norm{h^{t+1}_i - \nabla f_i(x^{t+1})}^2}\\
      &\leq \Exp{f(x^t)} - \frac{\gamma}{2}\Exp{\norm{\nabla f(x^t)}^2} \\
      &\quad + \frac{\gamma (2 \omega + 1)}{\probavailable} \Exp{\norm{g^{t} - h^{t}}^2} + \frac{\gamma (\left(2 \omega + 1\right)\probavailable - \probpairaa)}{n \probavailable^2} \Exp{\frac{1}{n}\sum_{i=1}^n\norm{g^{t}_i - h^{t}_i}^2} \\
      &\quad - \left(\frac{1}{2\gamma} - \frac{L}{2} - \frac{8 \gamma \omega \left(2 \omega + 1\right) \widehat{L}^2}{n \probavailable^2} - \frac{2\gamma \left(\probavailable - \probpairaa\right)\widehat{L}^2}{b n \probavailable^2} - \rho \frac{2(1 - \probavailable) \widehat{L}^2}{\probavailable} \right) \Exp{\norm{x^{t+1} - x^t}^2} \\
      &\quad + \frac{\gamma}{b} \Exp{\norm{h^{t} - \nabla f(x^{t})}^2} \\
      &\quad + \left(\frac{8 b^2 \gamma \omega (2 \omega + 1)}{n \probavailable^2} + \frac{2 \gamma b \left(\probavailable - \probpairaa\right)}{n \probavailable^2} + \rho \left(\frac{2 b^2 (1 - \probavailable)}{\probavailable} + (1 - b)^2\right) \right)\Exp{\frac{1}{n}\sum_{i=1}^n\norm{h^{t}_i - \nabla f_i(x^{t})}^2}.
    \end{align*}
    Note that $b = \frac{\probavailable}{2 - \probavailable},$ thus
    \begin{align*}
      &\left(\frac{8 b^2 \gamma \omega (2 \omega + 1)}{n \probavailable^2} + \frac{2 \gamma b \left(\probavailable - \probpairaa\right)}{n \probavailable^2} + \rho \left(\frac{2 b^2 (1 - \probavailable)}{\probavailable} + (1 - b)^2\right) \right) \\
      &\leq \left(\frac{8 b^2 \gamma \omega (2 \omega + 1)}{n \probavailable^2} + \frac{2 \gamma b \left(\probavailable - \probpairaa\right)}{n \probavailable^2} + \rho \left(1 - b\right) \right).
    \end{align*}
    And if we take $\rho = \frac{8 b \gamma \omega (2 \omega + 1)}{n \probavailable^2} + \frac{2 \gamma \left(\probavailable - \probpairaa\right)}{n \probavailable^2},$ then
    \begin{align*}
      \left(\frac{8 b^2 \gamma \omega (2 \omega + 1)}{n \probavailable^2} + \frac{2 \gamma b \left(\probavailable - \probpairaa\right)}{n \probavailable^2} + \rho \left(1 - b\right) \right) \leq \rho,
    \end{align*}
    and 
    \begin{align*}
      &\Exp{f(x^{t + 1})} + \frac{\gamma (2 \omega + 1)}{\probavailable} \Exp{\norm{g^{t+1} - h^{t+1}}^2} + \frac{\gamma (\left(2 \omega + 1\right)\probavailable - \probpairaa)}{n \probavailable^2} \Exp{\frac{1}{n}\sum_{i=1}^n\norm{g^{t+1}_i - h^{t+1}_i}^2}\\
      &\quad  + \frac{\gamma}{b} \Exp{\norm{h^{t+1} - \nabla f(x^{t+1})}^2} + \left(\frac{8 b \gamma \omega (2 \omega + 1)}{n \probavailable^2} + \frac{2 \gamma \left(\probavailable - \probpairaa\right)}{n \probavailable^2}\right) \Exp{\frac{1}{n}\sum_{i=1}^n\norm{h^{t+1}_i - \nabla f_i(x^{t+1})}^2}\\
      &\leq \Exp{f(x^t)} - \frac{\gamma}{2}\Exp{\norm{\nabla f(x^t)}^2} \\
      &\quad + \frac{\gamma (2 \omega + 1)}{\probavailable} \Exp{\norm{g^{t} - h^{t}}^2} + \frac{\gamma (\left(2 \omega + 1\right)\probavailable - \probpairaa)}{n \probavailable^2} \Exp{\frac{1}{n}\sum_{i=1}^n\norm{g^{t}_i - h^{t}_i}^2} \\
      &\quad - \Bigg(\frac{1}{2\gamma} - \frac{L}{2} - \frac{8 \gamma \omega \left(2 \omega + 1\right) \widehat{L}^2}{n \probavailable^2} - \frac{2\gamma \left(\probavailable - \probpairaa\right)\widehat{L}^2}{b n \probavailable^2} \\
      &\quad\qquad - \frac{16 b \gamma \omega (2 \omega + 1) (1 - \probavailable) \widehat{L}^2}{n \probavailable^3} - \frac{4 \gamma \left(\probavailable - \probpairaa\right) (1 - \probavailable) \widehat{L}^2}{n \probavailable^3} \Bigg) \Exp{\norm{x^{t+1} - x^t}^2} \\
      &\quad + \frac{\gamma}{b} \Exp{\norm{h^{t} - \nabla f(x^{t})}^2} + \left(\frac{8 b \gamma \omega (2 \omega + 1)}{n \probavailable^2} + \frac{2 \gamma \left(\probavailable - \probpairaa\right)}{n \probavailable^2}\right)\Exp{\frac{1}{n}\sum_{i=1}^n\norm{h^{t}_i - \nabla f_i(x^{t})}^2}.
    \end{align*}
    Let us simplify the last inequality. First, note that
    \begin{align*}
      &\frac{16 b \gamma \omega (2 \omega + 1) (1 - \probavailable) \widehat{L}^2}{n \probavailable^3} \leq \frac{16 \gamma \omega (2 \omega + 1) \widehat{L}^2}{n \probavailable^2},
    \end{align*}
    due to $b \leq \probavailable.$ Second,
    \begin{align*}
      \frac{2\gamma \left(\probavailable - \probpairaa\right)\widehat{L}^2}{b n \probavailable^2} \leq \frac{4\gamma \left(\probavailable - \probpairaa\right)\widehat{L}^2}{n \probavailable^3},
    \end{align*}
    due to $b \geq \frac{\probavailable}{2}.$ All in all, we have
    \begin{align*}
      &\Exp{f(x^{t + 1})} + \frac{\gamma (2 \omega + 1)}{\probavailable} \Exp{\norm{g^{t+1} - h^{t+1}}^2} + \frac{\gamma (\left(2 \omega + 1\right)\probavailable - \probpairaa)}{n \probavailable^2} \Exp{\frac{1}{n}\sum_{i=1}^n\norm{g^{t+1}_i - h^{t+1}_i}^2}\\
      &\quad  + \frac{\gamma}{b} \Exp{\norm{h^{t+1} - \nabla f(x^{t+1})}^2} + \left(\frac{8 b \gamma \omega (2 \omega + 1)}{n \probavailable^2} + \frac{2 \gamma \left(\probavailable - \probpairaa\right)}{n \probavailable^2}\right) \Exp{\frac{1}{n}\sum_{i=1}^n\norm{h^{t+1}_i - \nabla f_i(x^{t+1})}^2}\\
      &\leq \Exp{f(x^t)} - \frac{\gamma}{2}\Exp{\norm{\nabla f(x^t)}^2} \\
      &\quad + \frac{\gamma (2 \omega + 1)}{\probavailable} \Exp{\norm{g^{t} - h^{t}}^2} + \frac{\gamma (\left(2 \omega + 1\right)\probavailable - \probpairaa)}{n \probavailable^2} \Exp{\frac{1}{n}\sum_{i=1}^n\norm{g^{t}_i - h^{t}_i}^2} \\
      &\quad - \Bigg(\frac{1}{2\gamma} - \frac{L}{2} - \frac{24 \gamma \omega \left(2 \omega + 1\right) \widehat{L}^2}{n \probavailable^2} - \frac{8\gamma \left(\probavailable - \probpairaa\right)\widehat{L}^2}{n \probavailable^3} \Bigg) \Exp{\norm{x^{t+1} - x^t}^2} \\
      &\quad + \frac{\gamma}{b} \Exp{\norm{h^{t} - \nabla f(x^{t})}^2} + \left(\frac{8 b \gamma \omega (2 \omega + 1)}{n \probavailable^2} + \frac{2 \gamma \left(\probavailable - \probpairaa\right)}{n \probavailable^2}\right)\Exp{\frac{1}{n}\sum_{i=1}^n\norm{h^{t}_i - \nabla f_i(x^{t})}^2}.
    \end{align*}
    Using Lemma~\ref{lemma:gamma} and the assumption about $\gamma,$ we get
    \begin{align*}
      &\Exp{f(x^{t + 1})} + \frac{\gamma (2 \omega + 1)}{\probavailable} \Exp{\norm{g^{t+1} - h^{t+1}}^2} + \frac{\gamma (\left(2 \omega + 1\right)\probavailable - \probpairaa)}{n \probavailable^2} \Exp{\frac{1}{n}\sum_{i=1}^n\norm{g^{t+1}_i - h^{t+1}_i}^2}\\
      &\quad  + \frac{\gamma}{b} \Exp{\norm{h^{t+1} - \nabla f(x^{t+1})}^2} + \left(\frac{8 b \gamma \omega (2 \omega + 1)}{n \probavailable^2} + \frac{2 \gamma \left(\probavailable - \probpairaa\right)}{n \probavailable^2}\right) \Exp{\frac{1}{n}\sum_{i=1}^n\norm{h^{t+1}_i - \nabla f_i(x^{t+1})}^2}\\
      &\leq \Exp{f(x^t)} - \frac{\gamma}{2}\Exp{\norm{\nabla f(x^t)}^2} \\
      &\quad + \frac{\gamma (2 \omega + 1)}{\probavailable} \Exp{\norm{g^{t} - h^{t}}^2} + \frac{\gamma (\left(2 \omega + 1\right)\probavailable - \probpairaa)}{n \probavailable^2} \Exp{\frac{1}{n}\sum_{i=1}^n\norm{g^{t}_i - h^{t}_i}^2} \\
      &\quad + \frac{\gamma}{b} \Exp{\norm{h^{t} - \nabla f(x^{t})}^2} + \left(\frac{8 b \gamma \omega (2 \omega + 1)}{n \probavailable^2} + \frac{2 \gamma \left(\probavailable - \probpairaa\right)}{n \probavailable^2}\right)\Exp{\frac{1}{n}\sum_{i=1}^n\norm{h^{t}_i - \nabla f_i(x^{t})}^2}.
    \end{align*}
    It is left to apply Lemma~\ref{lemma:good_recursion} with 
    \begin{eqnarray*}
      \Psi^t &=& \frac{(2 \omega + 1)}{\probavailable} \Exp{\norm{g^{t} - h^{t}}^2} + \frac{(\left(2 \omega + 1\right)\probavailable - \probpairaa)}{n \probavailable^2} \Exp{\frac{1}{n}\sum_{i=1}^n\norm{g^{t}_i - h^{t}_i}^2} \\
        &\quad +& \frac{1}{b} \Exp{\norm{h^{t} - \nabla f(x^{t})}^2} + \left(\frac{8 b \omega (2 \omega + 1)}{n \probavailable^2} + \frac{2 \left(\probavailable - \probpairaa\right)}{n \probavailable^2}\right)\Exp{\frac{1}{n}\sum_{i=1}^n\norm{h^{t}_i - \nabla f_i(x^{t})}^2}
    \end{eqnarray*}
    to conclude the proof.
\end{proof}

\subsection{Proof for \algname{\algorithmname-PAGE}}

Let us denote
\begin{align*}
    &k^{t+1}_{i, 1} \eqdef \nabla f_i(x^{t+1}) - \nabla f_i(x^{t}) - \frac{b}{\probpage} \left(h^t_i - \nabla f_i(x^{t})\right), \\
    &k^{t+1}_{i, 2} \eqdef \frac{1}{B}\sum_{j \in I^t_i}\left(\nabla f_{ij}(x^{t+1}) - \nabla f_{ij}(x^{t})\right), \\
    &h^{t+1}_{i,1} \eqdef \begin{cases}
        h^t_i + \frac{1}{\probavailable} k^{t+1}_{i, 1},& i^{\textnormal{th}} \textnormal{ node is \textit{participating}}, \\
        h^t_i, & \textnormal{otherwise,} 
    \end{cases}  \\
    &h^{t+1}_{i,2} \eqdef \begin{cases}
        h^t_i + \frac{1}{\probavailable} k^{t+1}_{i, 2},& i^{\textnormal{th}} \textnormal{ node is \textit{participating}}, \\
        h^t_i, & \textnormal{otherwise,}
    \end{cases}  \\
\end{align*}
$h^{t+1}_{1} \eqdef \frac{1}{n}\sum_{i=1}^n h^{t+1}_{i,1},$ and $h^{t+1}_{2} \eqdef \frac{1}{n}\sum_{i=1}^n h^{t+1}_{i,2}.$ Note, that
\begin{align*}
  &h^{t+1} = \begin{cases}
    h^{t+1}_{1},& \textnormal{with probability $\probpage$,} \\
    h^{t+1}_{2},& \textnormal{with probability $1 - \probpage$.} 
  \end{cases}
\end{align*}

\begin{lemma}
  \label{lemma:gradient_page}
  Suppose that Assumptions \ref{ass:nodes_lipschitz_constant}, \ref{ass:max_lipschitz_constant}, and \ref{ass:partial_participation} hold. For $h^{t+1}_i$ and $k^{t+1}_i$ from Algorithm~\ref{alg:main_algorithm} (\algname{\algorithmname-PAGE}) we have
  \begin{enumerate}
  \item
      \begin{align*}
          &\ExpSub{B}{\ExpSub{\probavailable}{\ExpSub{\probpage}{\norm{h^{t+1} - \nabla f(x^{t+1})}^2}}} \\
          & \leq \left(\frac{2 \left(\probavailable - \probpairaa\right) \widehat{L}^2}{n \probavailable^2} + \frac{(1 - \probpage)L_{\max}^2}{n \probavailable B}\right) \norm{x^{t+1} - x^{t}}^2\\
          &\quad + \frac{2\left(\probavailable - \probpairaa\right) b^2}{n^2 \probavailable^2 \probpage}\sum_{i=1}^n\norm{ h^t_i - \nabla f_i(x^{t})}^2 + \left(\probpage\left(1 - \frac{b}{\probpage}\right)^2 + (1 - \probpage)\right)\norm{h^{t} - \nabla f(x^{t})}^2.
      \end{align*}
  \item
      \begin{align*}
          &\ExpSub{B}{\ExpSub{\probavailable}{\ExpSub{\probpage}{\norm{h^{t+1}_i - \nabla f_i(x^{t+1})}^2}}} \\
          & \leq \left(\frac{2\left(1 - \probavailable\right)L_i^2}{\probavailable} + \frac{(1 - \probpage)L_{\max}^2}{\probavailable B}\right) \norm{x^{t+1} - x^{t}}^2 \\
          &\quad +\left(\frac{2\left(1 - \probavailable\right)b^2}{\probavailable \probpage} + \probpage\left(1 - \frac{b}{\probpage}\right)^2 + (1 - \probpage)\right)\norm{h^{t}_i - \nabla f_i(x^{t})}^2, \quad \forall i \in [n].
      \end{align*}
  \item
      \begin{align*}
        &\ExpSub{B}{\ExpSub{\probpage}{\norm{k^{t+1}_i}^2}} \\
        &\leq \left(2 L_i^2 + \frac{(1 - \probpage)L_{\max}^2}{B}\right)\norm{x^{t+1} - x^{t}}^2 +  \frac{2b^2}{\probpage} \norm{h^t_i - \nabla f_i(x^{t})}^2, \quad \forall i \in [n].
      \end{align*}
  \end{enumerate}
\end{lemma}

\begin{proof}
  First, we prove the first inequality of the theorem:
  \begin{align*}
    &\ExpSub{B}{\ExpSub{\probavailable}{\ExpSub{\probpage}{\norm{h^{t+1} - \nabla f(x^{t+1})}^2}}} \\
    &=\probpage\ExpSub{B}{\ExpSub{\probavailable}{\norm{h^{t+1}_{1} - \nabla f(x^{t+1})}^2}} + (1 - \probpage)\ExpSub{B}{\ExpSub{\probavailable}{\norm{h^{t+1}_{2} - \nabla f(x^{t+1})}^2}}.\\
  \end{align*}
  Using 
  \begin{align*}
    &\ExpSub{B}{\ExpSub{\probavailable}{h^{t+1}_{i,1}}} = \\
    &=\probavailable h^t_i +  \nabla f_i(x^{t+1}) - \nabla f_i(x^{t}) - \frac{b}{\probpage} \left(h^t_i - \nabla f_i(x^{t})\right) + (1 - \probavailable) h^t_i \\
    &= h^t_i + \nabla f_i(x^{t+1}) - \nabla f_i(x^{t}) - \frac{b}{\probpage} \left(h^t_i - \nabla f_i(x^{t})\right).
  \end{align*}
  and 
  \begin{align*}
    &\ExpSub{B}{\ExpSub{\probavailable}{h^{t+1}_{i,2}}} = \\
    &=\probavailable h^t_i +  \ExpSub{B}{\frac{1}{B}\sum_{j \in I^t_i}\left(\nabla f_{ij}(x^{t+1}) - \nabla f_{ij}(x^{t})\right)} + (1 - \probavailable) h^t_i \\
    &= h^t_i + \nabla f_i(x^{t+1}) - \nabla f_i(x^{t}),
  \end{align*}
  we obtain
  \begin{align}
    &\ExpSub{B}{\ExpSub{\probavailable}{\ExpSub{\probpage}{\norm{h^{t+1} - \nabla f(x^{t+1})}^2}}} \nonumber\\
    &\overset{\eqref{auxiliary:variance_decomposition}}{=}\probpage\ExpSub{\probavailable}{\norm{h^{t+1}_{1} - \ExpSub{\probavailable}{h^{t+1}_{1}}}^2} + (1 - \probpage)\ExpSub{B}{\ExpSub{\probavailable}{\norm{h^{t+1}_{2} - \ExpSub{B}{\ExpSub{\probavailable}{h^{t+1}_{2}}}}^2}} \nonumber\\
    &\quad +\probpage\norm{\ExpSub{\probavailable}{h^{t+1}_{1}} - \nabla f(x^{t+1})}^2 + (1 - \probpage)\norm{\ExpSub{B}{\ExpSub{\probavailable}{h^{t+1}_{2}}} - \nabla f(x^{t+1})}^2 \nonumber \\
    &=\probpage\ExpSub{\probavailable}{\norm{h^{t+1}_{1} - \ExpSub{\probavailable}{h^{t+1}_{1}}}^2} + (1 - \probpage)\ExpSub{B}{\ExpSub{\probavailable}{\norm{h^{t+1}_{2} - \ExpSub{B}{\ExpSub{\probavailable}{h^{t+1}_{2}}}}^2}} \nonumber\\
    &\quad +\left(\probpage\left(1 - \frac{b}{\probpage}\right)^2 + (1 - \probpage)\right)\norm{h^{t} - \nabla f(x^{t})}^2. \label{eq:page_proof:orig}
  \end{align}
  Next, we consider $\ExpSub{\probavailable}{\norm{h^{t+1}_{1} - \ExpSub{\probavailable}{h^{t+1}_{1}}}^2}$. We can use Lemma~\ref{lemma:sampling} with $r_i = h^{t}_i$ and $s_i = k^{t+1}_{i,1}$ to obtain
  \begin{align*}
    &\ExpSub{\probavailable}{\norm{h^{t+1}_{1} - \ExpSub{\probavailable}{h^{t+1}_{1}}}^2}\\
    &\leq \frac{1}{n^2 \probavailable}\sum_{i=1}^n\norm{k^{t+1}_{i,1} - k^{t+1}_{i,1}}^2 +\frac{\probavailable - \probpairaa}{n^2 \probavailable^2}\sum_{i=1}^n\norm{k^{t+1}_{i,1}}^2 \\
    &= \frac{\probavailable - \probpairaa}{n^2 \probavailable^2}\sum_{i=1}^n\norm{\nabla f_i(x^{t+1}) - \nabla f_i(x^{t}) - \frac{b}{\probpage} \left(h^t_i - \nabla f_i(x^{t})\right)}^2 \\
    &\overset{\eqref{auxiliary:jensen_inequality}}{\leq} \frac{2\left(\probavailable - \probpairaa\right)}{n^2 \probavailable^2}\sum_{i=1}^n\norm{\nabla f_i(x^{t+1}) - \nabla f_i(x^{t})}^2 + \frac{2 \left(\probavailable - \probpairaa\right) b^2}{n^2 \probavailable^2 \probpage^2}\sum_{i=1}^n\norm{h^t_i - \nabla f_i(x^{t})}^2.
  \end{align*}
  From Assumption~\ref{ass:nodes_lipschitz_constant}, we have
  \begin{align}
    &\ExpSub{\probavailable}{\norm{h^{t+1}_{1} - \ExpSub{\probavailable}{h^{t+1}_{1}}}^2} \nonumber\\
    &\leq \frac{2\left(\probavailable - \probpairaa\right) \widehat{L}^2}{n \probavailable^2}\norm{x^{t+1} - x^{t}}^2 + \frac{2 \left(\probavailable - \probpairaa\right)b^2}{n^2 \probavailable^2 \probpage^2}\sum_{i=1}^n\norm{h^t_i - \nabla f_i(x^{t})}^2 \label{eq:page_proof:h_1}.
  \end{align}
  Now, we prove the bound for $\ExpSub{B}{\ExpSub{\probavailable}{\norm{h^{t+1}_{2} - \ExpSub{B}{\ExpSub{\probavailable}{h^{t+1}_{2}}}}^2}}.$
  Considering that mini-batches in the algorithm are independent, we can use Lemma~\ref{lemma:sampling} with $r_i = h^{t}_i$ and $s_i = k^{t+1}_{i,2}$ to obtain
  \begin{align*}
    &\ExpSub{B}{\ExpSub{\probavailable}{\norm{h^{t+1}_{2} - \ExpSub{B}{\ExpSub{\probavailable}{h^{t+1}_{2}}}}^2}}\\
    &\leq \frac{1}{n^2 \probavailable}\sum_{i=1}^n\ExpSub{B}{\norm{k^{t+1}_{i,2} - \ExpSub{B}{k^{t+1}_{i,2}}}^2} +\frac{\probavailable - \probpairaa}{n^2 \probavailable^2}\sum_{i=1}^n\norm{\ExpSub{B}{k^{t+1}_{i,2}}}^2 \\
    &= \frac{1}{n^2 \probavailable}\sum_{i=1}^n\ExpSub{B}{\norm{\frac{1}{B}\sum_{j \in I^t_i}\left(\nabla f_{ij}(x^{t+1}) - \nabla f_{ij}(x^{t})\right) - \left(\nabla f_{i}(x^{t+1}) - \nabla f_{i}(x^{t})\right)}^2} \\
    &\quad +\frac{\probavailable - \probpairaa}{n^2 \probavailable^2}\sum_{i=1}^n\norm{\nabla f_{i}(x^{t+1}) - \nabla f_{i}(x^{t})}^2 \\
    &= \frac{1}{n^2 \probavailable B^2}\sum_{i=1}^n\ExpSub{B}{\sum_{j \in I^t_i}\norm{\left(\nabla f_{ij}(x^{t+1}) - \nabla f_{ij}(x^{t})\right) - \left(\nabla f_{i}(x^{t+1}) - \nabla f_{i}(x^{t})\right)}^2} \\
    &\quad +\frac{\probavailable - \probpairaa}{n^2 \probavailable^2}\sum_{i=1}^n\norm{\nabla f_{i}(x^{t+1}) - \nabla f_{i}(x^{t})}^2 \\
    &= \frac{1}{n^2 \probavailable B m}\sum_{i=1}^n \sum_{j=1}^m \norm{\left(\nabla f_{ij}(x^{t+1}) - \nabla f_{ij}(x^{t})\right) - \left(\nabla f_{i}(x^{t+1}) - \nabla f_{i}(x^{t})\right)}^2 \\
    &\quad +\frac{\probavailable - \probpairaa}{n^2 \probavailable^2}\sum_{i=1}^n\norm{\nabla f_{i}(x^{t+1}) - \nabla f_{i}(x^{t})}^2 \\
    &\leq \frac{1}{n^2 \probavailable B m}\sum_{i=1}^n \sum_{j=1}^m \norm{\nabla f_{ij}(x^{t+1}) - \nabla f_{ij}(x^{t})}^2 +\frac{\probavailable - \probpairaa}{n^2 \probavailable^2}\sum_{i=1}^n\norm{\nabla f_{i}(x^{t+1}) - \nabla f_{i}(x^{t})}^2.
  \end{align*}
  Next, we use Assumptions \ref{ass:nodes_lipschitz_constant} and \ref{ass:max_lipschitz_constant} to get
  \begin{align}
    \ExpSub{B}{\ExpSub{\probavailable}{\norm{h^{t+1}_{2} - \ExpSub{B}{\ExpSub{\probavailable}{h^{t+1}_{2}}}}^2}} \leq \left(\frac{L_{\max}^2}{n \probavailable B} +\frac{\left(\probavailable - \probpairaa\right) \widehat{L}^2}{n \probavailable^2}\right)\norm{x^{t+1} - x^{t}}^2. \label{eq:page_proof:h_2}
  \end{align}
  Applying \eqref{eq:page_proof:h_1} and \eqref{eq:page_proof:h_2} into \eqref{eq:page_proof:orig}, we get
  \begin{align*}
    &\ExpSub{B}{\ExpSub{\probavailable}{\ExpSub{\probpage}{\norm{h^{t+1} - \nabla f(x^{t+1})}^2}}} \\
    &\leq\probpage\left(\frac{2 \left(\probavailable - \probpairaa\right) \widehat{L}^2}{n \probavailable^2} \norm{x^{t+1} - x^{t}}^2 + \frac{2\left(\probavailable - \probpairaa\right) b^2}{n^2 \probavailable^2 \probpage^2}\sum_{i=1}^n\norm{ h^t_i - \nabla f_i(x^{t})}^2\right) + \\
    &\quad + (1 - \probpage)\left(\frac{L_{\max}^2}{n \probavailable B} +\frac{\left(\probavailable - \probpairaa\right) \widehat{L}^2}{n \probavailable^2}\right)\norm{x^{t+1} - x^{t}}^2 \\
    &\quad +\left(\probpage\left(1 - \frac{b}{\probpage}\right)^2 + (1 - \probpage)\right)\norm{h^{t} - \nabla f(x^{t})}^2\\
    &\leq\left(\frac{2 \left(\probavailable - \probpairaa\right) \widehat{L}^2}{n \probavailable^2} + \frac{(1 - \probpage)L_{\max}^2}{n \probavailable B}\right) \norm{x^{t+1} - x^{t}}^2\\
    &\quad + \frac{2\left(\probavailable - \probpairaa\right) b^2}{n^2 \probavailable^2 \probpage}\sum_{i=1}^n\norm{ h^t_i - \nabla f_i(x^{t})}^2 + \left(\probpage\left(1 - \frac{b}{\probpage}\right)^2 + (1 - \probpage)\right)\norm{h^{t} - \nabla f(x^{t})}^2.
  \end{align*}
  The proof of the second inequality almost repeats the previous one:
  \begin{align}
    &\ExpSub{B}{\ExpSub{\probavailable}{\ExpSub{\probpage}{\norm{h^{t+1}_i - \nabla f_i(x^{t+1})}^2}}} \nonumber \\
    &=\probpage\ExpSub{B}{\ExpSub{\probavailable}{\norm{h^{t+1}_{i,1} - \nabla f_i(x^{t+1})}^2}} + (1 - \probpage)\ExpSub{B}{\ExpSub{\probavailable}{\norm{h^{t+1}_{i,2} - \nabla f_i(x^{t+1})}^2}} \nonumber\\
    &\overset{\eqref{auxiliary:variance_decomposition}}{=}\probpage\ExpSub{B}{\ExpSub{\probavailable}{\norm{h^{t+1}_{i,1} - \ExpSub{B}{\ExpSub{\probavailable}{h^{t+1}_{i,1}}}}^2}} + (1 - \probpage)\ExpSub{B}{\ExpSub{\probavailable}{\norm{h^{t+1}_{i,2} - \ExpSub{B}{\ExpSub{\probavailable}{h^{t+1}_{i,2}}}}^2}} \nonumber\\
    &\quad +\probpage\norm{\ExpSub{B}{\ExpSub{\probavailable}{h^{t+1}_{i,1}}} - \nabla f_i(x^{t+1})}^2 + (1 - \probpage)\norm{\ExpSub{B}{\ExpSub{\probavailable}{h^{t+1}_{i,2}}} - \nabla f_i(x^{t+1})}^2 \nonumber \\
    &=\probpage\ExpSub{B}{\ExpSub{\probavailable}{\norm{h^{t+1}_{i,1} - \ExpSub{B}{\ExpSub{\probavailable}{h^{t+1}_{i,1}}}}^2}} + (1 - \probpage)\ExpSub{B}{\ExpSub{\probavailable}{\norm{h^{t+1}_{i,2} - \ExpSub{B}{\ExpSub{\probavailable}{h^{t+1}_{i,2}}}}^2}} \nonumber\\
    &\quad +\left(\probpage\left(1 - \frac{b}{\probpage}\right)^2 + (1 - \probpage)\right)\norm{h^{t}_i - \nabla f_i(x^{t})}^2 \label{eq:page_proof:mini:orig}.
  \end{align}
  Let us consider $\ExpSub{B}{\ExpSub{\probavailable}{\norm{h^{t+1}_{i,1} - \ExpSub{B}{\ExpSub{\probavailable}{h^{t+1}_{i,1}}}}^2}}$:
  \begin{align*}
    &\ExpSub{B}{\ExpSub{\probavailable}{\norm{h^{t+1}_{i,1} - \ExpSub{B}{\ExpSub{\probavailable}{h^{t+1}_{i,1}}}}^2}}\\
    &=\ExpSub{\probavailable}{\norm{h^{t+1}_{i,1} - \ExpSub{B}{\ExpSub{\probavailable}{h^{t+1}_{i,1}}}}^2} \\
    &=\probavailable \norm{h^t_i + \frac{1}{\probavailable} k^{t+1}_{i, 1} - \left(h^t_i + \nabla f_i(x^{t+1}) - \nabla f_i(x^{t}) - \frac{b}{\probpage} \left(h^t_i - \nabla f_i(x^{t})\right)\right)}^2 \\
    &\quad + (1 - \probavailable) \norm{h^t_i - \left(h^t_i + \nabla f_i(x^{t+1}) - \nabla f_i(x^{t}) - \frac{b}{\probpage} \left(h^t_i - \nabla f_i(x^{t})\right)\right)}^2 \\
    &=\frac{(1 - \probavailable)^2}{\probavailable} \norm{\nabla f_i(x^{t+1}) - \nabla f_i(x^{t}) - \frac{b}{\probpage} \left(h^t_i - \nabla f_i(x^{t})\right)}^2 \\
    &\quad + (1 - \probavailable) \norm{\nabla f_i(x^{t+1}) - \nabla f_i(x^{t}) - \frac{b}{\probpage} \left(h^t_i - \nabla f_i(x^{t})\right)}^2 \\
    &=\frac{1 - \probavailable}{\probavailable} \norm{\nabla f_i(x^{t+1}) - \nabla f_i(x^{t}) - \frac{b}{\probpage} \left(h^t_i - \nabla f_i(x^{t})\right)}^2.
  \end{align*}
  Considering \eqref{auxiliary:jensen_inequality} and Assumption~\ref{ass:nodes_lipschitz_constant}, we obtain
  \begin{align}
    &\ExpSub{B}{\ExpSub{\probavailable}{\norm{h^{t+1}_{i,1} - \ExpSub{B}{\ExpSub{\probavailable}{h^{t+1}_{i,1}}}}^2}} \nonumber\\
    &\leq \frac{2\left(1 - \probavailable\right)L_i^2}{\probavailable} \norm{x^{t+1} - x^{t}}^2 + \frac{2\left(1 - \probavailable\right)b^2}{\probavailable \probpage^2} \norm{h^t_i - \nabla f_i(x^{t})}^2 \label{eq:page_proof:mini:h_1}.
  \end{align}
  Next, we obtain the bound for $\ExpSub{B}{\ExpSub{\probavailable}{\norm{h^{t+1}_{i,2} - \ExpSub{B}{\ExpSub{\probavailable}{h^{t+1}_{i,2}}}}^2}}$:
  \begin{align}
    &\ExpSub{B}{\ExpSub{\probavailable}{\norm{h^{t+1}_{i,2} - \ExpSub{B}{\ExpSub{\probavailable}{h^{t+1}_{i,2}}}}^2}} \nonumber \\
    &= \probavailable \ExpSub{B}{\norm{h^t_i + \frac{1}{\probavailable} k^{t+1}_{i, 2} - \left(h^t_i + \nabla f_i(x^{t+1}) - \nabla f_i(x^{t})\right)}^2} \nonumber\\
    &\quad +(1 - \probavailable) \ExpSub{B}{\norm{h^{t}_{i} - \left(h^t_i + \nabla f_i(x^{t+1}) - \nabla f_i(x^{t})\right)}^2}\nonumber \\
    &= \probavailable \ExpSub{B}{\norm{\frac{1}{\probavailable} k^{t+1}_{i, 2} - \left(\nabla f_i(x^{t+1}) - \nabla f_i(x^{t})\right)}^2}\nonumber \\
    &\quad +(1 - \probavailable) \norm{\nabla f_i(x^{t+1}) - \nabla f_i(x^{t})}^2\nonumber \\
    &\overset{\eqref{auxiliary:variance_decomposition}}{=} \frac{1}{\probavailable} \ExpSub{B}{\norm{k^{t+1}_{i, 2} - \left(\nabla f_i(x^{t+1}) - \nabla f_i(x^{t})\right)}^2} + \frac{(1 - \probavailable)^2}{\probavailable} \norm{\nabla f_i(x^{t+1}) - \nabla f_i(x^{t})}^2\nonumber \\
    &\quad +(1 - \probavailable) \norm{\nabla f_i(x^{t+1}) - \nabla f_i(x^{t})}^2 \nonumber \\
    &= \frac{1}{\probavailable} \ExpSub{B}{\norm{k^{t+1}_{i, 2} - \left(\nabla f_i(x^{t+1}) - \nabla f_i(x^{t})\right)}^2} + \frac{1 - \probavailable}{\probavailable} \norm{\nabla f_i(x^{t+1}) - \nabla f_i(x^{t})}^2 \nonumber \\
    &\leq \frac{1}{\probavailable} \ExpSub{B}{\norm{k^{t+1}_{i, 2} - \left(\nabla f_i(x^{t+1}) - \nabla f_i(x^{t})\right)}^2} + \frac{\left(1 - \probavailable\right)L^2_i}{\probavailable} \norm{x^{t+1} - x^{t}}^2, \label{eq:page_proof:mini:h_2}
  \end{align}
  where we used Assumption~\ref{ass:nodes_lipschitz_constant}. By plugging \eqref{eq:page_proof:mini:h_1} and \eqref{eq:page_proof:mini:h_2} into \eqref{eq:page_proof:mini:orig}, we get
  \begin{align*}
    &\ExpSub{B}{\ExpSub{\probavailable}{\ExpSub{\probpage}{\norm{h^{t+1}_i - \nabla f_i(x^{t+1})}^2}}} \nonumber \\
    &\leq\probpage\left(\frac{2\left(1 - \probavailable\right)L_i^2}{\probavailable} \norm{x^{t+1} - x^{t}}^2 + \frac{2\left(1 - \probavailable\right)b^2}{\probavailable \probpage^2} \norm{h^t_i - \nabla f_i(x^{t})}^2\right) \\
    &\quad + (1 - \probpage)\left(\frac{1}{\probavailable} \ExpSub{B}{\norm{k^{t+1}_{i, 2} - \left(\nabla f_i(x^{t+1}) - \nabla f_i(x^{t})\right)}^2} + \frac{\left(1 - \probavailable\right)L^2_i}{\probavailable} \norm{x^{t+1} - x^{t}}^2\right) \\
    &\quad +\left(\probpage\left(1 - \frac{b}{\probpage}\right)^2 + (1 - \probpage)\right)\norm{h^{t}_i - \nabla f_i(x^{t})}^2 \\
    &\leq \frac{2\left(1 - \probavailable\right)L_i^2}{\probavailable} \norm{x^{t+1} - x^{t}}^2 + \frac{1 - \probpage}{\probavailable} \ExpSub{B}{\norm{k^{t+1}_{i, 2} - \left(\nabla f_i(x^{t+1}) - \nabla f_i(x^{t})\right)}^2} \\
    &\quad +\left(\frac{2\left(1 - \probavailable\right)b^2}{\probavailable \probpage} + \probpage\left(1 - \frac{b}{\probpage}\right)^2 + (1 - \probpage)\right)\norm{h^{t}_i - \nabla f_i(x^{t})}^2.
  \end{align*}
  From the independence of elements in the mini-batch, we obtain
  \begin{align*}
    &\ExpSub{B}{\ExpSub{\probavailable}{\ExpSub{\probpage}{\norm{h^{t+1}_i - \nabla f_i(x^{t+1})}^2}}} \nonumber \\
    &\leq \frac{2\left(1 - \probavailable\right)L_i^2}{\probavailable} \norm{x^{t+1} - x^{t}}^2 + \frac{1 - \probpage}{\probavailable} \ExpSub{B}{\norm{\frac{1}{B}\sum_{j \in I^t_i}\left(\nabla f_{ij}(x^{t+1}) - \nabla f_{ij}(x^{t})\right) - \left(\nabla f_i(x^{t+1}) - \nabla f_i(x^{t})\right)}^2} \\
    &\quad +\left(\frac{2\left(1 - \probavailable\right)b^2}{\probavailable \probpage} + \probpage\left(1 - \frac{b}{\probpage}\right)^2 + (1 - \probpage)\right)\norm{h^{t}_i - \nabla f_i(x^{t})}^2 \\
    &= \frac{2\left(1 - \probavailable\right)L_i^2}{\probavailable} \norm{x^{t+1} - x^{t}}^2 + \frac{1 - \probpage}{\probavailable B^2} \ExpSub{B}{\sum_{j \in I^t_i} \norm{\left(\nabla f_{ij}(x^{t+1}) - \nabla f_{ij}(x^{t})\right) - \left(\nabla f_i(x^{t+1}) - \nabla f_i(x^{t})\right)}^2} \\
    &\quad +\left(\frac{2\left(1 - \probavailable\right)b^2}{\probavailable \probpage} + \probpage\left(1 - \frac{b}{\probpage}\right)^2 + (1 - \probpage)\right)\norm{h^{t}_i - \nabla f_i(x^{t})}^2 \\
    &= \frac{2\left(1 - \probavailable\right)L_i^2}{\probavailable} \norm{x^{t+1} - x^{t}}^2 + \frac{1 - \probpage}{m \probavailable B} \sum_{j=1}^m \norm{\left(\nabla f_{ij}(x^{t+1}) - \nabla f_{ij}(x^{t})\right) - \left(\nabla f_i(x^{t+1}) - \nabla f_i(x^{t})\right)}^2 \\
    &\quad +\left(\frac{2\left(1 - \probavailable\right)b^2}{\probavailable \probpage} + \probpage\left(1 - \frac{b}{\probpage}\right)^2 + (1 - \probpage)\right)\norm{h^{t}_i - \nabla f_i(x^{t})}^2 \\
    &\leq \frac{2\left(1 - \probavailable\right)L_i^2}{\probavailable} \norm{x^{t+1} - x^{t}}^2 + \frac{1 - \probpage}{m \probavailable B} \sum_{j=1}^m \norm{\nabla f_{ij}(x^{t+1}) - \nabla f_{ij}(x^{t})}^2 \\
    &\quad +\left(\frac{2\left(1 - \probavailable\right)b^2}{\probavailable \probpage} + \probpage\left(1 - \frac{b}{\probpage}\right)^2 + (1 - \probpage)\right)\norm{h^{t}_i - \nabla f_i(x^{t})}^2 \\
    &\leq \left(\frac{2\left(1 - \probavailable\right)L_i^2}{\probavailable} + \frac{(1 - \probpage)L_{\max}^2}{\probavailable B}\right) \norm{x^{t+1} - x^{t}}^2 \\
    &\quad +\left(\frac{2\left(1 - \probavailable\right)b^2}{\probavailable \probpage} + \probpage\left(1 - \frac{b}{\probpage}\right)^2 + (1 - \probpage)\right)\norm{h^{t}_i - \nabla f_i(x^{t})}^2, \\
  \end{align*}
  where we used Assumption~\ref{ass:max_lipschitz_constant}. Finally, we prove the last inequality:
  \begin{align*}
    &\ExpSub{B}{\ExpSub{\probpage}{\norm{k^{t+1}_i}^2}} \\
    &=\probpage \norm{\nabla f_i(x^{t+1}) - \nabla f_i(x^{t}) - \frac{b}{\probpage} \left(h^t_i - \nabla f_i(x^{t})\right)}^2 \\
    &\quad + (1 - \probpage) \ExpSub{B}{\norm{\frac{1}{B}\sum_{j \in I^t_i}\left(\nabla f_{ij}(x^{t+1}) - \nabla f_{ij}(x^{t})\right)}^2} \\
    &\overset{\eqref{auxiliary:variance_decomposition}}{=}\probpage \norm{\nabla f_i(x^{t+1}) - \nabla f_i(x^{t}) - \frac{b}{\probpage} \left(h^t_i - \nabla f_i(x^{t})\right)}^2 \\
    &\quad + (1 - \probpage) \ExpSub{B}{\norm{\frac{1}{B}\sum_{j \in I^t_i}\left(\nabla f_{ij}(x^{t+1}) - \nabla f_{ij}(x^{t})\right) - \left(\nabla f_{i}(x^{t+1}) - \nabla f_{i}(x^{t})\right)}^2} \\
    &\quad + (1 - \probpage) \norm{\nabla f_{i}(x^{t+1}) - \nabla f_{i}(x^{t})}^2\\
    &\overset{\eqref{auxiliary:jensen_inequality}}{\leq}2 \probpage \norm{\nabla f_i(x^{t+1}) - \nabla f_i(x^{t})}^2 +  \frac{2b^2}{\probpage} \norm{h^t_i - \nabla f_i(x^{t})}^2 \\
    &\quad + (1 - \probpage) \ExpSub{B}{\norm{\frac{1}{B}\sum_{j \in I^t_i}\left(\nabla f_{ij}(x^{t+1}) - \nabla f_{ij}(x^{t})\right) - \left(\nabla f_{i}(x^{t+1}) - \nabla f_{i}(x^{t})\right)}^2} \\
    &\quad + (1 - \probpage) \norm{\nabla f_{i}(x^{t+1}) - \nabla f_{i}(x^{t})}^2\\
    &\leq2 \norm{\nabla f_i(x^{t+1}) - \nabla f_i(x^{t})}^2 +  \frac{2b^2}{\probpage} \norm{h^t_i - \nabla f_i(x^{t})}^2 \\
    &\quad + (1 - \probpage) \ExpSub{B}{\norm{\frac{1}{B}\sum_{j \in I^t_i}\left(\nabla f_{ij}(x^{t+1}) - \nabla f_{ij}(x^{t})\right) - \left(\nabla f_{i}(x^{t+1}) - \nabla f_{i}(x^{t})\right)}^2}.
  \end{align*}
  Using the independence of elements in the mini-batch, we have
  \begin{align*}
    &\ExpSub{B}{\ExpSub{\probpage}{\norm{k^{t+1}_i}^2}} \\
    &\leq 2 \norm{\nabla f_i(x^{t+1}) - \nabla f_i(x^{t})}^2 +  \frac{2b^2}{\probpage} \norm{h^t_i - \nabla f_i(x^{t})}^2 \\
    &\quad + \frac{1 - \probpage}{B^2} \ExpSub{B}{\sum_{j \in I^t_i}\norm{\left(\nabla f_{ij}(x^{t+1}) - \nabla f_{ij}(x^{t})\right) - \left(\nabla f_{i}(x^{t+1}) - \nabla f_{i}(x^{t})\right)}^2} \\ 
    &= 2 \norm{\nabla f_i(x^{t+1}) - \nabla f_i(x^{t})}^2 +  \frac{2b^2}{\probpage} \norm{h^t_i - \nabla f_i(x^{t})}^2 \\
    &\quad + \frac{1 - \probpage}{B m} \sum_{j=1}^m \norm{\left(\nabla f_{ij}(x^{t+1}) - \nabla f_{ij}(x^{t})\right) - \left(\nabla f_{i}(x^{t+1}) - \nabla f_{i}(x^{t})\right)}^2 \\ 
    &\leq 2 \norm{\nabla f_i(x^{t+1}) - \nabla f_i(x^{t})}^2 +  \frac{2b^2}{\probpage} \norm{h^t_i - \nabla f_i(x^{t})}^2 \\
    &\quad + \frac{1 - \probpage}{B m} \sum_{j=1}^m \norm{\nabla f_{ij}(x^{t+1}) - \nabla f_{ij}(x^{t})}^2 \\.
  \end{align*}
  It it left to consider Assumptions~\ref{ass:nodes_lipschitz_constant} and \ref{ass:max_lipschitz_constant} to get
  \begin{align*}
    &\ExpSub{B}{\ExpSub{\probpage}{\norm{k^{t+1}_i}^2}} \\
    &\leq \left(2 L_i^2 + \frac{(1 - \probpage)L_{\max}^2}{B}\right)\norm{x^{t+1} - x^{t}}^2 +  \frac{2b^2}{\probpage} \norm{h^t_i - \nabla f_i(x^{t})}^2.
  \end{align*}
\end{proof}

\CONVERGENCEPAGE*

\begin{proof}
  Let us fix constants $\nu, \rho \in [0,\infty)$ that we will define later. Considering Lemma~\ref{lemma:main_lemma}, Lemma~\ref{lemma:gradient_page}, and the law of total expectation, we obtain
    \begin{align*}
      &\Exp{f(x^{t + 1})} + \frac{\gamma (2 \omega + 1)}{\probavailable} \Exp{\norm{g^{t+1} - h^{t+1}}^2} + \frac{\gamma (\left(2 \omega + 1\right)\probavailable - \probpairaa)}{n \probavailable^2} \Exp{\frac{1}{n}\sum_{i=1}^n\norm{g^{t+1}_i - h^{t+1}_i}^2}\\
      &\quad  + \nu \Exp{\norm{h^{t+1} - \nabla f(x^{t+1})}^2} + \rho \Exp{\frac{1}{n}\sum_{i=1}^n\norm{h^{t+1}_i - \nabla f_i(x^{t+1})}^2}\\
      &\leq \Exp{f(x^t) - \frac{\gamma}{2}\norm{\nabla f(x^t)}^2 - \left(\frac{1}{2\gamma} - \frac{L}{2}\right)
      \norm{x^{t+1} - x^t}^2 + \gamma \norm{h^{t} - \nabla f(x^t)}^2}\nonumber\\
      &\quad + \frac{\gamma (2 \omega + 1)}{\probavailable}\Exp{\norm{g^{t} - h^t}^2}+ \frac{\gamma (\left(2 \omega + 1\right)\probavailable - \probpairaa)}{n \probavailable^2}\Exp{\frac{1}{n} \sum_{i=1}^n\norm{g^t_i - h^{t}_i}^2} \\
      &\quad + \frac{4 \gamma \omega (2 \omega + 1)}{n \probavailable^2} \Exp{\frac{1}{n} \sum_{i=1}^n\norm{k^{t+1}_i}^2} \\
      &\quad  + \nu \Exp{\norm{h^{t+1} - \nabla f(x^{t+1})}^2} + \rho \Exp{\frac{1}{n}\sum_{i=1}^n\norm{h^{t+1}_i - \nabla f_i(x^{t+1})}^2}\\
      &= \Exp{f(x^t) - \frac{\gamma}{2}\norm{\nabla f(x^t)}^2 - \left(\frac{1}{2\gamma} - \frac{L}{2}\right)
      \norm{x^{t+1} - x^t}^2 + \gamma \norm{h^{t} - \nabla f(x^t)}^2}\nonumber\\
      &\quad + \frac{\gamma (2 \omega + 1)}{\probavailable}\Exp{\norm{g^{t} - h^t}^2}+ \frac{\gamma (\left(2 \omega + 1\right)\probavailable - \probpairaa)}{n \probavailable^2}\Exp{\frac{1}{n} \sum_{i=1}^n\norm{g^t_i - h^{t}_i}^2} \\
      &\quad + \frac{4 \gamma \omega (2 \omega + 1)}{n \probavailable^2} \Exp{\ExpSub{B}{\ExpSub{\probpage}{\frac{1}{n} \sum_{i=1}^n\norm{k^{t+1}_i}^2}}} \\
      &\quad  + \nu \Exp{\ExpSub{B}{\ExpSub{\probavailable}{\ExpSub{\probpage}{\norm{h^{t+1} - \nabla f(x^{t+1})}^2}}}} \\
      &\quad + \rho \Exp{\ExpSub{B}{\ExpSub{\probavailable}{\ExpSub{\probpage}{\frac{1}{n}\sum_{i=1}^n\norm{h^{t+1}_i - \nabla f_i(x^{t+1})}^2}}}}\\
      &\leq \Exp{f(x^t) - \frac{\gamma}{2}\norm{\nabla f(x^t)}^2 - \left(\frac{1}{2\gamma} - \frac{L}{2}\right)
      \norm{x^{t+1} - x^t}^2 + \gamma \norm{h^{t} - \nabla f(x^t)}^2}\nonumber\\
      &\quad + \frac{\gamma (2 \omega + 1)}{\probavailable}\Exp{\norm{g^{t} - h^t}^2}+ \frac{\gamma (\left(2 \omega + 1\right)\probavailable - \probpairaa)}{n \probavailable^2}\Exp{\frac{1}{n} \sum_{i=1}^n\norm{g^t_i - h^{t}_i}^2} \\
      &\quad + \frac{4 \gamma \omega (2 \omega + 1)}{n \probavailable^2} \Exp{\left(2 \widehat{L}^2 + \frac{(1 - \probpage)L_{\max}^2}{B}\right)\norm{x^{t+1} - x^{t}}^2 +  \frac{2b^2}{\probpage} \frac{1}{n}\sum_{i=1}^n \norm{h^t_i - \nabla f_i(x^{t})}^2} \\
      &\quad  + \nu {\rm E}\Bigg(\left(\frac{2 \left(\probavailable - \probpairaa\right) \widehat{L}^2}{n \probavailable^2} + \frac{(1 - \probpage)L_{\max}^2}{n \probavailable B}\right) \norm{x^{t+1} - x^{t}}^2\\
      &\qquad\quad + \frac{2\left(\probavailable - \probpairaa\right) b^2}{n^2 \probavailable^2 \probpage}\sum_{i=1}^n\norm{ h^t_i - \nabla f_i(x^{t})}^2 + \left(\probpage\left(1 - \frac{b}{\probpage}\right)^2 + (1 - \probpage)\right)\norm{h^{t} - \nabla f(x^{t})}^2\Bigg) \\
      &\quad + \rho {\rm E}\Bigg(\left(\frac{2\left(1 - \probavailable\right)\widehat{L}^2}{\probavailable} + \frac{(1 - \probpage)L_{\max}^2}{\probavailable B}\right) \norm{x^{t+1} - x^{t}}^2 \\
      &\qquad\quad +\left(\frac{2\left(1 - \probavailable\right)b^2}{\probavailable \probpage} + \probpage\left(1 - \frac{b}{\probpage}\right)^2 + (1 - \probpage)\right)\frac{1}{n}\sum_{i=1}^n\norm{h^{t}_i - \nabla f_i(x^{t})}^2\Bigg)\\
    \end{align*}
    After rearranging the terms, we get
    \begin{align*}
      &\Exp{f(x^{t + 1})} + \frac{\gamma (2 \omega + 1)}{\probavailable} \Exp{\norm{g^{t+1} - h^{t+1}}^2} + \frac{\gamma (\left(2 \omega + 1\right)\probavailable - \probpairaa)}{n \probavailable^2} \Exp{\frac{1}{n}\sum_{i=1}^n\norm{g^{t+1}_i - h^{t+1}_i}^2}\\
      &\quad  + \nu \Exp{\norm{h^{t+1} - \nabla f(x^{t+1})}^2} + \rho \Exp{\frac{1}{n}\sum_{i=1}^n\norm{h^{t+1}_i - \nabla f_i(x^{t+1})}^2}\\
      &\leq \Exp{f(x^t)} - \frac{\gamma}{2}\Exp{\norm{\nabla f(x^t)}^2} \\
      &\quad + \frac{\gamma (2 \omega + 1)}{\probavailable} \Exp{\norm{g^{t} - h^{t}}^2} + \frac{\gamma (\left(2 \omega + 1\right)\probavailable - \probpairaa)}{n \probavailable^2} \Exp{\frac{1}{n}\sum_{i=1}^n\norm{g^{t}_i - h^{t}_i}^2} \\
      &\quad - \Bigg(\frac{1}{2\gamma} - \frac{L}{2} - \frac{4 \gamma \omega (2 \omega + 1)}{n \probavailable^2} \left(2 \widehat{L}^2 + \frac{(1 - \probpage)L_{\max}^2}{B}\right) \\
      &\qquad\quad - \nu\left(\frac{2 \left(\probavailable - \probpairaa\right) \widehat{L}^2}{n \probavailable^2} + \frac{(1 - \probpage)L_{\max}^2}{n \probavailable B}\right) - \rho \left(\frac{2\left(1 - \probavailable\right)\widehat{L}^2}{\probavailable} + \frac{(1 - \probpage)L_{\max}^2}{\probavailable B}\right)\Bigg) \Exp{\norm{x^{t+1} - x^t}^2} \\
      &\quad + \left(\gamma + \nu \left(\probpage\left(1 - \frac{b}{\probpage}\right)^2 + (1 - \probpage)\right)\right) \Exp{\norm{h^{t} - \nabla f(x^{t})}^2} \\
      &\quad + \Bigg(\frac{8 b^2 \gamma \omega (2 \omega + 1)}{n \probavailable^2 \probpage} + \frac{2 \nu \left(\probavailable - \probpairaa\right) b^2}{n \probavailable^2 \probpage} \\
      &\qquad\quad+ \rho \left(\frac{2\left(1 - \probavailable\right)b^2}{\probavailable \probpage} + \probpage\left(1 - \frac{b}{\probpage}\right)^2 + (1 - \probpage)\right)\Bigg)\Exp{\frac{1}{n}\sum_{i=1}^n\norm{h^{t}_i - \nabla f_i(x^{t})}^2}.
    \end{align*}
    Due to $b = \frac{\probpage \probavailable}{2 - \probavailable} \leq \probpage,$ one can show that $\left(\probpage\left(1 - \frac{b}{\probpage}\right)^2 + (1 - \probpage)\right) \leq 1 - b.$ 
    Thus, if we take $\nu = \frac{\gamma}{b},$ then
    $$\left(\gamma + \nu \left(\probpage\left(1 - \frac{b}{\probpage}\right)^2 + (1 - \probpage)\right)\right) \leq \gamma + \nu (1 - b) = \nu,$$ therefore
    \begin{align*}
      &\Exp{f(x^{t + 1})} + \frac{\gamma (2 \omega + 1)}{\probavailable} \Exp{\norm{g^{t+1} - h^{t+1}}^2} + \frac{\gamma (\left(2 \omega + 1\right)\probavailable - \probpairaa)}{n \probavailable^2} \Exp{\frac{1}{n}\sum_{i=1}^n\norm{g^{t+1}_i - h^{t+1}_i}^2}\\
      &\quad  + \frac{\gamma}{b} \Exp{\norm{h^{t+1} - \nabla f(x^{t+1})}^2} + \rho \Exp{\frac{1}{n}\sum_{i=1}^n\norm{h^{t+1}_i - \nabla f_i(x^{t+1})}^2}\\
      &\leq \Exp{f(x^t)} - \frac{\gamma}{2}\Exp{\norm{\nabla f(x^t)}^2} \\
      &\quad + \frac{\gamma (2 \omega + 1)}{\probavailable} \Exp{\norm{g^{t} - h^{t}}^2} + \frac{\gamma (\left(2 \omega + 1\right)\probavailable - \probpairaa)}{n \probavailable^2} \Exp{\frac{1}{n}\sum_{i=1}^n\norm{g^{t}_i - h^{t}_i}^2} \\
      &\quad - \Bigg(\frac{1}{2\gamma} - \frac{L}{2} - \frac{4 \gamma \omega (2 \omega + 1)}{n \probavailable^2} \left(2 \widehat{L}^2 + \frac{(1 - \probpage)L_{\max}^2}{B}\right) \\
      &\qquad\quad - \frac{\gamma}{b}\left(\frac{2 \left(\probavailable - \probpairaa\right) \widehat{L}^2}{n \probavailable^2} + \frac{(1 - \probpage)L_{\max}^2}{n \probavailable B}\right) - \rho \left(\frac{2\left(1 - \probavailable\right)\widehat{L}^2}{\probavailable} + \frac{(1 - \probpage)L_{\max}^2}{\probavailable B}\right)\Bigg) \Exp{\norm{x^{t+1} - x^t}^2} \\
      &\quad + \frac{\gamma}{b}\Exp{\norm{h^{t} - \nabla f(x^{t})}^2} \\
      &\quad + \Bigg(\frac{8 b^2 \gamma \omega (2 \omega + 1)}{n \probavailable^2 \probpage} + \frac{2 \gamma \left(\probavailable - \probpairaa\right) b}{n \probavailable^2 \probpage} \\
      &\qquad\quad+ \rho \left(\frac{2\left(1 - \probavailable\right)b^2}{\probavailable \probpage} + \probpage\left(1 - \frac{b}{\probpage}\right)^2 + (1 - \probpage)\right)\Bigg)\Exp{\frac{1}{n}\sum_{i=1}^n\norm{h^{t}_i - \nabla f_i(x^{t})}^2}.
    \end{align*}
    Next, with the choice of $b = \frac{\probpage \probavailable}{2 - \probavailable},$ we ensure that
    $$\left(\frac{2\left(1 - \probavailable\right)b^2}{\probavailable \probpage} + \probpage\left(1 - \frac{b}{\probpage}\right)^2 + (1 - \probpage)\right) \leq 1 - b.$$ If we take $\rho = \frac{8 b \gamma \omega (2 \omega + 1)}{n \probavailable^2 \probpage} + \frac{2 \gamma \left(\probavailable - \probpairaa\right)}{n \probavailable^2 \probpage},$ then
    $$\Bigg(\frac{8 b^2 \gamma \omega (2 \omega + 1)}{n \probavailable^2 \probpage} + \frac{2 \gamma \left(\probavailable - \probpairaa\right) b}{n \probavailable^2 \probpage}+ \rho \left(\frac{2\left(1 - \probavailable\right)b^2}{\probavailable \probpage} + \probpage\left(1 - \frac{b}{\probpage}\right)^2 + (1 - \probpage)\right)\Bigg) \leq \rho,$$ therefore
    \begin{align*}
      &\Exp{f(x^{t + 1})} + \frac{\gamma (2 \omega + 1)}{\probavailable} \Exp{\norm{g^{t+1} - h^{t+1}}^2} + \frac{\gamma (\left(2 \omega + 1\right)\probavailable - \probpairaa)}{n \probavailable^2} \Exp{\frac{1}{n}\sum_{i=1}^n\norm{g^{t+1}_i - h^{t+1}_i}^2}\\
      &\quad  + \frac{\gamma}{b} \Exp{\norm{h^{t+1} - \nabla f(x^{t+1})}^2} + \left(\frac{8 b \gamma \omega (2 \omega + 1)}{n \probavailable^2 \probpage} + \frac{2 \gamma \left(\probavailable - \probpairaa\right)}{n \probavailable^2 \probpage}\right) \Exp{\frac{1}{n}\sum_{i=1}^n\norm{h^{t+1}_i - \nabla f_i(x^{t+1})}^2}\\
      &\leq \Exp{f(x^t)} - \frac{\gamma}{2}\Exp{\norm{\nabla f(x^t)}^2} \\
      &\quad + \frac{\gamma (2 \omega + 1)}{\probavailable} \Exp{\norm{g^{t} - h^{t}}^2} + \frac{\gamma (\left(2 \omega + 1\right)\probavailable - \probpairaa)}{n \probavailable^2} \Exp{\frac{1}{n}\sum_{i=1}^n\norm{g^{t}_i - h^{t}_i}^2} \\
      &\quad - \Bigg(\frac{1}{2\gamma} - \frac{L}{2} - \frac{4 \gamma \omega (2 \omega + 1)}{n \probavailable^2} \left(2 \widehat{L}^2 + \frac{(1 - \probpage)L_{\max}^2}{B}\right) \\
      &\qquad\quad - \frac{\gamma}{b n \probavailable}\left(2 \left(1 - \frac{\probpairaa}{\probavailable}\right) \widehat{L}^2 + \frac{(1 - \probpage)L_{\max}^2}{B}\right) \\
      &\qquad\quad- \left(\frac{8 b \gamma \omega (2 \omega + 1)}{n \probavailable^3 \probpage} + \frac{2 \gamma \left(1 - \frac{\probpairaa}{\probavailable}\right)}{n \probavailable^2 \probpage}\right) \left(2\left(1 - \probavailable\right)\widehat{L}^2 + \frac{(1 - \probpage)L_{\max}^2}{B}\right)\Bigg) \Exp{\norm{x^{t+1} - x^t}^2} \\
      &\quad + \frac{\gamma}{b}\Exp{\norm{h^{t} - \nabla f(x^{t})}^2} + \left(\frac{8 b \gamma \omega (2 \omega + 1)}{n \probavailable^2 \probpage} + \frac{2 \gamma \left(\probavailable - \probpairaa\right)}{n \probavailable^2 \probpage}\right)\Exp{\frac{1}{n}\sum_{i=1}^n\norm{h^{t}_i - \nabla f_i(x^{t})}^2}.
    \end{align*}
    Let us simplify the inequality. First, due to $b \geq \frac{\probpage \probavailable}{2},$ we have
    $$\frac{\gamma}{b n \probavailable}\left(2 \left(1 - \frac{\probpairaa}{\probavailable}\right) \widehat{L}^2 + \frac{(1 - \probpage)L_{\max}^2}{B}\right) \leq \frac{4 \gamma}{n \probavailable^2 \probpage}\left(\left(1 - \frac{\probpairaa}{\probavailable}\right) \widehat{L}^2 + \frac{(1 - \probpage)L_{\max}^2}{B}\right).$$
    Second, due to $b \leq \probavailable \probpage$ and $\probpairaa \leq \probavailable^2$, we get
    \begin{align*}
      &\left(\frac{8 b \gamma \omega (2 \omega + 1)}{n \probavailable^3 \probpage} + \frac{2 \gamma \left(1 - \frac{\probpairaa}{\probavailable}\right)}{n \probavailable^2 \probpage}\right) \left(2\left(1 - \probavailable\right)\widehat{L}^2 + \frac{(1 - \probpage)L_{\max}^2}{B}\right) \\
      &\leq \left(\frac{8 \gamma \omega (2 \omega + 1)}{n \probavailable^2} + \frac{2 \gamma \left(1 - \frac{\probpairaa}{\probavailable}\right)}{n \probavailable^2 \probpage}\right) \left(2\left(1 - \frac{\probpairaa}{\probavailable}\right)\widehat{L}^2 + \frac{(1 - \probpage)L_{\max}^2}{B}\right) \\
      &\leq \frac{16 \gamma \omega (2 \omega + 1)}{n \probavailable^2} \left(\left(1 - \frac{\probpairaa}{\probavailable}\right)\widehat{L}^2 + \frac{(1 - \probpage)L_{\max}^2}{B}\right) \\
      &\quad + \frac{4 \gamma \left(1 - \frac{\probpairaa}{\probavailable}\right)}{n \probavailable^2 \probpage} \left(\left(1 - \frac{\probpairaa}{\probavailable}\right)\widehat{L}^2 + \frac{(1 - \probpage)L_{\max}^2}{B}\right) \\
      &\leq \frac{16 \gamma \omega (2 \omega + 1)}{n \probavailable^2} \left(\widehat{L}^2 + \frac{(1 - \probpage)L_{\max}^2}{B}\right) \\
      &\quad + \frac{4 \gamma}{n \probavailable^2 \probpage} \left(\left(1 - \frac{\probpairaa}{\probavailable}\right)\widehat{L}^2 + \frac{(1 - \probpage)L_{\max}^2}{B}\right).
    \end{align*}
    Combining all bounds together, we obtain the following simplified inequality:
    \begin{align*}
      &\Exp{f(x^{t + 1})} + \frac{\gamma (2 \omega + 1)}{\probavailable} \Exp{\norm{g^{t+1} - h^{t+1}}^2} + \frac{\gamma (\left(2 \omega + 1\right)\probavailable - \probpairaa)}{n \probavailable^2} \Exp{\frac{1}{n}\sum_{i=1}^n\norm{g^{t+1}_i - h^{t+1}_i}^2}\\
      &\quad  + \frac{\gamma}{b} \Exp{\norm{h^{t+1} - \nabla f(x^{t+1})}^2} + \left(\frac{8 b \gamma \omega (2 \omega + 1)}{n \probavailable^2 \probpage} + \frac{2 \gamma \left(\probavailable - \probpairaa\right)}{n \probavailable^2 \probpage}\right) \Exp{\frac{1}{n}\sum_{i=1}^n\norm{h^{t+1}_i - \nabla f_i(x^{t+1})}^2}\\
      &\leq \Exp{f(x^t)} - \frac{\gamma}{2}\Exp{\norm{\nabla f(x^t)}^2} \\
      &\quad + \frac{\gamma (2 \omega + 1)}{\probavailable} \Exp{\norm{g^{t} - h^{t}}^2} + \frac{\gamma (\left(2 \omega + 1\right)\probavailable - \probpairaa)}{n \probavailable^2} \Exp{\frac{1}{n}\sum_{i=1}^n\norm{g^{t}_i - h^{t}_i}^2} \\
      &\quad - \Bigg(\frac{1}{2\gamma} - \frac{L}{2} - \frac{24 \gamma \omega (2 \omega + 1)}{n \probavailable^2} \left(\widehat{L}^2 + \frac{(1 - \probpage)L_{\max}^2}{B}\right) \\
      &\qquad\quad - \frac{8 \gamma}{n \probavailable^2 \probpage} \left(\left(1 - \frac{\probpairaa}{\probavailable}\right)\widehat{L}^2 + \frac{(1 - \probpage)L_{\max}^2}{B}\right)\Bigg) \Exp{\norm{x^{t+1} - x^t}^2} \\
      &\quad + \frac{\gamma}{b}\Exp{\norm{h^{t} - \nabla f(x^{t})}^2} + \left(\frac{8 b \gamma \omega (2 \omega + 1)}{n \probavailable^2 \probpage} + \frac{2 \gamma \left(\probavailable - \probpairaa\right)}{n \probavailable^2 \probpage}\right)\Exp{\frac{1}{n}\sum_{i=1}^n\norm{h^{t}_i - \nabla f_i(x^{t})}^2}.
    \end{align*}
    Using Lemma~\ref{lemma:gamma} and the assumption about $\gamma,$ we get
    \begin{align*}
      &\Exp{f(x^{t + 1})} + \frac{\gamma (2 \omega + 1)}{\probavailable} \Exp{\norm{g^{t+1} - h^{t+1}}^2} + \frac{\gamma (\left(2 \omega + 1\right)\probavailable - \probpairaa)}{n \probavailable^2} \Exp{\frac{1}{n}\sum_{i=1}^n\norm{g^{t+1}_i - h^{t+1}_i}^2}\\
      &\quad  + \frac{\gamma}{b} \Exp{\norm{h^{t+1} - \nabla f(x^{t+1})}^2} + \left(\frac{8 b \gamma \omega (2 \omega + 1)}{n \probavailable^2 \probpage} + \frac{2 \gamma \left(\probavailable - \probpairaa\right)}{n \probavailable^2 \probpage}\right) \Exp{\frac{1}{n}\sum_{i=1}^n\norm{h^{t+1}_i - \nabla f_i(x^{t+1})}^2}\\
      &\leq \Exp{f(x^t)} - \frac{\gamma}{2}\Exp{\norm{\nabla f(x^t)}^2} \\
      &\quad + \frac{\gamma (2 \omega + 1)}{\probavailable} \Exp{\norm{g^{t} - h^{t}}^2} + \frac{\gamma (\left(2 \omega + 1\right)\probavailable - \probpairaa)}{n \probavailable^2} \Exp{\frac{1}{n}\sum_{i=1}^n\norm{g^{t}_i - h^{t}_i}^2} \\
      &\quad + \frac{\gamma}{b}\Exp{\norm{h^{t} - \nabla f(x^{t})}^2} + \left(\frac{8 b \gamma \omega (2 \omega + 1)}{n \probavailable^2 \probpage} + \frac{2 \gamma \left(\probavailable - \probpairaa\right)}{n \probavailable^2 \probpage}\right)\Exp{\frac{1}{n}\sum_{i=1}^n\norm{h^{t}_i - \nabla f_i(x^{t})}^2}.
    \end{align*}
    It is left to apply Lemma~\ref{lemma:good_recursion} with 
    \begin{eqnarray*}
      \Psi^t &=& \frac{(2 \omega + 1)}{\probavailable} \Exp{\norm{g^{t} - h^{t}}^2} + \frac{(\left(2 \omega + 1\right)\probavailable - \probpairaa)}{n \probavailable^2} \Exp{\frac{1}{n}\sum_{i=1}^n\norm{g^{t}_i - h^{t}_i}^2} \\
      &\quad& + \frac{1}{b}\Exp{\norm{h^{t} - \nabla f(x^{t})}^2} + \left(\frac{8 b \omega (2 \omega + 1)}{n \probavailable^2 \probpage} + \frac{2 \left(\probavailable - \probpairaa\right)}{n \probavailable^2 \probpage}\right)\Exp{\frac{1}{n}\sum_{i=1}^n\norm{h^{t}_i - \nabla f_i(x^{t})}^2}
    \end{eqnarray*}
    to conclude the proof.
  \end{proof}

  \COROLLARYPAGE*

  \begin{proof}
    In the view of Theorem~\ref{theorem:page}, it is enough to do
    \begin{align*}
      T \eqdef \cO\left(\frac{\Delta_0}{\varepsilon}\left[L + \sqrt{\frac{\omega^2}{n \probavailable^2} \left(\widehat{L}^2 + \frac{(1 - \probpage)L_{\max}^2}{B}\right) + \frac{1}{n \probavailable^2 \probpage} \left(\left(1 - \frac{\probpairaa}{\probavailable}\right)\widehat{L}^2 + \frac{(1 - \probpage)L_{\max}^2}{B}\right)}\right]\right)
    \end{align*}
    steps to get $\varepsilon$-solution. Using the choice of $\probmega$ and the definition of $\mathbbm{1}_{\probavailable}$, we can get \eqref{eq:rate_dasha_pp_page}. 

    Note that the expected number of gradients calculations at each communication round equals $\probmega m + (1 - \probmega) B = \frac{2 m B}{m + B} \leq 2 B.$
  \end{proof}

  \COROLLARYPAGERANDK*

  \begin{proof}
    The communication complexity equals
    \begin{eqnarray*}
        \cO\left(d + K T\right) &=& \cO\left(d + \frac{\Delta_0}{\varepsilon}\left[K L + K \frac{\omega}{\probavailable\sqrt{n}}\left(\widehat{L} + \frac{L_{\max}}{\sqrt{B}}\right) + K \frac{1}{\probavailable}\sqrt{\frac{m}{n}}\left(\frac{\mathbbm{1}_{\probavailable}\widehat{L}}{\sqrt{B}} + \frac{L_{\max}}{B}\right)\right]\right).
    \end{eqnarray*}
    Since $B \leq \frac{L_{\max}^2}{\mathbbm{1}_{\probavailable}^2 \widehat{L}^2},$ we have $\frac{\mathbbm{1}_{\probavailable}\widehat{L}}{\sqrt{B}} + \frac{L_{\max}}{B} \leq \frac{2 L_{\max}}{B}$ and 
    \begin{eqnarray*}
      \cO\left(d + K T\right) &=& \cO\left(d + \frac{\Delta_0}{\varepsilon}\left[K L + K \frac{\omega}{\probavailable\sqrt{n}}\left(\widehat{L} + \frac{L_{\max}}{\sqrt{B}}\right) + K \frac{1}{\probavailable}\sqrt{\frac{m}{n}}\frac{L_{\max}}{B}\right]\right).
    \end{eqnarray*}
    Note that $K = \Theta\left(\frac{B d}{\sqrt{m}}\right) = \cO\left(\frac{d}{\probavailable\sqrt{n}}\right)$ and $\omega + 1 = \frac{d}{K}$ due to Theorem~\ref{theorem:rand_k}, thus 
    \begin{eqnarray*}
      \cO\left(d + K T\right) &=& \cO\left(d + \frac{\Delta_0}{\varepsilon}\left[\frac{d}{\probavailable\sqrt{n}} L + \frac{d}{\probavailable\sqrt{n}}\left(\widehat{L} + \frac{L_{\max}}{\sqrt{B}}\right) + \frac{d}{\probavailable \sqrt{n}}L_{\max}\right]\right) \\
      &=& \cO\left(d + \frac{L_{\max} \Delta_0 d}{\probavailable \varepsilon \sqrt{n}}\right).
    \end{eqnarray*}
    Using the same reasoning, the expected number of gradient calculations per node equals
    \begin{eqnarray*}
        \cO\left(m + B T\right) &=& \cO\left(m + \frac{\Delta_0}{\varepsilon}\left[B L + B \frac{\omega}{\probavailable\sqrt{n}}\left(\widehat{L} + \frac{L_{\max}}{\sqrt{B}}\right) + B \frac{1}{\probavailable}\sqrt{\frac{m}{n}}\left(\frac{\mathbbm{1}_{\probavailable}\widehat{L}}{\sqrt{B}} + \frac{L_{\max}}{B}\right)\right]\right) \\
        &=& \cO\left(m + \frac{\Delta_0}{\varepsilon}\left[B L + B \frac{d}{K \probavailable\sqrt{n}}\left(\widehat{L} + \frac{L_{\max}}{\sqrt{B}}\right) + B \frac{1}{\probavailable}\sqrt{\frac{m}{n}}\frac{L_{\max}}{B}\right]\right) \\
        &=& \cO\left(m + \frac{\Delta_0}{\varepsilon}\left[\frac{1}{\probavailable}\sqrt{\frac{m}{n}} L + \frac{\sqrt{m}}{\probavailable\sqrt{n}}\left(\widehat{L} + \frac{L_{\max}}{\sqrt{B}}\right) + \frac{1}{\probavailable}\sqrt{\frac{m}{n}}L_{\max}\right]\right) \\
        &=& \cO\left(m + \frac{L_{\max} \Delta_0 \sqrt{m}}{\probavailable \varepsilon \sqrt{n}}\right).
    \end{eqnarray*}
  \end{proof}

\subsection{Proof for \algname{\algorithmname-FINITE-MVR}}

\label{sec:proof_finite_mvr}

\begin{lemma}
  \label{lemma:finite_mvr}
  Suppose that Assumptions \ref{ass:nodes_lipschitz_constant}, \ref{ass:max_lipschitz_constant}, and \ref{ass:partial_participation} hold. For $h^{t+1}_i$, $h^{t+1}_{ij}$ and $k^{t+1}_i$ from Algorithm~\ref{alg:main_algorithm} (\algname{\algorithmname-FINITE-MVR}) we have
  \begin{enumerate}
  \item
      \begin{align*}
          &\ExpSub{B}{\ExpSub{\probavailable}{\norm{h^{t+1} - \nabla f(x^{t+1})}^2}} \\
          & \leq \left(\frac{2 L_{\max}^2}{n \probavailable B} + \frac{2 \left(\probavailable - \probpairaa\right) \widehat{L}^2}{n \probavailable^2}\right)\norm{x^{t+1} - x^{t}}^2  \\
          &\quad + \frac{2\left(\probavailable - \probpairaa\right)b^2}{n^2 \probavailable^2}\sum_{i=1}^n\norm{h^t_{i} - \nabla f_{i}(x^{t})}^2 + \frac{2b^2}{n^2 \probavailable B m}\sum_{i=1}^n\sum_{j=1}^m\norm{h^t_{ij} - \nabla f_{ij}(x^{t})}^2 \\
          &\quad + (1 - b)^2\norm{h^{t} - \nabla f(x^{t})}^2.
      \end{align*}
  \item
      \begin{align*}
          &\ExpSub{B}{\ExpSub{\probavailable}{\norm{h^{t+1}_i - \nabla f_i(x^{t+1})}^2}} \\
          & \leq \left(\frac{2 L_{\max}^2}{\probavailable B} +\frac{2(1 - \probavailable) L_i^2}{\probavailable}\right) \norm{x^{t+1} - x^{t}}^2 \\
          &\quad + \frac{2 b^2}{\probavailable B m} \sum_{j=1}^m\norm{h^t_{ij} - \nabla f_{ij}(x^{t})}^2 +\left(\frac{2 \left(1 - \probavailable\right) b^2}{\probavailable} + (1 - b)^2\right)\norm{h^{t}_i - \nabla f_i(x^{t})}^2, \quad \forall i \in [n].
      \end{align*}
  \item
      \begin{align*}
          &\ExpSub{B}{\ExpSub{\probavailable}{\norm{h^{t+1}_{ij} - \nabla f_{ij}(x^{t+1})}^2}} \\
          & \leq \frac{2\left(1 - \frac{\probavailable B}{m}\right) L_{\max}^2}{\frac{\probavailable B}{m}} \norm{x^{t+1} - x^{t}}^2 \\
          &\quad + \left(\frac{2\left(1 - \frac{\probavailable B}{m}\right) b^2}{\frac{\probavailable B}{m}} + (1 - b)^2\right) \norm{h^t_{ij} - \nabla f_{ij}(x^{t})}^2, \quad \forall i \in [n], \forall j \in [m].
      \end{align*}
  \item
      \begin{align*}
        &\ExpSub{B}{\norm{k^{t+1}_i}^2} \\
        &\leq \left(\frac{2 L_{\max}^2}{B} + 2 L_i^2\right)\norm{x^{t+1} - x^{t}}^2 \\
        &\quad + \frac{2 b^2}{B m}\sum_{j=1}^m\norm{h^t_{ij} - \nabla f_{ij}(x^{t})}^2 + 2 b^2 \norm{h^t_i - \nabla f_i(x^{t})}^2, \quad \forall i \in [n].
      \end{align*}
  \end{enumerate}
\end{lemma}

\begin{proof}
  We start by proving the first inequality. Note that
  \begin{align*}
    &\ExpSub{B}{\ExpSub{\probavailable}{h^{t+1}_i}} \\
    & =\probavailable \left(h^t_i + \frac{1}{\probavailable}\ExpSub{B}{k^{t+1}_i}\right) + (1 - \probavailable) h^t_i \\
    & =h^t_i + \frac{1}{m}\sum_{j=1}^m \frac{B}{m} \cdot \frac{m}{B}\left(\nabla f_{ij}(x^{t+1}) - \nabla f_{ij}(x^{t}) - b \left(h^t_{ij} - \nabla f_{ij}(x^{t})\right)\right) + \left(1 - \frac{B}{m}\right) \cdot 0 \\
    & =\nabla f_{i}(x^{t+1}) + (1 - b) \left(h^t_{i} - \nabla f_{i}(x^{t})\right),
  \end{align*}
  thus
  \begin{align*}
    &\ExpSub{B}{\ExpSub{\probavailable}{\norm{h^{t+1} - \nabla f(x^{t+1})}^2}} \\
    &\overset{\eqref{auxiliary:variance_decomposition}}{=}\ExpSub{B}{\ExpSub{\probavailable}{\norm{h^{t+1} - \ExpSub{B}{\ExpSub{\probavailable}{h^{t+1}}}}^2}} + (1 - b)^2\norm{h^{t} - \nabla f(x^{t})}^2.
  \end{align*}
  We can use Lemma~\ref{lemma:sampling} with $r_i = h^{t}_i$ and $s_i = k^{t+1}_i$ to obtain
  \begin{align*}
    &\ExpSub{B}{\ExpSub{\probavailable}{\norm{h^{t+1} - \nabla f(x^{t+1})}^2}} \\
    &\leq \frac{1}{n^2 \probavailable}\sum_{i=1}^n\ExpSub{B}{\norm{k^{t+1}_i - \ExpSub{B}{k^{t+1}_i}}^2} +\frac{\probavailable - \probpairaa}{n^2 \probavailable^2}\sum_{i=1}^n\norm{\ExpSub{B}{k^{t+1}_i}}^2 \\
    &\quad + (1 - b)^2\norm{h^{t} - \nabla f(x^{t})}^2\\
    &= \frac{1}{n^2 \probavailable}\sum_{i=1}^n\ExpSub{B}{\norm{\frac{1}{m}\sum_{j=1}^m k^{t+1}_{ij} - \left(\nabla f_{i}(x^{t+1}) - \nabla f_{i}(x^{t}) - b \left(h^t_{i} - \nabla f_{i}(x^{t})\right)\right)}^2} \\
    &\quad +\frac{\probavailable - \probpairaa}{n^2 \probavailable^2}\sum_{i=1}^n\norm{\nabla f_{i}(x^{t+1}) - \nabla f_{i}(x^{t}) - b \left(h^t_{i} - \nabla f_{i}(x^{t})\right)}^2 \\
    &\quad + (1 - b)^2\norm{h^{t} - \nabla f(x^{t})}^2.
  \end{align*}
  Next, we again use Lemma~\ref{lemma:sampling} with $r_i = 0,$ $s_i = \nabla f_{ij}(x^{t+1}) - \nabla f_{ij}(x^{t}) - b \left(h^t_{ij} - \nabla f_{ij}(x^{t})\right), $ $\probavailable = \frac{B}{m},$ and $\probpairaa = \frac{B(B-1)}{m(m-1)}$:
  \begin{align*}
    &\ExpSub{B}{\ExpSub{\probavailable}{\norm{h^{t+1} - \nabla f(x^{t+1})}^2}} \\
    &\leq \frac{1}{n^2 \probavailable}\sum_{i=1}^n\left(\frac{m - B}{B m (m - 1)}\sum_{j=1}^m\norm{\nabla f_{ij}(x^{t+1}) - \nabla f_{ij}(x^{t}) - b \left(h^t_{ij} - \nabla f_{ij}(x^{t})\right)}^2\right) \\
    &\quad +\frac{\probavailable - \probpairaa}{n^2 \probavailable^2}\sum_{i=1}^n\norm{\nabla f_{i}(x^{t+1}) - \nabla f_{i}(x^{t}) - b \left(h^t_{i} - \nabla f_{i}(x^{t})\right)}^2 \\
    &\quad + (1 - b)^2\norm{h^{t} - \nabla f(x^{t})}^2\\
    &\leq \frac{1}{n^2 \probavailable B m}\sum_{i=1}^n\sum_{j=1}^m\norm{\nabla f_{ij}(x^{t+1}) - \nabla f_{ij}(x^{t}) - b \left(h^t_{ij} - \nabla f_{ij}(x^{t})\right)}^2 \\
    &\quad +\frac{\probavailable - \probpairaa}{n^2 \probavailable^2}\sum_{i=1}^n\norm{\nabla f_{i}(x^{t+1}) - \nabla f_{i}(x^{t}) - b \left(h^t_{i} - \nabla f_{i}(x^{t})\right)}^2 \\
    &\quad + (1 - b)^2\norm{h^{t} - \nabla f(x^{t})}^2\\
    &\overset{\eqref{auxiliary:jensen_inequality}}{\leq} \frac{2}{n^2 \probavailable B m}\sum_{i=1}^n\sum_{j=1}^m\norm{\nabla f_{ij}(x^{t+1}) - \nabla f_{ij}(x^{t})}^2 + \frac{2b^2}{n^2 \probavailable B m}\sum_{i=1}^n\sum_{j=1}^m\norm{h^t_{ij} - \nabla f_{ij}(x^{t})}^2 \\
    &\quad +\frac{2 \left(\probavailable - \probpairaa\right)}{n^2 \probavailable^2}\sum_{i=1}^n\norm{\nabla f_{i}(x^{t+1}) - \nabla f_{i}(x^{t})}^2 + \frac{2\left(\probavailable - \probpairaa\right)b^2}{n^2 \probavailable^2}\sum_{i=1}^n\norm{h^t_{i} - \nabla f_{i}(x^{t})}^2 \\
    &\quad + (1 - b)^2\norm{h^{t} - \nabla f(x^{t})}^2.
  \end{align*}
  Due to Assumptions \ref{ass:nodes_lipschitz_constant} and \ref{ass:max_lipschitz_constant}, we have
  \begin{align*}
    &\ExpSub{B}{\ExpSub{\probavailable}{\norm{h^{t+1} - \nabla f(x^{t+1})}^2}} \\
    &\leq \left(\frac{2 L_{\max}^2}{n \probavailable B} + \frac{2 \left(\probavailable - \probpairaa\right) \widehat{L}^2}{n \probavailable^2}\right)\norm{x^{t+1} - x^{t}}^2  \\
    &\quad + \frac{2\left(\probavailable - \probpairaa\right)b^2}{n^2 \probavailable^2}\sum_{i=1}^n\norm{h^t_{i} - \nabla f_{i}(x^{t})}^2 + \frac{2b^2}{n^2 \probavailable B m}\sum_{i=1}^n\sum_{j=1}^m\norm{h^t_{ij} - \nabla f_{ij}(x^{t})}^2 \\
    &\quad + (1 - b)^2\norm{h^{t} - \nabla f(x^{t})}^2.
  \end{align*}
  Let us get the bound for the second inequality:
  \begin{align*}
    &\ExpSub{B}{\ExpSub{\probavailable}{\norm{h^{t+1}_i - \nabla f_i(x^{t+1})}^2}} \\
    &\overset{\eqref{auxiliary:variance_decomposition}}{=}\ExpSub{B}{\ExpSub{\probavailable}{ \norm{h^{t+1}_i - \left(\nabla f_i(x^{t+1}) + (1 - b)(h^t_i - \nabla f_i(x^{t}))\right)}^2}} \\
    &\quad +(1 - b)^2\norm{h^{t}_i - \nabla f_i(x^{t})}^2 \\
    &=\probavailable\ExpSub{B}{\norm{h^{t}_i + \frac{1}{\probavailable}k^{t+1}_i - \left(\nabla f_i(x^{t+1}) + (1 - b)(h^t_i - \nabla f_i(x^{t}))\right)}^2} \\
    &\quad +(1 - \probavailable)\norm{h^{t}_i - \left(\nabla f_i(x^{t+1}) + (1 - b)(h^t_i - \nabla f_i(x^{t}))\right)}^2 \\
    &\quad +(1 - b)^2\norm{h^{t}_i - \nabla f_i(x^{t})}^2 \\
    &\overset{\eqref{auxiliary:variance_decomposition}}{=}\frac{1}{\probavailable} \ExpSub{B}{\norm{k^{t+1}_i - \ExpSub{B}{k^{t+1}_i}}^2} \\
    &\quad +\frac{1 - \probavailable}{\probavailable} \norm{\nabla f_i(x^{t+1}) - \nabla f_i(x^{t}) - b(h^t_i - \nabla f_i(x^{t}))}^2 \\
    &\quad +(1 - b)^2\norm{h^{t}_i - \nabla f_i(x^{t})}^2.
  \end{align*}
  Let us use Lemma~\ref{lemma:sampling} with $r_i = 0,$ $s_i = \nabla f_{ij}(x^{t+1}) - \nabla f_{ij}(x^{t}) - b \left(h^t_{ij} - \nabla f_{ij}(x^{t})\right), $ $\probavailable = \frac{B}{m},$ and $\probpairaa = \frac{B(B-1)}{m(m-1)}$:
  \begin{align*}
    &\ExpSub{B}{\ExpSub{\probavailable}{\norm{h^{t+1}_i - \nabla f_i(x^{t+1})}^2}} \\
    &\leq\frac{1}{\probavailable} \left(\frac{m - B}{B m (m - 1)}\sum_{j=1}^m\norm{\nabla f_{ij}(x^{t+1}) - \nabla f_{ij}(x^{t}) - b \left(h^t_{ij} - \nabla f_{ij}(x^{t})\right)}^2\right) \\
    &\quad +\frac{1 - \probavailable}{\probavailable} \norm{\nabla f_i(x^{t+1}) - \nabla f_i(x^{t}) - b(h^t_i - \nabla f_i(x^{t}))}^2 \\
    &\quad +(1 - b)^2\norm{h^{t}_i - \nabla f_i(x^{t})}^2 \\
    &\leq\frac{1}{\probavailable B m} \sum_{j=1}^m\norm{\nabla f_{ij}(x^{t+1}) - \nabla f_{ij}(x^{t}) - b \left(h^t_{ij} - \nabla f_{ij}(x^{t})\right)}^2 \\
    &\quad +\frac{1 - \probavailable}{\probavailable} \norm{\nabla f_i(x^{t+1}) - \nabla f_i(x^{t}) - b(h^t_i - \nabla f_i(x^{t}))}^2 \\
    &\quad +(1 - b)^2\norm{h^{t}_i - \nabla f_i(x^{t})}^2 \\
    &\overset{\eqref{auxiliary:jensen_inequality}}{\leq}\frac{2}{\probavailable B m} \sum_{j=1}^m\norm{\nabla f_{ij}(x^{t+1}) - \nabla f_{ij}(x^{t})}^2 +\frac{2(1 - \probavailable)}{\probavailable} \norm{\nabla f_i(x^{t+1}) - \nabla f_i(x^{t})}^2 \\
    &\quad + \frac{2 b^2}{\probavailable B m} \sum_{j=1}^m\norm{h^t_{ij} - \nabla f_{ij}(x^{t})}^2 +\left(\frac{2 \left(1 - \probavailable\right) b^2}{\probavailable} + (1 - b)^2\right)\norm{h^{t}_i - \nabla f_i(x^{t})}^2 \\
    &\leq\left(\frac{2 L_{\max}^2}{\probavailable B} +\frac{2(1 - \probavailable) L_i^2}{\probavailable}\right) \norm{x^{t+1} - x^{t}}^2 \\
    &\quad + \frac{2 b^2}{\probavailable B m} \sum_{j=1}^m\norm{h^t_{ij} - \nabla f_{ij}(x^{t})}^2 +\left(\frac{2 \left(1 - \probavailable\right) b^2}{\probavailable} + (1 - b)^2\right)\norm{h^{t}_i - \nabla f_i(x^{t})}^2,
  \end{align*}
  where we used Assumptions \ref{ass:nodes_lipschitz_constant} and \ref{ass:max_lipschitz_constant}. We continue the proof by considering $\ExpSub{B}{\ExpSub{\probavailable}{\norm{h^{t+1}_{ij} - \nabla f_{ij}(x^{t+1})}^2}}$:
  \begin{align*}
    &\ExpSub{B}{\ExpSub{\probavailable}{\norm{h^{t+1}_{ij} - \nabla f_{ij}(x^{t+1})}^2}} \\
    &\overset{\eqref{auxiliary:variance_decomposition}}{=}\ExpSub{B}{\ExpSub{\probavailable}{\norm{h^{t+1}_{ij} - \left(\nabla f_{ij}(x^{t+1}) + (1 - b)(h^t_{ij} - \nabla f_{ij}(x^{t}))\right)}^2}} \\
    &\quad +(1 - b)^2\norm{h^{t}_{ij} - \nabla f_{ij}(x^{t})}^2 \\
    &=\frac{\probavailable B}{m} \ExpSub{B}{\norm{h^{t}_{ij} + \frac{m}{B\probavailable}\left(\nabla f_{ij}(x^{t+1}) - \nabla f_{ij}(x^{t}) - b \left(h^t_{ij} - \nabla f_{ij}(x^{t})\right)\right) - \left(\nabla f_{ij}(x^{t+1}) + (1 - b)(h^t_{ij} - \nabla f_{ij}(x^{t}))\right)}^2} \\
    &\quad + \left(1 - \frac{\probavailable B}{m}\right) \norm{h^{t}_{ij} - \left(\nabla f_{ij}(x^{t+1}) + (1 - b)(h^t_{ij} - \nabla f_{ij}(x^{t}))\right)}^2 \\
    &\quad +(1 - b)^2\norm{h^{t}_{ij} - \nabla f_{ij}(x^{t})}^2 \\
    &=\frac{\left(1 - \frac{\probavailable B}{m}\right)^2}{\frac{\probavailable B}{m}}\norm{\nabla f_{ij}(x^{t+1}) - \nabla f_{ij}(x^{t}) - b(h^t_{ij} - \nabla f_{ij}(x^{t}))}^2 \\
    &\quad + \left(1 - \frac{\probavailable B}{m}\right) \norm{\nabla f_{ij}(x^{t+1}) - \nabla f_{ij}(x^{t}) - b(h^t_{ij} - \nabla f_{ij}(x^{t}))}^2 \\
    &\quad +(1 - b)^2\norm{h^{t}_{ij} - \nabla f_{ij}(x^{t})}^2 \\
    &= \frac{\left(1 - \frac{\probavailable B}{m}\right)}{\frac{\probavailable B}{m}} \norm{\nabla f_{ij}(x^{t+1}) - \nabla f_{ij}(x^{t}) - b(h^t_{ij} - \nabla f_{ij}(x^{t}))}^2 \\
    &\quad +(1 - b)^2\norm{h^{t}_{ij} - \nabla f_{ij}(x^{t})}^2 \\
    &\overset{\eqref{auxiliary:jensen_inequality}}{\leq} \frac{2\left(1 - \frac{\probavailable B}{m}\right)}{\frac{\probavailable B}{m}} \norm{\nabla f_{ij}(x^{t+1}) - \nabla f_{ij}(x^{t})}^2 + \left(\frac{2\left(1 - \frac{\probavailable B}{m}\right) b^2}{\frac{\probavailable B}{m}} + (1 - b)^2\right) \norm{h^t_{ij} - \nabla f_{ij}(x^{t})}^2. \\
  \end{align*}
  It is left to consider Assumption \ref{ass:max_lipschitz_constant}:
  \begin{align*}
    &\ExpSub{B}{\ExpSub{\probavailable}{\norm{h^{t+1}_{ij} - \nabla f_{ij}(x^{t+1})}^2}} \\
    &\leq \frac{2\left(1 - \frac{\probavailable B}{m}\right) L_{\max}^2}{\frac{\probavailable B}{m}} \norm{x^{t+1} - x^{t}}^2 + \left(\frac{2\left(1 - \frac{\probavailable B}{m}\right) b^2}{\frac{\probavailable B}{m}} + (1 - b)^2\right) \norm{h^t_{ij} - \nabla f_{ij}(x^{t})}^2. \\
  \end{align*}
  Finally, we obtain the bound for the last inequality of the lemma:
  \begin{align*}
    &\ExpSub{B}{\norm{k^{t+1}_i}^2}\\
    &\overset{\eqref{auxiliary:variance_decomposition}}{=}\ExpSub{B}{\norm{k^{t+1}_i - \ExpSub{B}{k^{t+1}_i}}^2}\\
    &\quad + \norm{\nabla f_i(x^{t+1}) - \nabla f_i(x^{t}) - b(h^t_i - \nabla f_i(x^{t}))}^2.
  \end{align*}
  Using Lemma~\ref{lemma:sampling}, we get
  \begin{align*}
    &\ExpSub{B}{\norm{k^{t+1}_i}^2}\\
    &\leq\frac{m - B}{B m (m - 1)}\sum_{j=1}^m\norm{\nabla f_{ij}(x^{t+1}) - \nabla f_{ij}(x^{t}) - b \left(h^t_{ij} - \nabla f_{ij}(x^{t})\right)}^2\\
    &\quad + \norm{\nabla f_i(x^{t+1}) - \nabla f_i(x^{t}) - b(h^t_i - \nabla f_i(x^{t}))}^2 \\
    &\leq\frac{1}{B m}\sum_{j=1}^m\norm{\nabla f_{ij}(x^{t+1}) - \nabla f_{ij}(x^{t}) - b \left(h^t_{ij} - \nabla f_{ij}(x^{t})\right)}^2\\
    &\quad + \norm{\nabla f_i(x^{t+1}) - \nabla f_i(x^{t}) - b(h^t_i - \nabla f_i(x^{t}))}^2 \\
    &\overset{\eqref{auxiliary:jensen_inequality}}{\leq}\frac{2}{B m}\sum_{j=1}^m\norm{\nabla f_{ij}(x^{t+1}) - \nabla f_{ij}(x^{t})}^2 + 2 \norm{\nabla f_i(x^{t+1}) - \nabla f_i(x^{t})}^2\\
    &\quad + \frac{2 b^2}{B m}\sum_{j=1}^m\norm{h^t_{ij} - \nabla f_{ij}(x^{t})}^2 + 2 b^2 \norm{h^t_i - \nabla f_i(x^{t})}^2 \\
    &\leq\left(\frac{2 L_{\max}^2}{B} + 2 L_i^2\right)\norm{x^{t+1} - x^{t}}^2 \\
    &\quad + \frac{2 b^2}{B m}\sum_{j=1}^m\norm{h^t_{ij} - \nabla f_{ij}(x^{t})}^2 + 2 b^2 \norm{h^t_i - \nabla f_i(x^{t})}^2,
  \end{align*}
  where we used Assumptions \ref{ass:nodes_lipschitz_constant} and \ref{ass:max_lipschitz_constant}.
\end{proof}

\begin{restatable}{theorem}{CONVERGENCEFINITEMVR}
  \label{theorem:finitemvr}
  Suppose that Assumptions \ref{ass:lower_bound}, \ref{ass:lipschitz_constant}, \ref{ass:nodes_lipschitz_constant}, \ref{ass:max_lipschitz_constant}, \ref{ass:compressors}, and \ref{ass:partial_participation} hold. Let us take $a = \frac{\probavailable}{2 \omega + 1} ,$ $b = \frac{\frac{\probavailable B}{m}}{2 - \frac{\probavailable B}{m}},$
  $$\gamma \leq \left(L + \sqrt{\frac{148 \omega (2 \omega + 1)}{n \probavailable^2} \left(\widehat{L}^2 + \frac{L_{\max}^2}{B}\right) + \frac{72 m}{n \probavailable^2 B}\left(\left(1 - \frac{\probpairaa}{\probavailable}\right)\widehat{L}^2 + \frac{L_{\max}^2}{B}\right)}\right)^{-1},$$
  $g^{0}_i = h^{0}_i = \nabla f_i(x^0)$ for all $i \in [n]$ and $h^{0}_{ij} = \nabla f_{ij}(x^0)$ for all $i \in [n], j \in [m]$ in Algorithm~\ref{alg:main_algorithm} \algname{(\algorithmname-FINITE-MVR)}
  then $\Exp{\norm{\nabla f(\widehat{x}^T)}^2} \leq \frac{2 \Delta_0}{\gamma T}.$
\end{restatable}

\begin{proof}
  Let us fix constants $\nu, \rho, \delta \in [0,\infty)$ that we will define later. Considering Lemma~\ref{lemma:main_lemma}, Lemma~\ref{lemma:finite_mvr}, and the law of total expectation, we obtain
    \begin{align*}
      &\Exp{f(x^{t + 1})} + \frac{\gamma (2 \omega + 1)}{\probavailable} \Exp{\norm{g^{t+1} - h^{t+1}}^2} + \frac{\gamma (\left(2 \omega + 1\right)\probavailable - \probpairaa)}{n \probavailable^2} \Exp{\frac{1}{n}\sum_{i=1}^n\norm{g^{t+1}_i - h^{t+1}_i}^2}\\
      &\quad  + \nu \Exp{\norm{h^{t+1} - \nabla f(x^{t+1})}^2} + \rho \Exp{\frac{1}{n}\sum_{i=1}^n\norm{h^{t+1}_i - \nabla f_i(x^{t+1})}^2}\\
      &\quad + \delta \Exp{\frac{1}{nm}\sum_{i=1}^n\sum_{j=1}^m\norm{h^{t+1}_{ij} - \nabla f_{ij}(x^{t+1})}^2} \\
      &\leq \Exp{f(x^t) - \frac{\gamma}{2}\norm{\nabla f(x^t)}^2 - \left(\frac{1}{2\gamma} - \frac{L}{2}\right)
      \norm{x^{t+1} - x^t}^2 + \gamma \norm{h^{t} - \nabla f(x^t)}^2}\nonumber\\
      &\quad + \frac{\gamma (2 \omega + 1)}{\probavailable}\Exp{\norm{g^{t} - h^t}^2}+ \frac{\gamma (\left(2 \omega + 1\right)\probavailable - \probpairaa)}{n \probavailable^2}\Exp{\frac{1}{n} \sum_{i=1}^n\norm{g^t_i - h^{t}_i}^2} \\
      &\quad + \frac{4 \gamma \omega (2 \omega + 1)}{n \probavailable^2} \Exp{\frac{1}{n} \sum_{i=1}^n\norm{k^{t+1}_i}^2} \\
      &\quad  + \nu \Exp{\norm{h^{t+1} - \nabla f(x^{t+1})}^2} + \rho \Exp{\frac{1}{n}\sum_{i=1}^n\norm{h^{t+1}_i - \nabla f_i(x^{t+1})}^2}\\
      &\quad + \delta \Exp{\frac{1}{nm}\sum_{i=1}^n\sum_{j=1}^m\norm{h^{t+1}_{ij} - \nabla f_{ij}(x^{t+1})}^2} \\
      &= \Exp{f(x^t) - \frac{\gamma}{2}\norm{\nabla f(x^t)}^2 - \left(\frac{1}{2\gamma} - \frac{L}{2}\right)
      \norm{x^{t+1} - x^t}^2 + \gamma \norm{h^{t} - \nabla f(x^t)}^2}\nonumber\\
      &\quad + \frac{\gamma (2 \omega + 1)}{\probavailable}\Exp{\norm{g^{t} - h^t}^2}+ \frac{\gamma (\left(2 \omega + 1\right)\probavailable - \probpairaa)}{n \probavailable^2}\Exp{\frac{1}{n} \sum_{i=1}^n\norm{g^t_i - h^{t}_i}^2} \\
      &\quad + \frac{4 \gamma \omega (2 \omega + 1)}{n \probavailable^2} \Exp{\ExpSub{B}{\frac{1}{n} \sum_{i=1}^n\norm{k^{t+1}_i}^2}} \\
      &\quad  + \nu \Exp{\ExpSub{B}{\ExpSub{\probavailable}{\norm{h^{t+1} - \nabla f(x^{t+1})}^2}}} \\
      &\quad + \rho \Exp{\ExpSub{B}{\ExpSub{\probavailable}{\frac{1}{n}\sum_{i=1}^n\norm{h^{t+1}_i - \nabla f_i(x^{t+1})}^2}}}\\
      &\quad + \delta \Exp{\ExpSub{B}{\ExpSub{\probavailable}{\frac{1}{nm}\sum_{i=1}^n\sum_{j=1}^m\norm{h^{t+1}_{ij} - \nabla f_{ij}(x^{t+1})}^2}}} \\
      &\leq \Exp{f(x^t) - \frac{\gamma}{2}\norm{\nabla f(x^t)}^2 - \left(\frac{1}{2\gamma} - \frac{L}{2}\right)
      \norm{x^{t+1} - x^t}^2 + \gamma \norm{h^{t} - \nabla f(x^t)}^2}\nonumber\\
      &\quad + \frac{\gamma (2 \omega + 1)}{\probavailable}\Exp{\norm{g^{t} - h^t}^2}+ \frac{\gamma (\left(2 \omega + 1\right)\probavailable - \probpairaa)}{n \probavailable^2}\Exp{\frac{1}{n} \sum_{i=1}^n\norm{g^t_i - h^{t}_i}^2} \\
      &\quad + \frac{4 \gamma \omega (2 \omega + 1)}{n \probavailable^2} \Exp{\left(\frac{2 L_{\max}^2}{B} + 2 \widehat{L}^2\right)\norm{x^{t+1} - x^{t}}^2 + \frac{2 b^2}{B m n}\sum_{i=1}^n \sum_{j=1}^m\norm{h^t_{ij} - \nabla f_{ij}(x^{t})}^2 + \frac{2 b^2}{n} \sum_{i=1}^n \norm{h^t_i - \nabla f_i(x^{t})}^2} \\
      &\quad  + \nu {\rm E}\Bigg(\left(\frac{2 L_{\max}^2}{n \probavailable B} + \frac{2 \left(\probavailable - \probpairaa\right) \widehat{L}^2}{n \probavailable^2}\right)\norm{x^{t+1} - x^{t}}^2  \\
      &\qquad\quad + \frac{2\left(\probavailable - \probpairaa\right)b^2}{n^2 \probavailable^2}\sum_{i=1}^n\norm{h^t_{i} - \nabla f_{i}(x^{t})}^2 + \frac{2b^2}{n^2 \probavailable B m}\sum_{i=1}^n\sum_{j=1}^m\norm{h^t_{ij} - \nabla f_{ij}(x^{t})}^2\\
      &\qquad\quad + (1 - b)^2\norm{h^{t} - \nabla f(x^{t})}^2\Bigg) \\
      &\quad + \rho {\rm E}\Bigg(\left(\frac{2 L_{\max}^2}{\probavailable B} +\frac{2(1 - \probavailable) \widehat{L}^2}{\probavailable}\right) \norm{x^{t+1} - x^{t}}^2 \\
      &\qquad\quad + \frac{2 b^2}{\probavailable B n m} \sum_{i=1}^n \sum_{j=1}^m\norm{h^t_{ij} - \nabla f_{ij}(x^{t})}^2 +\left(\frac{2 \left(1 - \probavailable\right) b^2}{\probavailable} + (1 - b)^2\right) \frac{1}{n} \sum_{i=1}^n \norm{h^{t}_i - \nabla f_i(x^{t})}^2\Bigg)\\
      &\quad + \delta {\rm E}\Bigg(\frac{2\left(1 - \frac{\probavailable B}{m}\right) L_{\max}^2}{\frac{\probavailable B}{m}} \norm{x^{t+1} - x^{t}}^2 \\
      &\qquad\quad + \left(\frac{2\left(1 - \frac{\probavailable B}{m}\right) b^2}{\frac{\probavailable B}{m}} + (1 - b)^2\right) \frac{1}{nm} \sum_{i=1}^n \sum_{j=1}^m \norm{h^t_{ij} - \nabla f_{ij}(x^{t})}^2\Bigg).\\
    \end{align*}
    Due to $b = \frac{\frac{\probavailable B}{m}}{2 - \frac{\probavailable B}{m}} \leq \frac{\probavailable}{2 - \probavailable},$ we have 
    $$\left(\frac{2\left(1 - \frac{\probavailable B}{m}\right) b^2}{\frac{\probavailable B}{m}} + (1 - b)^2\right) \leq 1 - b$$
    and
    $$\left(\frac{2 \left(1 - \probavailable\right) b^2}{\probavailable} + (1 - b)^2\right) \leq 1 - b.$$
    Moreover, we consider that $1 - \frac{\probavailable B}{m} \leq 1,$ therefore
    \begin{align*}
      &\Exp{f(x^{t + 1})} + \frac{\gamma (2 \omega + 1)}{\probavailable} \Exp{\norm{g^{t+1} - h^{t+1}}^2} + \frac{\gamma (\left(2 \omega + 1\right)\probavailable - \probpairaa)}{n \probavailable^2} \Exp{\frac{1}{n}\sum_{i=1}^n\norm{g^{t+1}_i - h^{t+1}_i}^2}\\
      &\quad  + \nu \Exp{\norm{h^{t+1} - \nabla f(x^{t+1})}^2} + \rho \Exp{\frac{1}{n}\sum_{i=1}^n\norm{h^{t+1}_i - \nabla f_i(x^{t+1})}^2}\\
      &\quad + \delta \Exp{\frac{1}{nm}\sum_{i=1}^n\sum_{j=1}^m\norm{h^{t+1}_{ij} - \nabla f_{ij}(x^{t+1})}^2} \\
      &\leq \Exp{f(x^t) - \frac{\gamma}{2}\norm{\nabla f(x^t)}^2 - \left(\frac{1}{2\gamma} - \frac{L}{2}\right)
      \norm{x^{t+1} - x^t}^2 + \gamma \norm{h^{t} - \nabla f(x^t)}^2}\nonumber\\
      &\quad + \frac{\gamma (2 \omega + 1)}{\probavailable}\Exp{\norm{g^{t} - h^t}^2}+ \frac{\gamma (\left(2 \omega + 1\right)\probavailable - \probpairaa)}{n \probavailable^2}\Exp{\frac{1}{n} \sum_{i=1}^n\norm{g^t_i - h^{t}_i}^2} \\
      &\quad + \frac{4 \gamma \omega (2 \omega + 1)}{n \probavailable^2} \Exp{\left(\frac{2 L_{\max}^2}{B} + 2 \widehat{L}^2\right)\norm{x^{t+1} - x^{t}}^2 + \frac{2 b^2}{B m n}\sum_{i=1}^n \sum_{j=1}^m\norm{h^t_{ij} - \nabla f_{ij}(x^{t})}^2 + \frac{2 b^2}{n} \sum_{i=1}^n \norm{h^t_i - \nabla f_i(x^{t})}^2} \\
      &\quad  + \nu {\rm E}\Bigg(\left(\frac{2 L_{\max}^2}{n \probavailable B} + \frac{2 \left(\probavailable - \probpairaa\right) \widehat{L}^2}{n \probavailable^2}\right)\norm{x^{t+1} - x^{t}}^2  \\
      &\qquad\quad + \frac{2\left(\probavailable - \probpairaa\right)b^2}{n^2 \probavailable^2}\sum_{i=1}^n\norm{h^t_{i} - \nabla f_{i}(x^{t})}^2 + \frac{2b^2}{n^2 \probavailable B m}\sum_{i=1}^n\sum_{j=1}^m\norm{h^t_{ij} - \nabla f_{ij}(x^{t})}^2 \\
      &\qquad\quad + (1 - b)^2\norm{h^{t} - \nabla f(x^{t})}^2\Bigg) \\
      &\quad + \rho {\rm E}\Bigg(\left(\frac{2 L_{\max}^2}{\probavailable B} +\frac{2(1 - \probavailable) \widehat{L}^2}{\probavailable}\right) \norm{x^{t+1} - x^{t}}^2 \\
      &\qquad\quad + \frac{2 b^2}{\probavailable B n m} \sum_{i=1}^n \sum_{j=1}^m\norm{h^t_{ij} - \nabla f_{ij}(x^{t})}^2 +\left(1 - b\right) \frac{1}{n} \sum_{i=1}^n \norm{h^{t}_i - \nabla f_i(x^{t})}^2\Bigg)\\
      &\quad + \delta {\rm E}\Bigg(\frac{2 m L_{\max}^2}{\probavailable B} \norm{x^{t+1} - x^{t}}^2 + \left(1 - b\right) \frac{1}{nm} \sum_{i=1}^n \sum_{j=1}^m \norm{h^t_{ij} - \nabla f_{ij}(x^{t})}^2\Bigg).
    \end{align*}
    After rearranging the terms, we get
    \begin{align*}
      &\Exp{f(x^{t + 1})} + \frac{\gamma (2 \omega + 1)}{\probavailable} \Exp{\norm{g^{t+1} - h^{t+1}}^2} + \frac{\gamma (\left(2 \omega + 1\right)\probavailable - \probpairaa)}{n \probavailable^2} \Exp{\frac{1}{n}\sum_{i=1}^n\norm{g^{t+1}_i - h^{t+1}_i}^2}\\
      &\quad  + \nu \Exp{\norm{h^{t+1} - \nabla f(x^{t+1})}^2} + \rho \Exp{\frac{1}{n}\sum_{i=1}^n\norm{h^{t+1}_i - \nabla f_i(x^{t+1})}^2}\\
      &\quad + \delta \Exp{\frac{1}{nm}\sum_{i=1}^n\sum_{j=1}^m\norm{h^{t+1}_{ij} - \nabla f_{ij}(x^{t+1})}^2} \\
      &\leq \Exp{f(x^t)} - \frac{\gamma}{2}\Exp{\norm{\nabla f(x^t)}^2} \\
      &\quad + \frac{\gamma (2 \omega + 1)}{\probavailable} \Exp{\norm{g^{t} - h^{t}}^2} + \frac{\gamma (\left(2 \omega + 1\right)\probavailable - \probpairaa)}{n \probavailable^2} \Exp{\frac{1}{n}\sum_{i=1}^n\norm{g^{t}_i - h^{t}_i}^2} \\
      &\quad - \Bigg(\frac{1}{2\gamma} - \frac{L}{2} - \frac{4 \gamma \omega (2 \omega + 1)}{n \probavailable^2} \left(\frac{2 L_{\max}^2}{B} + 2 \widehat{L}^2\right) \\
      &\qquad\quad - \nu\left(\frac{2 L_{\max}^2}{n \probavailable B} + \frac{2 \left(\probavailable - \probpairaa\right) \widehat{L}^2}{n \probavailable^2}\right) - \rho \left(\frac{2 L_{\max}^2}{\probavailable B} +\frac{2(1 - \probavailable) \widehat{L}^2}{\probavailable}\right) - \delta \frac{2 m L_{\max}^2}{\probavailable B} \Bigg) \Exp{\norm{x^{t+1} - x^t}^2} \\
      &\quad + \left(\gamma + \nu \left(1 - b\right)^2\right) \Exp{\norm{h^{t} - \nabla f(x^{t})}^2} \\
      &\quad + \Bigg(\frac{8 b^2 \gamma \omega (2 \omega + 1)}{n \probavailable^2} + \frac{2 \nu \left(\probavailable - \probpairaa\right) b^2}{n \probavailable^2} + \rho \left(1 - b\right)\Bigg)\Exp{\frac{1}{n}\sum_{i=1}^n\norm{h^{t}_i - \nabla f_i(x^{t})}^2} \\
      &\quad + \Bigg(\frac{8 b^2 \gamma \omega (2 \omega + 1)}{n \probavailable^2 B} + \frac{2 \nu b^2}{n \probavailable B} + \frac{2 \rho b^2}{\probavailable B} + \delta \left(1 - b\right)\Bigg)\Exp{\frac{1}{nm}\sum_{i=1}^n\sum_{j=1}^m\norm{h^{t}_{ij} - \nabla f_{ij}(x^{t})}^2}.
    \end{align*}
    Thus, if we take $\nu = \frac{\gamma}{b},$ then $\gamma + \nu \left(1 - b\right)^2 \leq \nu$ and
    \begin{align*}
      &\Exp{f(x^{t + 1})} + \frac{\gamma (2 \omega + 1)}{\probavailable} \Exp{\norm{g^{t+1} - h^{t+1}}^2} + \frac{\gamma (\left(2 \omega + 1\right)\probavailable - \probpairaa)}{n \probavailable^2} \Exp{\frac{1}{n}\sum_{i=1}^n\norm{g^{t+1}_i - h^{t+1}_i}^2}\\
      &\quad  + \frac{\gamma}{b} \Exp{\norm{h^{t+1} - \nabla f(x^{t+1})}^2} + \rho \Exp{\frac{1}{n}\sum_{i=1}^n\norm{h^{t+1}_i - \nabla f_i(x^{t+1})}^2}\\
      &\quad + \delta \Exp{\frac{1}{nm}\sum_{i=1}^n\sum_{j=1}^m\norm{h^{t+1}_{ij} - \nabla f_{ij}(x^{t+1})}^2} \\
      &\leq \Exp{f(x^t)} - \frac{\gamma}{2}\Exp{\norm{\nabla f(x^t)}^2} \\
      &\quad + \frac{\gamma (2 \omega + 1)}{\probavailable} \Exp{\norm{g^{t} - h^{t}}^2} + \frac{\gamma (\left(2 \omega + 1\right)\probavailable - \probpairaa)}{n \probavailable^2} \Exp{\frac{1}{n}\sum_{i=1}^n\norm{g^{t}_i - h^{t}_i}^2} \\
      &\quad - \Bigg(\frac{1}{2\gamma} - \frac{L}{2} - \frac{4 \gamma \omega (2 \omega + 1)}{n \probavailable^2} \left(\frac{2 L_{\max}^2}{B} + 2 \widehat{L}^2\right) \\
      &\qquad\quad - \left(\frac{2 \gamma L_{\max}^2}{b n \probavailable B} + \frac{2 \gamma \left(\probavailable - \probpairaa\right) \widehat{L}^2}{b n \probavailable^2}\right) - \rho \left(\frac{2 L_{\max}^2}{\probavailable B} +\frac{2(1 - \probavailable) \widehat{L}^2}{\probavailable}\right) - \delta \frac{2 m L_{\max}^2}{\probavailable B} \Bigg) \Exp{\norm{x^{t+1} - x^t}^2} \\
      &\quad + \frac{\gamma}{b} \Exp{\norm{h^{t} - \nabla f(x^{t})}^2} \\
      &\quad + \Bigg(\frac{8 b^2 \gamma \omega (2 \omega + 1)}{n \probavailable^2} + \frac{2 \gamma \left(\probavailable - \probpairaa\right) b}{n \probavailable^2} + \rho \left(1 - b\right)\Bigg)\Exp{\frac{1}{n}\sum_{i=1}^n\norm{h^{t}_i - \nabla f_i(x^{t})}^2} \\
      &\quad + \Bigg(\frac{8 b^2 \gamma \omega (2 \omega + 1)}{n \probavailable^2 B} + \frac{2 \gamma b}{n \probavailable B} + \frac{2 \rho b^2}{\probavailable B} + \delta \left(1 - b\right)\Bigg)\Exp{\frac{1}{nm}\sum_{i=1}^n\sum_{j=1}^m\norm{h^{t}_{ij} - \nabla f_{ij}(x^{t})}^2}.
    \end{align*}
    Next, if we take $\rho = \frac{8 b \gamma \omega (2 \omega + 1)}{n \probavailable^2} + \frac{2 \gamma \left(\probavailable - \probpairaa\right)}{n \probavailable^2},$ then
    $$\Bigg(\frac{8 b^2 \gamma \omega (2 \omega + 1)}{n \probavailable^2} + \frac{2 \gamma \left(\probavailable - \probpairaa\right) b}{n \probavailable^2} + \rho \left(1 - b\right)\Bigg) = \rho,$$ therefore
    \begin{align*}
      &\Exp{f(x^{t + 1})} + \frac{\gamma (2 \omega + 1)}{\probavailable} \Exp{\norm{g^{t+1} - h^{t+1}}^2} + \frac{\gamma (\left(2 \omega + 1\right)\probavailable - \probpairaa)}{n \probavailable^2} \Exp{\frac{1}{n}\sum_{i=1}^n\norm{g^{t+1}_i - h^{t+1}_i}^2}\\
      &\quad  + \frac{\gamma}{b} \Exp{\norm{h^{t+1} - \nabla f(x^{t+1})}^2} + \left(\frac{8 b \gamma \omega (2 \omega + 1)}{n \probavailable^2} + \frac{2 \gamma \left(\probavailable - \probpairaa\right)}{n \probavailable^2}\right) \Exp{\frac{1}{n}\sum_{i=1}^n\norm{h^{t+1}_i - \nabla f_i(x^{t+1})}^2}\\
      &\quad + \delta \Exp{\frac{1}{nm}\sum_{i=1}^n\sum_{j=1}^m\norm{h^{t+1}_{ij} - \nabla f_{ij}(x^{t+1})}^2} \\
      &\leq \Exp{f(x^t)} - \frac{\gamma}{2}\Exp{\norm{\nabla f(x^t)}^2} \\
      &\quad + \frac{\gamma (2 \omega + 1)}{\probavailable} \Exp{\norm{g^{t} - h^{t}}^2} + \frac{\gamma (\left(2 \omega + 1\right)\probavailable - \probpairaa)}{n \probavailable^2} \Exp{\frac{1}{n}\sum_{i=1}^n\norm{g^{t}_i - h^{t}_i}^2} \\
      &\quad - \Bigg(\frac{1}{2\gamma} - \frac{L}{2} - \frac{4 \gamma \omega (2 \omega + 1)}{n \probavailable^2} \left(\frac{2 L_{\max}^2}{B} + 2 \widehat{L}^2\right) \\
      &\qquad\quad - \left(\frac{2 \gamma L_{\max}^2}{b n \probavailable B} + \frac{2 \gamma \left(\probavailable - \probpairaa\right) \widehat{L}^2}{b n \probavailable^2}\right) - \left(\frac{8 b \gamma \omega (2 \omega + 1)}{n \probavailable^2} + \frac{2 \gamma \left(\probavailable - \probpairaa\right)}{n \probavailable^2}\right) \left(\frac{2 L_{\max}^2}{\probavailable B} +\frac{2(1 - \probavailable) \widehat{L}^2}{\probavailable}\right) \\
      &\qquad\quad - \delta \frac{2 m L_{\max}^2}{\probavailable B} \Bigg) \Exp{\norm{x^{t+1} - x^t}^2} \\
      &\quad + \frac{\gamma}{b} \Exp{\norm{h^{t} - \nabla f(x^{t})}^2} \\
      &\quad + \left(\frac{8 b \gamma \omega (2 \omega + 1)}{n \probavailable^2} + \frac{2 \gamma \left(\probavailable - \probpairaa\right)}{n \probavailable^2}\right)\Exp{\frac{1}{n}\sum_{i=1}^n\norm{h^{t}_i - \nabla f_i(x^{t})}^2} \\
      &\quad + \Bigg(\frac{8 b^2 \gamma \omega (2 \omega + 1)}{n \probavailable^2 B} + \frac{2 \gamma b}{n \probavailable B} + \frac{16 b^3 \gamma \omega (2 \omega + 1)}{n \probavailable^3 B} + \frac{4 b^2 \gamma \left(\probavailable - \probpairaa\right)}{n B \probavailable^3} + \delta \left(1 - b\right)\Bigg)\Exp{\frac{1}{nm}\sum_{i=1}^n\sum_{j=1}^m\norm{h^{t}_{ij} - \nabla f_{ij}(x^{t})}^2}.
    \end{align*}
    Due to $b \leq \probavailable$ and $\frac{\probavailable - \probpairaa}{\probavailable} \leq 1,$ we have
    \begin{align*}
      &\frac{8 b^2 \gamma \omega (2 \omega + 1)}{n \probavailable^2 B} + \frac{2 \gamma b}{n \probavailable B} + \frac{16 b^3 \gamma \omega (2 \omega + 1)}{n \probavailable^3 B} + \frac{4 b^2 \gamma \left(\probavailable - \probpairaa\right)}{n B \probavailable^3} \\
      &\leq \frac{8 b^2 \gamma \omega (2 \omega + 1)}{n \probavailable^2 B} + \frac{2 \gamma b}{n \probavailable B} + \frac{16 b^2 \gamma \omega (2 \omega + 1)}{n \probavailable^2 B} + \frac{4 \gamma b}{n \probavailable B} \\
      &=\frac{24 b^2 \gamma \omega (2 \omega + 1)}{n \probavailable^2 B} + \frac{6 \gamma b}{n \probavailable B}.
    \end{align*}
    Let us take $\delta = \frac{24 b \gamma \omega (2 \omega + 1)}{n \probavailable^2 B} + \frac{6 \gamma }{n \probavailable B}.$ Thus
    \begin{align*}
      \Bigg(\frac{8 b^2 \gamma \omega (2 \omega + 1)}{n \probavailable^2 B} + \frac{2 \gamma b}{n \probavailable B} + \frac{16 b^3 \gamma \omega (2 \omega + 1)}{n \probavailable^3 B} + \frac{4 b^2 \gamma \left(\probavailable - \probpairaa\right)}{n B \probavailable^3} + \delta \left(1 - b\right)\Bigg) \leq \delta
    \end{align*}
    and
    \begin{align*}
      &\Exp{f(x^{t + 1})} + \frac{\gamma (2 \omega + 1)}{\probavailable} \Exp{\norm{g^{t+1} - h^{t+1}}^2} + \frac{\gamma (\left(2 \omega + 1\right)\probavailable - \probpairaa)}{n \probavailable^2} \Exp{\frac{1}{n}\sum_{i=1}^n\norm{g^{t+1}_i - h^{t+1}_i}^2}\\
      &\quad  + \frac{\gamma}{b} \Exp{\norm{h^{t+1} - \nabla f(x^{t+1})}^2} + \left(\frac{8 b \gamma \omega (2 \omega + 1)}{n \probavailable^2} + \frac{2 \gamma \left(\probavailable - \probpairaa\right)}{n \probavailable^2}\right) \Exp{\frac{1}{n}\sum_{i=1}^n\norm{h^{t+1}_i - \nabla f_i(x^{t+1})}^2}\\
      &\quad + \left(\frac{24 b \gamma \omega (2 \omega + 1)}{n \probavailable^2 B} + \frac{6 \gamma }{n \probavailable B}\right) \Exp{\frac{1}{nm}\sum_{i=1}^n\sum_{j=1}^m\norm{h^{t+1}_{ij} - \nabla f_{ij}(x^{t+1})}^2} \\
      &\leq \Exp{f(x^t)} - \frac{\gamma}{2}\Exp{\norm{\nabla f(x^t)}^2} \\
      &\quad + \frac{\gamma (2 \omega + 1)}{\probavailable} \Exp{\norm{g^{t} - h^{t}}^2} + \frac{\gamma (\left(2 \omega + 1\right)\probavailable - \probpairaa)}{n \probavailable^2} \Exp{\frac{1}{n}\sum_{i=1}^n\norm{g^{t}_i - h^{t}_i}^2} \\
      &\quad - \Bigg(\frac{1}{2\gamma} - \frac{L}{2} - \frac{4 \gamma \omega (2 \omega + 1)}{n \probavailable^2} \left(\frac{2 L_{\max}^2}{B} + 2 \widehat{L}^2\right) \\
      &\qquad\quad - \left(\frac{2 \gamma L_{\max}^2}{b n \probavailable B} + \frac{2 \gamma \left(\probavailable - \probpairaa\right) \widehat{L}^2}{b n \probavailable^2}\right) - \left(\frac{8 b \gamma \omega (2 \omega + 1)}{n \probavailable^2} + \frac{2 \gamma \left(\probavailable - \probpairaa\right)}{n \probavailable^2}\right) \left(\frac{2 L_{\max}^2}{\probavailable B} +\frac{2(1 - \probavailable) \widehat{L}^2}{\probavailable}\right) \\
      &\qquad\quad - \left(\frac{24 b \gamma \omega (2 \omega + 1)}{n \probavailable^2 B} + \frac{6 \gamma }{n \probavailable B}\right) \frac{2 m L_{\max}^2}{\probavailable B} \Bigg) \Exp{\norm{x^{t+1} - x^t}^2} \\
      &\quad + \frac{\gamma}{b} \Exp{\norm{h^{t} - \nabla f(x^{t})}^2} \\
      &\quad + \left(\frac{8 b \gamma \omega (2 \omega + 1)}{n \probavailable^2} + \frac{2 \gamma \left(\probavailable - \probpairaa\right)}{n \probavailable^2}\right)\Exp{\frac{1}{n}\sum_{i=1}^n\norm{h^{t}_i - \nabla f_i(x^{t})}^2} \\
      &\quad + \left(\frac{24 b \gamma \omega (2 \omega + 1)}{n \probavailable^2 B} + \frac{6 \gamma }{n \probavailable B}\right)\Exp{\frac{1}{nm}\sum_{i=1}^n\sum_{j=1}^m\norm{h^{t}_{ij} - \nabla f_{ij}(x^{t})}^2}.
    \end{align*}

    Let us simplify the term near $\Exp{\norm{x^{t+1} - x^t}^2}.$ Due to $b \leq \probavailable$, $\frac{\probavailable - \probpairaa}{\probavailable} \leq 1,$ and $1 - \probavailable \leq 1,$ 
    we have 
    \begin{align*}
      &\frac{4 \gamma \omega (2 \omega + 1)}{n \probavailable^2} \left(\frac{2 L_{\max}^2}{B} + 2 \widehat{L}^2\right) \\
      &\quad + \left(\frac{2 \gamma L_{\max}^2}{b n \probavailable B} + \frac{2 \gamma \left(\probavailable - \probpairaa\right) \widehat{L}^2}{b n \probavailable^2}\right) \\
      &\quad + \left(\frac{8 b \gamma \omega (2 \omega + 1)}{n \probavailable^2} + \frac{2 \gamma \left(\probavailable - \probpairaa\right)}{n \probavailable^2}\right) \left(\frac{2 L_{\max}^2}{\probavailable B} +\frac{2(1 - \probavailable) \widehat{L}^2}{\probavailable}\right) \\
      &\quad + \left(\frac{24 b \gamma \omega (2 \omega + 1)}{n \probavailable^2 B} + \frac{6 \gamma }{n \probavailable B}\right) \frac{2 m L_{\max}^2}{\probavailable B} \\
      &\leq \frac{12 \gamma \omega (2 \omega + 1)}{n \probavailable^2} \left(\frac{2 L_{\max}^2}{B} + 2 \widehat{L}^2\right) \\
      &\quad + \left(\frac{6 \gamma L_{\max}^2}{b n \probavailable B} + \frac{6 \gamma \left(\probavailable - \probpairaa\right) \widehat{L}^2}{b n \probavailable^2}\right) \\
      &\quad + \left(\frac{24 b \gamma \omega (2 \omega + 1)}{n \probavailable^2 B} + \frac{6 \gamma }{n \probavailable B}\right) \frac{2 m L_{\max}^2}{\probavailable B}
    \end{align*}
    Considering that $b \leq \frac{\probavailable B}{m}$ and $b \geq \frac{\probavailable B}{2 m},$ we obtain
    \begin{align*}
      &\frac{4 \gamma \omega (2 \omega + 1)}{n \probavailable^2} \left(\frac{2 L_{\max}^2}{B} + 2 \widehat{L}^2\right) \\
      &\quad + \left(\frac{2 \gamma L_{\max}^2}{b n \probavailable B} + \frac{2 \gamma \left(\probavailable - \probpairaa\right) \widehat{L}^2}{b n \probavailable^2}\right) \\
      &\quad + \left(\frac{8 b \gamma \omega (2 \omega + 1)}{n \probavailable^2} + \frac{2 \gamma \left(\probavailable - \probpairaa\right)}{n \probavailable^2}\right) \left(\frac{2 L_{\max}^2}{\probavailable B} +\frac{2(1 - \probavailable) \widehat{L}^2}{\probavailable}\right) \\
      &\quad + \left(\frac{24 b \gamma \omega (2 \omega + 1)}{n \probavailable^2 B} + \frac{6 \gamma }{n \probavailable B}\right) \frac{2 m L_{\max}^2}{\probavailable B} \\
      &\leq \frac{36 \gamma \omega (2 \omega + 1)}{n \probavailable^2} \left(\frac{2 L_{\max}^2}{B} + 2 \widehat{L}^2\right) + \left(\frac{18 \gamma L_{\max}^2}{b n \probavailable B} + \frac{6 \gamma \left(\probavailable - \probpairaa\right) \widehat{L}^2}{b n \probavailable^2}\right) \\
      &\leq \frac{36 \gamma \omega (2 \omega + 1)}{n \probavailable^2} \left(\frac{2 L_{\max}^2}{B} + 2 \widehat{L}^2\right) + \left(\frac{36 m \gamma L_{\max}^2}{n \probavailable^2 B^2} + \frac{12 m \gamma \left(\probavailable - \probpairaa\right) \widehat{L}^2}{B n \probavailable^3}\right).
    \end{align*}
    All in all, we have
    \begin{align*}
      &\Exp{f(x^{t + 1})} + \frac{\gamma (2 \omega + 1)}{\probavailable} \Exp{\norm{g^{t+1} - h^{t+1}}^2} + \frac{\gamma (\left(2 \omega + 1\right)\probavailable - \probpairaa)}{n \probavailable^2} \Exp{\frac{1}{n}\sum_{i=1}^n\norm{g^{t+1}_i - h^{t+1}_i}^2}\\
      &\quad  + \frac{\gamma}{b} \Exp{\norm{h^{t+1} - \nabla f(x^{t+1})}^2} + \left(\frac{8 b \gamma \omega (2 \omega + 1)}{n \probavailable^2} + \frac{2 \gamma \left(\probavailable - \probpairaa\right)}{n \probavailable^2}\right) \Exp{\frac{1}{n}\sum_{i=1}^n\norm{h^{t+1}_i - \nabla f_i(x^{t+1})}^2}\\
      &\quad + \left(\frac{24 b \gamma \omega (2 \omega + 1)}{n \probavailable^2 B} + \frac{6 \gamma }{n \probavailable B}\right) \Exp{\frac{1}{nm}\sum_{i=1}^n\sum_{j=1}^m\norm{h^{t+1}_{ij} - \nabla f_{ij}(x^{t+1})}^2} \\
      &\leq \Exp{f(x^t)} - \frac{\gamma}{2}\Exp{\norm{\nabla f(x^t)}^2} \\
      &\quad + \frac{\gamma (2 \omega + 1)}{\probavailable} \Exp{\norm{g^{t} - h^{t}}^2} + \frac{\gamma (\left(2 \omega + 1\right)\probavailable - \probpairaa)}{n \probavailable^2} \Exp{\frac{1}{n}\sum_{i=1}^n\norm{g^{t}_i - h^{t}_i}^2} \\
      &\quad - \Bigg(\frac{1}{2\gamma} - \frac{L}{2} - \frac{36 \gamma \omega (2 \omega + 1)}{n \probavailable^2} \left(\frac{2 L_{\max}^2}{B} + 2 \widehat{L}^2\right) - \left(\frac{36 m \gamma L_{\max}^2}{n \probavailable^2 B^2} + \frac{12 m \gamma \left(\probavailable - \probpairaa\right) \widehat{L}^2}{B n \probavailable^3}\right) \Bigg) \Exp{\norm{x^{t+1} - x^t}^2} \\
      &\quad + \frac{\gamma}{b} \Exp{\norm{h^{t} - \nabla f(x^{t})}^2} \\
      &\quad + \left(\frac{8 b \gamma \omega (2 \omega + 1)}{n \probavailable^2} + \frac{2 \gamma \left(\probavailable - \probpairaa\right)}{n \probavailable^2}\right)\Exp{\frac{1}{n}\sum_{i=1}^n\norm{h^{t}_i - \nabla f_i(x^{t})}^2} \\
      &\quad + \left(\frac{24 b \gamma \omega (2 \omega + 1)}{n \probavailable^2 B} + \frac{6 \gamma }{n \probavailable B}\right)\Exp{\frac{1}{nm}\sum_{i=1}^n\sum_{j=1}^m\norm{h^{t}_{ij} - \nabla f_{ij}(x^{t})}^2}.
    \end{align*}
    Using Lemma~\ref{lemma:gamma} and the assumption about $\gamma,$ we get
    \begin{align*}
      &\Exp{f(x^{t + 1})} + \frac{\gamma (2 \omega + 1)}{\probavailable} \Exp{\norm{g^{t+1} - h^{t+1}}^2} + \frac{\gamma (\left(2 \omega + 1\right)\probavailable - \probpairaa)}{n \probavailable^2} \Exp{\frac{1}{n}\sum_{i=1}^n\norm{g^{t+1}_i - h^{t+1}_i}^2}\\
      &\quad  + \frac{\gamma}{b} \Exp{\norm{h^{t+1} - \nabla f(x^{t+1})}^2} + \left(\frac{8 b \gamma \omega (2 \omega + 1)}{n \probavailable^2} + \frac{2 \gamma \left(\probavailable - \probpairaa\right)}{n \probavailable^2}\right) \Exp{\frac{1}{n}\sum_{i=1}^n\norm{h^{t+1}_i - \nabla f_i(x^{t+1})}^2}\\
      &\quad + \left(\frac{24 b \gamma \omega (2 \omega + 1)}{n \probavailable^2 B} + \frac{6 \gamma }{n \probavailable B}\right) \Exp{\frac{1}{nm}\sum_{i=1}^n\sum_{j=1}^m\norm{h^{t+1}_{ij} - \nabla f_{ij}(x^{t+1})}^2} \\
      &\leq \Exp{f(x^t)} - \frac{\gamma}{2}\Exp{\norm{\nabla f(x^t)}^2} \\
      &\quad + \frac{\gamma (2 \omega + 1)}{\probavailable} \Exp{\norm{g^{t} - h^{t}}^2} + \frac{\gamma (\left(2 \omega + 1\right)\probavailable - \probpairaa)}{n \probavailable^2} \Exp{\frac{1}{n}\sum_{i=1}^n\norm{g^{t}_i - h^{t}_i}^2} \\
      &\quad + \frac{\gamma}{b} \Exp{\norm{h^{t} - \nabla f(x^{t})}^2} \\
      &\quad + \left(\frac{8 b \gamma \omega (2 \omega + 1)}{n \probavailable^2} + \frac{2 \gamma \left(\probavailable - \probpairaa\right)}{n \probavailable^2}\right)\Exp{\frac{1}{n}\sum_{i=1}^n\norm{h^{t}_i - \nabla f_i(x^{t})}^2} \\
      &\quad + \left(\frac{24 b \gamma \omega (2 \omega + 1)}{n \probavailable^2 B} + \frac{6 \gamma }{n \probavailable B}\right)\Exp{\frac{1}{nm}\sum_{i=1}^n\sum_{j=1}^m\norm{h^{t}_{ij} - \nabla f_{ij}(x^{t})}^2}.
    \end{align*}
    It is left to apply Lemma~\ref{lemma:good_recursion} with 
    \begin{eqnarray*}
      \Psi^t &=& \frac{(2 \omega + 1)}{\probavailable} \Exp{\norm{g^{t} - h^{t}}^2} + \frac{(\left(2 \omega + 1\right)\probavailable - \probpairaa)}{n \probavailable^2} \Exp{\frac{1}{n}\sum_{i=1}^n\norm{g^{t}_i - h^{t}_i}^2} \\
      &\quad& + \frac{1}{b} \Exp{\norm{h^{t} - \nabla f(x^{t})}^2} \\
      &\quad& + \left(\frac{8 b \omega (2 \omega + 1)}{n \probavailable^2} + \frac{2 \left(\probavailable - \probpairaa\right)}{n \probavailable^2}\right)\Exp{\frac{1}{n}\sum_{i=1}^n\norm{h^{t}_i - \nabla f_i(x^{t})}^2} \\
      &\quad& + \left(\frac{24 b \omega (2 \omega + 1)}{n \probavailable^2 B} + \frac{6}{n \probavailable B}\right)\Exp{\frac{1}{nm}\sum_{i=1}^n\sum_{j=1}^m\norm{h^{t}_{ij} - \nabla f_{ij}(x^{t})}^2}
    \end{eqnarray*}
    to conclude the proof.
\end{proof}

\subsection{Proof for \algname{\algorithmname-MVR}}

Let us denote $\nabla f_i(x^{t+1};\xi^{t+1}_i) \eqdef \frac{1}{B} \sum_{j=1}^B \nabla f_i(x^{t+1};\xi^{t+1}_{ij}).$

\begin{lemma}
  \label{lemma:gradient_mvr}
  Suppose that Assumptions \ref{ass:nodes_lipschitz_constant}, \ref{ass:stochastic_unbiased_and_variance_bounded}, \ref{ass:mean_square_smoothness} and \ref{ass:partial_participation} hold. For $h^{t+1}_i$ and $k^{t+1}_i$ from Algorithm~\ref{alg:main_algorithm} (\algname{\algorithmname-MVR}) we have
  \begin{enumerate}
  \item
      \begin{align*}
          &\ExpSub{k}{\ExpSub{\probavailable}{\norm{h^{t+1} - \nabla f(x^{t+1})}^2}} \\
          & \leq \frac{2 b^2 \sigma^2}{n \probavailable B} + \left(\frac{2 (1 - b)^2 L_{\sigma}^2}{n \probavailable B} + \frac{2\left(\probavailable - \probpairaa\right) \widehat{L}^2}{n \probavailable^2} \right) \norm{x^{t+1} - x^{t}}^2\\
          &\quad + \frac{2 \left(\probavailable - \probpairaa\right) b^2}{n^2 \probavailable^2} \sum_{i=1}^n \norm{h^t_i - \nabla f_i(x^{t})}^2 + \left(1 - b\right)^2 \norm{h^{t} - \nabla f(x^{t})}^2.
      \end{align*}
  \item
      \begin{align*}
          &\ExpSub{k}{\ExpSub{\probavailable}{\norm{h^{t+1}_i - \nabla f_i(x^{t+1})}^2}} \\
          & \leq \frac{2 b^2 \sigma^2}{\probavailable B}  + \left(\frac{2 (1 - b)^2 L_{\sigma}^2}{\probavailable B} + \frac{2(1 - \probavailable) L_i^2}{\probavailable}\right)\norm{x^{t+1} - x^{t}}^2 \\
          &\quad + \left(\frac{2 (1 - \probavailable) b^2}{\probavailable} + (1 - b)^2\right)\norm{h^t_i - \nabla f_i(x^{t})}^2, \quad \forall i \in [n].
      \end{align*}
  \item
      \begin{align*}
        &\ExpSub{k}{\norm{k^{t+1}_i}^2} \leq \frac{2 b^2 \sigma^2}{B} + \left(\frac{2 (1 - b)^2 L_{\sigma}^2}{B} + 2 L_i^2\right)\norm{x^{t+1} - x^{t}}^2 + 2 b^2 \norm{h^t_i - \nabla f_i(x^{t})}^2, \quad \forall i \in [n].
      \end{align*}
  \end{enumerate}
\end{lemma}

\begin{proof}
  First, let us proof the bound for $\ExpSub{k}{\ExpSub{\probavailable}{\norm{h^{t+1} - \nabla f(x^{t+1})}^2}}$:
  \begin{align*}
      &\ExpSub{k}{\ExpSub{\probavailable}{\norm{h^{t+1} - \nabla f(x^{t+1})}^2}} \\
      &=\ExpSub{k}{\ExpSub{\probavailable}{\norm{h^{t+1} - \ExpSub{k}{\ExpSub{\probavailable}{h^{t+1}}}}^2}} + \norm{\ExpSub{k}{\ExpSub{\probavailable}{h^{t+1}}} - \nabla f(x^{t+1})}^2.
  \end{align*}
  Using
  \begin{align*}
      \ExpSub{k}{\ExpSub{\probavailable}{h^{t+1}_i}} = h^{t}_i + \ExpSub{k}{k^{t+1}_i} = h^{t}_i + \nabla f_i(x^{t+1}) - \nabla f_i(x^{t}) - b(h^t_i - \nabla f_i(x^{t}))
  \end{align*}
  and \eqref{auxiliary:variance_decomposition}, we have
  \begin{align*}
    &\ExpSub{k}{\ExpSub{\probavailable}{\norm{h^{t+1} - \nabla f(x^{t+1})}^2}} \\
    &=\ExpSub{k}{\ExpSub{\probavailable}{\norm{h^{t+1} - \ExpSub{k}{\ExpSub{\probavailable}{h^{t+1}}}}^2}} + \left(1 - b\right)^2 \norm{h^{t} - \nabla f(x^{t})}^2.
  \end{align*}
  We can use Lemma~\ref{lemma:sampling} with $r_i = h^{t}_i$ and $s_i = k^{t+1}_i$ to obtain
  \begin{align*}
    &\ExpSub{k}{\ExpSub{\probavailable}{\norm{h^{t+1} - \nabla f(x^{t+1})}^2}} \\
    &\leq \frac{1}{n^2 \probavailable}\sum_{i=1}^n\ExpSub{k}{\norm{k^{t+1}_i - \ExpSub{k}{k^{t+1}_i}}^2} +\frac{\probavailable - \probpairaa}{n^2 \probavailable^2}\sum_{i=1}^n\norm{\ExpSub{k}{k^{t+1}_i}}^2 + \left(1 - b\right)^2 \norm{h^{t} - \nabla f(x^{t})}^2\\
    &= \frac{1}{n^2 \probavailable}\sum_{i=1}^n\ExpSub{k}{\norm{\nabla f_i(x^{t+1};\xi^{t+1}_i) - \nabla f_i(x^{t};\xi^{t+1}_i) - b \left(h^t_i - \nabla f_i(x^{t};\xi^{t+1}_i)\right) \right.\right.\\
    &\left.\left. \quad- \left(\nabla f_i(x^{t+1}) - \nabla f_i(x^{t}) - b \left(h^t_i - \nabla f_i(x^{t})\right)\right)}^2} \\
    &\quad +\frac{\probavailable - \probpairaa}{n^2 \probavailable^2}\sum_{i=1}^n\norm{\nabla f_i(x^{t+1}) - \nabla f_i(x^{t}) - b \left(h^t_i - \nabla f_i(x^{t})\right)}^2 \\
    &\quad + \left(1 - b\right)^2 \norm{h^{t} - \nabla f(x^{t})}^2 \\
    &\overset{\eqref{auxiliary:jensen_inequality}}{\leq} \frac{2}{n^2 \probavailable}\sum_{i=1}^n \ExpSub{k}{\norm{b \left(\nabla f_i(x^{t+1};\xi^{t+1}_i) - \nabla f_i(x^{t+1})\right)}^2} \\
    & \quad + \frac{2}{n^2 \probavailable}\sum_{i=1}^n \ExpSub{k}{\norm{(1 - b)\left(\nabla f_i(x^{t+1};\xi^{t+1}_i) - \nabla f_i(x^{t};\xi^{t+1}_i) - \left(\nabla f_i(x^{t+1}) - \nabla f_i(x^{t})\right)\right)}^2} \\
    &\quad + \frac{\probavailable - \probpairaa}{n^2 \probavailable^2} \sum_{i=1}^n \norm{\nabla f_i(x^{t+1}) - \nabla f_i(x^{t}) - b(h^t_i - \nabla f_i(x^{t}))}^2 \\
    &\quad + \left(1 - b\right)^2 \norm{h^{t} - \nabla f(x^{t})}^2 \\
    &= \frac{2 b^2}{n^2 \probavailable}\sum_{i=1}^n \ExpSub{k}{\norm{\nabla f_i(x^{t+1};\xi^{t+1}_i) - \nabla f_i(x^{t+1})}^2} \\
    & \quad + \frac{2 (1 - b)^2}{n^2 \probavailable}\sum_{i=1}^n \ExpSub{k}{\norm{\nabla f_i(x^{t+1};\xi^{t+1}_i) - \nabla f_i(x^{t};\xi^{t+1}_i) - \left(\nabla f_i(x^{t+1}) - \nabla f_i(x^{t})\right)}^2} \\
    &\quad + \frac{\probavailable - \probpairaa}{n^2 \probavailable^2} \sum_{i=1}^n \norm{\nabla f_i(x^{t+1}) - \nabla f_i(x^{t}) - b(h^t_i - \nabla f_i(x^{t}))}^2 \\
    &\quad + \left(1 - b\right)^2 \norm{h^{t} - \nabla f(x^{t})}^2. \\
    &= \frac{2 b^2}{n^2 \probavailable B^2}\sum_{i=1}^n \sum_{j=1}^B \ExpSub{k}{\norm{\nabla f_i(x^{t+1};\xi^{t+1}_{ij}) - \nabla f_i(x^{t+1})}^2} \\
    & \quad + \frac{2 (1 - b)^2}{n^2 \probavailable}\sum_{i=1}^n \ExpSub{k}{\norm{\nabla f_i(x^{t+1};\xi^{t+1}_i) - \nabla f_i(x^{t};\xi^{t+1}_i) - \left(\nabla f_i(x^{t+1}) - \nabla f_i(x^{t})\right)}^2} \\
    &\quad + \frac{\probavailable - \probpairaa}{n^2 \probavailable^2} \sum_{i=1}^n \norm{\nabla f_i(x^{t+1}) - \nabla f_i(x^{t}) - b(h^t_i - \nabla f_i(x^{t}))}^2 \\
    &\quad + \left(1 - b\right)^2 \norm{h^{t} - \nabla f(x^{t})}^2.
  \end{align*}
  In the last equality, we use the independence of elements in the mini-batches. Due to Assumption~\ref{ass:stochastic_unbiased_and_variance_bounded}, we get
  \begin{align*}
    &\ExpSub{k}{\ExpSub{\probavailable}{\norm{h^{t+1} - \nabla f(x^{t+1})}^2}} \\
    &\leq \frac{2 b^2 \sigma^2}{n \probavailable B}\\
    & \quad + \frac{2 (1 - b)^2}{n^2 \probavailable}\sum_{i=1}^n \ExpSub{k}{\norm{\nabla f_i(x^{t+1};\xi^{t+1}_i) - \nabla f_i(x^{t};\xi^{t+1}_i) - \left(\nabla f_i(x^{t+1}) - \nabla f_i(x^{t})\right)}^2} \\
    &\quad + \frac{\probavailable - \probpairaa}{n^2 \probavailable^2} \sum_{i=1}^n \norm{\nabla f_i(x^{t+1}) - \nabla f_i(x^{t}) - b(h^t_i - \nabla f_i(x^{t}))}^2 \\
    &\quad + \left(1 - b\right)^2 \norm{h^{t} - \nabla f(x^{t})}^2 \\
    &\overset{\eqref{auxiliary:jensen_inequality}}{\leq} \frac{2 b^2 \sigma^2}{n \probavailable B}\\
    & \quad + \frac{2 (1 - b)^2}{n^2 \probavailable}\sum_{i=1}^n \ExpSub{k}{\norm{\nabla f_i(x^{t+1};\xi^{t+1}_i) - \nabla f_i(x^{t};\xi^{t+1}_i) - \left(\nabla f_i(x^{t+1}) - \nabla f_i(x^{t})\right)}^2} \\
    &\quad + \frac{2\left(\probavailable - \probpairaa\right)}{n^2 \probavailable^2} \sum_{i=1}^n \norm{\nabla f_i(x^{t+1}) - \nabla f_i(x^{t})}^2 + \frac{2 \left(\probavailable - \probpairaa\right) b^2}{n^2 \probavailable^2} \sum_{i=1}^n \norm{h^t_i - \nabla f_i(x^{t})}^2 \\
    &\quad + \left(1 - b\right)^2 \norm{h^{t} - \nabla f(x^{t})}^2. \\
    &= \frac{2 b^2 \sigma^2}{n \probavailable B}\\
    & \quad + \frac{2 (1 - b)^2}{n^2 \probavailable B^2}\sum_{i=1}^n \sum_{j=1}^B \ExpSub{k}{\norm{\nabla f_i(x^{t+1};\xi^{t+1}_{ij}) - \nabla f_i(x^{t};\xi^{t+1}_{ij}) - \left(\nabla f_i(x^{t+1}) - \nabla f_i(x^{t})\right)}^2} \\
    &\quad + \frac{2\left(\probavailable - \probpairaa\right)}{n^2 \probavailable^2} \sum_{i=1}^n \norm{\nabla f_i(x^{t+1}) - \nabla f_i(x^{t})}^2 + \frac{2 \left(\probavailable - \probpairaa\right) b^2}{n^2 \probavailable^2} \sum_{i=1}^n \norm{h^t_i - \nabla f_i(x^{t})}^2 \\
    &\quad + \left(1 - b\right)^2 \norm{h^{t} - \nabla f(x^{t})}^2, \\
  \end{align*}
  where we use the independence of elements in the mini-batches.
  Using Assumptions~\ref{ass:nodes_lipschitz_constant} and \ref{ass:mean_square_smoothness}, we obtain
  \begin{align*}
    &\ExpSub{k}{\ExpSub{\probavailable}{\norm{h^{t+1} - \nabla f(x^{t+1})}^2}} \\
    &\leq \frac{2 b^2 \sigma^2}{n \probavailable B} + \left(\frac{2 (1 - b)^2 L_{\sigma}^2}{n \probavailable B} + \frac{2\left(\probavailable - \probpairaa\right) \widehat{L}^2}{n \probavailable^2} \right) \norm{x^{t+1} - x^{t}}^2\\
    &\quad + \frac{2 \left(\probavailable - \probpairaa\right) b^2}{n^2 \probavailable^2} \sum_{i=1}^n \norm{h^t_i - \nabla f_i(x^{t})}^2 + \left(1 - b\right)^2 \norm{h^{t} - \nabla f(x^{t})}^2.
  \end{align*}
  Now, we prove the second inequality:
  \begin{align*}
    &\ExpSub{k}{\ExpSub{\probavailable}{\norm{h^{t+1}_i - \nabla f_i(x^{t+1})}^2}} \\
    &=\ExpSub{k}{\ExpSub{\probavailable}{\norm{h^{t+1}_i - \ExpSub{k}{\ExpSub{\probavailable}{h^{t+1}_i}}}^2}} \\
    &\quad + \norm{\ExpSub{k}{\ExpSub{\probavailable}{h^{t+1}_i}} - \nabla f_i(x^{t+1})}^2 \\
    &=\ExpSub{k}{\ExpSub{\probavailable}{\norm{h^{t+1}_i - \left(h^{t}_i + \nabla f_i(x^{t+1}) - \nabla f_i(x^{t}) - b(h^t_i - \nabla f_i(x^{t}))\right)}^2}} \\
    &\quad + \norm{h^{t}_i + \nabla f_i(x^{t+1}) - \nabla f_i(x^{t}) - b(h^t_i - \nabla f_i(x^{t})) - \nabla f_i(x^{t+1})}^2 \\
    &=\ExpSub{k}{\ExpSub{\probavailable}{\norm{h^{t+1}_i - \left(h^{t}_i + \nabla f_i(x^{t+1}) - \nabla f_i(x^{t}) - b(h^t_i - \nabla f_i(x^{t}))\right)}^2}} \\
    &\quad + (1 - b)^2\norm{h^t_i - \nabla f_i(x^{t})}^2 \\
    &=\probavailable \ExpSub{k}{\norm{h^{t}_i + \frac{1}{\probavailable}k^{t+1}_i - \left(h^{t}_i + \nabla f_i(x^{t+1}) - \nabla f_i(x^{t}) - b(h^t_i - \nabla f_i(x^{t}))\right)}^2} \\
    &\quad + (1 - \probavailable)\norm{h^{t}_i - \left(h^{t}_i + \nabla f_i(x^{t+1}) - \nabla f_i(x^{t}) - b(h^t_i - \nabla f_i(x^{t}))\right)}^2 \\
    &\quad + (1 - b)^2\norm{h^t_i - \nabla f_i(x^{t})}^2 \\
    &=\probavailable \ExpSub{k}{\norm{\frac{1}{\probavailable}k^{t+1}_i - \left(\nabla f_i(x^{t+1}) - \nabla f_i(x^{t}) - b(h^t_i - \nabla f_i(x^{t}))\right)}^2} \\
    &\quad + (1 - \probavailable)\norm{\nabla f_i(x^{t+1}) - \nabla f_i(x^{t}) - b(h^t_i - \nabla f_i(x^{t}))}^2 \\
    &\quad + (1 - b)^2\norm{h^t_i - \nabla f_i(x^{t})}^2 \\
    &\overset{\eqref{auxiliary:variance_decomposition}}{=}\frac{1}{\probavailable} \ExpSub{k}{\norm{k^{t+1}_i - \left(\nabla f_i(x^{t+1}) - \nabla f_i(x^{t}) - b(h^t_i - \nabla f_i(x^{t}))\right)}^2} \\
    &\quad + \frac{(1 - \probavailable)^2}{\probavailable} \norm{\nabla f_i(x^{t+1}) - \nabla f_i(x^{t}) - b(h^t_i - \nabla f_i(x^{t}))}^2 \\
    &\quad + (1 - \probavailable)\norm{\nabla f_i(x^{t+1}) - \nabla f_i(x^{t}) - b(h^t_i - \nabla f_i(x^{t}))}^2 \\
    &\quad + (1 - b)^2\norm{h^t_i - \nabla f_i(x^{t})}^2 \\
    &=\frac{1}{\probavailable} \ExpSub{k}{\norm{\nabla f_i(x^{t+1};\xi^{t+1}_{i}) - \nabla f_i(x^{t};\xi^{t+1}_{i}) - b \left(h^t_i - \nabla f_i(x^{t};\xi^{t+1}_{i})\right) - \left(\nabla f_i(x^{t+1}) - \nabla f_i(x^{t}) - b(h^t_i - \nabla f_i(x^{t}))\right)}^2} \\
    &\quad + \frac{1 - \probavailable}{\probavailable} \norm{\nabla f_i(x^{t+1}) - \nabla f_i(x^{t}) - b(h^t_i - \nabla f_i(x^{t}))}^2 \\
    &\quad + (1 - b)^2\norm{h^t_i - \nabla f_i(x^{t})}^2 \\
    &=\frac{1}{\probavailable} \ExpSub{k}{\norm{b \left(\nabla f_i(x^{t+1};\xi^{t+1}_{i}) - \nabla f_i(x^{t+1})\right) + (1 - b) \left(\nabla f_i(x^{t+1};\xi^{t+1}_{i}) - \nabla f_i(x^{t};\xi^{t+1}_{i}) - \left(\nabla f_i(x^{t+1}) - \nabla f_i(x^{t})\right)\right)}^2} \\
    &\quad + \frac{1 - \probavailable}{\probavailable} \norm{\nabla f_i(x^{t+1}) - \nabla f_i(x^{t}) - b(h^t_i - \nabla f_i(x^{t}))}^2 \\
    &\quad + (1 - b)^2\norm{h^t_i - \nabla f_i(x^{t})}^2 \\
    &\overset{\eqref{auxiliary:jensen_inequality}}{\leq} \frac{2 b^2}{\probavailable} \ExpSub{k}{\norm{\nabla f_i(x^{t+1};\xi^{t+1}_{i}) - \nabla f_i(x^{t+1})}^2} \\
    &\quad + \frac{2 (1 - b)^2}{\probavailable} \ExpSub{k}{\norm{\nabla f_i(x^{t+1};\xi^{t+1}_{i}) - \nabla f_i(x^{t};\xi^{t+1}_{i}) - \left(\nabla f_i(x^{t+1}) - \nabla f_i(x^{t})\right)}^2} \\
    &\quad + \frac{1 - \probavailable}{\probavailable} \norm{\nabla f_i(x^{t+1}) - \nabla f_i(x^{t}) - b(h^t_i - \nabla f_i(x^{t}))}^2 \\
    &\quad + (1 - b)^2\norm{h^t_i - \nabla f_i(x^{t})}^2.
  \end{align*}
  Considering the independence of elements in the mini-batch, we obtain
  \begin{align*}
    &\ExpSub{k}{\ExpSub{\probavailable}{\norm{h^{t+1}_i - \nabla f_i(x^{t+1})}^2}} \\
    &= \frac{2 b^2}{\probavailable B^2} \sum_{j=1}^B \ExpSub{k}{\norm{\nabla f_i(x^{t+1};\xi^{t+1}_{ij}) - \nabla f_i(x^{t+1})}^2} \\
    &\quad + \frac{2 (1 - b)^2}{\probavailable B^2} \sum_{j=1}^B \ExpSub{k}{\norm{\nabla f_i(x^{t+1};\xi^{t+1}_{ij}) - \nabla f_i(x^{t};\xi^{t+1}_{ij}) - \left(\nabla f_i(x^{t+1}) - \nabla f_i(x^{t})\right)}^2} \\
    &\quad + \frac{1 - \probavailable}{\probavailable} \norm{\nabla f_i(x^{t+1}) - \nabla f_i(x^{t}) - b(h^t_i - \nabla f_i(x^{t}))}^2 \\
    &\quad + (1 - b)^2\norm{h^t_i - \nabla f_i(x^{t})}^2. \\
    &\overset{\eqref{auxiliary:jensen_inequality}}{\leq} \frac{2 b^2}{\probavailable B^2} \sum_{j=1}^B \ExpSub{k}{\norm{\nabla f_i(x^{t+1};\xi^{t+1}_{ij}) - \nabla f_i(x^{t+1})}^2} \\
    &\quad + \frac{2 (1 - b)^2}{\probavailable B^2} \sum_{j=1}^B \ExpSub{k}{\norm{\nabla f_i(x^{t+1};\xi^{t+1}_{ij}) - \nabla f_i(x^{t};\xi^{t+1}_{ij}) - \left(\nabla f_i(x^{t+1}) - \nabla f_i(x^{t})\right)}^2} \\
    &\quad + \frac{2(1 - \probavailable)}{\probavailable} \norm{\nabla f_i(x^{t+1}) - \nabla f_i(x^{t})}^2 + \left(\frac{2 (1 - \probavailable) b^2}{\probavailable} + (1 - b)^2\right)\norm{h^t_i - \nabla f_i(x^{t})}^2 \\
  \end{align*}
  Next, we use Assumptions~\ref{ass:nodes_lipschitz_constant}, \ref{ass:mean_square_smoothness}, \ref{ass:stochastic_unbiased_and_variance_bounded}, to get
  \begin{align*}
    &\ExpSub{k}{\ExpSub{\probavailable}{\norm{h^{t+1}_i - \nabla f_i(x^{t+1})}^2}} \\
    &\leq \frac{2 b^2 \sigma^2}{\probavailable B}  + \left(\frac{2 (1 - b)^2 L_{\sigma}^2}{\probavailable B} + \frac{2(1 - \probavailable) L_i^2}{\probavailable}\right)\norm{x^{t+1} - x^{t}}^2 \\
    &\quad + \left(\frac{2 (1 - \probavailable) b^2}{\probavailable} + (1 - b)^2\right)\norm{h^t_i - \nabla f_i(x^{t})}^2.\\
  \end{align*}
  It is left to prove the bound for $\ExpSub{k}{\norm{k^{t+1}_i}^2}$:
  \begin{align*}
    &\ExpSub{k}{\norm{k^{t+1}_i}^2} \\
    &= \ExpSub{k}{\norm{\nabla f_i(x^{t+1};\xi^{t+1}_{i}) - \nabla f_i(x^{t};\xi^{t+1}_{i}) - b \left(h^t_i - \nabla f_i(x^{t};\xi^{t+1}_{i})\right)}^2} \\
    &\overset{\eqref{auxiliary:variance_decomposition}}{=} \ExpSub{k}{\norm{\nabla f_i(x^{t+1};\xi^{t+1}_{i}) - \nabla f_i(x^{t};\xi^{t+1}_{i}) - b \left(h^t_i - \nabla f_i(x^{t};\xi^{t+1}_{i})\right) - \left(\nabla f_i(x^{t+1}) - \nabla f_i(x^{t}) - b(h^t_i - \nabla f_i(x^{t}))\right)}^2} \\
    &\quad + \norm{\nabla f_i(x^{t+1}) - \nabla f_i(x^{t}) - b(h^t_i - \nabla f_i(x^{t}))}^2 \\
    &= \ExpSub{k}{\norm{b \left(\nabla f_i(x^{t+1};\xi^{t+1}_{i}) - \nabla f_i(x^{t+1})\right) + (1 - b) \left(\nabla f_i(x^{t+1};\xi^{t+1}_{i}) - \nabla f_i(x^{t};\xi^{t+1}_{i}) - \left(\nabla f_i(x^{t+1}) - \nabla f_i(x^{t})\right)\right)}^2} \\
    &\quad + \norm{\nabla f_i(x^{t+1}) - \nabla f_i(x^{t}) - b(h^t_i - \nabla f_i(x^{t}))}^2 \\
    &\overset{\eqref{auxiliary:jensen_inequality}}{\leq} 2 b^2 \ExpSub{k}{\norm{\nabla f_i(x^{t+1};\xi^{t+1}_{i}) - \nabla f_i(x^{t+1})}^2} \\
    &\quad + 2 (1 - b)^2 \ExpSub{k}{\norm{\nabla f_i(x^{t+1};\xi^{t+1}_{i}) - \nabla f_i(x^{t};\xi^{t+1}_{i}) - \left(\nabla f_i(x^{t+1}) - \nabla f_i(x^{t})\right)}^2} \\
    &\quad + 2 \norm{\nabla f_i(x^{t+1}) - \nabla f_i(x^{t})}^2 + 2 b^2 \norm{h^t_i - \nabla f_i(x^{t})}^2.
  \end{align*}
  Using Assumptions~\ref{ass:nodes_lipschitz_constant}, \ref{ass:mean_square_smoothness}, \ref{ass:stochastic_unbiased_and_variance_bounded} and the independence of elements in the mini-batch, we get
  \begin{align*}
    &\ExpSub{k}{\norm{k^{t+1}_i}^2} \\
    &\leq \frac{2 b^2 \sigma^2}{B} + \left(\frac{2 (1 - b)^2 L_{\sigma}^2}{B} + 2 L_i^2\right)\norm{x^{t+1} - x^{t}}^2 + 2 b^2 \norm{h^t_i - \nabla f_i(x^{t})}^2.
  \end{align*}
\end{proof}

\CONVERGENCEMVR*

\begin{proof}
  Let us fix constants $\nu, \rho \in [0,\infty)$ that we will define later. Considering Lemma~\ref{lemma:main_lemma}, Lemma~\ref{lemma:gradient_mvr}, and the law of total expectation, we obtain
    \begin{align*}
      &\Exp{f(x^{t + 1})} + \frac{\gamma (2 \omega + 1)}{\probavailable} \Exp{\norm{g^{t+1} - h^{t+1}}^2} + \frac{\gamma (\left(2 \omega + 1\right)\probavailable - \probpairaa)}{n \probavailable^2} \Exp{\frac{1}{n}\sum_{i=1}^n\norm{g^{t+1}_i - h^{t+1}_i}^2}\\
      &\quad  + \nu \Exp{\norm{h^{t+1} - \nabla f(x^{t+1})}^2} + \rho \Exp{\frac{1}{n}\sum_{i=1}^n\norm{h^{t+1}_i - \nabla f_i(x^{t+1})}^2}\\
      &\leq \Exp{f(x^t) - \frac{\gamma}{2}\norm{\nabla f(x^t)}^2 - \left(\frac{1}{2\gamma} - \frac{L}{2}\right)
      \norm{x^{t+1} - x^t}^2 + \gamma \norm{h^{t} - \nabla f(x^t)}^2}\nonumber\\
      &\quad + \frac{\gamma (2 \omega + 1)}{\probavailable}\Exp{\norm{g^{t} - h^t}^2}+ \frac{\gamma (\left(2 \omega + 1\right)\probavailable - \probpairaa)}{n \probavailable^2}\Exp{\frac{1}{n} \sum_{i=1}^n\norm{g^t_i - h^{t}_i}^2} \\
      &\quad + \frac{4 \gamma \omega (2 \omega + 1)}{n \probavailable^2} \Exp{\frac{1}{n} \sum_{i=1}^n\norm{k^{t+1}_i}^2} \\
      &\quad  + \nu \Exp{\norm{h^{t+1} - \nabla f(x^{t+1})}^2} + \rho \Exp{\frac{1}{n}\sum_{i=1}^n\norm{h^{t+1}_i - \nabla f_i(x^{t+1})}^2}\\
      &= \Exp{f(x^t) - \frac{\gamma}{2}\norm{\nabla f(x^t)}^2 - \left(\frac{1}{2\gamma} - \frac{L}{2}\right)
      \norm{x^{t+1} - x^t}^2 + \gamma \norm{h^{t} - \nabla f(x^t)}^2}\nonumber\\
      &\quad + \frac{\gamma (2 \omega + 1)}{\probavailable}\Exp{\norm{g^{t} - h^t}^2}+ \frac{\gamma (\left(2 \omega + 1\right)\probavailable - \probpairaa)}{n \probavailable^2}\Exp{\frac{1}{n} \sum_{i=1}^n\norm{g^t_i - h^{t}_i}^2} \\
      &\quad + \frac{4 \gamma \omega (2 \omega + 1)}{n \probavailable^2} \Exp{\ExpSub{k}{\frac{1}{n} \sum_{i=1}^n\norm{k^{t+1}_i}^2}} \\
      &\quad  + \nu \Exp{\ExpSub{B}{\ExpSub{\probavailable}{\norm{h^{t+1} - \nabla f(x^{t+1})}^2}}} \\
      &\quad + \rho \Exp{\ExpSub{B}{\ExpSub{\probavailable}{\frac{1}{n}\sum_{i=1}^n\norm{h^{t+1}_i - \nabla f_i(x^{t+1})}^2}}}\\
      &\leq \Exp{f(x^t) - \frac{\gamma}{2}\norm{\nabla f(x^t)}^2 - \left(\frac{1}{2\gamma} - \frac{L}{2}\right)
      \norm{x^{t+1} - x^t}^2 + \gamma \norm{h^{t} - \nabla f(x^t)}^2}\nonumber\\
      &\quad + \frac{\gamma (2 \omega + 1)}{\probavailable}\Exp{\norm{g^{t} - h^t}^2}+ \frac{\gamma (\left(2 \omega + 1\right)\probavailable - \probpairaa)}{n \probavailable^2}\Exp{\frac{1}{n} \sum_{i=1}^n\norm{g^t_i - h^{t}_i}^2} \\
      &\quad + \frac{4 \gamma \omega (2 \omega + 1)}{n \probavailable^2} \Exp{\frac{2 b^2 \sigma^2}{B} + \left(\frac{2 (1 - b)^2 L_{\sigma}^2}{B} + 2 \widehat{L}^2\right)\norm{x^{t+1} - x^{t}}^2 + 2 b^2 \frac{1}{n} \sum_{i=1}^n \norm{h^t_i - \nabla f_i(x^{t})}^2} \\
      &\quad  + \nu {\rm E}\Bigg(\frac{2 b^2 \sigma^2}{n \probavailable B} + \left(\frac{2 (1 - b)^2 L_{\sigma}^2}{n \probavailable B} + \frac{2\left(\probavailable - \probpairaa\right) \widehat{L}^2}{n \probavailable^2} \right) \norm{x^{t+1} - x^{t}}^2\\
      &\qquad\quad + \frac{2 \left(\probavailable - \probpairaa\right) b^2}{n^2 \probavailable^2} \sum_{i=1}^n \norm{h^t_i - \nabla f_i(x^{t})}^2 + \left(1 - b\right)^2 \norm{h^{t} - \nabla f(x^{t})}^2\Bigg) \\
      &\quad + \rho {\rm E}\Bigg(\frac{2 b^2 \sigma^2}{\probavailable B}  + \left(\frac{2 (1 - b)^2 L_{\sigma}^2}{\probavailable B} + \frac{2(1 - \probavailable) \widehat{L}^2}{\probavailable}\right)\norm{x^{t+1} - x^{t}}^2 \\
      &\qquad\quad + \left(\frac{2 (1 - \probavailable) b^2}{\probavailable} + (1 - b)^2\right) \frac{1}{n} \sum_{i=1}^n \norm{h^t_i - \nabla f_i(x^{t})}^2\Bigg).
    \end{align*}
    After rearranging the terms, we get
    \begin{align*}
      &\Exp{f(x^{t + 1})} + \frac{\gamma (2 \omega + 1)}{\probavailable} \Exp{\norm{g^{t+1} - h^{t+1}}^2} + \frac{\gamma (\left(2 \omega + 1\right)\probavailable - \probpairaa)}{n \probavailable^2} \Exp{\frac{1}{n}\sum_{i=1}^n\norm{g^{t+1}_i - h^{t+1}_i}^2}\\
      &\quad  + \nu \Exp{\norm{h^{t+1} - \nabla f(x^{t+1})}^2} + \rho \Exp{\frac{1}{n}\sum_{i=1}^n\norm{h^{t+1}_i - \nabla f_i(x^{t+1})}^2}\\
      &\leq \Exp{f(x^t)} - \frac{\gamma}{2}\Exp{\norm{\nabla f(x^t)}^2} \\
      &\quad + \frac{\gamma (2 \omega + 1)}{\probavailable} \Exp{\norm{g^{t} - h^{t}}^2} + \frac{\gamma (\left(2 \omega + 1\right)\probavailable - \probpairaa)}{n \probavailable^2} \Exp{\frac{1}{n}\sum_{i=1}^n\norm{g^{t}_i - h^{t}_i}^2} \\
      &\quad - \Bigg(\frac{1}{2\gamma} - \frac{L}{2} - \frac{4 \gamma \omega (2 \omega + 1)}{n \probavailable^2} \left(\frac{2 (1 - b)^2 L_{\sigma}^2}{B} + 2 \widehat{L}^2\right) \\
      &\qquad\quad - \nu\left(\frac{2 (1 - b)^2 L_{\sigma}^2}{n \probavailable B} + \frac{2\left(\probavailable - \probpairaa\right) \widehat{L}^2}{n \probavailable^2} \right) - \rho \left(\frac{2 (1 - b)^2 L_{\sigma}^2}{\probavailable B} + \frac{2(1 - \probavailable) \widehat{L}^2}{\probavailable}\right)\Bigg) \Exp{\norm{x^{t+1} - x^t}^2} \\
      &\quad + \left(\gamma + \nu \left(1 - b\right)^2\right) \Exp{\norm{h^{t} - \nabla f(x^{t})}^2} \\
      &\quad + \Bigg(\frac{8 b^2 \gamma \omega (2 \omega + 1)}{n \probavailable^2} + \frac{2 \nu \left(\probavailable - \probpairaa\right) b^2}{n \probavailable^2} + \rho \left(\frac{2 (1 - \probavailable) b^2}{\probavailable} + (1 - b)^2\right)\Bigg)\Exp{\frac{1}{n}\sum_{i=1}^n\norm{h^{t}_i - \nabla f_i(x^{t})}^2} \\
      &\quad + \left(\frac{8 b^2 \gamma \omega (2 \omega + 1)}{n \probavailable^2} + \nu \frac{2 b^2}{n \probavailable} + \rho \frac{2 b^2}{\probavailable}\right) \frac{\sigma^2}{B}.
    \end{align*}
    By taking $\nu = \frac{\gamma}{b},$ one can show that $\left(\gamma + \nu (1 - b)^2\right) \leq \nu,$ and
    \begin{align*}
      &\Exp{f(x^{t + 1})} + \frac{\gamma (2 \omega + 1)}{\probavailable} \Exp{\norm{g^{t+1} - h^{t+1}}^2} + \frac{\gamma (\left(2 \omega + 1\right)\probavailable - \probpairaa)}{n \probavailable^2} \Exp{\frac{1}{n}\sum_{i=1}^n\norm{g^{t+1}_i - h^{t+1}_i}^2}\\
      &\quad  + \frac{\gamma}{b} \Exp{\norm{h^{t+1} - \nabla f(x^{t+1})}^2} + \rho \Exp{\frac{1}{n}\sum_{i=1}^n\norm{h^{t+1}_i - \nabla f_i(x^{t+1})}^2}\\
      &\leq \Exp{f(x^t)} - \frac{\gamma}{2}\Exp{\norm{\nabla f(x^t)}^2} \\
      &\quad + \frac{\gamma (2 \omega + 1)}{\probavailable} \Exp{\norm{g^{t} - h^{t}}^2} + \frac{\gamma (\left(2 \omega + 1\right)\probavailable - \probpairaa)}{n \probavailable^2} \Exp{\frac{1}{n}\sum_{i=1}^n\norm{g^{t}_i - h^{t}_i}^2} \\
      &\quad - \Bigg(\frac{1}{2\gamma} - \frac{L}{2} - \frac{4 \gamma \omega (2 \omega + 1)}{n \probavailable^2} \left(\frac{2 (1 - b)^2 L_{\sigma}^2}{B} + 2 \widehat{L}^2\right) \\
      &\qquad\quad - \frac{\gamma}{b}\left(\frac{2 (1 - b)^2 L_{\sigma}^2}{n \probavailable B} + \frac{2\left(\probavailable - \probpairaa\right) \widehat{L}^2}{n \probavailable^2} \right) - \rho \left(\frac{2 (1 - b)^2 L_{\sigma}^2}{\probavailable B} + \frac{2(1 - \probavailable) \widehat{L}^2}{\probavailable}\right)\Bigg) \Exp{\norm{x^{t+1} - x^t}^2} \\
      &\quad + \frac{\gamma}{b} \Exp{\norm{h^{t} - \nabla f(x^{t})}^2} \\
      &\quad + \Bigg(\frac{8 b^2 \gamma \omega (2 \omega + 1)}{n \probavailable^2} + \frac{2 \gamma \left(\probavailable - \probpairaa\right) b}{n \probavailable^2} + \rho \left(\frac{2 (1 - \probavailable) b^2}{\probavailable} + (1 - b)^2\right)\Bigg)\Exp{\frac{1}{n}\sum_{i=1}^n\norm{h^{t}_i - \nabla f_i(x^{t})}^2} \\
      &\quad + \left(\frac{8 b^2 \gamma \omega (2 \omega + 1)}{n \probavailable^2} + \frac{2 \gamma b}{n \probavailable} + \rho \frac{2 b^2}{\probavailable}\right) \frac{\sigma^2}{B}.
    \end{align*}
    Note that $b \leq \frac{\probavailable}{2 - \probavailable},$ thus
    \begin{align*}
      &\Bigg(\frac{8 b^2 \gamma \omega (2 \omega + 1)}{n \probavailable^2} + \frac{2 \gamma \left(\probavailable - \probpairaa\right) b}{n \probavailable^2} + \rho \left(\frac{2 (1 - \probavailable) b^2}{\probavailable} + (1 - b)^2\right)\Bigg) \\
      &\leq \Bigg(\frac{8 b^2 \gamma \omega (2 \omega + 1)}{n \probavailable^2} + \frac{2 \gamma \left(\probavailable - \probpairaa\right) b}{n \probavailable^2} + \rho \left(1 - b\right)\Bigg).
    \end{align*}
    And if we take $\rho = \frac{8 b \gamma \omega (2 \omega + 1)}{n \probavailable^2} + \frac{2 \gamma \left(\probavailable - \probpairaa\right)}{n \probavailable^2},$ then
    \begin{align*}
      \Bigg(\frac{8 b^2 \gamma \omega (2 \omega + 1)}{n \probavailable^2} + \frac{2 \gamma \left(\probavailable - \probpairaa\right) b}{n \probavailable^2} + \rho \left(1 - b\right)\Bigg) \leq \rho,
    \end{align*}
    and 
    \begin{align*}
      &\Exp{f(x^{t + 1})} + \frac{\gamma (2 \omega + 1)}{\probavailable} \Exp{\norm{g^{t+1} - h^{t+1}}^2} + \frac{\gamma (\left(2 \omega + 1\right)\probavailable - \probpairaa)}{n \probavailable^2} \Exp{\frac{1}{n}\sum_{i=1}^n\norm{g^{t+1}_i - h^{t+1}_i}^2}\\
      &\quad  + \frac{\gamma}{b} \Exp{\norm{h^{t+1} - \nabla f(x^{t+1})}^2} + \left(\frac{8 b \gamma \omega (2 \omega + 1)}{n \probavailable^2} + \frac{2 \gamma \left(\probavailable - \probpairaa\right)}{n \probavailable^2}\right) \Exp{\frac{1}{n}\sum_{i=1}^n\norm{h^{t+1}_i - \nabla f_i(x^{t+1})}^2}\\
      &\leq \Exp{f(x^t)} - \frac{\gamma}{2}\Exp{\norm{\nabla f(x^t)}^2} \\
      &\quad + \frac{\gamma (2 \omega + 1)}{\probavailable} \Exp{\norm{g^{t} - h^{t}}^2} + \frac{\gamma (\left(2 \omega + 1\right)\probavailable - \probpairaa)}{n \probavailable^2} \Exp{\frac{1}{n}\sum_{i=1}^n\norm{g^{t}_i - h^{t}_i}^2} \\
      &\quad - \Bigg(\frac{1}{2\gamma} - \frac{L}{2} - \frac{4 \gamma \omega (2 \omega + 1)}{n \probavailable^2} \left(\frac{2 (1 - b)^2 L_{\sigma}^2}{B} + 2 \widehat{L}^2\right) \\
      &\qquad\quad - \frac{\gamma}{n \probavailable b}\left(\frac{2 (1 - b)^2 L_{\sigma}^2}{B} + 2\left(1 - \frac{\probpairaa}{\probavailable}\right) \widehat{L}^2 \right) \\
      &\qquad\quad - \left(\frac{8 b \gamma \omega (2 \omega + 1)}{n \probavailable^3} + \frac{2 \gamma \left(1 - \frac{\probpairaa}{\probavailable}\right)}{n \probavailable^2}\right) \left(\frac{2 (1 - b)^2 L_{\sigma}^2}{B} + 2(1 - \probavailable) \widehat{L}^2\right)\Bigg) \Exp{\norm{x^{t+1} - x^t}^2} \\
      &\quad + \frac{\gamma}{b} \Exp{\norm{h^{t} - \nabla f(x^{t})}^2} + \left(\frac{8 b \gamma \omega (2 \omega + 1)}{n \probavailable^2} + \frac{2 \gamma \left(\probavailable - \probpairaa\right)}{n \probavailable^2}\right)\Exp{\frac{1}{n}\sum_{i=1}^n\norm{h^{t}_i - \nabla f_i(x^{t})}^2} \\
      &\quad + \left(\frac{8 b^2 \gamma \omega (2 \omega + 1)}{n \probavailable^2} + \frac{2 \gamma b}{n \probavailable} + \left(\frac{8 b \gamma \omega (2 \omega + 1)}{n \probavailable^2} + \frac{2 \gamma \left(\probavailable - \probpairaa\right)}{n \probavailable^2}\right) \frac{2 b^2}{\probavailable}\right) \frac{\sigma^2}{B}.
    \end{align*}
    Let us simplify the inequality. First, due to $b \leq \probavailable$ and $\left(1 - \probavailable\right) \leq \left(1 - \frac{\probpairaa}{\probavailable}\right),$ we have
    \begin{align*}
      &\left(\frac{8 b \gamma \omega (2 \omega + 1)}{n \probavailable^3} + \frac{2 \gamma \left(1 - \frac{\probpairaa}{\probavailable}\right)}{n \probavailable^2}\right) \left(\frac{2 (1 - b)^2 L_{\sigma}^2}{B} + 2(1 - \probavailable) \widehat{L}^2\right) \\
      &=\frac{8 b \gamma \omega (2 \omega + 1)}{n \probavailable^3} \left(\frac{2 (1 - b)^2 L_{\sigma}^2}{B} + 2(1 - \probavailable) \widehat{L}^2\right) \\
      &\quad + \frac{2 \gamma \left(1 - \frac{\probpairaa}{\probavailable}\right)}{n \probavailable^2} \left(\frac{2 (1 - b)^2 L_{\sigma}^2}{B} + 2(1 - \probavailable) \widehat{L}^2\right) \\
      &\leq \frac{8 \gamma \omega (2 \omega + 1)}{n \probavailable^2} \left(\frac{2 (1 - b)^2 L_{\sigma}^2}{B} + 2 \widehat{L}^2\right) \\
      &\quad + \frac{2 \gamma}{n \probavailable b} \left(\frac{2 (1 - b)^2 L_{\sigma}^2}{B} + 2\left(1 - \frac{\probpairaa}{\probavailable}\right) \widehat{L}^2\right),
    \end{align*}
    therefore
    \begin{align*}
      &\Exp{f(x^{t + 1})} + \frac{\gamma (2 \omega + 1)}{\probavailable} \Exp{\norm{g^{t+1} - h^{t+1}}^2} + \frac{\gamma (\left(2 \omega + 1\right)\probavailable - \probpairaa)}{n \probavailable^2} \Exp{\frac{1}{n}\sum_{i=1}^n\norm{g^{t+1}_i - h^{t+1}_i}^2}\\
      &\quad  + \frac{\gamma}{b} \Exp{\norm{h^{t+1} - \nabla f(x^{t+1})}^2} + \left(\frac{8 b \gamma \omega (2 \omega + 1)}{n \probavailable^2} + \frac{2 \gamma \left(\probavailable - \probpairaa\right)}{n \probavailable^2}\right) \Exp{\frac{1}{n}\sum_{i=1}^n\norm{h^{t+1}_i - \nabla f_i(x^{t+1})}^2}\\
      &\leq \Exp{f(x^t)} - \frac{\gamma}{2}\Exp{\norm{\nabla f(x^t)}^2} \\
      &\quad + \frac{\gamma (2 \omega + 1)}{\probavailable} \Exp{\norm{g^{t} - h^{t}}^2} + \frac{\gamma (\left(2 \omega + 1\right)\probavailable - \probpairaa)}{n \probavailable^2} \Exp{\frac{1}{n}\sum_{i=1}^n\norm{g^{t}_i - h^{t}_i}^2} \\
      &\quad - \Bigg(\frac{1}{2\gamma} - \frac{L}{2} - \frac{12 \gamma \omega (2 \omega + 1)}{n \probavailable^2} \left(\frac{2 (1 - b)^2 L_{\sigma}^2}{B} + 2 \widehat{L}^2\right) \\
      &\qquad\quad - \frac{3 \gamma}{n \probavailable b}\left(\frac{2 (1 - b)^2 L_{\sigma}^2}{B} + 2\left(1 - \frac{\probpairaa}{\probavailable}\right) \widehat{L}^2 \right)\Bigg) \Exp{\norm{x^{t+1} - x^t}^2} \\
      &\quad + \frac{\gamma}{b} \Exp{\norm{h^{t} - \nabla f(x^{t})}^2} + \left(\frac{8 b \gamma \omega (2 \omega + 1)}{n \probavailable^2} + \frac{2 \gamma \left(\probavailable - \probpairaa\right)}{n \probavailable^2}\right)\Exp{\frac{1}{n}\sum_{i=1}^n\norm{h^{t}_i - \nabla f_i(x^{t})}^2} \\
      &\quad + \left(\frac{8 b^2 \gamma \omega (2 \omega + 1)}{n \probavailable^2} + \frac{2 \gamma b}{n \probavailable} + \left(\frac{8 b \gamma \omega (2 \omega + 1)}{n \probavailable^2} + \frac{2 \gamma \left(\probavailable - \probpairaa\right)}{n \probavailable^2}\right) \frac{2 b^2}{\probavailable}\right) \frac{\sigma^2}{B} \\
      &= \Exp{f(x^t)} - \frac{\gamma}{2}\Exp{\norm{\nabla f(x^t)}^2} \\
      &\quad + \frac{\gamma (2 \omega + 1)}{\probavailable} \Exp{\norm{g^{t} - h^{t}}^2} + \frac{\gamma (\left(2 \omega + 1\right)\probavailable - \probpairaa)}{n \probavailable^2} \Exp{\frac{1}{n}\sum_{i=1}^n\norm{g^{t}_i - h^{t}_i}^2} \\
      &\quad - \Bigg(\frac{1}{2\gamma} - \frac{L}{2} - \frac{24 \gamma \omega (2 \omega + 1)}{n \probavailable^2} \left(\frac{(1 - b)^2 L_{\sigma}^2}{B} + \widehat{L}^2\right) \\
      &\qquad\quad - \frac{6 \gamma}{n \probavailable b}\left(\frac{(1 - b)^2 L_{\sigma}^2}{B} + \left(1 - \frac{\probpairaa}{\probavailable}\right) \widehat{L}^2 \right)\Bigg) \Exp{\norm{x^{t+1} - x^t}^2} \\
      &\quad + \frac{\gamma}{b} \Exp{\norm{h^{t} - \nabla f(x^{t})}^2} + \left(\frac{8 b \gamma \omega (2 \omega + 1)}{n \probavailable^2} + \frac{2 \gamma \left(\probavailable - \probpairaa\right)}{n \probavailable^2}\right)\Exp{\frac{1}{n}\sum_{i=1}^n\norm{h^{t}_i - \nabla f_i(x^{t})}^2} \\
      &\quad + \left(\frac{8 b^2 \gamma \omega (2 \omega + 1)}{n \probavailable^2} + \frac{2 \gamma b}{n \probavailable} + \left(\frac{8 b \gamma \omega (2 \omega + 1)}{n \probavailable^2} + \frac{2 \gamma \left(\probavailable - \probpairaa\right)}{n \probavailable^2}\right) \frac{2 b^2}{\probavailable}\right) \frac{\sigma^2}{B}.
    \end{align*}
    Also, we can simplify the last term:
    \begin{align*}
      &\left(\frac{8 b \gamma \omega (2 \omega + 1)}{n \probavailable^2} + \frac{2 \gamma \left(\probavailable - \probpairaa\right)}{n \probavailable^2}\right) \frac{2 b^2}{\probavailable} \\
      &= \frac{16 b^3 \gamma \omega (2 \omega + 1)}{n \probavailable^3} + \frac{4 b^2 \gamma \left(1 - \frac{\probpairaa}{\probavailable}\right)}{n \probavailable^2} \\
      &\leq \frac{16 b^2 \gamma \omega (2 \omega + 1)}{n \probavailable^2} + \frac{4 b \gamma}{n \probavailable},
    \end{align*}
    thus
    \begin{align*}
      &\Exp{f(x^{t + 1})} + \frac{\gamma (2 \omega + 1)}{\probavailable} \Exp{\norm{g^{t+1} - h^{t+1}}^2} + \frac{\gamma (\left(2 \omega + 1\right)\probavailable - \probpairaa)}{n \probavailable^2} \Exp{\frac{1}{n}\sum_{i=1}^n\norm{g^{t+1}_i - h^{t+1}_i}^2}\\
      &\quad  + \frac{\gamma}{b} \Exp{\norm{h^{t+1} - \nabla f(x^{t+1})}^2} + \left(\frac{8 b \gamma \omega (2 \omega + 1)}{n \probavailable^2} + \frac{2 \gamma \left(\probavailable - \probpairaa\right)}{n \probavailable^2}\right) \Exp{\frac{1}{n}\sum_{i=1}^n\norm{h^{t+1}_i - \nabla f_i(x^{t+1})}^2}\\
      &\leq \Exp{f(x^t)} - \frac{\gamma}{2}\Exp{\norm{\nabla f(x^t)}^2} \\
      &\quad + \frac{\gamma (2 \omega + 1)}{\probavailable} \Exp{\norm{g^{t} - h^{t}}^2} + \frac{\gamma (\left(2 \omega + 1\right)\probavailable - \probpairaa)}{n \probavailable^2} \Exp{\frac{1}{n}\sum_{i=1}^n\norm{g^{t}_i - h^{t}_i}^2} \\
      &\quad - \Bigg(\frac{1}{2\gamma} - \frac{L}{2} - \frac{24 \gamma \omega (2 \omega + 1)}{n \probavailable^2} \left(\frac{(1 - b)^2 L_{\sigma}^2}{B} + \widehat{L}^2\right) \\
      &\qquad\quad - \frac{6 \gamma}{n \probavailable b}\left(\frac{(1 - b)^2 L_{\sigma}^2}{B} + \left(1 - \frac{\probpairaa}{\probavailable}\right) \widehat{L}^2 \right)\Bigg) \Exp{\norm{x^{t+1} - x^t}^2} \\
      &\quad + \frac{\gamma}{b} \Exp{\norm{h^{t} - \nabla f(x^{t})}^2} + \left(\frac{8 b \gamma \omega (2 \omega + 1)}{n \probavailable^2} + \frac{2 \gamma \left(\probavailable - \probpairaa\right)}{n \probavailable^2}\right)\Exp{\frac{1}{n}\sum_{i=1}^n\norm{h^{t}_i - \nabla f_i(x^{t})}^2} \\
      &\quad + \left(\frac{24 b^2 \gamma \omega (2 \omega + 1)}{n \probavailable^2} + \frac{6 \gamma b}{n \probavailable}\right) \frac{\sigma^2}{B}.
    \end{align*}
    Using Lemma~\ref{lemma:gamma} and the assumption about $\gamma,$ we get
    \begin{align*}
      &\Exp{f(x^{t + 1})} + \frac{\gamma (2 \omega + 1)}{\probavailable} \Exp{\norm{g^{t+1} - h^{t+1}}^2} + \frac{\gamma (\left(2 \omega + 1\right)\probavailable - \probpairaa)}{n \probavailable^2} \Exp{\frac{1}{n}\sum_{i=1}^n\norm{g^{t+1}_i - h^{t+1}_i}^2}\\
      &\quad  + \frac{\gamma}{b} \Exp{\norm{h^{t+1} - \nabla f(x^{t+1})}^2} + \left(\frac{8 b \gamma \omega (2 \omega + 1)}{n \probavailable^2} + \frac{2 \gamma \left(\probavailable - \probpairaa\right)}{n \probavailable^2}\right) \Exp{\frac{1}{n}\sum_{i=1}^n\norm{h^{t+1}_i - \nabla f_i(x^{t+1})}^2}\\
      &\leq \Exp{f(x^t)} - \frac{\gamma}{2}\Exp{\norm{\nabla f(x^t)}^2} \\
      &\quad + \frac{\gamma (2 \omega + 1)}{\probavailable} \Exp{\norm{g^{t} - h^{t}}^2} + \frac{\gamma (\left(2 \omega + 1\right)\probavailable - \probpairaa)}{n \probavailable^2} \Exp{\frac{1}{n}\sum_{i=1}^n\norm{g^{t}_i - h^{t}_i}^2} \\
      &\quad + \frac{\gamma}{b} \Exp{\norm{h^{t} - \nabla f(x^{t})}^2} + \left(\frac{8 b \gamma \omega (2 \omega + 1)}{n \probavailable^2} + \frac{2 \gamma \left(\probavailable - \probpairaa\right)}{n \probavailable^2}\right)\Exp{\frac{1}{n}\sum_{i=1}^n\norm{h^{t}_i - \nabla f_i(x^{t})}^2} \\
      &\quad + \left(\frac{24 b^2 \gamma \omega (2 \omega + 1)}{n \probavailable^2} + \frac{6 \gamma b}{n \probavailable}\right) \frac{\sigma^2}{B}.
    \end{align*}
    It is left to apply Lemma~\ref{lemma:good_recursion} with 
    \begin{eqnarray*}
      \Psi^t &=& \frac{(2 \omega + 1)}{\probavailable} \Exp{\norm{g^{t} - h^{t}}^2} + \frac{(\left(2 \omega + 1\right)\probavailable - \probpairaa)}{n \probavailable^2} \Exp{\frac{1}{n}\sum_{i=1}^n\norm{g^{t}_i - h^{t}_i}^2} \\
      &+& \frac{1}{b} \Exp{\norm{h^{t} - \nabla f(x^{t})}^2} + \left(\frac{8 b \omega (2 \omega + 1)}{n \probavailable^2} + \frac{2 \left(\probavailable - \probpairaa\right)}{n \probavailable^2}\right)\Exp{\frac{1}{n}\sum_{i=1}^n\norm{h^{t}_i - \nabla f_i(x^{t})}^2}
    \end{eqnarray*}
    and $C = \left(\frac{24 b^2 \omega (2 \omega + 1)}{\probavailable^2} + \frac{6b}{\probavailable}\right) \frac{\sigma^2}{n B}$
    to conclude the proof.
  \end{proof}

  \COROLLARYSTOCHASTIC*

  \begin{proof}
    Using the result from Theorem~\ref{theorem:stochastic}, we have
    \begin{align*}
        &\Exp{\norm{\nabla f(\widehat{x}^T)}^2} \\
        &\leq \frac{1}{T}\vast[2 \Delta_0\left(L + \sqrt{\frac{48 \omega (2 \omega + 1)}{n \probavailable^2} \left(\widehat{L}^2 + \frac{(1 - b)^2 L_{\sigma}^2}{B}\right) + \frac{12}{n \probavailable b}\left(\left(1 - \frac{\probpairaa}{\probavailable}\right) \widehat{L}^2 + \frac{(1 - b)^2 L_{\sigma}^2}{B}\right)}\right) \\
        &\quad + \frac{2}{b} \norm{h^{0} - \nabla f(x^{0})}^2 + \left(\frac{32 b \omega (2 \omega + 1)}{n \probavailable^2} + \frac{4 \left(1 - \frac{\probpairaa}{\probavailable}\right)}{n \probavailable}\right)\left(\frac{1}{n}\sum_{i=1}^n\norm{h^{0}_i - \nabla f_i(x^{0})}^2\right)\vast] \\
        &\quad+ \left(\frac{48 b^2 \omega (2 \omega + 1)}{\probavailable^2} + \frac{12 b}{\probavailable}\right) \frac{\sigma^2}{n B}
    \end{align*}
    We choose $b$ to ensure $\left(\frac{48 b^2 \omega (2 \omega + 1)}{\probavailable^2} + \frac{12 b}{\probavailable}\right) \frac{\sigma^2}{n B} = \Theta\left(\varepsilon\right).$ Note that $\frac{1}{b} = \Theta\left(\max\left\{\frac{\omega}{\probavailable} \sqrt{\frac{\sigma^2}{n \varepsilon B}}, \frac{\sigma^2}{\probavailable n \varepsilon B}\right\}\right) \leq \Theta\left(\max\left\{\frac{\omega^2}{\probavailable}, \frac{\sigma^2}{\probavailable n \varepsilon B}\right\}\right),$ thus
    \begin{align*}
        &\Exp{\norm{\nabla f(\widehat{x}^T)}^2} \\
        &= \cO\vast(\frac{1}{T}\vast[\Delta_0\left(L + \frac{\omega}{\probavailable \sqrt{n}} \left(\widehat{L} + \frac{L_{\sigma}}{\sqrt{B}}\right) + \sqrt{\frac{\sigma^2}{\probavailable^2 \varepsilon n^2 B}}\left(\mathbbm{1}_{\probavailable}\widehat{L} + \frac{L_{\sigma}}{\sqrt{B}}\right)\right)\\
        &\quad + \frac{1}{b} \norm{h^{0} - \nabla f(x^{0})}^2 + \left(\frac{b \omega^2}{n \probavailable^2} + \frac{1}{n \probavailable}\right)\left(\frac{1}{n}\sum_{i=1}^n\norm{h^{0}_i - \nabla f_i(x^{0})}^2\right)\vast] + \varepsilon\vast),
    \end{align*}
    where $\mathbbm{1}_{\probavailable} = \sqrt{1 - \frac{\probpairaa}{\probavailable}}.$
    It enough to take the following $T$ to get $\varepsilon$-solution.
    \begin{align*}
        &T=\cO\vast(\frac{1}{\varepsilon}\vast[\Delta_0\left(L + \frac{\omega}{\probavailable \sqrt{n}} \left(\widehat{L} + \frac{L_{\sigma}}{\sqrt{B}}\right) + \sqrt{\frac{\sigma^2}{\probavailable^2 \varepsilon n^2 B}}\left(\mathbbm{1}_{\probavailable}\widehat{L} + \frac{L_{\sigma}}{\sqrt{B}}\right)\right)\\
        &+ \frac{1}{b} \norm{h^{0} - \nabla f(x^{0})}^2 + \left(\frac{b \omega^2}{n \probavailable^2} + \frac{1}{n \probavailable}\right)\left(\frac{1}{n}\sum_{i=1}^n\norm{h^{0}_i - \nabla f_i(x^{0})}^2\right)\vast]\vast).
    \end{align*}
    Let us bound the norms:
    \begin{eqnarray*}
        \Exp{\norm{h^{0} - \nabla f(x^{0})}^2} &=& \Exp{\norm{\frac{1}{n} \sum_{i=1}^n \frac{1}{B_{\textnormal{init}}} \sum_{k = 1}^{B_{\textnormal{init}}} \nabla f_i(x^0; \xi^0_{ik}) - \nabla f(x^{0})}^2} \\
        &=& \frac{1}{n^2 B_{\textnormal{init}}^2}\sum_{i=1}^n \sum_{k = 1}^{B_{\textnormal{init}}} \Exp{\norm{\nabla f_i(x^0; \xi^0_{ik}) - \nabla f_i(x^{0})}^2}\\
        &\leq& \frac{\sigma^2}{n B_{\textnormal{init}}}.
    \end{eqnarray*}
    Using the same reasoning, one cat get $\frac{1}{n}\sum_{i=1}^n\Exp{\norm{h^{0}_i - \nabla f_i(x^{0})}^2} \leq \frac{\sigma^2}{B_{\textnormal{init}}}.$ 
    Combining all inequalities, we have

    \begin{align*}
      T&=\cO\vast(\frac{1}{\varepsilon}\vast[\Delta_0\left(L + \frac{\omega}{\probavailable \sqrt{n}} \left(\widehat{L} + \frac{L_{\sigma}}{\sqrt{B}}\right) + \sqrt{\frac{\sigma^2}{\probavailable^2 \varepsilon n^2 B}}\left(\mathbbm{1}_{\probavailable}\widehat{L} + \frac{L_{\sigma}}{\sqrt{B}}\right)\right)\\
      &\qquad + \frac{\sigma^2}{b n B_{\textnormal{init}}}  + \frac{b \omega^2 \sigma^2}{n \probavailable^2 B_{\textnormal{init}}} + \frac{\sigma^2}{n \probavailable B_{\textnormal{init}}}\vast]\vast).
    \end{align*}

    Using the choice of $B_{\textnormal{init}}$ and $b,$ we obtain

    \begin{align*}
      T&=\cO\vast(\frac{1}{\varepsilon}\vast[\Delta_0\left(L + \frac{\omega}{\probavailable \sqrt{n}} \left(\widehat{L} + \frac{L_{\sigma}}{\sqrt{B}}\right) + \sqrt{\frac{\sigma^2}{\probavailable^2 \varepsilon n^2 B}}\left(\mathbbm{1}_{\probavailable}\widehat{L} + \frac{L_{\sigma}}{\sqrt{B}}\right)\right)\\
      &\qquad + \frac{\sigma^2}{\sqrt{\probavailable} n B}  + \frac{b^2 \omega^2 \sigma^2}{n \probavailable^{5/2} B} + \frac{b \sigma^2}{\probavailable^{3/2} n B}\vast]\vast) \\
      &=\cO\vast(\frac{1}{\varepsilon}\vast[\Delta_0\left(L + \frac{\omega}{\probavailable \sqrt{n}} \left(\widehat{L} + \frac{L_{\sigma}}{\sqrt{B}}\right) + \sqrt{\frac{\sigma^2}{\probavailable^2 \varepsilon n^2 B}}\left(\mathbbm{1}_{\probavailable}\widehat{L} + \frac{L_{\sigma}}{\sqrt{B}}\right)\right)\\
      &\qquad + \frac{\sigma^2}{\sqrt{\probavailable} n B}  + \frac{\varepsilon}{\sqrt{\probavailable}}\vast]\vast) \\
      &=\cO\vast(\frac{\Delta_0}{\varepsilon}\vast[L + \frac{\omega}{\probavailable \sqrt{n}} \left(\widehat{L} + \frac{L_{\sigma}}{\sqrt{B}}\right) + \sqrt{\frac{\sigma^2}{\probavailable^2 \varepsilon n^2 B}}\left(\mathbbm{1}_{\probavailable}\widehat{L} + \frac{L_{\sigma}}{\sqrt{B}}\right)\vast] + \frac{\sigma^2}{\sqrt{\probavailable} n \varepsilon B}  + \frac{1}{\sqrt{\probavailable}} \vast).
    \end{align*}

    Using $\frac{\sigma^2}{n \varepsilon B} \geq 1,$ we can conclude the proof of the inequality. The number of stochastic gradients that each node calculates equals $B_{\textnormal{init}} + 2BT = \cO(B_{\textnormal{init}} + BT).$
\end{proof}

\COROLLARYSTOCHASTICRANDK*

\begin{proof}
  The communication complexity equals
  \begin{eqnarray*}
      \cO\left(d + K T\right) &=& \cO\left(d + \frac{\Delta_0}{\varepsilon}\vast[K L + K\frac{\omega}{\probavailable \sqrt{n}} \left(\widehat{L} + \frac{L_{\sigma}}{\sqrt{B}}\right) + K \sqrt{\frac{\sigma^2}{\probavailable^2 \varepsilon n^2 B}}\left(\mathbbm{1}_{\probavailable}\widehat{L} + \frac{L_{\sigma}}{\sqrt{B}}\right)\vast] + K\frac{\sigma^2}{\sqrt{\probavailable} n \varepsilon B}\right).
  \end{eqnarray*}
  Due to $B \leq \frac{L_{\sigma}^2}{\mathbbm{1}_{\probavailable}^2\widehat{L}^2}, $ we have $\mathbbm{1}_{\probavailable}\widehat{L} + \frac{L_{\sigma}}{\sqrt{B}} \leq \frac{2L_{\sigma}}{\sqrt{B}}$ and
  \begin{eqnarray*}
    \cO\left(d + K T\right) &=& \cO\left(d + \frac{\Delta_0}{\varepsilon}\vast[K L + K\frac{\omega}{\probavailable \sqrt{n}} \left(\widehat{L} + \frac{L_{\sigma}}{\sqrt{B}}\right) + K \sqrt{\frac{\sigma^2}{\probavailable^2 \varepsilon n^2 B}}\frac{L_{\sigma}}{\sqrt{B}}\vast] + K\frac{\sigma^2}{\sqrt{\probavailable} n \varepsilon B}\right).
\end{eqnarray*}
  From Theorem~\ref{theorem:rand_k}, we have $\omega + 1 = \frac{d}{K}.$ Since $K = \Theta\left(\frac{B d \sqrt{\varepsilon n}}{\sigma}\right) = \cO\left(\frac{d}{\probavailable\sqrt{n}}\right),$ the communication complexity equals
  \begin{eqnarray*}
      \cO\left(d + K T\right) &=& \cO\left(d + \frac{\Delta_0}{\varepsilon}\vast[\frac{d}{\probavailable\sqrt{n}}L + \frac{d}{\probavailable \sqrt{n}} \left(\widehat{L} + \frac{L_{\sigma}}{\sqrt{B}}\right) + \frac{d}{\probavailable \sqrt{n}} L_{\sigma}\vast] + \frac{d \sigma}{\sqrt{\probavailable} \sqrt{n \varepsilon}}\right) \\
      &=& \cO\left(\frac{d \sigma}{\sqrt{\probavailable} \sqrt{n \varepsilon}} + \frac{L_{\sigma} \Delta_0 d}{\probavailable \sqrt{n} \varepsilon}\right) \\
  \end{eqnarray*}
  And the expected number of stochastic gradient calculations per node equals
  \begin{align*}
      &\cO\left(B_{\textnormal{init}} + B T\right) \\
      &= \cO\left(\frac{\sigma^2}{\sqrt{\probavailable} n \varepsilon} +  \frac{B \omega}{\sqrt{\probavailable}} \sqrt{\frac{\sigma^2}{n \varepsilon B}} + \frac{\Delta_0}{\varepsilon}\vast[B L + B\frac{\omega}{\probavailable \sqrt{n}} \left(\widehat{L} + \frac{L_{\sigma}}{\sqrt{B}}\right) + B \sqrt{\frac{\sigma^2}{\probavailable^2 \varepsilon n^2 B}}\left(\mathbbm{1}_{\probavailable}\widehat{L} + \frac{L_{\sigma}}{\sqrt{B}}\right)\vast] + B\frac{\sigma^2}{\sqrt{\probavailable} n \varepsilon B}\right) \\
      &= \cO\left(\frac{\sigma^2}{\sqrt{\probavailable} n \varepsilon} +  \frac{B d}{K \sqrt{\probavailable}} \sqrt{\frac{\sigma^2}{n \varepsilon B}} + \frac{\Delta_0}{\varepsilon}\vast[B L + B\frac{d}{K \probavailable \sqrt{n}} \left(\widehat{L} + \frac{L_{\sigma}}{\sqrt{B}}\right) + B \sqrt{\frac{\sigma^2}{\probavailable^2 \varepsilon n^2 B}}\frac{L_{\sigma}}{\sqrt{B}}\vast] + \frac{\sigma^2}{\sqrt{\probavailable} n \varepsilon}\right) \\
      &= \cO\left(\frac{\sigma^2}{\sqrt{\probavailable} n \varepsilon} + \frac{\sigma^2}{\sqrt{\probavailable} n \varepsilon \sqrt{B}} + \frac{\Delta_0}{\varepsilon}\vast[\frac{\sigma}{\probavailable\sqrt{\varepsilon} n} L + \frac{\sigma}{\probavailable \sqrt{\varepsilon} n} \left(\widehat{L} + \frac{L_{\sigma}}{\sqrt{B}}\right) + \frac{\sigma}{\probavailable \sqrt{\varepsilon} n}L_{\sigma}\vast]\right) \\
      &= \cO\left(\frac{\sigma^2}{\sqrt{\probavailable} n \varepsilon} + \frac{L_{\sigma} \Delta_0 \sigma}{\probavailable \varepsilon^{\nicefrac{3}{2}} n}\right).
  \end{align*}
\end{proof}

\newpage
\section{Analysis of \algname{\algorithmname} under Polyak-\L ojasiewicz Condition}
\label{sec:pl_condition}

In this section, we provide the theoretical convergence rates of \algname{\algorithmname} under Polyak-\L ojasiewiczc Condition.

\begin{assumption}
  \label{ass:pl_condition}
  The function $f$ satisfy (Polyak-\L ojasiewicz) P\L-condition:
  \begin{align}
      \label{eq:pl_condition}
      \norm{\nabla f(x)}^2 \geq 2\mu(f(x) - f^*), \quad \forall x \in \R,
  \end{align}
  where $f^* = \inf_{x \in \R^d} f(x) > -\infty.$
\end{assumption}

Under Polyak-\L ojasiewicz condition, a (random) point $\widehat{x}$ is $\varepsilon$-solution, if $\Exp{f(\widehat{x})} - f^* \leq \varepsilon.$

We now provide the convergence rates of \algname{\algorithmname} under P\L-condition.

\subsection{Gradient Setting}

\begin{restatable}{theorem}{CONVERGENCEPL}
  \label{theorem:gradient_oracle_pl}
  Suppose that Assumption \ref{ass:lower_bound}, \ref{ass:lipschitz_constant}, \ref{ass:nodes_lipschitz_constant}, \ref{ass:compressors}, \ref{ass:partial_participation} and \ref{ass:pl_condition} hold. Let us take $a = \frac{\probavailable}{2 \omega + 1},$ $b = \frac{\probavailable}{2 - \probavailable},$ 
  $$\gamma \leq \min\left\{\left(L + \sqrt{\frac{200 \omega \left(2 \omega + 1\right)}{n \probavailable^2} + \frac{48}{n \probavailable^2}\left(1 - \frac{\probpairaa}{\probavailable}\right)}\widehat{L}\right)^{-1}, \frac{a}{4\mu}\right\}, $$ 
  and $h^{0}_i = g^{0}_i = \nabla f_i(x^0)$ for all $i \in [n]$ in Algorithm~\ref{alg:main_algorithm} \algname{(\algorithmname)},
  then $\Exp{f(x^{T})} - f^* \leq (1 - \gamma \mu)^{T} \Delta_0.$
\end{restatable}

Let us provide bounds up to logarithmic factors and use $\widetilde{\cO}\left(\cdot\right)$ notation. The provided theorem states that to get $\varepsilon$-solution \algname{\algorithmname} have to run
\begin{align*}
  \widetilde{\cO}\left(\frac{\omega + 1}{\probavailable} + \frac{L}{\mu} + \frac{\omega \widehat{L}}{\probavailable \mu \sqrt{n}} + \frac{\widehat{L}}{\probavailable \mu \sqrt{n}}\right),
\end{align*}
communication rounds. The method \algname{DASHA} from \citep{tyurin2022dasha}, have to run
\begin{align*}
  \widetilde{\cO}\left(\omega + \frac{L}{\mu} + \frac{\omega \widehat{L}}{\mu \sqrt{n}}\right),
\end{align*}
communication rounds to get $\varepsilon$-solution. The difference is the same as in the general nonconvex case (see Section~\ref{sec:gradien_setting}). Up to Lipschitz constants factors, we get the degeneration up to $\nicefrac{1}{\probavailable}$ factor due to the partial participation.

\subsection{Finite-Sum Setting}

\begin{restatable}{theorem}{CONVERGENCEPLPAGE}
  \label{theorem:page_pl}
  Suppose that Assumption \ref{ass:lower_bound}, \ref{ass:lipschitz_constant}, \ref{ass:nodes_lipschitz_constant}, \ref{ass:compressors}, \ref{ass:max_lipschitz_constant}, \ref{ass:partial_participation}, and \ref{ass:pl_condition} hold. Let us take $a = \frac{\probavailable}{2 \omega + 1},$ probability $\probpage = \frac{B}{m + B}, $ $b = \frac{\probpage \probavailable}{2 - \probavailable},$
  {\scriptsize $$\gamma \leq \min\left\{\left(L + \sqrt{\frac{200 \omega (2 \omega + 1)}{n \probavailable^2} \left(\widehat{L}^2 + \frac{(1 - \probpage)L_{\max}^2}{B}\right) + \frac{48}{n \probavailable^2 \probpage}\left(\left(1 - \frac{\probpairaa}{\probavailable}\right) \widehat{L}^2 + \frac{(1 - \probpage)L_{\max}^2}{B}\right)}\right)^{-1}, \frac{a}{2\mu} , \frac{b}{2\mu}\right\},$$}
  and $h^{0}_i = g^{0}_i = \nabla f_i(x^0)$ for all $i \in [n]$ in Algorithm~\ref{alg:main_algorithm} \algname{(\algorithmname-PAGE)},
  then $\Exp{f(x^{T})} - f^* \leq (1 - \gamma \mu)^{T} \Delta_0.$
\end{restatable}

The provided theorem states that to get $\varepsilon$-solution \algname{\algorithmname} have to run
\begin{align*}
  \widetilde{\cO}\left(\frac{\omega + 1}{\probavailable} + \frac{m}{\probavailable B} + \frac{L}{\mu} + \frac{\omega}{\probavailable \mu \sqrt{n}} \left(\widehat{L} + \frac{L_{\max}}{\sqrt{B}}\right) + \frac{\sqrt{m}}{\probavailable \mu \sqrt{n B}} \left(\widehat{L} + \frac{L_{\max}}{\sqrt{B}}\right)\right),
\end{align*}
communication rounds. The method \algname{DASHA-PAGE} from \citep{tyurin2022dasha}, have to run
\begin{align*}
  \widetilde{\cO}\left(\omega + \frac{m}{B} + \frac{L}{\mu} + \frac{\omega}{\mu \sqrt{n}} \left(\widehat{L} + \frac{L_{\max}}{\sqrt{B}}\right) + \frac{\sqrt{m}}{\mu \sqrt{n B}} \left(\frac{L_{\max}}{\sqrt{B}}\right)\right),
\end{align*}
communication rounds to get $\varepsilon$-solution. We can guarantee the degeneration up to $\nicefrac{1}{\probavailable}$ factor due to the partial participation only if $B = \cO\left(\frac{L_{\max}^2}{\widehat{L}^2}\right)$. The same conclusion we have in Section~\ref{sec:finite_sum_setting}.

\subsection{Stochastic Setting}

\begin{restatable}{theorem}{CONVERGENCEPLSTOCHASTIC}
  \label{theorem:stochastic_pl}
  Suppose that Assumption \ref{ass:lower_bound}, \ref{ass:lipschitz_constant}, \ref{ass:nodes_lipschitz_constant}, \ref{ass:compressors}, \ref{ass:stochastic_unbiased_and_variance_bounded}, \ref{ass:mean_square_smoothness}, \ref{ass:partial_participation} and \ref{ass:pl_condition} hold. Let us take $a = \frac{\probavailable}{2\omega + 1}$, $b~\in~\left(0, \frac{\probavailable}{2 - \probavailable}\right],$
  $$\gamma \leq \min\left\{\left(L + \sqrt{\frac{200 \omega (2 \omega + 1)}{n \probavailable^2} \left(\frac{(1 - b)^2 L_{\sigma}^2}{B} + \widehat{L}^2\right) + \frac{40}{n \probavailable b}\left(\frac{(1 - b)^2 L_{\sigma}^2}{B} + \left(1 - \frac{\probpairaa}{\probavailable}\right) \widehat{L}^2 \right)}\right)^{-1}, \frac{a}{2\mu} , \frac{b}{2\mu}\right\},$$
  and $h^{0}_i = g^{0}_i$ for all $i \in [n]$ in Algorithm~\ref{alg:main_algorithm} \algname{(\algorithmname-MVR)},
  then 
  \begin{align*}
      &\Exp{f(x^{T}) - f^*} \\
      &\leq (1 - \gamma \mu)^{T}\left(\Delta_0 + \frac{2 \gamma}{b} \norm{h^{0} - \nabla f(x^{0})}^2 + \left(\frac{40 \gamma b \omega (2 \omega + 1)}{n \probavailable^2} + \frac{8 \gamma \left(\probavailable - \probpairaa\right)}{n \probavailable^2}\right)\frac{1}{n}\sum_{i=1}^n\norm{h^{0}_i - \nabla f_i(x^{0})}^2\right) \\
      &\quad + \frac{1}{\mu}\left(\frac{100 b^2 \omega (2 \omega + 1)}{\probavailable^2} + \frac{20 b}{\probavailable}\right) \frac{\sigma^2}{n B}.
  \end{align*}
\end{restatable}

The provided theorems states that to get $\varepsilon$-solution \algname{\algorithmname} have to run
\begin{align}
  \label{eq:pl_compare:new:mvr}
  \widetilde{\cO}\left(\frac{\omega + 1}{\probavailable} + \underbrace{\frac{\omega}{\probavailable}\sqrt{\frac{\sigma^2}{\mu n \varepsilon B}}}_{\mathcal{P}_2} + \frac{\sigma^2}{\probavailable \mu n \varepsilon B} + \frac{L}{\mu} + \frac{\omega}{\probavailable \mu \sqrt{n}} \left(\widehat{L} + \frac{L_{\sigma}}{\sqrt{B}}\right) + \underbrace{\frac{\sigma}{\probavailable n \mu^{\nicefrac{3}{2}} \sqrt{\varepsilon B}}\left(\widehat{L} + \frac{L_{\sigma}}{\sqrt{B}}\right)}_{\mathcal{P}_1}\right)
\end{align}
communication rounds. We take $b = \Theta\left(\min \left\{\ \frac{\probavailable}{\omega} \sqrt{\frac{\mu n \varepsilon B}{\sigma^2}}, \frac{\probavailable \mu n \varepsilon B}{\sigma^2}\right\}\right) \geq \Theta\left(\min \left\{\ \frac{\probavailable}{\omega^2}, \frac{\probavailable \mu n \varepsilon B}{\sigma^2}\right\}\right).$

The method \algname{DASHA-SYNC-MVR} from \citep{tyurin2022dasha}, have to run
\begin{align}
  \label{eq:pl_compare:old:mvr}
  \widetilde{\cO}\left(\omega + \frac{\sigma^2}{\mu n \varepsilon B} + \frac{L}{\mu} + \frac{\omega}{\mu \sqrt{n}} \left(\widehat{L} + \frac{L_{\sigma}}{\sqrt{B}}\right) + \frac{\sigma}{n \mu^{\nicefrac{3}{2}} \sqrt{\varepsilon B}}\left(\frac{L_{\sigma}}{\sqrt{B}}\right)\right)
\end{align}
communication rounds to get $\varepsilon$-solution\footnote{For simplicity, we omitted $\frac{d}{\zeta_{\cC}}$ term from the complexity in the stochastic setting, where $\zeta_{\cC}$ is defined in Definition~\ref{def:expected_density}. For instance, for the Rand$K$ compressor (see Definition~\ref{def:rand_k} and Theorem~\ref{theorem:rand_k}), $\zeta_{\cC} = K$ and $\frac{d}{\zeta_{\cC}} = \Theta\left(\omega\right).$}.

In the stochastic setting, the comparison is a little bit more complicated. As in the finite-sum setting, we have to take 
$B = \cO\left(\frac{L_{\sigma}^2}{\widehat{L}^2}\right)$ to guarantee the degeneration up to $\nicefrac{1}{\probavailable}$ of the term $\mathcal{P}_1$ from \eqref{eq:pl_compare:new:mvr}. However, \algname{\algorithmname-MVR} has also suboptimal term $\mathcal{P}_2$. This suboptimality is tightly connected with the suboptimality of $B_{\textnormal{init}}$ in the general nonconvex case, which we discuss in Section~\ref{sec:stochastic_setting}, and it also appears in the analysis of \algname{DASHA-MVR} \citep{tyurin2022dasha}. Let us provide the counterpart of Corollary~\ref{cor:stochastic:randk}. The corollary reveals that we can escape regimes when $\mathcal{P}_2$ is the bottleneck by choosing the parameters of the compressors.

\begin{restatable}{corollary}{CONVERGENCEPLSTOCHASTICRANDK}
  \label{cor:stochastic:pl:randk}
  Suppose that assumptions of Theorem~\ref{theorem:stochastic_pl} hold, 
  batch size $B \leq \min\left\{\frac{\sigma}{\probavailable\sqrt{\mu \varepsilon} n}, \frac{L_{\sigma}^2}{\widehat{L}^2}\right\},$ 
  we take Rand$K$ compressors with $K = \Theta\left(\frac{B d \sqrt{\mu \varepsilon n}}{\sigma}\right).$ Then
  the communication complexity equals 
  \begin{align*}
    \widetilde{\cO}\left(\frac{d \sigma}{\probavailable \sqrt{\mu \varepsilon n}} + \frac{d L_{\sigma}}{\probavailable \mu \sqrt{n}}\right),
  \end{align*}
  and the expected number of stochastic gradient calculations per node equals
  \begin{align*}
    \widetilde{\cO}\left(\frac{\sigma^2}{\probavailable \mu n \varepsilon} + \frac{\sigma L_{\sigma}}{\probavailable n \mu^{\nicefrac{3}{2}} \sqrt{\varepsilon}}\right).
  \end{align*}
\end{restatable}
Up to Lipschitz constants, \algname{\algorithmname-MVR} has the state-of-the-art oracle complexity under P\L-condition (see \citep{PAGE}). Moreover, \algname{\algorithmname-MVR} has the state-of-the-art communication complexity of \algname{DASHA} for a small enough $\mu$.

\subsection{Proofs of Theorems}
The following proofs almost repeat the proofs from Section~\ref{sec:proof_of_theorems}. And one of the main changes is that instead of Lemma~\ref{lemma:good_recursion}, we use the following lemma.

\subsubsection{Standard Lemma under Polyak-\L ojasiewicz Condition}

\begin{lemma}
  \label{lemma:good_recursion_pl}
  Suppose that Assumptions \ref{ass:lower_bound} and \ref{ass:pl_condition} hold and
  \begin{align*}
      \Exp{f(x^{t+1})} + \gamma \Psi^{t+1} \leq \Exp{f(x^t)} - \frac{\gamma}{2}\Exp{\norm{\nabla f(x^t)}^2} + (1 - \gamma \mu)\gamma \Psi^{t} + \gamma C,
  \end{align*}
  where $\Psi^{t}$ is a sequence of numbers, $\Psi^{t} \geq 0$ for all $t \in [T]$, constant $C \geq 0,$ constant $\mu > 0,$ and constant $\gamma \in (0, 1 / \mu).$ Then 
  \begin{align}
      \label{eq:good_recursion_pl}
      \Exp{f(x^{T}) - f^*} \leq (1 - \gamma \mu)^{T}\left(\left(f(x^0) - f^*\right) + \gamma \Psi^{0}\right) + \frac{C}{\mu}.
  \end{align}
\end{lemma}

\begin{proof}
  We subtract $f^*$ and use P\L-condition \eqref{eq:pl_condition} to get
  \begin{eqnarray*}
      \Exp{f(x^{t+1}) - f^*} + \gamma \Psi^{t+1} &\leq& \Exp{f(x^t) - f^*} - \frac{\gamma}{2}\Exp{\norm{\nabla f(x^t)}^2} + \gamma \Psi^{t} + \gamma C \\
      &\leq& (1 - \gamma \mu)\Exp{f(x^t) - f^*} + (1 - \gamma \mu)\gamma \Psi^{t} + \gamma C \\
      &=& (1 - \gamma \mu)\left(\Exp{f(x^t) - f^*} + \gamma \Psi^{t}\right) + \gamma C.
  \end{eqnarray*}
  Unrolling the inequality, we have
  \begin{eqnarray*}
      \Exp{f(x^{t+1}) - f^*} + \gamma \Psi^{t+1} &\leq& (1 - \gamma \mu)^{t+1}\left(\left(f(x^0) - f^*\right) + \gamma \Psi^{0}\right) + \gamma C \sum_{i = 0}^{t} (1 - \gamma \mu)^i \\
      &\leq& (1 - \gamma \mu)^{t+1}\left(\left(f(x^0) - f^*\right) + \gamma \Psi^{0}\right) + \frac{C}{\mu}.
  \end{eqnarray*}
  It is left to note that $\Psi^{t} \geq 0$ for all $t \in [T]$.
\end{proof}

\subsubsection{Generic Lemma}
We now provide the counterpart of Lemma~\ref{lemma:main_lemma}.

\begin{lemma}
  \label{lemma:main_lemma_pl}
  Suppose that Assumptions \ref{ass:lipschitz_constant}, \ref{ass:compressors}, \ref{ass:partial_participation} and \ref{ass:pl_condition} hold and let us take $a = \frac{\probavailable}{2 \omega + 1},$ then
  \begin{align*}
    &\Exp{f(x^{t + 1})} + \frac{2\gamma (2\omega + 1)}{\probavailable} \Exp{\norm{g^{t+1} - h^{t+1}}^2} + \frac{4 \gamma (\left(2 \omega + 1\right)\probavailable - \probpairaa)}{n \probavailable^2} \Exp{\frac{1}{n}\sum_{i=1}^n\norm{g^{t+1}_i - h^{t+1}_i}^2}\\
    &\leq \Exp{f(x^t) - \frac{\gamma}{2}\norm{\nabla f(x^t)}^2 - \left(\frac{1}{2\gamma} - \frac{L}{2}\right)
    \norm{x^{t+1} - x^t}^2 + \gamma \norm{h^{t} - \nabla f(x^t)}^2}\nonumber\\
    &\quad + \left(1 - \gamma \mu\right)\frac{2\gamma(2\omega + 1)}{\probavailable}\Exp{\norm{g^{t} - h^t}^2} + \left(1 - \gamma \mu\right)\frac{4 \gamma (\left(2 \omega + 1\right)\probavailable - \probpairaa)}{n \probavailable^2}\Exp{\frac{1}{n} \sum_{i=1}^n\norm{g^t_i - h^{t}_i}^2}\nonumber\\
    &\quad + \frac{10 \gamma (2 \omega + 1)\omega}{n \probavailable^2}\Exp{\frac{1}{n} \sum_{i=1}^n\norm{k^{t+1}_i}^2}.
  \end{align*}
\end{lemma}

\begin{proof}
  Let us fix some constants $\kappa, \eta \in [0,\infty)$ that we will define later. Using the same reasoning as in Lemma~\ref{lemma:main_lemma}, we can get
  \begin{align*}
      &\Exp{f(x^{t + 1})} \nonumber \\
      &\quad  + \kappa \Exp{\norm{g^{t+1} - h^{t+1}}^2} + \eta \Exp{\frac{1}{n}\sum_{i=1}^n\norm{g^{t+1}_i - h^{t+1}_i}^2} \nonumber\\
      &\leq \Exp{f(x^t) - \frac{\gamma}{2}\norm{\nabla f(x^t)}^2 - \left(\frac{1}{2\gamma} - \frac{L}{2}\right)
      \norm{x^{t+1} - x^t}^2 + \gamma \norm{h^{t} - \nabla f(x^t)}^2}\nonumber\\
      &\quad + \left(\gamma + \kappa \left(1 - a\right)^2\right)\Exp{\norm{g^{t} - h^t}^2}\nonumber\\
      &\quad + \left(\frac{\kappa a^2 (\left(2 \omega + 1\right)\probavailable - \probpairaa)}{n \probavailable^2} + \eta\left(\frac{a^2(2\omega + 1 - \probavailable)}{\probavailable} + (1 - a)^2\right)\right)\Exp{\frac{1}{n} \sum_{i=1}^n\norm{g^t_i - h^{t}_i}^2}\nonumber\\
      &\quad + \left(\frac{2 \kappa \omega}{n \probavailable} + \frac{2 \eta \omega}{\probavailable}\right)\Exp{\frac{1}{n} \sum_{i=1}^n\norm{k^{t+1}_i}^2}.
  \end{align*}
  Let us take $\kappa = \frac{2\gamma}{a}.$ One can show that $\gamma + \kappa \left(1 - a\right)^2 \leq \left(1 - \frac{a}{2}\right)\kappa, $ and thus
  \begin{align*}
      &\Exp{f(x^{t + 1})} \\
      &\quad  + \frac{2\gamma}{a} \Exp{\norm{g^{t+1} - h^{t+1}}^2} + \eta \Exp{\frac{1}{n}\sum_{i=1}^n\norm{g^{t+1}_i - h^{t+1}_i}^2}\\
      &\leq \Exp{f(x^t) - \frac{\gamma}{2}\norm{\nabla f(x^t)}^2 - \left(\frac{1}{2\gamma} - \frac{L}{2}\right)
      \norm{x^{t+1} - x^t}^2 + \gamma \norm{h^{t} - \nabla f(x^t)}^2}\nonumber\\
      &\quad + \left(1 - \frac{a}{2}\right)\frac{2\gamma}{a}\Exp{\norm{g^{t} - h^t}^2}\nonumber\\
      &\quad + \left(\frac{2\gamma a (\left(2 \omega + 1\right)\probavailable - \probpairaa)}{n \probavailable^2} + \eta\left(\frac{a^2(2\omega + 1 - \probavailable)}{\probavailable} + (1 - a)^2\right)\right)\Exp{\frac{1}{n} \sum_{i=1}^n\norm{g^t_i - h^{t}_i}^2}\nonumber\\
      &\quad + \left(\frac{4 \gamma \omega}{a n \probavailable} + \frac{2 \eta \omega}{\probavailable}\right)\Exp{\frac{1}{n} \sum_{i=1}^n\norm{k^{t+1}_i}^2}.
  \end{align*}
  Considering the choice of $a$, one can show that $\left(\frac{a^2(2\omega + 1 - \probavailable)}{\probavailable} + (1 - a)^2\right) \leq 1 - a.$ If we take $\eta = \frac{4 \gamma (\left(2 \omega + 1\right)\probavailable - \probpairaa)}{n \probavailable^2},$ then $\left(\frac{2\gamma a (\left(2 \omega + 1\right)\probavailable - \probpairaa)}{n \probavailable^2} + \eta\left(\frac{a^2(2\omega + 1 - \probavailable)}{\probavailable} + (1 - a)^2\right)\right) \leq \left(1 - \frac{a}{2}\right)\eta$ and

  \begin{align*}
    &\Exp{f(x^{t + 1})} \\
    &\quad  + \frac{2\gamma (2\omega + 1)}{\probavailable} \Exp{\norm{g^{t+1} - h^{t+1}}^2} + \frac{4 \gamma (\left(2 \omega + 1\right)\probavailable - \probpairaa)}{n \probavailable^2} \Exp{\frac{1}{n}\sum_{i=1}^n\norm{g^{t+1}_i - h^{t+1}_i}^2}\\
    &\leq \Exp{f(x^t) - \frac{\gamma}{2}\norm{\nabla f(x^t)}^2 - \left(\frac{1}{2\gamma} - \frac{L}{2}\right)
    \norm{x^{t+1} - x^t}^2 + \gamma \norm{h^{t} - \nabla f(x^t)}^2}\nonumber\\
    &\quad + \left(1 - \frac{a}{2}\right)\frac{2\gamma(2\omega + 1)}{\probavailable}\Exp{\norm{g^{t} - h^t}^2} + \left(1 - \frac{a}{2}\right)\frac{4 \gamma (\left(2 \omega + 1\right)\probavailable - \probpairaa)}{n \probavailable^2}\Exp{\frac{1}{n} \sum_{i=1}^n\norm{g^t_i - h^{t}_i}^2}\nonumber\\
    &\quad + \left(\frac{2 \gamma (2 \omega + 1)\omega}{n \probavailable^2} + \frac{8 \gamma (\left(2 \omega + 1\right)\probavailable - \probpairaa) \omega}{n \probavailable^3}\right)\Exp{\frac{1}{n} \sum_{i=1}^n\norm{k^{t+1}_i}^2} \\
    &\leq \Exp{f(x^t) - \frac{\gamma}{2}\norm{\nabla f(x^t)}^2 - \left(\frac{1}{2\gamma} - \frac{L}{2}\right)
    \norm{x^{t+1} - x^t}^2 + \gamma \norm{h^{t} - \nabla f(x^t)}^2}\nonumber\\
    &\quad + \left(1 - \frac{a}{2}\right)\frac{2\gamma(2\omega + 1)}{\probavailable}\Exp{\norm{g^{t} - h^t}^2} + \left(1 - \frac{a}{2}\right)\frac{4 \gamma (\left(2 \omega + 1\right)\probavailable - \probpairaa)}{n \probavailable^2}\Exp{\frac{1}{n} \sum_{i=1}^n\norm{g^t_i - h^{t}_i}^2}\nonumber\\
    &\quad + \frac{10 \gamma (2 \omega + 1)\omega}{n \probavailable^2}\Exp{\frac{1}{n} \sum_{i=1}^n\norm{k^{t+1}_i}^2}.
  \end{align*}
  It it left to consider that $\gamma \leq \frac{a}{2\mu},$ and therefore $1 - \frac{a}{2} \leq 1 - \gamma \mu.$
\end{proof}

\subsubsection{Proof for \algname{\algorithmname} under P\L-condition}

\CONVERGENCEPL*

\begin{proof}
  Let us fix constants $\nu, \rho \in [0,\infty)$ that we will define later. Considering Lemma~\ref{lemma:main_lemma_pl}, Lemma~\ref{lemma:gradient}, and the law of total expectation, we obtain
    \begin{align*}
      &\Exp{f(x^{t + 1})} + \frac{2\gamma (2\omega + 1)}{\probavailable} \Exp{\norm{g^{t+1} - h^{t+1}}^2} + \frac{4 \gamma (\left(2 \omega + 1\right)\probavailable - \probpairaa)}{n \probavailable^2} \Exp{\frac{1}{n}\sum_{i=1}^n\norm{g^{t+1}_i - h^{t+1}_i}^2}\\
      &\quad  + \nu \Exp{\norm{h^{t+1} - \nabla f(x^{t+1})}^2} + \rho \Exp{\frac{1}{n}\sum_{i=1}^n\norm{h^{t+1}_i - \nabla f_i(x^{t+1})}^2}\\
      &\leq \Exp{f(x^t) - \frac{\gamma}{2}\norm{\nabla f(x^t)}^2 - \left(\frac{1}{2\gamma} - \frac{L}{2}\right)
      \norm{x^{t+1} - x^t}^2 + \gamma \norm{h^{t} - \nabla f(x^t)}^2}\nonumber\\
      &\quad + \left(1 - \gamma \mu\right)\frac{2\gamma(2\omega + 1)}{\probavailable}\Exp{\norm{g^{t} - h^t}^2} + \left(1 - \gamma \mu\right)\frac{4 \gamma (\left(2 \omega + 1\right)\probavailable - \probpairaa)}{n \probavailable^2}\Exp{\frac{1}{n} \sum_{i=1}^n\norm{g^t_i - h^{t}_i}^2} \\
      &\quad + \frac{10 \gamma \omega (2 \omega + 1)}{n \probavailable^2} \Exp{2\widehat{L}^2\norm{x^{t+1} - x^{t}}^2 + 2b^2 \frac{1}{n}\sum_{i=1}^n \norm{h^t_i - \nabla f_i(x^{t})}^2} \\
      &\quad + \nu \Exp{\frac{2\left(\probavailable - \probpairaa\right)\widehat{L}^2}{n \probavailable^2} \norm{x^{t+1} - x^{t}}^2 + \frac{2 b^2 \left(\probavailable - \probpairaa\right)}{n^2 \probavailable^2} \sum_{i=1}^n \norm{h^t_i - \nabla f_i(x^{t})}^2 + \left(1 - b\right)^2 \norm{h^{t} - \nabla f(x^{t})}^2} \\
      &\quad + \rho \Exp{\frac{2(1 - \probavailable)}{\probavailable} \widehat{L}^2 \norm{x^{t+1} - x^{t}}^2 + \left(\frac{2 b^2 (1 - \probavailable)}{\probavailable} + (1 - b)^2\right) \frac{1}{n}\sum_{i=1}^n \norm{h^{t}_i - \nabla f_i(x^{t})}^2}.
    \end{align*}
    After rearranging the terms, we get
    \begin{align*}
      &\Exp{f(x^{t + 1})} + \frac{2\gamma (2\omega + 1)}{\probavailable} \Exp{\norm{g^{t+1} - h^{t+1}}^2} + \frac{4 \gamma (\left(2 \omega + 1\right)\probavailable - \probpairaa)}{n \probavailable^2} \Exp{\frac{1}{n}\sum_{i=1}^n\norm{g^{t+1}_i - h^{t+1}_i}^2}\\
      &\quad  + \nu \Exp{\norm{h^{t+1} - \nabla f(x^{t+1})}^2} + \rho \Exp{\frac{1}{n}\sum_{i=1}^n\norm{h^{t+1}_i - \nabla f_i(x^{t+1})}^2}\\
      &\leq \Exp{f(x^t)} - \frac{\gamma}{2}\Exp{\norm{\nabla f(x^t)}^2} \\
      &\quad + \left(1 - \gamma \mu\right)\frac{2\gamma(2\omega + 1)}{\probavailable}\Exp{\norm{g^{t} - h^t}^2} + \left(1 - \gamma \mu\right)\frac{4 \gamma (\left(2 \omega + 1\right)\probavailable - \probpairaa)}{n \probavailable^2}\Exp{\frac{1}{n} \sum_{i=1}^n\norm{g^t_i - h^{t}_i}^2} \\
      &\quad - \left(\frac{1}{2\gamma} - \frac{L}{2} - \frac{20 \gamma \omega \left(2 \omega + 1\right) \widehat{L}^2}{n \probavailable^2} - \nu \frac{2\left(\probavailable - \probpairaa\right)\widehat{L}^2}{n \probavailable^2} - \rho \frac{2(1 - \probavailable) \widehat{L}^2}{\probavailable} \right) \Exp{\norm{x^{t+1} - x^t}^2} \\
      &\quad + \left(\gamma + \nu (1 - b)^2\right) \Exp{\norm{h^{t} - \nabla f(x^{t})}^2} \\
      &\quad + \left(\frac{20 b^2 \gamma \omega (2 \omega + 1)}{n \probavailable^2} + \nu \frac{2 b^2 \left(\probavailable - \probpairaa\right)}{n \probavailable^2} + \rho \left(\frac{2 b^2 (1 - \probavailable)}{\probavailable} + (1 - b)^2\right) \right)\Exp{\frac{1}{n}\sum_{i=1}^n\norm{h^{t}_i - \nabla f_i(x^{t})}^2}.
    \end{align*}
    By taking $\nu = \frac{2\gamma}{b},$ one can show that $\left(\gamma + \nu (1 - b)^2\right) \leq \left(1 - \frac{b}{2}\right)\nu,$ and
    \begin{align*}
      &\Exp{f(x^{t + 1})} + \frac{2\gamma (2\omega + 1)}{\probavailable} \Exp{\norm{g^{t+1} - h^{t+1}}^2} + \frac{4 \gamma (\left(2 \omega + 1\right)\probavailable - \probpairaa)}{n \probavailable^2} \Exp{\frac{1}{n}\sum_{i=1}^n\norm{g^{t+1}_i - h^{t+1}_i}^2}\\
      &\quad  + \frac{2\gamma}{b} \Exp{\norm{h^{t+1} - \nabla f(x^{t+1})}^2} + \rho \Exp{\frac{1}{n}\sum_{i=1}^n\norm{h^{t+1}_i - \nabla f_i(x^{t+1})}^2}\\
      &\leq \Exp{f(x^t)} - \frac{\gamma}{2}\Exp{\norm{\nabla f(x^t)}^2} \\
      &\quad + \left(1 - \gamma \mu\right)\frac{2\gamma(2\omega + 1)}{\probavailable}\Exp{\norm{g^{t} - h^t}^2} + \left(1 - \gamma \mu\right)\frac{4 \gamma (\left(2 \omega + 1\right)\probavailable - \probpairaa)}{n \probavailable^2}\Exp{\frac{1}{n} \sum_{i=1}^n\norm{g^t_i - h^{t}_i}^2} \\
      &\quad - \left(\frac{1}{2\gamma} - \frac{L}{2} - \frac{20 \gamma \omega \left(2 \omega + 1\right) \widehat{L}^2}{n \probavailable^2} - \frac{4\gamma\left(\probavailable - \probpairaa\right)\widehat{L}^2}{b n \probavailable^2} - \rho \frac{2(1 - \probavailable) \widehat{L}^2}{\probavailable} \right) \Exp{\norm{x^{t+1} - x^t}^2} \\
      &\quad + \left(1 - \frac{b}{2}\right)\frac{2\gamma}{b} \Exp{\norm{h^{t} - \nabla f(x^{t})}^2} \\
      &\quad + \left(\frac{20 b^2 \gamma \omega (2 \omega + 1)}{n \probavailable^2} + \frac{4 \gamma b \left(\probavailable - \probpairaa\right)}{n \probavailable^2} + \rho \left(\frac{2 b^2 (1 - \probavailable)}{\probavailable} + (1 - b)^2\right) \right)\Exp{\frac{1}{n}\sum_{i=1}^n\norm{h^{t}_i - \nabla f_i(x^{t})}^2}.
    \end{align*}
    Note that $b = \frac{\probavailable}{2 - \probavailable},$ thus
    \begin{align*}
      &\left(\frac{20 b^2 \gamma \omega (2 \omega + 1)}{n \probavailable^2} + \frac{4 \gamma b \left(\probavailable - \probpairaa\right)}{n \probavailable^2} + \rho \left(\frac{2 b^2 (1 - \probavailable)}{\probavailable} + (1 - b)^2\right) \right) \\
      &\leq \left(\frac{20 b^2 \gamma \omega (2 \omega + 1)}{n \probavailable^2} + \frac{4 \gamma b \left(\probavailable - \probpairaa\right)}{n \probavailable^2} + \rho \left(1 - b\right) \right).
    \end{align*}
    And if we take $\rho = \frac{40 b \gamma \omega (2 \omega + 1)}{n \probavailable^2} + \frac{8 \gamma \left(\probavailable - \probpairaa\right)}{n \probavailable^2},$ then
    \begin{align*}
      \left(\frac{20 b^2 \gamma \omega (2 \omega + 1)}{n \probavailable^2} + \frac{4 \gamma b \left(\probavailable - \probpairaa\right)}{n \probavailable^2} + \rho \left(1 - b\right) \right) \leq \left(1 - \frac{b}{2}\right)\rho,
    \end{align*}
    and 
    \begin{align*}
      &\Exp{f(x^{t + 1})} + \frac{2\gamma (2 \omega + 1)}{\probavailable} \Exp{\norm{g^{t+1} - h^{t+1}}^2} + \frac{4\gamma (\left(2 \omega + 1\right)\probavailable - \probpairaa)}{n \probavailable^2} \Exp{\frac{1}{n}\sum_{i=1}^n\norm{g^{t+1}_i - h^{t+1}_i}^2}\\
      &\quad  + \frac{2\gamma}{b} \Exp{\norm{h^{t+1} - \nabla f(x^{t+1})}^2} + \left(\frac{40 b \gamma \omega (2 \omega + 1)}{n \probavailable^2} + \frac{8 \gamma \left(\probavailable - \probpairaa\right)}{n \probavailable^2}\right) \Exp{\frac{1}{n}\sum_{i=1}^n\norm{h^{t+1}_i - \nabla f_i(x^{t+1})}^2}\\
      &\leq \Exp{f(x^t)} - \frac{\gamma}{2}\Exp{\norm{\nabla f(x^t)}^2} \\
      &\quad + \left(1 - \gamma \mu\right)\frac{2\gamma (2 \omega + 1)}{\probavailable} \Exp{\norm{g^{t} - h^{t}}^2} + \left(1 - \gamma \mu\right)\frac{4\gamma (\left(2 \omega + 1\right)\probavailable - \probpairaa)}{n \probavailable^2} \Exp{\frac{1}{n}\sum_{i=1}^n\norm{g^{t}_i - h^{t}_i}^2} \\
      &\quad - \Bigg(\frac{1}{2\gamma} - \frac{L}{2} - \frac{20 \gamma \omega \left(2 \omega + 1\right) \widehat{L}^2}{n \probavailable^2} - \frac{4\gamma \left(\probavailable - \probpairaa\right)\widehat{L}^2}{b n \probavailable^2} \\
      &\quad\qquad - \frac{80 b \gamma \omega (2 \omega + 1) (1 - \probavailable) \widehat{L}^2}{n \probavailable^3} - \frac{16 \gamma \left(\probavailable - \probpairaa\right) (1 - \probavailable) \widehat{L}^2}{n \probavailable^3} \Bigg) \Exp{\norm{x^{t+1} - x^t}^2} \\
      &\quad + \left(1 - \frac{b}{2}\right)\frac{2\gamma}{b} \Exp{\norm{h^{t} - \nabla f(x^{t})}^2} + \left(1 - \frac{b}{2}\right)\left(\frac{40 b \gamma \omega (2 \omega + 1)}{n \probavailable^2} + \frac{8 \gamma \left(\probavailable - \probpairaa\right)}{n \probavailable^2}\right)\Exp{\frac{1}{n}\sum_{i=1}^n\norm{h^{t}_i - \nabla f_i(x^{t})}^2}.
    \end{align*}
    Due to $\frac{\probavailable}{2} \leq b \leq \probavailable,$ we have
    \begin{align*}
      &\Exp{f(x^{t + 1})} + \frac{2\gamma (2 \omega + 1)}{\probavailable} \Exp{\norm{g^{t+1} - h^{t+1}}^2} + \frac{4\gamma (\left(2 \omega + 1\right)\probavailable - \probpairaa)}{n \probavailable^2} \Exp{\frac{1}{n}\sum_{i=1}^n\norm{g^{t+1}_i - h^{t+1}_i}^2}\\
      &\quad  + \frac{2\gamma}{b} \Exp{\norm{h^{t+1} - \nabla f(x^{t+1})}^2} + \left(\frac{40 b \gamma \omega (2 \omega + 1)}{n \probavailable^2} + \frac{8 \gamma \left(\probavailable - \probpairaa\right)}{n \probavailable^2}\right) \Exp{\frac{1}{n}\sum_{i=1}^n\norm{h^{t+1}_i - \nabla f_i(x^{t+1})}^2}\\
      &\leq \Exp{f(x^t)} - \frac{\gamma}{2}\Exp{\norm{\nabla f(x^t)}^2} \\
      &\quad + \left(1 - \gamma \mu\right)\frac{2\gamma (2 \omega + 1)}{\probavailable} \Exp{\norm{g^{t} - h^{t}}^2} + \left(1 - \gamma \mu\right)\frac{4\gamma (\left(2 \omega + 1\right)\probavailable - \probpairaa)}{n \probavailable^2} \Exp{\frac{1}{n}\sum_{i=1}^n\norm{g^{t}_i - h^{t}_i}^2} \\
      &\quad - \Bigg(\frac{1}{2\gamma} - \frac{L}{2} - \frac{100 \gamma \omega \left(2 \omega + 1\right) \widehat{L}^2}{n \probavailable^2} - \frac{24\gamma \left(\probavailable - \probpairaa\right)\widehat{L}^2}{n \probavailable^3}\Bigg) \Exp{\norm{x^{t+1} - x^t}^2} \\
      &\quad + \left(1 - \frac{b}{2}\right)\frac{2\gamma}{b} \Exp{\norm{h^{t} - \nabla f(x^{t})}^2} + \left(1 - \frac{b}{2}\right)\left(\frac{40 b \gamma \omega (2 \omega + 1)}{n \probavailable^2} + \frac{8 \gamma \left(\probavailable - \probpairaa\right)}{n \probavailable^2}\right)\Exp{\frac{1}{n}\sum_{i=1}^n\norm{h^{t}_i - \nabla f_i(x^{t})}^2}.
    \end{align*}
    Using Lemma~\ref{lemma:gamma} and the assumption about $\gamma,$ we get
    \begin{align*}
      &\Exp{f(x^{t + 1})} + \frac{2\gamma (2 \omega + 1)}{\probavailable} \Exp{\norm{g^{t+1} - h^{t+1}}^2} + \frac{4\gamma (\left(2 \omega + 1\right)\probavailable - \probpairaa)}{n \probavailable^2} \Exp{\frac{1}{n}\sum_{i=1}^n\norm{g^{t+1}_i - h^{t+1}_i}^2}\\
      &\quad  + \frac{2\gamma}{b} \Exp{\norm{h^{t+1} - \nabla f(x^{t+1})}^2} + \left(\frac{40 b \gamma \omega (2 \omega + 1)}{n \probavailable^2} + \frac{8 \gamma \left(\probavailable - \probpairaa\right)}{n \probavailable^2}\right) \Exp{\frac{1}{n}\sum_{i=1}^n\norm{h^{t+1}_i - \nabla f_i(x^{t+1})}^2}\\
      &\leq \Exp{f(x^t)} - \frac{\gamma}{2}\Exp{\norm{\nabla f(x^t)}^2} \\
      &\quad + \left(1 - \gamma \mu\right)\frac{2\gamma (2 \omega + 1)}{\probavailable} \Exp{\norm{g^{t} - h^{t}}^2} + \left(1 - \gamma \mu\right)\frac{4\gamma (\left(2 \omega + 1\right)\probavailable - \probpairaa)}{n \probavailable^2} \Exp{\frac{1}{n}\sum_{i=1}^n\norm{g^{t}_i - h^{t}_i}^2} \\
      &\quad + \left(1 - \frac{b}{2}\right)\frac{2\gamma}{b} \Exp{\norm{h^{t} - \nabla f(x^{t})}^2} + \left(1 - \frac{b}{2}\right)\left(\frac{40 b \gamma \omega (2 \omega + 1)}{n \probavailable^2} + \frac{8 \gamma \left(\probavailable - \probpairaa\right)}{n \probavailable^2}\right)\Exp{\frac{1}{n}\sum_{i=1}^n\norm{h^{t}_i - \nabla f_i(x^{t})}^2}.
    \end{align*}
    Note that $\gamma \leq \frac{a}{4\mu} \leq \frac{\probavailable}{4\mu} \leq \frac{b}{2\mu},$ thus $1 - \frac{b}{2} \leq 1 - \gamma \mu$ and 
    \begin{align*}
      &\Exp{f(x^{t + 1})} + \frac{2\gamma (2 \omega + 1)}{\probavailable} \Exp{\norm{g^{t+1} - h^{t+1}}^2} + \frac{4\gamma (\left(2 \omega + 1\right)\probavailable - \probpairaa)}{n \probavailable^2} \Exp{\frac{1}{n}\sum_{i=1}^n\norm{g^{t+1}_i - h^{t+1}_i}^2}\\
      &\quad  + \frac{2\gamma}{b} \Exp{\norm{h^{t+1} - \nabla f(x^{t+1})}^2} + \left(\frac{40 b \gamma \omega (2 \omega + 1)}{n \probavailable^2} + \frac{8 \gamma \left(\probavailable - \probpairaa\right)}{n \probavailable^2}\right) \Exp{\frac{1}{n}\sum_{i=1}^n\norm{h^{t+1}_i - \nabla f_i(x^{t+1})}^2}\\
      &\leq \Exp{f(x^t)} - \frac{\gamma}{2}\Exp{\norm{\nabla f(x^t)}^2} \\
      &\quad + \left(1 - \gamma \mu\right)\frac{2\gamma (2 \omega + 1)}{\probavailable} \Exp{\norm{g^{t} - h^{t}}^2} + \left(1 - \gamma \mu\right)\frac{4\gamma (\left(2 \omega + 1\right)\probavailable - \probpairaa)}{n \probavailable^2} \Exp{\frac{1}{n}\sum_{i=1}^n\norm{g^{t}_i - h^{t}_i}^2} \\
      &\quad + \left(1 - \gamma \mu\right)\frac{2\gamma}{b} \Exp{\norm{h^{t} - \nabla f(x^{t})}^2} + \left(1 - \gamma \mu\right)\left(\frac{40 b \gamma \omega (2 \omega + 1)}{n \probavailable^2} + \frac{8 \gamma \left(\probavailable - \probpairaa\right)}{n \probavailable^2}\right)\Exp{\frac{1}{n}\sum_{i=1}^n\norm{h^{t}_i - \nabla f_i(x^{t})}^2}.
    \end{align*}
    In the view of Lemma~\ref{lemma:good_recursion_pl} with
    \begin{eqnarray*}
        \Psi^t &=& \frac{2(2 \omega + 1)}{\probavailable} \Exp{\norm{g^{t} - h^{t}}^2} + \frac{4(\left(2 \omega + 1\right)\probavailable - \probpairaa)}{n \probavailable^2} \Exp{\frac{1}{n}\sum_{i=1}^n\norm{g^{t}_i - h^{t}_i}^2} \\
        &\quad +& \frac{2}{b} \Exp{\norm{h^{t} - \nabla f(x^{t})}^2} + \left(\frac{40 b \omega (2 \omega + 1)}{n \probavailable^2} + \frac{8 \gamma \left(\probavailable - \probpairaa\right)}{n \probavailable^2}\right)\Exp{\frac{1}{n}\sum_{i=1}^n\norm{h^{t}_i - \nabla f_i(x^{t})}^2},
    \end{eqnarray*}
    we can conclude the proof of the theorem.
\end{proof}

\subsubsection{Proof for \algname{\algorithmname-PAGE} under P\L-condition}

\CONVERGENCEPLPAGE*

\begin{proof}
  Let us fix constants $\nu, \rho \in [0,\infty)$ that we will define later. Considering Lemma~\ref{lemma:main_lemma_pl}, Lemma~\ref{lemma:gradient_page}, and the law of total expectation, we obtain
    \begin{align*}
      &\Exp{f(x^{t + 1})} + \frac{2\gamma (2\omega + 1)}{\probavailable} \Exp{\norm{g^{t+1} - h^{t+1}}^2} + \frac{4 \gamma (\left(2 \omega + 1\right)\probavailable - \probpairaa)}{n \probavailable^2} \Exp{\frac{1}{n}\sum_{i=1}^n\norm{g^{t+1}_i - h^{t+1}_i}^2}\\
      &\quad  + \nu \Exp{\norm{h^{t+1} - \nabla f(x^{t+1})}^2} + \rho \Exp{\frac{1}{n}\sum_{i=1}^n\norm{h^{t+1}_i - \nabla f_i(x^{t+1})}^2}\\
      &\leq \Exp{f(x^t) - \frac{\gamma}{2}\norm{\nabla f(x^t)}^2 - \left(\frac{1}{2\gamma} - \frac{L}{2}\right)
      \norm{x^{t+1} - x^t}^2 + \gamma \norm{h^{t} - \nabla f(x^t)}^2}\nonumber\\
      &\quad + \left(1 - \gamma \mu\right)\frac{2\gamma(2\omega + 1)}{\probavailable}\Exp{\norm{g^{t} - h^t}^2} + \left(1 - \gamma \mu\right)\frac{4 \gamma (\left(2 \omega + 1\right)\probavailable - \probpairaa)}{n \probavailable^2}\Exp{\frac{1}{n} \sum_{i=1}^n\norm{g^t_i - h^{t}_i}^2}\nonumber\\
      &\quad + \frac{10 \gamma (2 \omega + 1)\omega}{n \probavailable^2}\Exp{\frac{1}{n} \sum_{i=1}^n\norm{k^{t+1}_i}^2} \\
      &\quad  + \nu \Exp{\norm{h^{t+1} - \nabla f(x^{t+1})}^2} + \rho \Exp{\frac{1}{n}\sum_{i=1}^n\norm{h^{t+1}_i - \nabla f_i(x^{t+1})}^2}\\
      &\leq \Exp{f(x^t) - \frac{\gamma}{2}\norm{\nabla f(x^t)}^2 - \left(\frac{1}{2\gamma} - \frac{L}{2}\right)
      \norm{x^{t+1} - x^t}^2 + \gamma \norm{h^{t} - \nabla f(x^t)}^2}\nonumber\\
      &\quad + \left(1 - \gamma \mu\right)\frac{2\gamma(2\omega + 1)}{\probavailable}\Exp{\norm{g^{t} - h^t}^2} + \left(1 - \gamma \mu\right)\frac{4 \gamma (\left(2 \omega + 1\right)\probavailable - \probpairaa)}{n \probavailable^2}\Exp{\frac{1}{n} \sum_{i=1}^n\norm{g^t_i - h^{t}_i}^2}\nonumber\\
      &\quad + \frac{10 \gamma (2 \omega + 1)\omega}{n \probavailable^2}\Exp{\left(2 \widehat{L}^2 + \frac{(1 - \probpage)L_{\max}^2}{B}\right)\norm{x^{t+1} - x^{t}}^2 +  \frac{2b^2}{\probpage} \frac{1}{n}\sum_{i=1}^n \norm{h^t_i - \nabla f_i(x^{t})}^2} \\
      &\quad  + \nu {\rm E}\Bigg(\left(\frac{2 \left(\probavailable - \probpairaa\right) \widehat{L}^2}{n \probavailable^2} + \frac{(1 - \probpage)L_{\max}^2}{n \probavailable B}\right) \norm{x^{t+1} - x^{t}}^2\\
      &\qquad\quad + \frac{2\left(\probavailable - \probpairaa\right) b^2}{n^2 \probavailable^2 \probpage}\sum_{i=1}^n\norm{ h^t_i - \nabla f_i(x^{t})}^2 + \left(\probpage\left(1 - \frac{b}{\probpage}\right)^2 + (1 - \probpage)\right)\norm{h^{t} - \nabla f(x^{t})}^2\Bigg) \\
      &\quad + \rho {\rm E}\Bigg(\left(\frac{2\left(1 - \probavailable\right)\widehat{L}^2}{\probavailable} + \frac{(1 - \probpage)L_{\max}^2}{\probavailable B}\right) \norm{x^{t+1} - x^{t}}^2 \\
      &\qquad\quad +\left(\frac{2\left(1 - \probavailable\right)b^2}{\probavailable \probpage} + \probpage\left(1 - \frac{b}{\probpage}\right)^2 + (1 - \probpage)\right)\frac{1}{n}\sum_{i=1}^n\norm{h^{t}_i - \nabla f_i(x^{t})}^2\Bigg).
    \end{align*}
    After rearranging the terms, we get
    \begin{align*}
      &\Exp{f(x^{t + 1})} + \frac{2\gamma (2\omega + 1)}{\probavailable} \Exp{\norm{g^{t+1} - h^{t+1}}^2} + \frac{4 \gamma (\left(2 \omega + 1\right)\probavailable - \probpairaa)}{n \probavailable^2} \Exp{\frac{1}{n}\sum_{i=1}^n\norm{g^{t+1}_i - h^{t+1}_i}^2}\\
      &\quad  + \nu \Exp{\norm{h^{t+1} - \nabla f(x^{t+1})}^2} + \rho \Exp{\frac{1}{n}\sum_{i=1}^n\norm{h^{t+1}_i - \nabla f_i(x^{t+1})}^2}\\
      &\leq \Exp{f(x^t)} - \frac{\gamma}{2}\Exp{\norm{\nabla f(x^t)}^2} \\
      &\quad + \left(1 - \gamma \mu\right)\frac{2\gamma(2\omega + 1)}{\probavailable}\Exp{\norm{g^{t} - h^t}^2} + \left(1 - \gamma \mu\right)\frac{4 \gamma (\left(2 \omega + 1\right)\probavailable - \probpairaa)}{n \probavailable^2}\Exp{\frac{1}{n} \sum_{i=1}^n\norm{g^t_i - h^{t}_i}^2}\\
      &\quad - \Bigg(\frac{1}{2\gamma} - \frac{L}{2} - \frac{10 \gamma \omega (2 \omega + 1)}{n \probavailable^2} \left(2 \widehat{L}^2 + \frac{(1 - \probpage)L_{\max}^2}{B}\right) \\
      &\qquad\quad - \nu\left(\frac{2 \left(\probavailable - \probpairaa\right) \widehat{L}^2}{n \probavailable^2} + \frac{(1 - \probpage)L_{\max}^2}{n \probavailable B}\right) - \rho \left(\frac{2\left(1 - \probavailable\right)\widehat{L}^2}{\probavailable} + \frac{(1 - \probpage)L_{\max}^2}{\probavailable B}\right)\Bigg) \Exp{\norm{x^{t+1} - x^t}^2} \\
      &\quad + \left(\gamma + \nu \left(\probpage\left(1 - \frac{b}{\probpage}\right)^2 + (1 - \probpage)\right)\right) \Exp{\norm{h^{t} - \nabla f(x^{t})}^2} \\
      &\quad + \Bigg(\frac{20 b^2 \gamma \omega (2 \omega + 1)}{n \probavailable^2 \probpage} + \frac{2 \nu \left(\probavailable - \probpairaa\right) b^2}{n \probavailable^2 \probpage} \\
      &\qquad\quad+ \rho \left(\frac{2\left(1 - \probavailable\right)b^2}{\probavailable \probpage} + \probpage\left(1 - \frac{b}{\probpage}\right)^2 + (1 - \probpage)\right)\Bigg)\Exp{\frac{1}{n}\sum_{i=1}^n\norm{h^{t}_i - \nabla f_i(x^{t})}^2}.
    \end{align*}
    Due to $b = \frac{\probpage \probavailable}{2 - \probavailable} \leq \probpage,$ one can show that $\left(\probpage\left(1 - \frac{b}{\probpage}\right)^2 + (1 - \probpage)\right) \leq 1 - b.$ 
    Thus, if we take $\nu = \frac{2\gamma}{b},$ then
    $$\left(\gamma + \nu \left(\probpage\left(1 - \frac{b}{\probpage}\right)^2 + (1 - \probpage)\right)\right) \leq \gamma + \nu (1 - b) = \left(1 - \frac{b}{2}\right)\nu,$$ therefore
    \begin{align*}
      &\Exp{f(x^{t + 1})} + \frac{2\gamma (2\omega + 1)}{\probavailable} \Exp{\norm{g^{t+1} - h^{t+1}}^2} + \frac{4 \gamma (\left(2 \omega + 1\right)\probavailable - \probpairaa)}{n \probavailable^2} \Exp{\frac{1}{n}\sum_{i=1}^n\norm{g^{t+1}_i - h^{t+1}_i}^2}\\
      &\quad  + \frac{2\gamma}{b} \Exp{\norm{h^{t+1} - \nabla f(x^{t+1})}^2} + \rho \Exp{\frac{1}{n}\sum_{i=1}^n\norm{h^{t+1}_i - \nabla f_i(x^{t+1})}^2}\\
      &\leq \Exp{f(x^t)} - \frac{\gamma}{2}\Exp{\norm{\nabla f(x^t)}^2} \\
      &\quad + \left(1 - \gamma \mu\right)\frac{2\gamma(2\omega + 1)}{\probavailable}\Exp{\norm{g^{t} - h^t}^2} + \left(1 - \gamma \mu\right)\frac{4 \gamma (\left(2 \omega + 1\right)\probavailable - \probpairaa)}{n \probavailable^2}\Exp{\frac{1}{n} \sum_{i=1}^n\norm{g^t_i - h^{t}_i}^2}\\
      &\quad - \Bigg(\frac{1}{2\gamma} - \frac{L}{2} - \frac{10 \gamma \omega (2 \omega + 1)}{n \probavailable^2} \left(2 \widehat{L}^2 + \frac{(1 - \probpage)L_{\max}^2}{B}\right) \\
      &\qquad\quad - \frac{2\gamma}{b n \probavailable}\left(2 \left(1 - \frac{\probpairaa}{\probavailable}\right) \widehat{L}^2 + \frac{(1 - \probpage)L_{\max}^2}{B}\right) - \rho \left(\frac{2\left(1 - \probavailable\right)\widehat{L}^2}{\probavailable} + \frac{(1 - \probpage)L_{\max}^2}{\probavailable B}\right)\Bigg) \Exp{\norm{x^{t+1} - x^t}^2} \\
      &\quad + \left(1 - \frac{b}{2}\right)\frac{2\gamma}{b} \Exp{\norm{h^{t} - \nabla f(x^{t})}^2} \\
      &\quad + \Bigg(\frac{20 b^2 \gamma \omega (2 \omega + 1)}{n \probavailable^2 \probpage} + \frac{4 \gamma \left(\probavailable - \probpairaa\right) b}{n \probavailable^2 \probpage} \\
      &\qquad\quad+ \rho \left(\frac{2\left(1 - \probavailable\right)b^2}{\probavailable \probpage} + \probpage\left(1 - \frac{b}{\probpage}\right)^2 + (1 - \probpage)\right)\Bigg)\Exp{\frac{1}{n}\sum_{i=1}^n\norm{h^{t}_i - \nabla f_i(x^{t})}^2}.
    \end{align*}
    Next, with the choice of $b = \frac{\probpage \probavailable}{2 - \probavailable},$ we ensure that
    $$\left(\frac{2\left(1 - \probavailable\right)b^2}{\probavailable \probpage} + \probpage\left(1 - \frac{b}{\probpage}\right)^2 + (1 - \probpage)\right) \leq 1 - b.$$ If we take $\rho = \frac{40 b \gamma \omega (2 \omega + 1)}{n \probavailable^2 \probpage} + \frac{8 \gamma \left(\probavailable - \probpairaa\right)}{n \probavailable^2 \probpage},$ then
    $$\Bigg(\frac{20 b^2 \gamma \omega (2 \omega + 1)}{n \probavailable^2 \probpage} + \frac{4 \gamma \left(\probavailable - \probpairaa\right) b}{n \probavailable^2 \probpage} + \rho \left(\frac{2\left(1 - \probavailable\right)b^2}{\probavailable \probpage} + \probpage\left(1 - \frac{b}{\probpage}\right)^2 + (1 - \probpage)\right)\Bigg) \leq \left(1 - \frac{b}{2}\right)\rho,$$ therefore
    \begin{align*}
      &\Exp{f(x^{t + 1})} + \frac{2\gamma (2 \omega + 1)}{\probavailable} \Exp{\norm{g^{t+1} - h^{t+1}}^2} + \frac{4 \gamma (\left(2 \omega + 1\right)\probavailable - \probpairaa)}{n \probavailable^2} \Exp{\frac{1}{n}\sum_{i=1}^n\norm{g^{t+1}_i - h^{t+1}_i}^2}\\
      &\quad  + \frac{2\gamma}{b} \Exp{\norm{h^{t+1} - \nabla f(x^{t+1})}^2} + \left(\frac{40 b \gamma \omega (2 \omega + 1)}{n \probavailable^2 \probpage} + \frac{8 \gamma \left(\probavailable - \probpairaa\right)}{n \probavailable^2 \probpage}\right) \Exp{\frac{1}{n}\sum_{i=1}^n\norm{h^{t+1}_i - \nabla f_i(x^{t+1})}^2}\\
      &\leq \Exp{f(x^t)} - \frac{\gamma}{2}\Exp{\norm{\nabla f(x^t)}^2} \\
      &\quad + \left(1 - \gamma \mu\right)\frac{2 \gamma (2 \omega + 1)}{\probavailable} \Exp{\norm{g^{t} - h^{t}}^2} + \left(1 - \gamma \mu\right)\frac{4 \gamma (\left(2 \omega + 1\right)\probavailable - \probpairaa)}{n \probavailable^2} \Exp{\frac{1}{n}\sum_{i=1}^n\norm{g^{t}_i - h^{t}_i}^2} \\
      &\quad - \Bigg(\frac{1}{2\gamma} - \frac{L}{2} - \frac{10 \gamma \omega (2 \omega + 1)}{n \probavailable^2} \left(2 \widehat{L}^2 + \frac{(1 - \probpage)L_{\max}^2}{B}\right) \\
      &\qquad\quad - \frac{2 \gamma}{b n \probavailable}\left(2 \left(1 - \frac{\probpairaa}{\probavailable}\right) \widehat{L}^2 + \frac{(1 - \probpage)L_{\max}^2}{B}\right) \\
      &\qquad\quad- \left(\frac{40 b \gamma \omega (2 \omega + 1)}{n \probavailable^3 \probpage} + \frac{8 \gamma \left(1 - \frac{\probpairaa}{\probavailable}\right)}{n \probavailable^2 \probpage}\right) \left(2\left(1 - \probavailable\right)\widehat{L}^2 + \frac{(1 - \probpage)L_{\max}^2}{B}\right)\Bigg) \Exp{\norm{x^{t+1} - x^t}^2} \\
      &\quad + \left(1 - \frac{b}{2}\right)\frac{2\gamma}{b}\Exp{\norm{h^{t} - \nabla f(x^{t})}^2} + \left(1 - \frac{b}{2}\right)\left(\frac{40 b \gamma \omega (2 \omega + 1)}{n \probavailable^2 \probpage} + \frac{8 \gamma \left(\probavailable - \probpairaa\right)}{n \probavailable^2 \probpage}\right)\Exp{\frac{1}{n}\sum_{i=1}^n\norm{h^{t}_i - \nabla f_i(x^{t})}^2}.
    \end{align*}
    Let us simplify the inequality. First, due to $b \geq \frac{\probpage \probavailable}{2},$ we have
    $$\frac{2\gamma}{b n \probavailable}\left(2 \left(1 - \frac{\probpairaa}{\probavailable}\right) \widehat{L}^2 + \frac{(1 - \probpage)L_{\max}^2}{B}\right) \leq \frac{8 \gamma}{n \probavailable^2 \probpage}\left(\left(1 - \frac{\probpairaa}{\probavailable}\right) \widehat{L}^2 + \frac{(1 - \probpage)L_{\max}^2}{B}\right).$$
    Second, due to $b \leq \probavailable \probpage$ and $\probpairaa \leq \probavailable^2$, we get
    \begin{align*}
      &\left(\frac{40 b \gamma \omega (2 \omega + 1)}{n \probavailable^3 \probpage} + \frac{8 \gamma \left(1 - \frac{\probpairaa}{\probavailable}\right)}{n \probavailable^2 \probpage}\right) \left(2\left(1 - \probavailable\right)\widehat{L}^2 + \frac{(1 - \probpage)L_{\max}^2}{B}\right) \\
      &\leq \left(\frac{40 \gamma \omega (2 \omega + 1)}{n \probavailable^2} + \frac{8 \gamma \left(1 - \frac{\probpairaa}{\probavailable}\right)}{n \probavailable^2 \probpage}\right) \left(2\left(1 - \frac{\probpairaa}{\probavailable}\right)\widehat{L}^2 + \frac{(1 - \probpage)L_{\max}^2}{B}\right) \\
      &\leq \frac{80 \gamma \omega (2 \omega + 1)}{n \probavailable^2} \left(\left(1 - \frac{\probpairaa}{\probavailable}\right)\widehat{L}^2 + \frac{(1 - \probpage)L_{\max}^2}{B}\right) \\
      &\quad + \frac{16 \gamma \left(1 - \frac{\probpairaa}{\probavailable}\right)}{n \probavailable^2 \probpage} \left(\left(1 - \frac{\probpairaa}{\probavailable}\right)\widehat{L}^2 + \frac{(1 - \probpage)L_{\max}^2}{B}\right) \\
      &\leq \frac{80 \gamma \omega (2 \omega + 1)}{n \probavailable^2} \left(\widehat{L}^2 + \frac{(1 - \probpage)L_{\max}^2}{B}\right) \\
      &\quad + \frac{16 \gamma}{n \probavailable^2 \probpage} \left(\left(1 - \frac{\probpairaa}{\probavailable}\right)\widehat{L}^2 + \frac{(1 - \probpage)L_{\max}^2}{B}\right).
    \end{align*}
    Combining all bounds together, we obtain the following inequality:
    \begin{align*}
      &\Exp{f(x^{t + 1})} + \frac{2\gamma (2 \omega + 1)}{\probavailable} \Exp{\norm{g^{t+1} - h^{t+1}}^2} + \frac{4 \gamma (\left(2 \omega + 1\right)\probavailable - \probpairaa)}{n \probavailable^2} \Exp{\frac{1}{n}\sum_{i=1}^n\norm{g^{t+1}_i - h^{t+1}_i}^2}\\
      &\quad  + \frac{2\gamma}{b} \Exp{\norm{h^{t+1} - \nabla f(x^{t+1})}^2} + \left(\frac{40 b \gamma \omega (2 \omega + 1)}{n \probavailable^2 \probpage} + \frac{8 \gamma \left(\probavailable - \probpairaa\right)}{n \probavailable^2 \probpage}\right) \Exp{\frac{1}{n}\sum_{i=1}^n\norm{h^{t+1}_i - \nabla f_i(x^{t+1})}^2}\\
      &\leq \Exp{f(x^t)} - \frac{\gamma}{2}\Exp{\norm{\nabla f(x^t)}^2} \\
      &\quad + \left(1 - \gamma \mu\right)\frac{2 \gamma (2 \omega + 1)}{\probavailable} \Exp{\norm{g^{t} - h^{t}}^2} + \left(1 - \gamma \mu\right)\frac{4 \gamma (\left(2 \omega + 1\right)\probavailable - \probpairaa)}{n \probavailable^2} \Exp{\frac{1}{n}\sum_{i=1}^n\norm{g^{t}_i - h^{t}_i}^2} \\
      &\quad - \Bigg(\frac{1}{2\gamma} - \frac{L}{2} - \frac{100 \gamma \omega (2 \omega + 1)}{n \probavailable^2} \left(\widehat{L}^2 + \frac{(1 - \probpage)L_{\max}^2}{B}\right) \\
      &\qquad\quad - \frac{24 \gamma}{n \probavailable^2 \probpage}\left(\left(1 - \frac{\probpairaa}{\probavailable}\right) \widehat{L}^2 + \frac{(1 - \probpage)L_{\max}^2}{B}\right)\Bigg) \Exp{\norm{x^{t+1} - x^t}^2} \\
      &\quad + \left(1 - \frac{b}{2}\right)\frac{2\gamma}{b}\Exp{\norm{h^{t} - \nabla f(x^{t})}^2} + \left(1 - \frac{b}{2}\right)\left(\frac{40 b \gamma \omega (2 \omega + 1)}{n \probavailable^2 \probpage} + \frac{8 \gamma \left(\probavailable - \probpairaa\right)}{n \probavailable^2 \probpage}\right)\Exp{\frac{1}{n}\sum_{i=1}^n\norm{h^{t}_i - \nabla f_i(x^{t})}^2}.
    \end{align*}
    Using Lemma~\ref{lemma:gamma} and the assumption about $\gamma,$ we get
    \begin{align*}
      &\Exp{f(x^{t + 1})} + \frac{2\gamma (2 \omega + 1)}{\probavailable} \Exp{\norm{g^{t+1} - h^{t+1}}^2} + \frac{4 \gamma (\left(2 \omega + 1\right)\probavailable - \probpairaa)}{n \probavailable^2} \Exp{\frac{1}{n}\sum_{i=1}^n\norm{g^{t+1}_i - h^{t+1}_i}^2}\\
      &\quad  + \frac{2\gamma}{b} \Exp{\norm{h^{t+1} - \nabla f(x^{t+1})}^2} + \left(\frac{40 b \gamma \omega (2 \omega + 1)}{n \probavailable^2 \probpage} + \frac{8 \gamma \left(\probavailable - \probpairaa\right)}{n \probavailable^2 \probpage}\right) \Exp{\frac{1}{n}\sum_{i=1}^n\norm{h^{t+1}_i - \nabla f_i(x^{t+1})}^2}\\
      &\leq \Exp{f(x^t)} - \frac{\gamma}{2}\Exp{\norm{\nabla f(x^t)}^2} \\
      &\quad + \left(1 - \gamma \mu\right)\frac{2 \gamma (2 \omega + 1)}{\probavailable} \Exp{\norm{g^{t} - h^{t}}^2} + \left(1 - \gamma \mu\right)\frac{4 \gamma (\left(2 \omega + 1\right)\probavailable - \probpairaa)}{n \probavailable^2} \Exp{\frac{1}{n}\sum_{i=1}^n\norm{g^{t}_i - h^{t}_i}^2} \\
      &\quad + \left(1 - \frac{b}{2}\right)\frac{2\gamma}{b}\Exp{\norm{h^{t} - \nabla f(x^{t})}^2} + \left(1 - \frac{b}{2}\right)\left(\frac{40 b \gamma \omega (2 \omega + 1)}{n \probavailable^2 \probpage} + \frac{8 \gamma \left(\probavailable - \probpairaa\right)}{n \probavailable^2 \probpage}\right)\Exp{\frac{1}{n}\sum_{i=1}^n\norm{h^{t}_i - \nabla f_i(x^{t})}^2}.
    \end{align*}
    Note that $\gamma \leq \frac{b}{2\mu},$ thus $1 - \frac{b}{2} \leq 1 - \gamma \mu$ and 
    \begin{align*}
      &\Exp{f(x^{t + 1})} + \frac{2\gamma (2 \omega + 1)}{\probavailable} \Exp{\norm{g^{t+1} - h^{t+1}}^2} + \frac{4 \gamma (\left(2 \omega + 1\right)\probavailable - \probpairaa)}{n \probavailable^2} \Exp{\frac{1}{n}\sum_{i=1}^n\norm{g^{t+1}_i - h^{t+1}_i}^2}\\
      &\quad  + \frac{2\gamma}{b} \Exp{\norm{h^{t+1} - \nabla f(x^{t+1})}^2} + \left(\frac{40 b \gamma \omega (2 \omega + 1)}{n \probavailable^2 \probpage} + \frac{8 \gamma \left(\probavailable - \probpairaa\right)}{n \probavailable^2 \probpage}\right) \Exp{\frac{1}{n}\sum_{i=1}^n\norm{h^{t+1}_i - \nabla f_i(x^{t+1})}^2}\\
      &\leq \Exp{f(x^t)} - \frac{\gamma}{2}\Exp{\norm{\nabla f(x^t)}^2} \\
      &\quad + \left(1 - \gamma \mu\right)\frac{2 \gamma (2 \omega + 1)}{\probavailable} \Exp{\norm{g^{t} - h^{t}}^2} + \left(1 - \gamma \mu\right)\frac{4 \gamma (\left(2 \omega + 1\right)\probavailable - \probpairaa)}{n \probavailable^2} \Exp{\frac{1}{n}\sum_{i=1}^n\norm{g^{t}_i - h^{t}_i}^2} \\
      &\quad + \left(1 - \gamma \mu\right)\frac{2\gamma}{b}\Exp{\norm{h^{t} - \nabla f(x^{t})}^2} + \left(1 - \gamma \mu\right)\left(\frac{40 b \gamma \omega (2 \omega + 1)}{n \probavailable^2 \probpage} + \frac{8 \gamma \left(\probavailable - \probpairaa\right)}{n \probavailable^2 \probpage}\right)\Exp{\frac{1}{n}\sum_{i=1}^n\norm{h^{t}_i - \nabla f_i(x^{t})}^2}.
    \end{align*}
    It is left to apply Lemma~\ref{lemma:good_recursion_pl} with 
    \begin{align*}
      \Psi^t &= \frac{2 (2 \omega + 1)}{\probavailable} \Exp{\norm{g^{t} - h^{t}}^2} + \frac{4 (\left(2 \omega + 1\right)\probavailable - \probpairaa)}{n \probavailable^2} \Exp{\frac{1}{n}\sum_{i=1}^n\norm{g^{t}_i - h^{t}_i}^2} \\
      &+ \frac{2}{b}\Exp{\norm{h^{t} - \nabla f(x^{t})}^2} + \left(\frac{40 b \omega (2 \omega + 1)}{n \probavailable^2 \probpage} + \frac{8 \left(\probavailable - \probpairaa\right)}{n \probavailable^2 \probpage}\right)\Exp{\frac{1}{n}\sum_{i=1}^n\norm{h^{t}_i - \nabla f_i(x^{t})}^2}
    \end{align*}
    to conclude the proof.
  \end{proof}

  \subsubsection{Proof for \algname{\algorithmname-MVR} under P\L-condition}

  \CONVERGENCEPLSTOCHASTIC*

  \begin{align*}
    &\Exp{f(x^{t + 1})} + \frac{2\gamma (2\omega + 1)}{\probavailable} \Exp{\norm{g^{t+1} - h^{t+1}}^2} + \frac{4 \gamma (\left(2 \omega + 1\right)\probavailable - \probpairaa)}{n \probavailable^2} \Exp{\frac{1}{n}\sum_{i=1}^n\norm{g^{t+1}_i - h^{t+1}_i}^2}\\
    &\leq \Exp{f(x^t) - \frac{\gamma}{2}\norm{\nabla f(x^t)}^2 - \left(\frac{1}{2\gamma} - \frac{L}{2}\right)
    \norm{x^{t+1} - x^t}^2 + \gamma \norm{h^{t} - \nabla f(x^t)}^2}\nonumber\\
    &\quad + \left(1 - \gamma \mu\right)\frac{2\gamma(2\omega + 1)}{\probavailable}\Exp{\norm{g^{t} - h^t}^2} + \left(1 - \gamma \mu\right)\frac{4 \gamma (\left(2 \omega + 1\right)\probavailable - \probpairaa)}{n \probavailable^2}\Exp{\frac{1}{n} \sum_{i=1}^n\norm{g^t_i - h^{t}_i}^2}\nonumber\\
    &\quad + \frac{10 \gamma (2 \omega + 1)\omega}{n \probavailable^2}\Exp{\frac{1}{n} \sum_{i=1}^n\norm{k^{t+1}_i}^2}.
  \end{align*}

  \begin{proof}
    Let us fix constants $\nu, \rho \in [0,\infty)$ that we will define later. Considering Lemma~\ref{lemma:main_lemma_pl}, Lemma~\ref{lemma:gradient_mvr}, and the law of total expectation, we obtain
      \begin{align*}
        &\Exp{f(x^{t + 1})} + \frac{2\gamma (2\omega + 1)}{\probavailable} \Exp{\norm{g^{t+1} - h^{t+1}}^2} + \frac{4 \gamma (\left(2 \omega + 1\right)\probavailable - \probpairaa)}{n \probavailable^2} \Exp{\frac{1}{n}\sum_{i=1}^n\norm{g^{t+1}_i - h^{t+1}_i}^2}\\
        &\quad  + \nu \Exp{\norm{h^{t+1} - \nabla f(x^{t+1})}^2} + \rho \Exp{\frac{1}{n}\sum_{i=1}^n\norm{h^{t+1}_i - \nabla f_i(x^{t+1})}^2}\\
        &\leq \Exp{f(x^t) - \frac{\gamma}{2}\norm{\nabla f(x^t)}^2 - \left(\frac{1}{2\gamma} - \frac{L}{2}\right)
        \norm{x^{t+1} - x^t}^2 + \gamma \norm{h^{t} - \nabla f(x^t)}^2}\nonumber\\
        &\quad + \left(1 - \gamma \mu\right)\frac{2\gamma(2\omega + 1)}{\probavailable}\Exp{\norm{g^{t} - h^t}^2} + \left(1 - \gamma \mu\right)\frac{4 \gamma (\left(2 \omega + 1\right)\probavailable - \probpairaa)}{n \probavailable^2}\Exp{\frac{1}{n} \sum_{i=1}^n\norm{g^t_i - h^{t}_i}^2}\nonumber\\
        &\quad + \frac{10 \gamma (2 \omega + 1)\omega}{n \probavailable^2}\Exp{\frac{1}{n} \sum_{i=1}^n\norm{k^{t+1}_i}^2} \\
        &\quad  + \nu \Exp{\norm{h^{t+1} - \nabla f(x^{t+1})}^2} + \rho \Exp{\frac{1}{n}\sum_{i=1}^n\norm{h^{t+1}_i - \nabla f_i(x^{t+1})}^2}\\
        &\leq \Exp{f(x^t) - \frac{\gamma}{2}\norm{\nabla f(x^t)}^2 - \left(\frac{1}{2\gamma} - \frac{L}{2}\right)
        \norm{x^{t+1} - x^t}^2 + \gamma \norm{h^{t} - \nabla f(x^t)}^2}\nonumber\\
        &\quad + \left(1 - \gamma \mu\right)\frac{2 \gamma (2 \omega + 1)}{\probavailable}\Exp{\norm{g^{t} - h^t}^2}+ \left(1 - \gamma \mu\right)\frac{4\gamma (\left(2 \omega + 1\right)\probavailable - \probpairaa)}{n \probavailable^2}\Exp{\frac{1}{n} \sum_{i=1}^n\norm{g^t_i - h^{t}_i}^2} \\
        &\quad + \frac{10 \gamma \omega (2 \omega + 1)}{n \probavailable^2} \Exp{\frac{2 b^2 \sigma^2}{B} + \left(\frac{2 (1 - b)^2 L_{\sigma}^2}{B} + 2 \widehat{L}^2\right)\norm{x^{t+1} - x^{t}}^2 + 2 b^2 \frac{1}{n} \sum_{i=1}^n \norm{h^t_i - \nabla f_i(x^{t})}^2} \\
        &\quad  + \nu {\rm E}\Bigg(\frac{2 b^2 \sigma^2}{n \probavailable B} + \left(\frac{2 (1 - b)^2 L_{\sigma}^2}{n \probavailable B} + \frac{2\left(\probavailable - \probpairaa\right) \widehat{L}^2}{n \probavailable^2} \right) \norm{x^{t+1} - x^{t}}^2\\
        &\qquad\quad + \frac{2 \left(\probavailable - \probpairaa\right) b^2}{n^2 \probavailable^2} \sum_{i=1}^n \norm{h^t_i - \nabla f_i(x^{t})}^2 + \left(1 - b\right)^2 \norm{h^{t} - \nabla f(x^{t})}^2\Bigg) \\
        &\quad + \rho {\rm E}\Bigg(\frac{2 b^2 \sigma^2}{\probavailable B}  + \left(\frac{2 (1 - b)^2 L_{\sigma}^2}{\probavailable B} + \frac{2(1 - \probavailable) \widehat{L}^2}{\probavailable}\right)\norm{x^{t+1} - x^{t}}^2 \\
        &\qquad\quad + \left(\frac{2 (1 - \probavailable) b^2}{\probavailable} + (1 - b)^2\right) \frac{1}{n} \sum_{i=1}^n \norm{h^t_i - \nabla f_i(x^{t})}^2\Bigg).
      \end{align*}
      After rearranging the terms, we get
      \begin{align*}
        &\Exp{f(x^{t + 1})} + \left(1 - \gamma \mu\right)\frac{2\gamma (2 \omega + 1)}{\probavailable} \Exp{\norm{g^{t+1} - h^{t+1}}^2} + \left(1 - \gamma \mu\right)\frac{4\gamma (\left(2 \omega + 1\right)\probavailable - \probpairaa)}{n \probavailable^2} \Exp{\frac{1}{n}\sum_{i=1}^n\norm{g^{t+1}_i - h^{t+1}_i}^2}\\
        &\quad  + \nu \Exp{\norm{h^{t+1} - \nabla f(x^{t+1})}^2} + \rho \Exp{\frac{1}{n}\sum_{i=1}^n\norm{h^{t+1}_i - \nabla f_i(x^{t+1})}^2}\\
        &\leq \Exp{f(x^t)} - \frac{\gamma}{2}\Exp{\norm{\nabla f(x^t)}^2} \\
        &\quad + \left(1 - \gamma \mu\right)\frac{2\gamma (2 \omega + 1)}{\probavailable} \Exp{\norm{g^{t} - h^{t}}^2} + \left(1 - \gamma \mu\right)\frac{4\gamma (\left(2 \omega + 1\right)\probavailable - \probpairaa)}{n \probavailable^2} \Exp{\frac{1}{n}\sum_{i=1}^n\norm{g^{t}_i - h^{t}_i}^2} \\
        &\quad - \Bigg(\frac{1}{2\gamma} - \frac{L}{2} - \frac{10 \gamma \omega (2 \omega + 1)}{n \probavailable^2} \left(\frac{2 (1 - b)^2 L_{\sigma}^2}{B} + 2 \widehat{L}^2\right) \\
        &\qquad\quad - \nu\left(\frac{2 (1 - b)^2 L_{\sigma}^2}{n \probavailable B} + \frac{2\left(\probavailable - \probpairaa\right) \widehat{L}^2}{n \probavailable^2} \right) - \rho \left(\frac{2 (1 - b)^2 L_{\sigma}^2}{\probavailable B} + \frac{2(1 - \probavailable) \widehat{L}^2}{\probavailable}\right)\Bigg) \Exp{\norm{x^{t+1} - x^t}^2} \\
        &\quad + \left(\gamma + \nu \left(1 - b\right)^2\right) \Exp{\norm{h^{t} - \nabla f(x^{t})}^2} \\
        &\quad + \Bigg(\frac{20 b^2 \gamma \omega (2 \omega + 1)}{n \probavailable^2} + \frac{2 \nu \left(\probavailable - \probpairaa\right) b^2}{n \probavailable^2} + \rho \left(\frac{2 (1 - \probavailable) b^2}{\probavailable} + (1 - b)^2\right)\Bigg)\Exp{\frac{1}{n}\sum_{i=1}^n\norm{h^{t}_i - \nabla f_i(x^{t})}^2} \\
        &\quad + \left(\frac{20 b^2 \gamma \omega (2 \omega + 1)}{n \probavailable^2} + \nu \frac{2 b^2}{n \probavailable} + \rho \frac{2 b^2}{\probavailable}\right) \frac{\sigma^2}{B}.
      \end{align*}
      By taking $\nu = \frac{2\gamma}{b},$ one can show that $\left(\gamma + \nu (1 - b)^2\right) \leq \left(1 - \frac{b}{2}\right)\nu,$ and
      \begin{align*}
        &\Exp{f(x^{t + 1})} + \left(1 - \gamma \mu\right)\frac{2\gamma (2 \omega + 1)}{\probavailable} \Exp{\norm{g^{t+1} - h^{t+1}}^2} + \left(1 - \gamma \mu\right)\frac{4\gamma (\left(2 \omega + 1\right)\probavailable - \probpairaa)}{n \probavailable^2} \Exp{\frac{1}{n}\sum_{i=1}^n\norm{g^{t+1}_i - h^{t+1}_i}^2}\\
        &\quad  + \frac{2\gamma}{b} \Exp{\norm{h^{t+1} - \nabla f(x^{t+1})}^2} + \rho \Exp{\frac{1}{n}\sum_{i=1}^n\norm{h^{t+1}_i - \nabla f_i(x^{t+1})}^2}\\
        &\leq \Exp{f(x^t)} - \frac{\gamma}{2}\Exp{\norm{\nabla f(x^t)}^2} \\
        &\quad + \left(1 - \gamma \mu\right)\frac{2\gamma (2 \omega + 1)}{\probavailable} \Exp{\norm{g^{t} - h^{t}}^2} + \left(1 - \gamma \mu\right)\frac{4\gamma (\left(2 \omega + 1\right)\probavailable - \probpairaa)}{n \probavailable^2} \Exp{\frac{1}{n}\sum_{i=1}^n\norm{g^{t}_i - h^{t}_i}^2} \\
        &\quad - \Bigg(\frac{1}{2\gamma} - \frac{L}{2} - \frac{10 \gamma \omega (2 \omega + 1)}{n \probavailable^2} \left(\frac{2 (1 - b)^2 L_{\sigma}^2}{B} + 2 \widehat{L}^2\right) \\
        &\qquad\quad - \frac{2\gamma}{b}\left(\frac{2 (1 - b)^2 L_{\sigma}^2}{n \probavailable B} + \frac{2\left(\probavailable - \probpairaa\right) \widehat{L}^2}{n \probavailable^2} \right) - \rho \left(\frac{2 (1 - b)^2 L_{\sigma}^2}{\probavailable B} + \frac{2(1 - \probavailable) \widehat{L}^2}{\probavailable}\right)\Bigg) \Exp{\norm{x^{t+1} - x^t}^2} \\
        &\quad + \left(1 - \frac{b}{2}\right)\frac{2\gamma}{b} \Exp{\norm{h^{t} - \nabla f(x^{t})}^2} \\
        &\quad + \Bigg(\frac{20 b^2 \gamma \omega (2 \omega + 1)}{n \probavailable^2} + \frac{4 \gamma \left(\probavailable - \probpairaa\right) b}{n \probavailable^2} + \rho \left(\frac{2 (1 - \probavailable) b^2}{\probavailable} + (1 - b)^2\right)\Bigg)\Exp{\frac{1}{n}\sum_{i=1}^n\norm{h^{t}_i - \nabla f_i(x^{t})}^2} \\
        &\quad + \left(\frac{20 b^2 \gamma \omega (2 \omega + 1)}{n \probavailable^2} + \frac{4 \gamma b}{n \probavailable} + \rho \frac{2 b^2}{\probavailable}\right) \frac{\sigma^2}{B}.
      \end{align*}
      Note that $b \leq \frac{\probavailable}{2 - \probavailable},$ thus
      \begin{align*}
        &\Bigg(\frac{20 b^2 \gamma \omega (2 \omega + 1)}{n \probavailable^2} + \frac{4 \gamma \left(\probavailable - \probpairaa\right) b}{n \probavailable^2} + \rho \left(\frac{2 (1 - \probavailable) b^2}{\probavailable} + (1 - b)^2\right)\Bigg) \\
        &\leq \Bigg(\frac{20 b^2 \gamma \omega (2 \omega + 1)}{n \probavailable^2} + \frac{4 \gamma \left(\probavailable - \probpairaa\right) b}{n \probavailable^2} + \rho \left(1 - b\right)\Bigg).
      \end{align*}
      And if we take $\rho = \frac{40 b \gamma \omega (2 \omega + 1)}{n \probavailable^2} + \frac{8 \gamma \left(\probavailable - \probpairaa\right)}{n \probavailable^2},$ then
      \begin{align*}
        \Bigg(\frac{20 b^2 \gamma \omega (2 \omega + 1)}{n \probavailable^2} + \frac{4 \gamma \left(\probavailable - \probpairaa\right) b}{n \probavailable^2} + \rho \left(1 - b\right)\Bigg) \leq \rho,
      \end{align*}
      and 
      \begin{align*}
        &\Exp{f(x^{t + 1})} + \left(1 - \gamma \mu\right)\frac{2\gamma (2 \omega + 1)}{\probavailable} \Exp{\norm{g^{t+1} - h^{t+1}}^2} + \left(1 - \gamma \mu\right)\frac{4\gamma (\left(2 \omega + 1\right)\probavailable - \probpairaa)}{n \probavailable^2} \Exp{\frac{1}{n}\sum_{i=1}^n\norm{g^{t+1}_i - h^{t+1}_i}^2}\\
        &\quad  + \frac{2\gamma}{b} \Exp{\norm{h^{t+1} - \nabla f(x^{t+1})}^2} + \left(\frac{40 b \gamma \omega (2 \omega + 1)}{n \probavailable^2} + \frac{8 \gamma \left(\probavailable - \probpairaa\right)}{n \probavailable^2}\right) \Exp{\frac{1}{n}\sum_{i=1}^n\norm{h^{t+1}_i - \nabla f_i(x^{t+1})}^2}\\
        &\leq \Exp{f(x^t)} - \frac{\gamma}{2}\Exp{\norm{\nabla f(x^t)}^2} \\
        &\quad + \left(1 - \gamma \mu\right)\frac{2\gamma (2 \omega + 1)}{\probavailable} \Exp{\norm{g^{t} - h^{t}}^2} + \left(1 - \gamma \mu\right)\frac{4\gamma (\left(2 \omega + 1\right)\probavailable - \probpairaa)}{n \probavailable^2} \Exp{\frac{1}{n}\sum_{i=1}^n\norm{g^{t}_i - h^{t}_i}^2} \\
        &\quad - \Bigg(\frac{1}{2\gamma} - \frac{L}{2} - \frac{10 \gamma \omega (2 \omega + 1)}{n \probavailable^2} \left(\frac{2 (1 - b)^2 L_{\sigma}^2}{B} + 2 \widehat{L}^2\right) \\
        &\qquad\quad - \frac{2 \gamma}{n \probavailable b}\left(\frac{2 (1 - b)^2 L_{\sigma}^2}{B} + 2\left(1 - \frac{\probpairaa}{\probavailable}\right) \widehat{L}^2 \right) \\
        &\qquad\quad - \left(\frac{40 b \gamma \omega (2 \omega + 1)}{n \probavailable^3} + \frac{8 \gamma \left(1 - \frac{\probpairaa}{\probavailable}\right)}{n \probavailable^2}\right) \left(\frac{2 (1 - b)^2 L_{\sigma}^2}{B} + 2(1 - \probavailable) \widehat{L}^2\right)\Bigg) \Exp{\norm{x^{t+1} - x^t}^2} \\
        &\quad + \left(1 - \frac{b}{2}\right)\frac{2\gamma}{b} \Exp{\norm{h^{t} - \nabla f(x^{t})}^2} + \left(1 - \frac{b}{2}\right)\left(\frac{40 b \gamma \omega (2 \omega + 1)}{n \probavailable^2} + \frac{8 \gamma \left(\probavailable - \probpairaa\right)}{n \probavailable^2}\right)\Exp{\frac{1}{n}\sum_{i=1}^n\norm{h^{t}_i - \nabla f_i(x^{t})}^2} \\
        &\quad + \left(\frac{20 b^2 \gamma \omega (2 \omega + 1)}{n \probavailable^2} + \frac{4 \gamma b}{n \probavailable} + \left(\frac{40 b \gamma \omega (2 \omega + 1)}{n \probavailable^2} + \frac{8 \gamma \left(\probavailable - \probpairaa\right)}{n \probavailable^2}\right) \frac{2 b^2}{\probavailable}\right) \frac{\sigma^2}{B}.
      \end{align*}
      Let us simplify the inequality. First, due to $b \leq \probavailable$ and $\left(1 - \probavailable\right) \leq \left(1 - \frac{\probpairaa}{\probavailable}\right),$ we have
      \begin{align*}
        &\left(\frac{40 b \gamma \omega (2 \omega + 1)}{n \probavailable^3} + \frac{2 \gamma \left(1 - \frac{\probpairaa}{\probavailable}\right)}{n \probavailable^2}\right) \left(\frac{2 (1 - b)^2 L_{\sigma}^2}{B} + 8(1 - \probavailable) \widehat{L}^2\right) \\
        &=\frac{40 b \gamma \omega (2 \omega + 1)}{n \probavailable^3} \left(\frac{2 (1 - b)^2 L_{\sigma}^2}{B} + 2(1 - \probavailable) \widehat{L}^2\right) \\
        &\quad + \frac{8 \gamma \left(1 - \frac{\probpairaa}{\probavailable}\right)}{n \probavailable^2} \left(\frac{2 (1 - b)^2 L_{\sigma}^2}{B} + 2(1 - \probavailable) \widehat{L}^2\right) \\
        &\leq \frac{40 \gamma \omega (2 \omega + 1)}{n \probavailable^2} \left(\frac{2 (1 - b)^2 L_{\sigma}^2}{B} + 2 \widehat{L}^2\right) \\
        &\quad + \frac{8 \gamma}{n \probavailable b} \left(\frac{2 (1 - b)^2 L_{\sigma}^2}{B} + 2\left(1 - \frac{\probpairaa}{\probavailable}\right) \widehat{L}^2\right),
      \end{align*}
      therefore
      \begin{align*}
        &\Exp{f(x^{t + 1})} + \left(1 - \gamma \mu\right)\frac{2\gamma (2 \omega + 1)}{\probavailable} \Exp{\norm{g^{t+1} - h^{t+1}}^2} + \left(1 - \gamma \mu\right)\frac{4\gamma (\left(2 \omega + 1\right)\probavailable - \probpairaa)}{n \probavailable^2} \Exp{\frac{1}{n}\sum_{i=1}^n\norm{g^{t+1}_i - h^{t+1}_i}^2}\\
        &\quad  + \frac{2\gamma}{b} \Exp{\norm{h^{t+1} - \nabla f(x^{t+1})}^2} + \left(\frac{40 b \gamma \omega (2 \omega + 1)}{n \probavailable^2} + \frac{8 \gamma \left(\probavailable - \probpairaa\right)}{n \probavailable^2}\right) \Exp{\frac{1}{n}\sum_{i=1}^n\norm{h^{t+1}_i - \nabla f_i(x^{t+1})}^2}\\
        &\leq \Exp{f(x^t)} - \frac{\gamma}{2}\Exp{\norm{\nabla f(x^t)}^2} \\
        &\quad + \left(1 - \gamma \mu\right)\frac{2\gamma (2 \omega + 1)}{\probavailable} \Exp{\norm{g^{t} - h^{t}}^2} + \left(1 - \gamma \mu\right)\frac{4\gamma (\left(2 \omega + 1\right)\probavailable - \probpairaa)}{n \probavailable^2} \Exp{\frac{1}{n}\sum_{i=1}^n\norm{g^{t}_i - h^{t}_i}^2} \\
        &\quad - \Bigg(\frac{1}{2\gamma} - \frac{L}{2} - \frac{50 \gamma \omega (2 \omega + 1)}{n \probavailable^2} \left(\frac{2 (1 - b)^2 L_{\sigma}^2}{B} + 2 \widehat{L}^2\right) \\
        &\qquad\quad - \frac{10 \gamma}{n \probavailable b}\left(\frac{2 (1 - b)^2 L_{\sigma}^2}{B} + 2\left(1 - \frac{\probpairaa}{\probavailable}\right) \widehat{L}^2 \right)\Bigg) \Exp{\norm{x^{t+1} - x^t}^2} \\
        &\quad + \left(1 - \frac{b}{2}\right)\frac{2\gamma}{b} \Exp{\norm{h^{t} - \nabla f(x^{t})}^2} + \left(1 - \frac{b}{2}\right)\left(\frac{40 b \gamma \omega (2 \omega + 1)}{n \probavailable^2} + \frac{8 \gamma \left(\probavailable - \probpairaa\right)}{n \probavailable^2}\right)\Exp{\frac{1}{n}\sum_{i=1}^n\norm{h^{t}_i - \nabla f_i(x^{t})}^2} \\
        &\quad + \left(\frac{20 b^2 \gamma \omega (2 \omega + 1)}{n \probavailable^2} + \frac{4 \gamma b}{n \probavailable} + \left(\frac{40 b \gamma \omega (2 \omega + 1)}{n \probavailable^2} + \frac{8 \gamma \left(\probavailable - \probpairaa\right)}{n \probavailable^2}\right) \frac{2 b^2}{\probavailable}\right) \frac{\sigma^2}{B} \\
        &\leq \Exp{f(x^t)} - \frac{\gamma}{2}\Exp{\norm{\nabla f(x^t)}^2} \\
        &\quad + \left(1 - \gamma \mu\right)\frac{2\gamma (2 \omega + 1)}{\probavailable} \Exp{\norm{g^{t} - h^{t}}^2} + \left(1 - \gamma \mu\right)\frac{4\gamma (\left(2 \omega + 1\right)\probavailable - \probpairaa)}{n \probavailable^2} \Exp{\frac{1}{n}\sum_{i=1}^n\norm{g^{t}_i - h^{t}_i}^2} \\
        &\quad - \Bigg(\frac{1}{2\gamma} - \frac{L}{2} - \frac{100 \gamma \omega (2 \omega + 1)}{n \probavailable^2} \left(\frac{(1 - b)^2 L_{\sigma}^2}{B} + \widehat{L}^2\right) \\
        &\qquad\quad - \frac{20 \gamma}{n \probavailable b}\left(\frac{(1 - b)^2 L_{\sigma}^2}{B} + \left(1 - \frac{\probpairaa}{\probavailable}\right) \widehat{L}^2 \right)\Bigg) \Exp{\norm{x^{t+1} - x^t}^2} \\
        &\quad + \left(1 - \frac{b}{2}\right)\frac{2\gamma}{b} \Exp{\norm{h^{t} - \nabla f(x^{t})}^2} + \left(1 - \frac{b}{2}\right)\left(\frac{40 b \gamma \omega (2 \omega + 1)}{n \probavailable^2} + \frac{8 \gamma \left(\probavailable - \probpairaa\right)}{n \probavailable^2}\right)\Exp{\frac{1}{n}\sum_{i=1}^n\norm{h^{t}_i - \nabla f_i(x^{t})}^2} \\
        &\quad + \left(\frac{20 b^2 \gamma \omega (2 \omega + 1)}{n \probavailable^2} + \frac{4 \gamma b}{n \probavailable} + \left(\frac{40 b \gamma \omega (2 \omega + 1)}{n \probavailable^2} + \frac{8 \gamma \left(\probavailable - \probpairaa\right)}{n \probavailable^2}\right) \frac{2 b^2}{\probavailable}\right) \frac{\sigma^2}{B}.
      \end{align*}
      Also, we can simplify the last term:
      \begin{align*}
        &\left(\frac{40 b \gamma \omega (2 \omega + 1)}{n \probavailable^2} + \frac{8 \gamma \left(\probavailable - \probpairaa\right)}{n \probavailable^2}\right) \frac{2 b^2}{\probavailable} \\
        &= \frac{80 b^3 \gamma \omega (2 \omega + 1)}{n \probavailable^3} + \frac{16 b^2 \gamma \left(1 - \frac{\probpairaa}{\probavailable}\right)}{n \probavailable^2} \\
        &\leq \frac{80 b^2 \gamma \omega (2 \omega + 1)}{n \probavailable^2} + \frac{16 b \gamma}{n \probavailable},
      \end{align*}
      thus
      \begin{align*}
        &\Exp{f(x^{t + 1})} + \left(1 - \gamma \mu\right)\frac{2\gamma (2 \omega + 1)}{\probavailable} \Exp{\norm{g^{t+1} - h^{t+1}}^2} + \left(1 - \gamma \mu\right)\frac{4\gamma (\left(2 \omega + 1\right)\probavailable - \probpairaa)}{n \probavailable^2} \Exp{\frac{1}{n}\sum_{i=1}^n\norm{g^{t+1}_i - h^{t+1}_i}^2}\\
        &\quad  + \frac{2\gamma}{b} \Exp{\norm{h^{t+1} - \nabla f(x^{t+1})}^2} + \left(\frac{40 b \gamma \omega (2 \omega + 1)}{n \probavailable^2} + \frac{8 \gamma \left(\probavailable - \probpairaa\right)}{n \probavailable^2}\right) \Exp{\frac{1}{n}\sum_{i=1}^n\norm{h^{t+1}_i - \nabla f_i(x^{t+1})}^2}\\
        &\leq \Exp{f(x^t)} - \frac{\gamma}{2}\Exp{\norm{\nabla f(x^t)}^2} \\
        &\quad + \left(1 - \gamma \mu\right)\frac{2\gamma (2 \omega + 1)}{\probavailable} \Exp{\norm{g^{t} - h^{t}}^2} + \left(1 - \gamma \mu\right)\frac{4\gamma (\left(2 \omega + 1\right)\probavailable - \probpairaa)}{n \probavailable^2} \Exp{\frac{1}{n}\sum_{i=1}^n\norm{g^{t}_i - h^{t}_i}^2} \\
        &\quad - \Bigg(\frac{1}{2\gamma} - \frac{L}{2} - \frac{100 \gamma \omega (2 \omega + 1)}{n \probavailable^2} \left(\frac{(1 - b)^2 L_{\sigma}^2}{B} + \widehat{L}^2\right) \\
        &\qquad\quad - \frac{20 \gamma}{n \probavailable b}\left(\frac{(1 - b)^2 L_{\sigma}^2}{B} + \left(1 - \frac{\probpairaa}{\probavailable}\right) \widehat{L}^2 \right)\Bigg) \Exp{\norm{x^{t+1} - x^t}^2} \\
        &\quad + \left(1 - \frac{b}{2}\right)\frac{2\gamma}{b} \Exp{\norm{h^{t} - \nabla f(x^{t})}^2} + \left(1 - \frac{b}{2}\right)\left(\frac{40 b \gamma \omega (2 \omega + 1)}{n \probavailable^2} + \frac{8 \gamma \left(\probavailable - \probpairaa\right)}{n \probavailable^2}\right)\Exp{\frac{1}{n}\sum_{i=1}^n\norm{h^{t}_i - \nabla f_i(x^{t})}^2} \\
        &\quad + \left(\frac{100 b^2 \gamma \omega (2 \omega + 1)}{n \probavailable^2} + \frac{20 \gamma b}{n \probavailable}\right) \frac{\sigma^2}{B}.
      \end{align*}
      Using Lemma~\ref{lemma:gamma} and the assumption about $\gamma,$ we get
      \begin{align*}
        &\Exp{f(x^{t + 1})} + \left(1 - \gamma \mu\right)\frac{2\gamma (2 \omega + 1)}{\probavailable} \Exp{\norm{g^{t+1} - h^{t+1}}^2} + \left(1 - \gamma \mu\right)\frac{4\gamma (\left(2 \omega + 1\right)\probavailable - \probpairaa)}{n \probavailable^2} \Exp{\frac{1}{n}\sum_{i=1}^n\norm{g^{t+1}_i - h^{t+1}_i}^2}\\
        &\quad  + \frac{2\gamma}{b} \Exp{\norm{h^{t+1} - \nabla f(x^{t+1})}^2} + \left(\frac{40 b \gamma \omega (2 \omega + 1)}{n \probavailable^2} + \frac{8 \gamma \left(\probavailable - \probpairaa\right)}{n \probavailable^2}\right) \Exp{\frac{1}{n}\sum_{i=1}^n\norm{h^{t+1}_i - \nabla f_i(x^{t+1})}^2}\\
        &\leq \Exp{f(x^t)} - \frac{\gamma}{2}\Exp{\norm{\nabla f(x^t)}^2} \\
        &\quad + \left(1 - \gamma \mu\right)\frac{2\gamma (2 \omega + 1)}{\probavailable} \Exp{\norm{g^{t} - h^{t}}^2} + \left(1 - \gamma \mu\right)\frac{4\gamma (\left(2 \omega + 1\right)\probavailable - \probpairaa)}{n \probavailable^2} \Exp{\frac{1}{n}\sum_{i=1}^n\norm{g^{t}_i - h^{t}_i}^2} \\
        &\quad + \left(1 - \frac{b}{2}\right)\frac{2\gamma}{b} \Exp{\norm{h^{t} - \nabla f(x^{t})}^2} + \left(1 - \frac{b}{2}\right)\left(\frac{40 b \gamma \omega (2 \omega + 1)}{n \probavailable^2} + \frac{8 \gamma \left(\probavailable - \probpairaa\right)}{n \probavailable^2}\right)\Exp{\frac{1}{n}\sum_{i=1}^n\norm{h^{t}_i - \nabla f_i(x^{t})}^2} \\
        &\quad + \left(\frac{100 b^2 \gamma \omega (2 \omega + 1)}{n \probavailable^2} + \frac{20 \gamma b}{n \probavailable}\right) \frac{\sigma^2}{B}.
      \end{align*}
      Note that $\gamma \leq \frac{b}{2\mu},$ thus $1 - \frac{b}{2} \leq 1 - \gamma \mu$ and 
      \begin{align*}
        &\Exp{f(x^{t + 1})} + \left(1 - \gamma \mu\right)\frac{2\gamma (2 \omega + 1)}{\probavailable} \Exp{\norm{g^{t+1} - h^{t+1}}^2} + \left(1 - \gamma \mu\right)\frac{4\gamma (\left(2 \omega + 1\right)\probavailable - \probpairaa)}{n \probavailable^2} \Exp{\frac{1}{n}\sum_{i=1}^n\norm{g^{t+1}_i - h^{t+1}_i}^2}\\
        &\quad  + \frac{2\gamma}{b} \Exp{\norm{h^{t+1} - \nabla f(x^{t+1})}^2} + \left(\frac{40 b \gamma \omega (2 \omega + 1)}{n \probavailable^2} + \frac{8 \gamma \left(\probavailable - \probpairaa\right)}{n \probavailable^2}\right) \Exp{\frac{1}{n}\sum_{i=1}^n\norm{h^{t+1}_i - \nabla f_i(x^{t+1})}^2}\\
        &\leq \Exp{f(x^t)} - \frac{\gamma}{2}\Exp{\norm{\nabla f(x^t)}^2} \\
        &\quad + \left(1 - \gamma \mu\right)\frac{2\gamma (2 \omega + 1)}{\probavailable} \Exp{\norm{g^{t} - h^{t}}^2} + \left(1 - \gamma \mu\right)\frac{4\gamma (\left(2 \omega + 1\right)\probavailable - \probpairaa)}{n \probavailable^2} \Exp{\frac{1}{n}\sum_{i=1}^n\norm{g^{t}_i - h^{t}_i}^2} \\
        &\quad + \left(1 - \gamma \mu\right)\frac{2\gamma}{b} \Exp{\norm{h^{t} - \nabla f(x^{t})}^2} + \left(1 - \gamma \mu\right)\left(\frac{40 b \gamma \omega (2 \omega + 1)}{n \probavailable^2} + \frac{8 \gamma \left(\probavailable - \probpairaa\right)}{n \probavailable^2}\right)\Exp{\frac{1}{n}\sum_{i=1}^n\norm{h^{t}_i - \nabla f_i(x^{t})}^2} \\
        &\quad + \left(\frac{100 b^2 \gamma \omega (2 \omega + 1)}{n \probavailable^2} + \frac{20 \gamma b}{n \probavailable}\right) \frac{\sigma^2}{B}.
      \end{align*}
      It is left to apply Lemma~\ref{lemma:good_recursion_pl} with 
      \begin{eqnarray*}
        \Psi^t &=& \frac{2 (2 \omega + 1)}{\probavailable} \Exp{\norm{g^{t} - h^{t}}^2} + \frac{4 (\left(2 \omega + 1\right)\probavailable - \probpairaa)}{n \probavailable^2} \Exp{\frac{1}{n}\sum_{i=1}^n\norm{g^{t}_i - h^{t}_i}^2} \\
        &\quad +& \frac{2}{b} \Exp{\norm{h^{t} - \nabla f(x^{t})}^2} + \left(\frac{40 b \omega (2 \omega + 1)}{n \probavailable^2} + \frac{8 \left(\probavailable - \probpairaa\right)}{n \probavailable^2}\right)\Exp{\frac{1}{n}\sum_{i=1}^n\norm{h^{t}_i - \nabla f_i(x^{t})}^2}
      \end{eqnarray*}
      and $C = \left(\frac{100 b^2 \omega (2 \omega + 1)}{\probavailable^2} + \frac{20 b}{\probavailable}\right) \frac{\sigma^2}{n B}$
      to conclude the proof.
    \end{proof}

    \CONVERGENCEPLSTOCHASTICRANDK*


    \begin{proof}
      In the view of Theorem~\ref{theorem:stochastic_pl}, \algname{\algorithmname} have to run
      \begin{align*}
        \widetilde{\cO}\left(\frac{\omega + 1}{\probavailable} + \frac{\omega}{\probavailable}\sqrt{\frac{\sigma^2}{\mu n \varepsilon B}} + \frac{\sigma^2}{\probavailable \mu n \varepsilon B} + \frac{L}{\mu} + \frac{\omega}{\probavailable \mu \sqrt{n}} \left(\widehat{L} + \frac{L_{\sigma}}{\sqrt{B}}\right) + \frac{\sigma}{\probavailable n \mu^{\nicefrac{3}{2}} \sqrt{\varepsilon B}}\left(\widehat{L} + \frac{L_{\sigma}}{\sqrt{B}}\right)\right)
      \end{align*}
    communication rounds in the stochastic settings to get $\varepsilon$-solution. Note that $K = \cO\left(\frac{d}{\probavailable\sqrt{n}}\right).$ Moreover, we can skip the initialization procedure and initialize $h^0_i$ and $g^0_i$, for instance, with zeros because the initialization error is under a logarithm. Considering Theorem~\ref{theorem:rand_k}, the communication complexity equals 
    \begin{align*}
      &\widetilde{\cO}\left(K\frac{\omega + 1}{\probavailable} + K\frac{\omega}{\probavailable}\sqrt{\frac{\sigma^2}{\mu n \varepsilon B}} + K\frac{\sigma^2}{\probavailable \mu n \varepsilon B} + K\frac{L}{\mu} + K\frac{\omega}{\probavailable \mu \sqrt{n}} \left(\widehat{L} + \frac{L_{\sigma}}{\sqrt{B}}\right) + K\frac{\sigma}{\probavailable n \mu^{\nicefrac{3}{2}} \sqrt{\varepsilon B}}\left(\widehat{L} + \frac{L_{\sigma}}{\sqrt{B}}\right)\right) \\
      &=\widetilde{\cO}\left(K\frac{\omega + 1}{\probavailable} + K\frac{\omega}{\probavailable}\sqrt{\frac{\sigma^2}{\mu n \varepsilon B}} + K\frac{\sigma^2}{\probavailable \mu n \varepsilon B} + K\frac{L}{\mu} + K\frac{\omega}{\probavailable \mu \sqrt{n}} \left(\widehat{L} + \frac{L_{\sigma}}{\sqrt{B}}\right) + K\frac{\sigma L_{\sigma}}{\probavailable n \mu^{\nicefrac{3}{2}} \sqrt{\varepsilon} B}\right) \\
      &= \widetilde{\cO}\left(\frac{d}{\probavailable} + \frac{d}{\probavailable}\sqrt{\frac{\sigma^2}{\mu n \varepsilon B}} + \frac{K \sigma^2}{\probavailable \mu n \varepsilon B} + \frac{d L}{\probavailable \mu \sqrt{n}} + \frac{d}{\probavailable \mu \sqrt{n}} \left(\widehat{L} + \frac{L_{\sigma}}{\sqrt{B}}\right) + \frac{K\sigma L_{\sigma}}{\probavailable n \mu^{\nicefrac{3}{2}} \sqrt{\varepsilon} B}\right) \\
      &= \widetilde{\cO}\left(\frac{d}{\probavailable} + \frac{d \sigma}{\probavailable \sqrt{\mu n \varepsilon B}} + \frac{d \sigma}{\probavailable \sqrt{\mu \varepsilon n}} + \frac{d L}{\probavailable \mu \sqrt{n}} + \frac{d}{\probavailable \mu \sqrt{n}} \left(\widehat{L} + \frac{L_{\sigma}}{\sqrt{B}}\right) + \frac{d L_{\sigma}}{\probavailable \mu \sqrt{n}}\right) \\
      &= \widetilde{\cO}\left(\frac{d \sigma}{\probavailable \sqrt{\mu \varepsilon n}} + \frac{d L_{\sigma}}{\probavailable \mu \sqrt{n}}\right).
    \end{align*}
    The expected number of stochastic gradient calculations per node equals
    \begin{align*}
      &\widetilde{\cO}\left(B\frac{\omega + 1}{\probavailable} + B\frac{\omega}{\probavailable}\sqrt{\frac{\sigma^2}{\mu n \varepsilon B}} + B\frac{\sigma^2}{\probavailable \mu n \varepsilon B} + B\frac{L}{\mu} + B\frac{\omega}{\probavailable \mu \sqrt{n}} \left(\widehat{L} + \frac{L_{\sigma}}{\sqrt{B}}\right) + B\frac{\sigma}{\probavailable n \mu^{\nicefrac{3}{2}} \sqrt{\varepsilon B}}\left(\widehat{L} + \frac{L_{\sigma}}{\sqrt{B}}\right)\right) \\
      &=\widetilde{\cO}\left(B\frac{\omega + 1}{\probavailable} + B\frac{\omega}{\probavailable}\sqrt{\frac{\sigma^2}{\mu n \varepsilon B}} + B\frac{\sigma^2}{\probavailable \mu n \varepsilon B} + B\frac{L}{\mu} + B\frac{\omega}{\probavailable \mu \sqrt{n}} \left(\widehat{L} + \frac{L_{\sigma}}{\sqrt{B}}\right) + B\frac{\sigma}{\probavailable n \mu^{\nicefrac{3}{2}} \sqrt{\varepsilon B}}\left(\frac{L_{\sigma}}{\sqrt{B}}\right)\right) \\
      &=\widetilde{\cO}\left(\frac{Bd}{K\probavailable} + \frac{B d}{K \probavailable}\sqrt{\frac{\sigma^2}{\mu n \varepsilon B}} + \frac{\sigma^2}{\probavailable \mu n \varepsilon} + B\frac{L}{\mu} + \frac{B d}{K \probavailable \mu \sqrt{n}} \left(\widehat{L} + \frac{L_{\sigma}}{\sqrt{B}}\right) + \frac{\sigma L_{\sigma}}{\probavailable n \mu^{\nicefrac{3}{2}} \sqrt{\varepsilon}}\right) \\
      &=\widetilde{\cO}\left(\frac{\sigma}{\probavailable \sqrt{\mu \varepsilon n}} + \frac{\sigma^2}{\probavailable \mu \varepsilon n \sqrt{B}} + \frac{\sigma^2}{\probavailable \mu n \varepsilon} + \frac{\sigma L}{\probavailable \mu^{\nicefrac{3}{2}}\sqrt{\varepsilon} n} + \frac{\sigma}{\probavailable \mu^{\nicefrac{3}{2}} \sqrt{\varepsilon} n} \left(\widehat{L} + \frac{L_{\sigma}}{\sqrt{B}}\right) + \frac{\sigma L_{\sigma}}{\probavailable n \mu^{\nicefrac{3}{2}} \sqrt{\varepsilon}}\right) \\
      &=\widetilde{\cO}\left(\frac{\sigma^2}{\probavailable \mu n \varepsilon} + \frac{\sigma L_{\sigma}}{\probavailable n \mu^{\nicefrac{3}{2}} \sqrt{\varepsilon}}\right).
    \end{align*}
    \end{proof}

\newpage
  \section{Description of \algname{\algorithmname-SYNC-MVR}}
  \label{sec:main_algorithm_mvr_sync}

   By analogy to \citep{tyurin2022dasha}, we provide a ``synchronized'' version of the algorithm. With a small probability, participating nodes calculate and send a mega batch without compression. This helps us to resolve the suboptimality of \algname{\algorithmname-MVR} w.r.t.\,$\omega.$ Note that this suboptimality is not a problem. 
   We show in Corollary~\ref{cor:stochastic:randk} that \algname{\algorithmname-MVR} can have the optimal oracle complexity and SOTA communication complexity with the particular choices of parameters of the compressors.
  
  \begin{algorithm}[h]
      \caption{\algname{\algorithmname-SYNC-MVR}}
      \label{alg:main_algorithm_mvr_sync}
      \begin{algorithmic}[1]
      \STATE \textbf{Input:} starting point $x^0 \in \R^d$, stepsize $\gamma > 0$, momentum $a \in (0, 1]$, momentum $b \in (0, 1]$, probability $\probmega \in (0, 1]$, batch size $B'$ and $B$, probability $\probavailable \in (0, 1]$ that a node is \textit{participating}\textsuperscript{\red (a)}, number of iterations $T \geq 1$.
      \STATE Initialize $g^0_i$, $h^0_i$ on the nodes and $g^0 = \frac{1}{n}\sum_{i=1}^n g^0_i$ on the server
      \FOR{$t = 0, 1, \dots, T - 1$}
      \STATE $x^{t+1} = x^t - \gamma g^t$
      \STATE $c^{t+1} = 
      \begin{cases}
          1, \textnormal{with probability $\probmega$}, \\
          0, \textnormal{with probability $1 - \probmega$}
      \end{cases}$
      \STATE Broadcast $x^{t+1}, x^{t}$ to all \textit{participating}\textsuperscript{\red (a)} nodes
      \FOR{$i = 1, \dots, n$ in parallel}
          \IF{$i^{\textnormal{th}} \textnormal{ node is \textit{participating}}$\textsuperscript{\red (a)}}
              \IF{$c^{t+1} = 1$}
              \STATE Generate i.i.d.\,samples $\{\xi_{ik}^{t+1}\}_{k=1}^{B'}$ of size $B'$ from $\mathcal{D}_i.$
              \STATE $k^{t+1}_i = \frac{1}{B'} \sum_{k=1}^{B'} \nabla f_i(x^{t+1};\xi_{ik}^{t+1}) - \frac{1}{B'} \sum_{k=1}^{B'} \nabla f_i(x^{t};\xi_{ik}^{t+1})-~\frac{b}{\probmega}\left(h^t_i - \frac{1}{B'} \sum_{k=1}^{B'} \nabla f_i(x^{t};\xi_{ik}^{t+1})\right)$
              \STATE $m^{t+1}_i = \frac{1}{\probavailable}k^{t+1}_i - \frac{a}{\probavailable} \left(g^t_i - h^t_i\right)$ 
          \ELSE
              \STATE Generate i.i.d.\,samples $\{\xi^{t+1}_{ij}\}_{j=1}^B$ of size $B$ from $\mathcal{D}_i.$
              \STATE $k^{t+1}_i = \frac{1}{B} \sum_{j=1}^{B}\nabla f_i(x^{t+1};\xi^{t+1}_{ij}) - \frac{1}{B} \sum_{j=1}^{B}\nabla f_i(x^{t};\xi^{t+1}_{ij})$
              \STATE $m^{t+1}_i = \cC_i\left(\frac{1}{\probavailable}k^{t+1}_i - \frac{a}{\probavailable} \left(g^t_i - h^{t}_i\right)\right)$ 
          \ENDIF
          \STATE $h^{t+1}_i = h^t_i + \frac{1}{\probavailable}k^{t+1}_i$ 
          \STATE $g^{t+1}_i = g^t_i + m^{t+1}_i$
          \STATE Send $m^{t+1}_i$ to the server
      \ELSE
          \STATE $h^{t+1}_i = h^{t}_i$
          \STATE $m^{t+1}_i = 0$
          \STATE $g^{t+1}_i = g^{t}_i$
      \ENDIF
      \ENDFOR
      \STATE $g^{t+1} = g^t + \frac{1}{n} \sum_{i=1}^{n} m^{t+1}_i$
      \ENDFOR
      \STATE \textbf{Output:} $\hat{x}^T$ chosen uniformly at random from $\{x^t\}_{k=0}^{T-1}$

      {\red (a)}: For the formal description see Section~\ref{sec:partial_participation}.
      \end{algorithmic}
  \end{algorithm}

  In the following theorem, we provide the convergence rate of \algname{\algorithmname-SYNC-MVR}.

  \begin{restatable}{theorem}{CONVERGENCESYNCMVR}
    \label{theorem:sync_stochastic}
    Suppose that Assumptions \ref{ass:lower_bound}, \ref{ass:lipschitz_constant}, \ref{ass:nodes_lipschitz_constant}, \ref{ass:stochastic_unbiased_and_variance_bounded}, \ref{ass:mean_square_smoothness}, \ref{ass:compressors} and \ref{ass:partial_participation} hold. Let us take $a = \frac{\probavailable}{2\omega + 1}$, $b = \frac{\probmega \probavailable}{2 - \probavailable},$probability $\probmega \in (0, 1],$ batch size $B' \geq B \geq 1$ $$\gamma \leq \left(L + \sqrt{\frac{8 \left(2\omega + 1\right) \omega}{n \probavailable^2}\left(\widehat{L}^2 + \frac{L_{\sigma}^2}{B}\right) + \frac{16}{n \probmega \probavailable^2} \left(\left(1 - \frac{\probpairaa}{\probavailable}\right) \widehat{L}^2 + \frac{L_{\sigma}^2}{B}\right)}\right)^{-1},$$ and $h^{0}_i = g^{0}_i$ for all $i \in [n]$
    in Algorithm~\ref{alg:main_algorithm_mvr_sync}. Then 
    \begin{align*}
      \Exp{\norm{\nabla f(\widehat{x}^T)}^2} &\leq \frac{1}{T}\vast[\frac{2 \Delta_0}{\gamma} + \frac{4}{\probmega \probavailable} \norm{h^{0} - \nabla f(x^{0})}^2 + \frac{4 \left(1 - \frac{\probpairaa}{\probavailable}\right)}{n \probmega \probavailable}\frac{1}{n}\sum_{i=1}^n\norm{h^{0}_i - \nabla f_i(x^{0})}^2 \vast] \\
      &\quad + \frac{12 \sigma^2}{n B'}.
    \end{align*}
  \end{restatable}

  First, we introduce the expected density of compressors \citep{MARINA,tyurin2022dasha}.

  \begin{definition}
    \label{def:expected_density}
    The expected density of the compressor $\cC_i$ is $\zeta_{\cC_i} \eqdef \sup_{x \in \R^d} \Exp{\norm{\cC_i(x)}_0}$, where $\norm{x}_0$ is the number of nonzero components of $x \in \R^d.$ Let  $\zeta_{\cC} = \max_{i \in [n]} \zeta_{\cC_i}.$
  \end{definition}

  Note that $\zeta_{\cC}$ is finite and $\zeta_{\cC} \leq d.$

  In the next corollary, we choose particular algorithm parameters to reveal the communication and oracle complexity.

  \begin{restatable}{corollary}{COROLLARYSYNCSTOCHASTIC}
    \label{cor:sync_stochastic}
    Suppose that assumptions from Theorem~\ref{theorem:sync_stochastic} hold, probability $\probmega = \min \left\{\ \frac{\zeta_{\cC}}{d}, \frac{n \varepsilon B}{\sigma^2}\right\},$ batch size 
    $B' = \Theta\left(\frac{\sigma^2}{n \varepsilon}\right),$
    and $h^{0}_i = g^{0}_i = \frac{1}{B_{\textnormal{init}}} \sum_{k = 1}^{B_{\textnormal{init}}} \nabla f_i(x^0; \xi^0_{ik})$ for all $i \in [n],$ initial batch size 
    $B_{\textnormal{init}} = \Theta\left(\frac{B}{\probmega \sqrt{\probavailable}}\right) = \Theta\left(\max \left\{\frac{B d}{ \sqrt{\probavailable} \zeta_{\cC}},  \frac{\sigma^2}{\sqrt{\probavailable} n \varepsilon} \right\}\right),$ 
    then \algname{\algorithmname-SYNC-MVR}
    needs
    $$T \eqdef \cO\left(\frac{\Delta_0}{\varepsilon}\vast[L + \left(\frac{\omega}{\probavailable\sqrt{n}} + \sqrt{\frac{d}{\probavailable^2 \zeta_{\cC} n}}\right) \left(\widehat{L} + \frac{L_{\sigma}}{\sqrt{B}}\right)  + \frac{\sigma}{\probavailable \sqrt{\varepsilon} n} \left(\frac{\widehat{L}}{\sqrt{B}} + \frac{L_{\sigma}}{B}\right)\vast] + \frac{\sigma^2}{\sqrt{\probavailable} n \varepsilon B}\right).$$
  
    communication rounds to get an $\varepsilon$-solution, the expected communication complexity is equal to $\cO\left(d + \zeta_{\cC} T\right),$ and the expected number of stochastic gradient calculations per node equals $\cO(B_{\textnormal{init}} + BT),$ where $\zeta_{\cC}$ is the expected density from Definition~\ref{def:expected_density}.
  \end{restatable}

  The main improvement of Corollary~\ref{cor:sync_stochastic} over Corollary~\ref{cor:stochastic} is the size of the initial batch size $B_{\textnormal{init}}$. However, Corollary~\ref{cor:stochastic:randk} reveals that we can avoid regimes when \algname{\algorithmname-MVR} is suboptimal.

  We also provide a theorem under P\L-condition (see Assumption~\ref{ass:pl_condition}).

  \begin{restatable}{theorem}{CONVERGENCESYNCPLMVR}
    \label{theorem:sync_stochastic_pl}
    Suppose that Assumptions \ref{ass:lower_bound}, \ref{ass:lipschitz_constant}, \ref{ass:nodes_lipschitz_constant}, \ref{ass:stochastic_unbiased_and_variance_bounded}, \ref{ass:mean_square_smoothness}, \ref{ass:compressors}, \ref{ass:partial_participation} and \ref{ass:pl_condition} hold. Let us take $a = \frac{\probavailable}{2\omega + 1}$, $b = \frac{\probmega \probavailable}{2 - \probavailable},$probability $\probmega \in (0, 1],$ batch size $B' \geq B \geq 1,$ $$\gamma \leq \min\left\{\left(L + \sqrt{\frac{16 \left(2 \omega + 1\right)\omega}{n \probavailable^2}\left(\frac{L_{\sigma}^2}{B} + \widehat{L}^2\right) + \left(\frac{48 L_{\sigma}^2}{n \probmega \probavailable^2 B} + \frac{24 \left(1 - \frac{\probpairaa}{\probavailable}\right) \widehat{L}^2}{n \probmega \probavailable^2}\right)}\right)^{-1}, \frac{a}{2\mu} , \frac{b}{2\mu}\right\},$$ and $h^{0}_i = g^{0}_i$ for all $i \in [n]$
    in Algorithm~\ref{alg:main_algorithm_mvr_sync}. Then 
    \begin{align*}
      &\Exp{f(x^{T}) - f^*} \\
      &\leq (1 - \gamma \mu)^{T}\left(\Delta_0 + \frac{2 \gamma}{b} \norm{h^{0} - \nabla f(x^{0})}^2 + \frac{8 \gamma \left(\probavailable - \probpairaa\right)}{n \probavailable^2 \probmega}\frac{1}{n}\sum_{i=1}^n\norm{h^{0}_i - \nabla f_i(x^{0})}^2\right) + \frac{20 \sigma^2}{\mu n B'}.
    \end{align*}
  \end{restatable}

  Let us provide bounds up to logarithmic factors and use $\widetilde{\cO}\left(\cdot\right)$ notation.

  \begin{restatable}{corollary}{COROLLARYSYNCPLSTOCHASTIC}
    Suppose that assumptions from Theorem~\ref{theorem:sync_stochastic_pl} hold, probability $\probmega = \min \left\{\ \frac{\zeta_{\cC}}{d}, \frac{\mu n \varepsilon B}{\sigma^2}\right\},$ batch size 
    $B' = \Theta\left(\frac{\sigma^2}{\mu n \varepsilon}\right)$ 
    then \algname{\algorithmname-SYNC-MVR}
    needs
    $$T \eqdef \widetilde{\cO}\left(\frac{\omega + 1}{\probavailable} + \frac{d}{\probavailable\zeta_{\cC}} + \frac{\sigma^2}{\probavailable \mu n \varepsilon B} + \frac{L}{\mu} + \frac{\omega}{\probavailable \mu \sqrt{n}} \left(\frac{L_{\sigma}}{\sqrt{B}} + \widehat{L}\right) + \left(\frac{\sqrt{d}}{\probavailable \mu \sqrt{\zeta_{\cC} n}} + \frac{\sigma}{\probavailable n \mu^{\nicefrac{3}{2}} \sqrt{\varepsilon B}}\right)\left(\frac{L_{\sigma}}{\sqrt{B}} + \widehat{L}\right) \right).$$
    communication rounds to get an $\varepsilon$-solution, the expected communication complexity is equal to $\widetilde{\cO}\left(\zeta_{\cC} T\right),$ and the expected number of stochastic gradient calculations per node equals $\widetilde{\cO}(BT),$ where $\zeta_{\cC}$ is the expected density from Definition~\ref{def:expected_density}.
  \end{restatable}
  The proof of this corollary almost repeats the proof of Corollary~\ref{cor:sync_stochastic}. Note that we can skip the initialization procedure and initialize $h^0_i$ and $g^0_i$, for instance, with zeros because the initialization error is under a logarithm.

  Let us assume that $\frac{d}{\zeta_{\cC}} = \Theta\left(\omega\right)$ (holds for the Rand$K$ compressor), then the convergence rate of \algname{\algorithmname-SYNC-MVR} is 
  \begin{align}
    \label{eq:pl_compare:new:sync_mvr}
    \widetilde{\cO}\left(\frac{\omega + 1}{\probavailable} + \frac{\sigma^2}{\probavailable \mu n \varepsilon B} + \frac{L}{\mu} + \frac{\omega}{\probavailable \mu \sqrt{n}} \left(\frac{L_{\sigma}}{\sqrt{B}} + \widehat{L}\right) +  \frac{\sigma}{\probavailable n \mu^{\nicefrac{3}{2}} \sqrt{\varepsilon B}}\left(\frac{L_{\sigma}}{\sqrt{B}} + \widehat{L}\right) \right).
  \end{align}
  Comparing \eqref{eq:pl_compare:new:sync_mvr} with the rate of \algname{\algorithmname-MVR} \eqref{eq:pl_compare:new:mvr}, one can see that \algname{\algorithmname-SYNC-MVR} improves the suboptimal term $\mathcal{P}_2$ from \eqref{eq:pl_compare:new:mvr}. However, Corollary~\ref{cor:stochastic:pl:randk} reveals that we can escape these suboptimal regimes by choosing the parameter $K$ of Rand$K$ compressors in a particular way.

\subsection{Proof for \algname{\algorithmname-SYNC-MVR}}
In this section, we provide the proof of the convergence rate for \algname{\algorithmname-SYNC-MVR}. There are four different sources of randomness in Algorithm~\ref{alg:main_algorithm_mvr_sync}: the first one from random samples $\xi^{t+1}_{\cdot}$, the second one from compressors $\{\cC_i\}_{i=1}^n$, the third one from availability of nodes, and the fourth one from $c^{t+1}.$ We define $\ExpSub{k}{\cdot}$, $\ExpSub{\cC}{\cdot},$ $\ExpSub{\probavailable}{\cdot}$ and $\ExpSub{\probmega}{\cdot}$ to be conditional expectations w.r.t.\,$\xi^{t+1}_{\cdot}$, $\{\cC_i\}_{i=1}^n, $ availability, and $c^{t+1},$ accordingly, conditioned on all previous randomness. Moreover, we define $\ExpSub{t+1}{\cdot}$ to be a conditional expectation w.r.t. all randomness in iteration $t+1$ conditioned on all previous randomness. 

Let us denote
\begin{align*}
    &k^{t+1}_{i, 1} \eqdef \frac{1}{B'} \sum_{k=1}^{B'} \nabla f_i(x^{t+1};\xi_{ik}^{t+1}) - \frac{1}{B'} \sum_{k=1}^{B'} \nabla f_i(x^{t};\xi_{ik}^{t+1})-~\frac{b}{\probmega}\left(h^t_i - \frac{1}{B'} \sum_{k=1}^{B'} \nabla f_i(x^{t};\xi_{ik}^{t+1})\right), \\
    &k^{t+1}_{i, 2} \eqdef \frac{1}{B} \sum_{j=1}^{B}\nabla f_i(x^{t+1};\xi^{t+1}_{ij}) - \frac{1}{B} \sum_{j=1}^{B}\nabla f_i(x^{t};\xi^{t+1}_{ij}), \\
    &h^{t+1}_{i,1} \eqdef \begin{cases}
        h^t_i + \frac{1}{\probavailable} k^{t+1}_{i, 1},& i^{\textnormal{th}} \textnormal{ node is \textit{participating}}, \\
        h^t_i, & \textnormal{otherwise},
    \end{cases}  \\
    &h^{t+1}_{i,2} \eqdef \begin{cases}
        h^t_i + \frac{1}{\probavailable} k^{t+1}_{i, 2},& i^{\textnormal{th}} \textnormal{ node is \textit{participating}}, \\
        h^t_i, & \textnormal{otherwise},
    \end{cases}  \\
    &g^{t+1}_{i,1} \eqdef \begin{cases}
      g^{t}_i + \frac{1}{\probavailable}k^{t+1}_{i, 1} - \frac{a}{\probavailable} \left(g^t_i - h^t_i\right),& i^{\textnormal{th}} \textnormal{ node is \textit{participating}}, \\
      g^t_i, & \textnormal{otherwise},
    \end{cases}  \\
    &g^{t+1}_{i,2} \eqdef \begin{cases}
        g^t_i + \cC_i\left(\frac{1}{\probavailable}k^{t+1}_{i, 2} - \frac{a}{\probavailable} \left(g^t_i - h^{t}_i\right)\right),& i^{\textnormal{th}} \textnormal{ node is \textit{participating}}, \\
        g^t_i, & \textnormal{otherwise},
    \end{cases}  \\
\end{align*}
$h^{t+1}_{1} \eqdef \frac{1}{n}\sum_{i=1}^n h^{t+1}_{i,1},$ $h^{t+1}_{2} \eqdef \frac{1}{n}\sum_{i=1}^n h^{t+1}_{i,2},$ $g^{t+1}_{1} \eqdef \frac{1}{n}\sum_{i=1}^n g^{t+1}_{i,1},$ and $g^{t+1}_{2} \eqdef \frac{1}{n}\sum_{i=1}^n g^{t+1}_{i,2}.$ Note, that
\begin{align*}
  &h^{t+1} = \begin{cases}
    h^{t+1}_{1},& c^{t+1} = 1, \\
    h^{t+1}_{2},& c^{t+1} = 0,
\end{cases}
\end{align*}
and
\begin{align*}
  &g^{t+1} = \begin{cases}
    g^{t+1}_{1},& c^{t+1} = 1, \\
    g^{t+1}_{2},& c^{t+1} = 0 
\end{cases}
\end{align*}

First, we will prove two lemmas.

\begin{lemma}
  \label{lemma:sync:prob_compressor}
  Suppose that Assumptions \ref{ass:nodes_lipschitz_constant}, \ref{ass:stochastic_unbiased_and_variance_bounded}, \ref{ass:compressors} and \ref{ass:partial_participation} hold and let us consider sequences $\{g^{t+1}_i\}_{i=1}^n$ and $\{h^{t+1}_i\}_{i=1}^n$ from Algorithm~\ref{alg:main_algorithm_mvr_sync}, then
  \begin{align*}
    &\ExpSub{\cC}{\ExpSub{\probavailable}{\ExpSub{\probmega}{\norm{g^{t+1} - h^{t+1}}^2}}} \\
    &\leq \frac{2 \left(1 - \probmega\right) \omega}{n^2 \probavailable}\sum_{i=1}^n\norm{k^{t+1}_{i, 2}}^2 + \left(\frac{\left(\probavailable - \probpairaa\right)a^2}{n^2 \probavailable^2} + \frac{2 \left(1 - \probmega\right) a^2 \omega}{n^2 \probavailable}\right)\sum_{i=1}^n\norm{g^{t}_i - h^{t}_i}^2 \\
    &\quad + (1 - a)^2\norm{g^{t} - h^{t}}^2,
  \end{align*}
  and
  \begin{align*}
    &\ExpSub{\cC}{\ExpSub{\probavailable}{\ExpSub{\probmega}{\norm{g^{t+1}_i - h^{t+1}_i}^2}}} \\
    &\leq \frac{2\left(1 - \probmega\right) \omega}{\probavailable} \norm{k^{t+1}_{i, 2}}^2 + \left(\frac{(1 - \probavailable) a^2}{\probavailable} + \frac{2 \left(1 - \probmega\right) a^2 \omega}{\probavailable}\right)\norm{g^{t}_{i} - h^{t}_{i}}^2 \\
    &\quad + (1 - a)^2 \norm{g^{t}_{i} - h^{t}_{i}}^2, \quad \forall i \in [n].
  \end{align*}
\end{lemma}

\begin{proof}
  First, we get the bound for $\ExpSub{t+1}{\norm{g^{t+1} - h^{t+1}}^2}$:
  \begin{align*}
    &\ExpSub{\cC}{\ExpSub{\probavailable}{\ExpSub{\probmega}{\norm{g^{t+1} - h^{t+1}}^2}}} \\
    &= \probmega \ExpSub{\probavailable}{\norm{g^{t+1}_1 - h^{t+1}_1}^2} + \left(1 - \probmega\right) \ExpSub{\cC}{\ExpSub{\probavailable}{\norm{g^{t+1}_2 - h^{t+1}_2}^2}}.
  \end{align*}
  Using
  \begin{align*}
    \ExpSub{\probavailable}{g^{t+1}_{i,1} - h^{t+1}_{i,1}} = g^{t}_i + k^{t+1}_{i, 1} - a\left(g^t_i - h^t_i\right) - h^t_i - k^{t+1}_{i, 1} = \left(1 - a\right) \left(g^t_i - h^t_i\right)
  \end{align*}
  and 
  \begin{align*}
    \ExpSub{\cC}{\ExpSub{\probavailable}{g^{t+1}_{i,2} - h^{t+1}_{i,2}}} = g^t_i + k^{t+1}_{i, 2} - a \left(g^t_i - h^{t}_i\right) - h^t_i - k^{t+1}_{i, 2} = (1 - a) \left(g^t_i - h^{t}_i\right),
  \end{align*}
  we have
  \begin{align*}
    &\ExpSub{\cC}{\ExpSub{\probavailable}{\ExpSub{\probmega}{\norm{g^{t+1} - h^{t+1}}^2}}} \\
    &\overset{\eqref{auxiliary:variance_decomposition}}{=} \probmega \ExpSub{\probavailable}{\norm{g^{t+1}_1 - h^{t+1}_1 - \ExpSub{\probavailable}{g^{t+1}_1 - h^{t+1}_1}}^2} \\
    &\quad + \left(1 - \probmega\right) \ExpSub{\cC}{\ExpSub{\probavailable}{\norm{g^{t+1}_2 - h^{t+1}_2 - \ExpSub{\probavailable}{g^{t+1}_2 - h^{t+1}_2}}^2}} \\
    &\quad + (1 - a)^2 \norm{g^{t} - h^{t}}^2.
  \end{align*}
  We can use Lemma~\ref{lemma:sampling} two times with i) $r_i = g^{t}_i - h^{t}_i$ and $s_i = -a\left(g^{t}_i - h^{t}_i\right)$ and ii) $r_i = g^{t}_i - h^{t}_i$ and $s_i = \probavailable\cC_i\left(\frac{1}{\probavailable}k^{t+1}_{i, 2} - \frac{a}{\probavailable} \left(g^t_i - h^{t}_i\right)\right) - k^{t+1}_{i, 2}$, to obtain
  \begin{align*}
    &\ExpSub{\cC}{\ExpSub{\probavailable}{\ExpSub{\probmega}{\norm{g^{t+1} - h^{t+1}}^2}}} \\
    &\leq \frac{\probmega a^2 \left(\probavailable - \probpairaa\right)}{n^2 \probavailable^2}\sum_{i=1}^n\norm{g^{t}_i - h^{t}_i}^2 \\
    &\quad + \left(1 - \probmega\right) \left(\frac{1}{n^2 \probavailable}\sum_{i=1}^n\ExpSub{\cC}{\norm{\probavailable\cC_i\left(\frac{1}{\probavailable}k^{t+1}_{i, 2} - \frac{a}{\probavailable} \left(g^t_i - h^{t}_i\right)\right) - \left(k^{t+1}_{i, 2} - a\left(g^t_i - h^{t}_i\right)\right)}^2}\right)\\
    &\quad + \left(1 - \probmega\right) \left(\frac{a^2 \left(\probavailable - \probpairaa\right)}{n^2 \probavailable^2}\sum_{i=1}^n\norm{g^t_i - h^{t}_i}^2\right) \\
    &\quad + (1 - a)^2\norm{g^{t} - h^{t}}^2 \\
    &= \frac{a^2\left(\probavailable - \probpairaa\right)}{n^2 \probavailable^2}\sum_{i=1}^n\norm{g^{t}_i - h^{t}_i}^2 \\
    &\quad + \left(1 - \probmega\right) \left(\frac{\probavailable}{n^2}\sum_{i=1}^n\ExpSub{\cC}{\norm{\cC_i\left(\frac{1}{\probavailable}k^{t+1}_{i, 2} - \frac{a}{\probavailable} \left(g^t_i - h^{t}_i\right)\right) - \left(\frac{1}{\probavailable}k^{t+1}_{i, 2} - \frac{a}{\probavailable}\left(g^t_i - h^{t}_i\right)\right)}^2}\right)\\
    &\quad + (1 - a)^2\norm{g^{t} - h^{t}}^2 \\
    &\leq \frac{a^2\left(\probavailable - \probpairaa\right)}{n^2 \probavailable^2}\sum_{i=1}^n\norm{g^{t}_i - h^{t}_i}^2 \\
    &\quad + \frac{\left(1 - \probmega\right) \probavailable \omega}{n^2}\sum_{i=1}^n\norm{\frac{1}{\probavailable}k^{t+1}_{i, 2} - \frac{a}{\probavailable}\left(g^t_i - h^{t}_i\right)}^2\\
    &\quad + (1 - a)^2\norm{g^{t} - h^{t}}^2 \\
    &= \frac{a^2\left(\probavailable - \probpairaa\right)}{n^2 \probavailable^2}\sum_{i=1}^n\norm{g^{t}_i - h^{t}_i}^2 \\
    &\quad + \frac{\left(1 - \probmega\right) \omega}{n^2 \probavailable}\sum_{i=1}^n\norm{k^{t+1}_{i, 2} - a\left(g^t_i - h^{t}_i\right)}^2\\
    &\quad + (1 - a)^2\norm{g^{t} - h^{t}}^2.
  \end{align*}
  In the last inequality, we use Assumption~\ref{ass:compressors}. Next, using \eqref{auxiliary:jensen_inequality}, we have
  \begin{align*}
    &\ExpSub{\cC}{\ExpSub{\probavailable}{\ExpSub{\probmega}{\norm{g^{t+1} - h^{t+1}}^2}}} \\
    &\leq \frac{2 \left(1 - \probmega\right) \omega}{n^2 \probavailable}\sum_{i=1}^n\norm{k^{t+1}_{i, 2}}^2 + \left(\frac{\left(\probavailable - \probpairaa\right)a^2}{n^2 \probavailable^2} + \frac{2 \left(1 - \probmega\right) \omega a^2}{n^2 \probavailable}\right)\sum_{i=1}^n\norm{g^{t}_i - h^{t}_i}^2 \\
    &\quad + (1 - a)^2\norm{g^{t} - h^{t}}^2.
  \end{align*}
  The second inequality can be proved almost in the same way:
  \begin{align*}
    &\ExpSub{\cC}{\ExpSub{\probavailable}{\ExpSub{\probmega}{\norm{g^{t+1}_i - h^{t+1}_i}^2}}} \\
    &= \probmega \ExpSub{\probavailable}{\norm{g^{t+1}_{i,1} - h^{t+1}_{i,1}}^2} + \left(1 - \probmega\right) \ExpSub{\cC}{\ExpSub{\probavailable}{\norm{g^{t+1}_{i,2} - h^{t+1}_{i,2}}^2}} \\
    &\overset{\eqref{auxiliary:variance_decomposition}}{=} \probmega \ExpSub{\probavailable}{\norm{g^{t+1}_{i,1} - h^{t+1}_{i,1} - (1 - a)\left(g^{t}_{i} - h^{t}_{i}\right)}^2} + \left(1 - \probmega\right) \ExpSub{\cC}{\ExpSub{\probavailable}{\norm{g^{t+1}_{i,2} - h^{t+1}_{i,2}}^2}} \\
    &\quad + \probmega (1 - a)^2 \norm{g^{t}_{i} - h^{t}_{i}}^2 \\
    &= \frac{\probmega(1 - \probavailable) a^2}{\probavailable}\norm{g^{t}_{i} - h^{t}_{i}}^2 + \left(1 - \probmega\right) \ExpSub{\cC}{\ExpSub{\probavailable}{\norm{g^{t+1}_{i,2} - h^{t+1}_{i,2}}^2}}\\
    &\quad + \probmega (1 - a)^2 \norm{g^{t}_{i} - h^{t}_{i}}^2 \\
    &\overset{\eqref{auxiliary:variance_decomposition}}{=} \frac{\probmega(1 - \probavailable) a^2}{\probavailable}  \norm{g^{t}_{i} - h^{t}_{i}}^2 + \left(1 - \probmega\right) \ExpSub{\cC}{\ExpSub{\probavailable}{\norm{g^{t+1}_{i,2} - h^{t+1}_{i,2} - (1 - a)\left(g^{t}_{i} - h^{t}_{i}\right)}^2}}\\
    &\quad + (1 - a)^2 \norm{g^{t}_{i} - h^{t}_{i}}^2 \\
    &= \frac{\probmega(1 - \probavailable) a^2}{\probavailable}  \norm{g^{t}_{i} - h^{t}_{i}}^2 \\
    &\quad + \left(1 - \probmega\right) \probavailable \ExpSub{\cC}{\norm{g^t_i + \cC_i\left(\frac{1}{\probavailable}k^{t+1}_{i, 2} - \frac{a}{\probavailable} \left(g^t_i - h^{t}_i\right)\right) - \left(h^t_i + \frac{1}{\probavailable} k^{t+1}_{i, 2}\right) - (1 - a)\left(g^{t}_{i} - h^{t}_{i}\right)}^2} \\
    &\quad + \left(1 - \probmega\right) \left(1 - \probavailable\right)\norm{g^{t}_{i} - h^{t}_{i} - (1 - a)\left(g^{t}_{i} - h^{t}_{i}\right)}^2\\
    &\quad + (1 - a)^2 \norm{g^{t}_{i} - h^{t}_{i}}^2 \\
    &= \frac{\probmega(1 - \probavailable) a^2}{\probavailable}  \norm{g^{t}_{i} - h^{t}_{i}}^2 \\
    &\quad + \left(1 - \probmega\right) \probavailable \ExpSub{\cC}{\norm{\cC_i\left(\frac{1}{\probavailable}k^{t+1}_{i, 2} - \frac{a}{\probavailable} \left(g^t_i - h^{t}_i\right)\right) - \left(\frac{1}{\probavailable} k^{t+1}_{i, 2} - a\left(g^{t}_{i} - h^{t}_{i}\right)\right)}^2} \\
    &\quad + \left(1 - \probmega\right) \left(1 - \probavailable\right) a^2 \norm{g^{t}_{i} - h^{t}_{i}}^2\\
    &\quad + (1 - a)^2 \norm{g^{t}_{i} - h^{t}_{i}}^2 \\
    &\overset{\eqref{auxiliary:variance_decomposition}}{=} \left(\frac{\probmega(1 - \probavailable) a^2}{\probavailable} + \frac{\left(1 - \probmega\right) \left(1 - \probavailable\right) a^2}{\probavailable}\right) \norm{g^{t}_{i} - h^{t}_{i}}^2 \\
    &\quad + \left(1 - \probmega\right) \probavailable \ExpSub{\cC}{\norm{\cC_i\left(\frac{1}{\probavailable}k^{t+1}_{i, 2} - \frac{a}{\probavailable} \left(g^t_i - h^{t}_i\right)\right) - \left(\frac{1}{\probavailable} k^{t+1}_{i, 2} - \frac{a}{\probavailable}\left(g^{t}_{i} - h^{t}_{i}\right)\right)}^2} \\
    &\quad + (1 - a)^2 \norm{g^{t}_{i} - h^{t}_{i}}^2 \\
    &= \frac{(1 - \probavailable) a^2}{\probavailable} \norm{g^{t}_{i} - h^{t}_{i}}^2 \\
    &\quad + \left(1 - \probmega\right) \probavailable \ExpSub{\cC}{\norm{\cC_i\left(\frac{1}{\probavailable}k^{t+1}_{i, 2} - \frac{a}{\probavailable} \left(g^t_i - h^{t}_i\right)\right) - \left(\frac{1}{\probavailable} k^{t+1}_{i, 2} - \frac{a}{\probavailable}\left(g^{t}_{i} - h^{t}_{i}\right)\right)}^2} \\
    &\quad + (1 - a)^2 \norm{g^{t}_{i} - h^{t}_{i}}^2 \\
    &\leq \frac{(1 - \probavailable) a^2}{\probavailable}\norm{g^{t}_{i} - h^{t}_{i}}^2 \\
    &\quad + \frac{\left(1 - \probmega\right) \omega}{\probavailable} \norm{k^{t+1}_{i, 2} - a \left(g^t_i - h^{t}_i\right)}^2 \\
    &\quad + (1 - a)^2 \norm{g^{t}_{i} - h^{t}_{i}}^2 \\
    &\overset{\eqref{auxiliary:jensen_inequality}}{\leq} \frac{2\left(1 - \probmega\right) \omega}{\probavailable} \norm{k^{t+1}_{i, 2}}^2 + \left(\frac{(1 - \probavailable) a^2}{\probavailable} + \frac{2 \left(1 - \probmega\right) a^2 \omega}{\probavailable}\right)\norm{g^{t}_{i} - h^{t}_{i}}^2 \\
    &\quad + (1 - a)^2 \norm{g^{t}_{i} - h^{t}_{i}}^2.
  \end{align*}
\end{proof}

\begin{lemma}
  \label{lemma:sync_mvr}
  Suppose that Assumptions \ref{ass:nodes_lipschitz_constant}, \ref{ass:stochastic_unbiased_and_variance_bounded}, \ref{ass:mean_square_smoothness} and \ref{ass:partial_participation} hold and let us consider sequence $\{h^{t+1}_i\}_{i=1}^n$ from Algorithm~\ref{alg:main_algorithm_mvr_sync}, then
  \begin{align*}
    &\ExpSub{k}{\ExpSub{\probavailable}{\ExpSub{\probmega}{\norm{h^{t+1} - \nabla f(x^{t+1})}^2}}} \nonumber\\
    &\leq \frac{2 b^2 \sigma^2}{n \probmega \probavailable B'} + \left(\frac{2 \probmega L_{\sigma}^2}{n \probavailable B'}\left(1 - \frac{b}{\probmega}\right)^2 + \frac{\left(1 - \probmega\right)L_{\sigma}^2}{n \probavailable B} + \frac{2\left(\probavailable - \probpairaa\right) \widehat{L}^2}{n \probavailable^2}\right)\norm{x^{t+1} - x^{t}}^2 \nonumber\\
    &\quad + \frac{2\left(\probavailable - \probpairaa\right) b^2}{n^2 \probavailable^2 \probmega}\sum_{i=1}^n\norm{h^t_i -  \nabla f_i(x^{t})}^2 + \left(\probmega \left(1 - \frac{b}{\probmega}\right)^2 + (1 - \probmega)\right) \norm{h^{t} - \nabla f(x^{t})}^2,
  \end{align*}
  \begin{align*}
    &\ExpSub{k}{\ExpSub{\probavailable}{\ExpSub{\probmega}{\norm{h^{t+1}_i - \nabla f_i(x^{t+1})}^2}}} \\
    &\leq \frac{2 b^2 \sigma^2}{\probavailable \probmega B'} + \left(\frac{2 \probmega L_{\sigma}^2}{\probavailable B'}\left(1 - \frac{b}{\probmega}\right)^2 + \frac{(1 - \probmega) L_{\sigma}^2}{\probavailable B} + \frac{2(1 - \probavailable) L_{i}^2}{\probavailable}\right) \norm{x^{t+1} - x^{t}}^2 \\
    &\quad + \frac{2\left(1 - \probavailable\right)b^2}{\probmega \probavailable} \norm{h^t_i - \nabla f_i(x^{t})}^2 + \left(\probmega \left(1 - \frac{b}{\probmega}\right)^2 + (1 - \probmega)\right) \norm{h^{t}_i - \nabla f_i(x^{t})}^2, \quad \forall i \in [n],
  \end{align*}
  and 
  \begin{align*}
    &\ExpSub{k}{\norm{k^{t+1}_{i, 2}}^2} \leq \left(\frac{L_{\sigma}^2}{B} + L_{i}^2\right) \norm{x^{t+1} - x^{t}}^2,  \quad \forall i \in [n],
  \end{align*}
\end{lemma}

\begin{proof}
  First, we prove the bound for $\ExpSub{k}{\ExpSub{\probavailable}{\ExpSub{\probmega}{\norm{h^{t+1} - \nabla f(x^{t+1})}^2}}}$.
  Using
  \begin{align*}
    &\ExpSub{k}{\ExpSub{\probavailable}{h^{t+1}_{i,1}}} \\
    &= h^t_i + \ExpSub{k}{\frac{1}{B'} \sum_{k=1}^{B'} \nabla f_i(x^{t+1};\xi_{ik}^{t+1}) - \frac{1}{B'} \sum_{k=1}^{B'} \nabla f_i(x^{t};\xi_{ik}^{t+1})-~\frac{b}{\probmega}\left(h^t_i - \frac{1}{B'} \sum_{k=1}^{B'} \nabla f_i(x^{t};\xi_{ik}^{t+1})\right)} \\
    &=h^t_i + \nabla f_i(x^{t+1}) - \nabla f_i(x^{t})-\frac{b}{\probmega}\left(h^t_i - \nabla f_i(x^{t})\right)
  \end{align*}
  and 
  \begin{align*}
    &\ExpSub{k}{\ExpSub{\probavailable}{h^{t+1}_{i,2}}} \\
    &= h^t_i + \ExpSub{k}{\frac{1}{B} \sum_{j=1}^{B}\nabla f_i(x^{t+1};\xi^{t+1}_{ij}) - \frac{1}{B} \sum_{j=1}^{B}\nabla f_i(x^{t};\xi^{t+1}_{ij})} \\
    &= h^t_i + \nabla f_i(x^{t+1}) - \nabla f_i(x^{t}),
  \end{align*}
  we have
  \begin{align*}
    &\ExpSub{k}{\ExpSub{\probavailable}{\ExpSub{\probmega}{\norm{h^{t+1} - \nabla f(x^{t+1})}^2}}} \\
    &=\probmega \ExpSub{k}{\ExpSub{\probavailable}{\norm{h^{t+1}_{1} - \nabla f(x^{t+1})}^2}} + (1 - \probmega) \ExpSub{k}{\ExpSub{\probavailable}{\norm{h^{t+1}_{2} - \nabla f(x^{t+1})}^2}} \\
    &\overset{\eqref{auxiliary:variance_decomposition}}{=}\probmega \ExpSub{k}{\ExpSub{\probavailable}{\norm{h^{t+1}_{1} - \ExpSub{k}{\ExpSub{\probavailable}{h^{t+1}_{1}}}}^2}} + (1 - \probmega) \ExpSub{k}{\ExpSub{\probavailable}{\norm{h^{t+1}_{2} - \ExpSub{k}{\ExpSub{\probavailable}{h^{t+1}_{2}}}}^2}} \\
    &\quad + \left(\probmega \left(1 - \frac{b}{\probmega}\right)^2 + (1 - \probmega)\right) \norm{h^{t} - \nabla f(x^{t})}^2.
  \end{align*}
  We can use Lemma~\ref{lemma:sampling} two times with i) $r_i = h^{t}_i$ and $s_i = k^{t+1}_{i, 1}$ and ii) $r_i = h^{t}_i$ and $s_i = k^{t+1}_{i, 2}$, to obtain
  \begin{align}
    &\ExpSub{k}{\ExpSub{\probavailable}{\ExpSub{\probmega}{\norm{h^{t+1} - \nabla f(x^{t+1})}^2}}} \nonumber\\
    &\leq \probmega \left(\frac{1}{n^2 \probavailable}\sum_{i=1}^n\ExpSub{k}{\norm{k^{t+1}_{i, 1} - \ExpSub{k}{k^{t+1}_{i, 1}}}^2} +\frac{\probavailable - \probpairaa}{n^2 \probavailable^2}\sum_{i=1}^n\norm{\nabla f_i(x^{t+1}) - \nabla f_i(x^{t})- \frac{b}{\probmega}\left(h^t_i -  \nabla f_i(x^{t})\right)}^2\right) \nonumber\\
    &\quad + (1 - \probmega) \left(\frac{1}{n^2 \probavailable}\sum_{i=1}^n\ExpSub{k}{\norm{k^{t+1}_{i, 2} - \ExpSub{k}{k^{t+1}_{i, 2}}}^2} +\frac{\probavailable - \probpairaa}{n^2 \probavailable^2}\sum_{i=1}^n\norm{\nabla f_i(x^{t+1}) - \nabla f_i(x^{t})}^2\right) \nonumber\\
    &\quad + \left(\probmega \left(1 - \frac{b}{\probmega}\right)^2 + (1 - \probmega)\right) \norm{h^{t} - \nabla f(x^{t})}^2 \nonumber\\
    &\overset{\eqref{auxiliary:jensen_inequality}}{\leq} \frac{\probmega}{n^2 \probavailable}\sum_{i=1}^n\ExpSub{k}{\norm{k^{t+1}_{i, 1} - \ExpSub{k}{k^{t+1}_{i, 1}}}^2} \nonumber\\
    &\quad +  \frac{1 - \probmega}{n^2 \probavailable}\sum_{i=1}^n\ExpSub{k}{\norm{k^{t+1}_{i, 2} - \ExpSub{k}{k^{t+1}_{i, 2}}}^2} \nonumber\\
    &\quad + \frac{2\left(\probavailable - \probpairaa\right)}{n^2 \probavailable^2}\sum_{i=1}^n\norm{\nabla f_i(x^{t+1}) - \nabla f_i(x^{t})}^2 \nonumber\\
    &\quad + \frac{2\left(\probavailable - \probpairaa\right) b^2}{n^2 \probavailable^2 \probmega}\sum_{i=1}^n\norm{h^t_i -  \nabla f_i(x^{t})}^2 + \left(\probmega \left(1 - \frac{b}{\probmega}\right)^2 + (1 - \probmega)\right) \norm{h^{t} - \nabla f(x^{t})}^2. \label{eq:sync_mvr:h}
  \end{align}
  Let us consider $\ExpSub{k}{\norm{k^{t+1}_{i, 1} - \ExpSub{k}{k^{t+1}_{i, 1}}}^2}.$
  \begin{align*}
    &\ExpSub{k}{\norm{k^{t+1}_{i, 1} - \ExpSub{k}{k^{t+1}_{i, 1}}}^2} \\
    &=\ExpSub{k}{\norm{\frac{1}{B'} \sum_{k=1}^{B'} \nabla f_i(x^{t+1};\xi_{ik}^{t+1}) - \frac{1}{B'} \sum_{k=1}^{B'} \nabla f_i(x^{t};\xi_{ik}^{t+1})-~\frac{b}{\probmega}\left(h^t_i - \frac{1}{B'} \sum_{k=1}^{B'} \nabla f_i(x^{t};\xi_{ik}^{t+1})\right) \right.\right.\\
    &\qquad\qquad -\left.\left. \left(\nabla f_i(x^{t+1}) - \nabla f_i(x^{t})- \frac{b}{\probmega}\left(h^t_i -  \nabla f_i(x^{t})\right)\right)}^2} \\
    &=\ExpSub{k}{\norm{\frac{1}{B'} \sum_{k=1}^{B'} \nabla f_i(x^{t+1};\xi_{ik}^{t+1}) - \frac{1}{B'} \sum_{k=1}^{B'} \nabla f_i(x^{t};\xi_{ik}^{t+1})+\frac{b}{\probmega}\left(\frac{1}{B'} \sum_{k=1}^{B'} \nabla f_i(x^{t};\xi_{ik}^{t+1})\right) \right.\right.\\
    &\qquad\qquad -\left.\left. \left(\nabla f_i(x^{t+1}) - \nabla f_i(x^{t}) + \frac{b}{\probmega}\left(\nabla f_i(x^{t})\right)\right)}^2} \\
    &=\frac{1}{B'^2} \sum_{k=1}^{B'} \ExpSub{k}{\norm{\frac{b}{\probmega}\left(\nabla f_i(x^{t+1};\xi_{ik}^{t+1}) - \nabla f_i(x^{t+1})\right) \right.\right.\\
    &\qquad\qquad +\left.\left. \left(1 - \frac{b}{\probmega}\right)\left(\nabla f_i(x^{t+1};\xi_{ik}^{t+1}) - \nabla f_i(x^{t};\xi_{ik}^{t+1}) - \left(\nabla f_i(x^{t+1}) - \nabla f_i(x^{t})\right)\right)}^2},
  \end{align*}
  where we used independence of the mini-batch samples. Using \eqref{auxiliary:jensen_inequality}, we get
  \begin{align*}
    &\ExpSub{k}{\norm{k^{t+1}_{i, 1} - \ExpSub{k}{k^{t+1}_{i, 1}}}^2} \\
    &\leq \frac{2 b^2}{B'^2 \probmega^2} \sum_{k=1}^{B'} \ExpSub{k}{\norm{\nabla f_i(x^{t+1};\xi_{ik}^{t+1}) - \nabla f_i(x^{t+1})}^2}\\
    &\quad + \frac{2}{B'^2}\left(1 - \frac{b}{\probmega}\right)^2 \sum_{k=1}^{B'} \ExpSub{k}{\norm{\nabla f_i(x^{t+1};\xi_{ik}^{t+1}) - \nabla f_i(x^{t};\xi_{ik}^{t+1}) - \left(\nabla f_i(x^{t+1}) - \nabla f_i(x^{t})\right)}^2}.
  \end{align*}
  Due to Assumptions~\ref{ass:stochastic_unbiased_and_variance_bounded} and \ref{ass:mean_square_smoothness}, we have
  \begin{align}
    \ExpSub{k}{\norm{k^{t+1}_{i, 1} - \ExpSub{k}{k^{t+1}_{i, 1}}}^2} \leq \frac{2 b^2 \sigma^2}{B' \probmega^2} + \frac{2 L_{\sigma}^2}{B'}\left(1 - \frac{b}{\probmega}\right)^2 \norm{x^{t+1} - x^{t}}^2 \label{eq:sync_mvr:k_1}.
  \end{align}
  Next, we estimate the bound for $\ExpSub{k}{\norm{k^{t+1}_{i, 2} - \ExpSub{k}{k^{t+1}_{i, 2}}}^2}.$
  \begin{align*}
    &\ExpSub{k}{\norm{k^{t+1}_{i, 2} - \ExpSub{k}{k^{t+1}_{i, 2}}}^2} \\
    &=\ExpSub{k}{\norm{\frac{1}{B} \sum_{j=1}^{B}\nabla f_i(x^{t+1};\xi^{t+1}_{ij}) - \frac{1}{B} \sum_{j=1}^{B}\nabla f_i(x^{t};\xi^{t+1}_{ij}) - \left(\nabla f_i(x^{t+1}) - \nabla f_i(x^{t})\right)}^2} \\
    &=\frac{1}{B^2} \sum_{j=1}^{B} \ExpSub{k}{\norm{\nabla f_i(x^{t+1};\xi^{t+1}_{ij}) - \nabla f_i(x^{t};\xi^{t+1}_{ij}) - \left(\nabla f_i(x^{t+1}) - \nabla f_i(x^{t})\right)}^2}. \\
  \end{align*}
  Due to Assumptions \ref{ass:mean_square_smoothness}, we have
  \begin{align}
    &\ExpSub{k}{\norm{k^{t+1}_{i, 2} - \ExpSub{k}{k^{t+1}_{i, 2}}}^2} \leq\frac{L_{\sigma}^2}{B} \norm{x^{t+1} - x^{t}}^2. \label{eq:sync_mvr:k_2}
  \end{align}
  Plugging \eqref{eq:sync_mvr:k_1} and \eqref{eq:sync_mvr:k_2} into \eqref{eq:sync_mvr:h}, we obtain
  \begin{align*}
    &\ExpSub{k}{\ExpSub{\probavailable}{\ExpSub{\probmega}{\norm{h^{t+1} - \nabla f(x^{t+1})}^2}}} \nonumber\\
    &\leq \frac{\probmega}{n \probavailable}\left(\frac{2 b^2 \sigma^2}{B' \probmega^2} + \frac{2 L_{\sigma}^2}{B'}\left(1 - \frac{b}{\probmega}\right)^2 \norm{x^{t+1} - x^{t}}^2\right) \nonumber\\
    &\quad +  \frac{\left(1 - \probmega\right)L_{\sigma}^2}{n \probavailable B} \norm{x^{t+1} - x^{t}}^2 \nonumber\\
    &\quad + \frac{2\left(\probavailable - \probpairaa\right)}{n^2 \probavailable^2}\sum_{i=1}^n\norm{\nabla f_i(x^{t+1}) - \nabla f_i(x^{t})}^2 \nonumber\\
    &\quad + \frac{2\left(\probavailable - \probpairaa\right) b^2}{n^2 \probavailable^2 \probmega}\sum_{i=1}^n\norm{h^t_i -  \nabla f_i(x^{t})}^2 + \left(\probmega \left(1 - \frac{b}{\probmega}\right)^2 + (1 - \probmega)\right) \norm{h^{t} - \nabla f(x^{t})}^2.
  \end{align*}
  Using Assumption~\ref{ass:nodes_lipschitz_constant}, we get
  \begin{align*}
    &\ExpSub{k}{\ExpSub{\probavailable}{\ExpSub{\probmega}{\norm{h^{t+1} - \nabla f(x^{t+1})}^2}}} \nonumber\\
    &\leq \frac{2 b^2 \sigma^2}{n \probmega \probavailable B'} + \left(\frac{2 \probmega L_{\sigma}^2}{n \probavailable B'}\left(1 - \frac{b}{\probmega}\right)^2 + \frac{\left(1 - \probmega\right)L_{\sigma}^2}{n \probavailable B} + \frac{2\left(\probavailable - \probpairaa\right) \widehat{L}^2}{n \probavailable^2}\right)\norm{x^{t+1} - x^{t}}^2 \nonumber\\
    &\quad + \frac{2\left(\probavailable - \probpairaa\right) b^2}{n^2 \probavailable^2 \probmega}\sum_{i=1}^n\norm{h^t_i -  \nabla f_i(x^{t})}^2 + \left(\probmega \left(1 - \frac{b}{\probmega}\right)^2 + (1 - \probmega)\right) \norm{h^{t} - \nabla f(x^{t})}^2.
  \end{align*}
  Using almost the same derivations, we can prove the second inequality:
  \begin{align*}
    &\ExpSub{k}{\ExpSub{\probavailable}{\ExpSub{\probmega}{\norm{h^{t+1}_i - \nabla f_i(x^{t+1})}^2}}} \\
    &=\probmega \ExpSub{k}{\ExpSub{\probavailable}{\norm{h^{t+1}_{i,1} - \nabla f_i(x^{t+1})}^2}} + (1 - \probmega) \ExpSub{k}{\ExpSub{\probavailable}{\norm{h^{t+1}_{i,2} - \nabla f_i(x^{t+1})}^2}} \\
    &\overset{\eqref{auxiliary:variance_decomposition}}{=}\probmega \ExpSub{k}{\ExpSub{\probavailable}{\norm{h^{t+1}_{i,1} - \ExpSub{k}{\ExpSub{\probavailable}{h^{t+1}_{i,1}}}}^2}} + (1 - \probmega) \ExpSub{k}{\ExpSub{\probavailable}{\norm{h^{t+1}_{i,2} - \ExpSub{k}{\ExpSub{\probavailable}{h^{t+1}_{i,2}}}}^2}} \\
    &\quad + \left(\probmega \left(1 - \frac{b}{\probmega}\right)^2 + (1 - \probmega)\right) \norm{h^{t}_i - \nabla f_i(x^{t})}^2 \\
    &=\probmega \probavailable \ExpSub{k}{\norm{h^t_i + \frac{1}{\probavailable} k^{t+1}_{i, 1} - \left(h^t_i + \ExpSub{k}{k^{t+1}_{i, 1}}\right)}^2} \\
    &\quad + \probmega\left(1 - \probavailable\right) \norm{h^{t}_{i} - \left(h^t_i + \ExpSub{k}{k^{t+1}_{i, 1}}\right)}^2 \\
    &\quad + (1 - \probmega) \probavailable \ExpSub{k}{\norm{h^t_i + \frac{1}{\probavailable} k^{t+1}_{i, 2} - \left(h^t_i +  \ExpSub{k}{k^{t+1}_{i, 2}}\right)}^2} \\
    &\quad + (1 - \probmega) (1 - \probavailable)\norm{h^t_i - \left(h^t_i +  \ExpSub{k}{k^{t+1}_{i, 2}}\right)}^2 \\
    &\quad + \left(\probmega \left(1 - \frac{b}{\probmega}\right)^2 + (1 - \probmega)\right) \norm{h^{t}_i - \nabla f_i(x^{t})}^2 \\
    &=\probmega \probavailable \ExpSub{k}{\norm{\frac{1}{\probavailable} k^{t+1}_{i, 1} - \ExpSub{k}{k^{t+1}_{i, 1}}}^2} \\
    &\quad + \probmega\left(1 - \probavailable\right) \norm{\nabla f_i(x^{t+1}) - \nabla f_i(x^{t}) - \frac{b}{\probmega} \left(h^t_i - \nabla f_i(x^{t})\right)}^2 \\
    &\quad + (1 - \probmega) \probavailable \ExpSub{k}{\norm{\frac{1}{\probavailable} k^{t+1}_{i, 2} - \ExpSub{k}{k^{t+1}_{i, 2}}}^2} \\
    &\quad + (1 - \probmega) (1 - \probavailable)\norm{\nabla f_i(x^{t+1}) - \nabla f_i(x^{t})}^2 \\
    &\quad + \left(\probmega \left(1 - \frac{b}{\probmega}\right)^2 + (1 - \probmega)\right) \norm{h^{t}_i - \nabla f_i(x^{t})}^2 \\
    &\overset{\eqref{auxiliary:variance_decomposition}}{=}\frac{\probmega}{\probavailable} \ExpSub{k}{\norm{k^{t+1}_{i, 1} - \ExpSub{k}{k^{t+1}_{i, 1}}}^2} \\
    &\quad + \frac{(1 - \probmega)}{\probavailable} \ExpSub{k}{\norm{k^{t+1}_{i, 2} - \ExpSub{k}{k^{t+1}_{i, 2}}}^2} \\
    &\quad + \frac{\probmega\left(1 - \probavailable\right)}{\probavailable} \norm{\nabla f_i(x^{t+1}) - \nabla f_i(x^{t}) - \frac{b}{\probmega} \left(h^t_i - \nabla f_i(x^{t})\right)}^2 \\
    &\quad + \frac{(1 - \probmega)(1 - \probavailable)}{\probavailable} \norm{\nabla f_i(x^{t+1}) - \nabla f_i(x^{t})}^2 \\
    &\quad + \left(\probmega \left(1 - \frac{b}{\probmega}\right)^2 + (1 - \probmega)\right) \norm{h^{t}_i - \nabla f_i(x^{t})}^2 \\
    &\overset{\eqref{auxiliary:jensen_inequality}}{\leq} \frac{\probmega}{\probavailable} \ExpSub{k}{\norm{k^{t+1}_{i, 1} - \ExpSub{k}{k^{t+1}_{i, 1}}}^2} \\
    &\quad + \frac{(1 - \probmega)}{\probavailable} \ExpSub{k}{\norm{k^{t+1}_{i, 2} - \ExpSub{k}{k^{t+1}_{i, 2}}}^2} \\
    &\quad + \frac{2(1 - \probavailable)}{\probavailable} \norm{\nabla f_i(x^{t+1}) - \nabla f_i(x^{t})}^2 \\
    &\quad + \frac{2\left(1 - \probavailable\right)b^2}{\probmega \probavailable} \norm{h^t_i - \nabla f_i(x^{t})}^2 + \left(\probmega \left(1 - \frac{b}{\probmega}\right)^2 + (1 - \probmega)\right) \norm{h^{t}_i - \nabla f_i(x^{t})}^2.
  \end{align*}
  Using \eqref{eq:sync_mvr:k_1} and \eqref{eq:sync_mvr:k_2}, we get
  \begin{align*}
    &\ExpSub{k}{\ExpSub{\probavailable}{\ExpSub{\probmega}{\norm{h^{t+1}_i - \nabla f_i(x^{t+1})}^2}}} \\
    &\leq \frac{2 b^2 \sigma^2}{\probavailable \probmega B'} + \frac{2 \probmega L_{\sigma}^2}{\probavailable B'}\left(1 - \frac{b}{\probmega}\right)^2 \norm{x^{t+1} - x^{t}}^2 \\
    &\quad + \frac{(1 - \probmega) L_{\sigma}^2}{\probavailable B} \norm{x^{t+1} - x^{t}}^2 \\
    &\quad + \frac{2(1 - \probavailable)}{\probavailable} \norm{\nabla f_i(x^{t+1}) - \nabla f_i(x^{t})}^2 \\
    &\quad + \frac{2\left(1 - \probavailable\right)b^2}{\probmega \probavailable} \norm{h^t_i - \nabla f_i(x^{t})}^2 + \left(\probmega \left(1 - \frac{b}{\probmega}\right)^2 + (1 - \probmega)\right) \norm{h^{t}_i - \nabla f_i(x^{t})}^2.
  \end{align*}
  Next, due to Assumption~\ref{ass:nodes_lipschitz_constant}, we obtain
  \begin{align*}
    &\ExpSub{k}{\ExpSub{\probavailable}{\ExpSub{\probmega}{\norm{h^{t+1}_i - \nabla f_i(x^{t+1})}^2}}} \\
    &\leq \frac{2 b^2 \sigma^2}{\probavailable \probmega B'} + \left(\frac{2 \probmega L_{\sigma}^2}{\probavailable B'}\left(1 - \frac{b}{\probmega}\right)^2 + \frac{(1 - \probmega) L_{\sigma}^2}{\probavailable B} + \frac{2(1 - \probavailable) L_{i}^2}{\probavailable}\right) \norm{x^{t+1} - x^{t}}^2 \\
    &\quad + \frac{2\left(1 - \probavailable\right)b^2}{\probmega \probavailable} \norm{h^t_i - \nabla f_i(x^{t})}^2 + \left(\probmega \left(1 - \frac{b}{\probmega}\right)^2 + (1 - \probmega)\right) \norm{h^{t}_i - \nabla f_i(x^{t})}^2.
  \end{align*}
  The third inequality can be proved with the help of \eqref{eq:sync_mvr:k_2} and Assumption~\ref{ass:nodes_lipschitz_constant}.
  \begin{align*}
    &\ExpSub{k}{\norm{k^{t+1}_{i, 2}}^2} \\
    &\overset{\eqref{auxiliary:variance_decomposition}}{=}\ExpSub{k}{\norm{k^{t+1}_{i, 2} - \ExpSub{k}{k^{t+1}_{i, 2}}}^2} + \norm{\nabla f_i(x^{t+1}) - \nabla f_i(x^{t})}^2 \\
    &\leq \frac{L_{\sigma}^2}{B} \norm{x^{t+1} - x^{t}}^2 + \norm{\nabla f_i(x^{t+1}) - \nabla f_i(x^{t})}^2 \\
    &\leq \left(\frac{L_{\sigma}^2}{B} + L_{i}^2\right) \norm{x^{t+1} - x^{t}}^2.
  \end{align*}
\end{proof}

\CONVERGENCESYNCMVR*

\begin{proof}
  Due to Lemma \ref{lemma:page_lemma} and the update step from Line~\ref{alg:main_algorithm:x_update} in Algorithm~\ref{alg:main_algorithm_mvr_sync}, we have
  \begin{align*}
    &\ExpSub{t+1}{f(x^{t + 1})} \nonumber\\
    &\leq \ExpSub{t+1}{f(x^t) - \frac{\gamma}{2}\norm{\nabla f(x^t)}^2 - \left(\frac{1}{2\gamma} - \frac{L}{2}\right)
      \norm{x^{t+1} - x^t}^2 + \frac{\gamma}{2}\norm{g^{t} - \nabla f(x^t)}^2} \nonumber \\
      &= \ExpSub{t+1}{f(x^t) - \frac{\gamma}{2}\norm{\nabla f(x^t)}^2 - \left(\frac{1}{2\gamma} - \frac{L}{2}\right)
      \norm{x^{t+1} - x^t}^2 + \frac{\gamma}{2}\norm{g^{t} - h^t + h^t - \nabla f(x^t)}^2} \nonumber \\
      &\overset{\eqref{auxiliary:variance_decomposition}}{\leq} \ExpSub{t+1}{f(x^t) - \frac{\gamma}{2}\norm{\nabla f(x^t)}^2 - \left(\frac{1}{2\gamma} - \frac{L}{2}\right)
      \norm{x^{t+1} - x^t}^2 + \gamma\left(\norm{g^{t} - h^t}^2 + \norm{h^t - \nabla f(x^t)}^2}\right). \nonumber \\
      \nonumber
  \end{align*}
  Let us fix constants $\kappa, \eta, \nu, \rho \in [0,\infty)$ that we will define later. Considering Lemma~\ref{lemma:sync:prob_compressor}, Lemma~\ref{lemma:sync_mvr}, and the law of total expectation, we obtain
  \begin{align*}
    &\Exp{f(x^{t + 1})} + \kappa \Exp{\norm{g^{t+1} - h^{t+1}}^2} + \eta \Exp{\frac{1}{n}\sum_{i=1}^n\norm{g^{t+1}_i - h^{t+1}_i}^2}\\
    &\quad  + \nu \Exp{\norm{h^{t+1} - \nabla f(x^{t+1})}^2} + \rho \Exp{\frac{1}{n}\sum_{i=1}^n\norm{h^{t+1}_i - \nabla f_i(x^{t+1})}^2}\\
    &\leq \Exp{f(x^t) - \frac{\gamma}{2}\norm{\nabla f(x^t)}^2 - \left(\frac{1}{2\gamma} - \frac{L}{2}\right)
    \norm{x^{t+1} - x^t}^2 + \gamma\left(\norm{g^{t} - h^t}^2 + \norm{h^t - \nabla f(x^t)}^2\right)}\nonumber\\
    &\quad + \kappa \Exp{\ExpSub{k}{\ExpSub{\cC}{\ExpSub{\probavailable}{\ExpSub{\probmega}{\norm{g^{t+1} - h^{t+1}}^2}}}}} \\
    &\quad + \eta \Exp{\ExpSub{k}{\ExpSub{\cC}{\ExpSub{\probavailable}{\ExpSub{\probmega}{\frac{1}{n}\sum_{i=1}^n\norm{g^{t+1}_i - h^{t+1}_i}^2}}}}}\\
    &\quad  + \nu \Exp{\ExpSub{k}{\ExpSub{\probavailable}{\ExpSub{\probmega}{\norm{h^{t+1} - \nabla f(x^{t+1})}^2}}}} \\
    &\quad + \rho \Exp{\ExpSub{k}{\ExpSub{\probavailable}{\ExpSub{\probmega}{\frac{1}{n}\sum_{i=1}^n\norm{h^{t+1}_i - \nabla f_i(x^{t+1})}^2}}}} \\
    &\leq \Exp{f(x^t) - \frac{\gamma}{2}\norm{\nabla f(x^t)}^2 - \left(\frac{1}{2\gamma} - \frac{L}{2}\right)
    \norm{x^{t+1} - x^t}^2 + \gamma\left(\norm{g^{t} - h^t}^2 + \norm{h^t - \nabla f(x^t)}^2\right)}\nonumber\\
    &\quad  + \kappa {\rm E}\Bigg(\frac{2 \left(1 - \probmega\right) \omega}{n \probavailable}\left(\frac{L_{\sigma}^2}{B} + \widehat{L}^2\right) \norm{x^{t+1} - x^{t}}^2 \\
    &\qquad\quad + \left(\frac{\left(\probavailable - \probpairaa\right)a^2}{n^2 \probavailable^2} + \frac{2 \left(1 - \probmega\right) a^2 \omega}{n^2 \probavailable}\right)\sum_{i=1}^n\norm{g^{t}_i - h^{t}_i}^2 + (1 - a)^2\norm{g^{t} - h^{t}}^2\Bigg) \\
      &\quad  + \eta {\rm E}\Bigg(\frac{2\left(1 - \probmega\right) \omega}{\probavailable} \left(\frac{L_{\sigma}^2}{B} + \widehat{L}^2\right) \norm{x^{t+1} - x^{t}}^2 \\
      &\qquad\quad+ \left(\frac{(1 - \probavailable) a^2}{\probavailable} + \frac{2 \left(1 - \probmega\right) a^2 \omega}{\probavailable}\right)\frac{1}{n}\sum_{i=1}^n\norm{g^{t}_{i} - h^{t}_{i}}^2 + (1 - a)^2 \norm{g^{t}_{i} - h^{t}_{i}}^2\Bigg) \\
      &\quad  + \nu {\rm E}\Bigg(\frac{2 b^2 \sigma^2}{n \probmega \probavailable B'} + \left(\frac{2 \probmega L_{\sigma}^2}{n \probavailable B'}\left(1 - \frac{b}{\probmega}\right)^2 + \frac{\left(1 - \probmega\right)L_{\sigma}^2}{n \probavailable B} + \frac{2\left(\probavailable - \probpairaa\right) \widehat{L}^2}{n \probavailable^2}\right)\norm{x^{t+1} - x^{t}}^2\\
      &\qquad\quad + \frac{2\left(\probavailable - \probpairaa\right) b^2}{n^2 \probavailable^2 \probmega}\sum_{i=1}^n\norm{h^t_i -  \nabla f_i(x^{t})}^2 + \left(\probmega \left(1 - \frac{b}{\probmega}\right)^2 + (1 - \probmega)\right) \norm{h^{t} - \nabla f(x^{t})}^2\Bigg) \\
      &\quad  + \rho {\rm E}\Bigg(\frac{2 b^2 \sigma^2}{\probavailable \probmega B'} + \left(\frac{2 \probmega L_{\sigma}^2}{\probavailable B'}\left(1 - \frac{b}{\probmega}\right)^2 + \frac{(1 - \probmega) L_{\sigma}^2}{\probavailable B} + \frac{2(1 - \probavailable) \widehat{L}^2}{\probavailable}\right) \norm{x^{t+1} - x^{t}}^2\\
      &\qquad\quad + \frac{2\left(1 - \probavailable\right)b^2}{n \probmega \probavailable} \sum_{i=1}^n\norm{h^t_i - \nabla f_i(x^{t})}^2 + \left(\probmega \left(1 - \frac{b}{\probmega}\right)^2 + (1 - \probmega)\right) \frac{1}{n} \sum_{i=1}^n \norm{h^{t}_i - \nabla f_i(x^{t})}^2\Bigg).
  \end{align*}

  Let us simplify the last inequality. Since $B' \geq B$ and $b = \frac{\probmega \probavailable}{2 - \probavailable} \leq \probmega,$ we have $1 - \probmega \leq 1,$ $$\frac{2 \probmega L_{\sigma}^2}{\probavailable B'}\left(1 - \frac{b}{\probmega}\right)^2 \leq \frac{2 \probmega L_{\sigma}^2}{\probavailable B},$$ $$\left(\probmega \left(1 - \frac{b}{\probmega}\right)^2 + (1 - \probmega)\right) \leq 1 - b,$$ and $$\left(\frac{2\left(1 - \probavailable\right)b^2}{\probmega \probavailable} + \probmega \left(1 - \frac{b}{\probmega}\right)^2 + (1 - \probmega)\right) \leq 1 - b.$$
  Thus
  \begin{align*}
    &\Exp{f(x^{t + 1})} + \kappa \Exp{\norm{g^{t+1} - h^{t+1}}^2} + \eta \Exp{\frac{1}{n}\sum_{i=1}^n\norm{g^{t+1}_i - h^{t+1}_i}^2}\\
    &\quad  + \nu \Exp{\norm{h^{t+1} - \nabla f(x^{t+1})}^2} + \rho \Exp{\frac{1}{n}\sum_{i=1}^n\norm{h^{t+1}_i - \nabla f_i(x^{t+1})}^2}\\
    &\leq \Exp{f(x^t) - \frac{\gamma}{2}\norm{\nabla f(x^t)}^2 - \left(\frac{1}{2\gamma} - \frac{L}{2}\right)
    \norm{x^{t+1} - x^t}^2 + \gamma\left(\norm{g^{t} - h^t}^2 + \norm{h^t - \nabla f(x^t)}^2\right)}\nonumber\\
    &\quad  + \kappa {\rm E}\Bigg(\frac{2 \omega}{n \probavailable}\left(\frac{L_{\sigma}^2}{B} + \widehat{L}^2\right) \norm{x^{t+1} - x^{t}}^2 \\
    &\qquad\quad + \frac{\left(\left(2 \omega + 1\right)\probavailable - \probpairaa\right)a^2}{n^2 \probavailable^2}\sum_{i=1}^n\norm{g^{t}_i - h^{t}_i}^2 + (1 - a)^2\norm{g^{t} - h^{t}}^2\Bigg) \\
      &\quad  + \eta {\rm E}\Bigg(\frac{2 \omega}{\probavailable} \left(\frac{L_{\sigma}^2}{B} + \widehat{L}^2\right) \norm{x^{t+1} - x^{t}}^2 \\
      &\qquad\quad+ \frac{(2\omega + 1 - \probavailable) a^2}{\probavailable}\frac{1}{n}\sum_{i=1}^n\norm{g^{t}_{i} - h^{t}_{i}}^2 + (1 - a)^2 \norm{g^{t}_{i} - h^{t}_{i}}^2\Bigg) \\
      &\quad  + \nu {\rm E}\Bigg(\frac{2 b^2 \sigma^2}{n \probmega \probavailable B'} + \left(\frac{2 L_{\sigma}^2}{n \probavailable B} + \frac{2\left(\probavailable - \probpairaa\right) \widehat{L}^2}{n \probavailable^2}\right)\norm{x^{t+1} - x^{t}}^2\\
      &\qquad\quad + \frac{2\left(\probavailable - \probpairaa\right) b^2}{n^2 \probavailable^2 \probmega}\sum_{i=1}^n\norm{h^t_i -  \nabla f_i(x^{t})}^2 + \left(1 - b\right) \norm{h^{t} - \nabla f(x^{t})}^2\Bigg) \\
      &\quad  + \rho {\rm E}\Bigg(\frac{2 b^2 \sigma^2}{\probavailable \probmega B'} + \left(\frac{2 L_{\sigma}^2}{\probavailable B} + \frac{2(1 - \probavailable) \widehat{L}^2}{\probavailable}\right) \norm{x^{t+1} - x^{t}}^2\\
      &\qquad\quad + \left(1 - b\right) \frac{1}{n} \sum_{i=1}^n \norm{h^{t}_i - \nabla f_i(x^{t})}^2\Bigg).
  \end{align*}
  After rearranging the terms, we get
  \begin{align*}
    &\Exp{f(x^{t + 1})} + \kappa \Exp{\norm{g^{t+1} - h^{t+1}}^2} + \eta \Exp{\frac{1}{n}\sum_{i=1}^n\norm{g^{t+1}_i - h^{t+1}_i}^2}\\
    &\quad  + \nu \Exp{\norm{h^{t+1} - \nabla f(x^{t+1})}^2} + \rho \Exp{\frac{1}{n}\sum_{i=1}^n\norm{h^{t+1}_i - \nabla f_i(x^{t+1})}^2}\\
    &\leq \Exp{f(x^t)} - \frac{\gamma}{2}\Exp{\norm{\nabla f(x^t)}^2} \\
    &\quad - \Bigg(\frac{1}{2\gamma} - \frac{L}{2} - \frac{2 \kappa  \omega}{n \probavailable}\left(\frac{L_{\sigma}^2}{B} + \widehat{L}^2\right) - \frac{2\eta \omega}{\probavailable} \left(\frac{L_{\sigma}^2}{B} + \widehat{L}^2\right) \\
    &\qquad\quad - \nu \left(\frac{2 L_{\sigma}^2}{n \probavailable B} + \frac{2\left(\probavailable - \probpairaa\right) \widehat{L}^2}{n \probavailable^2}\right) - \rho \left(\frac{2 L_{\sigma}^2}{\probavailable B} + \frac{2(1 - \probavailable) \widehat{L}^2}{\probavailable}\right)\Bigg) \Exp{\norm{x^{t+1} - x^t}^2} \\
    &\quad + \left(\gamma + \kappa \left(1 - a\right)^2\right) \Exp{\norm{g^{t} - h^{t}}^2} \\
    &\quad + \Bigg(\kappa \frac{\left(\left(2 \omega + 1\right)\probavailable - \probpairaa\right)a^2}{n \probavailable^2} + \eta \left(\frac{(2\omega + 1 - \probavailable) a^2}{\probavailable} + (1 - a)^2\right)\Bigg)\Exp{\frac{1}{n}\sum_{i=1}^n\norm{g^{t}_i - h^{t}_i}^2} \\
    &\quad + \left(\gamma + \nu \left(1 - b\right)\right) \Exp{\norm{h^{t} - \nabla f(x^{t})}^2} \\
    &\quad + \Bigg(\nu \frac{2\left(\probavailable - \probpairaa\right) b^2}{n \probavailable^2 \probmega} + \rho (1 - b)\Bigg)\Exp{\frac{1}{n}\sum_{i=1}^n\norm{h^{t}_i - \nabla f_i(x^{t})}^2} \\
    &\quad + \left(\frac{2 \nu b^2}{n \probmega \probavailable} + \frac{2 \rho b^2}{\probavailable \probmega}\right) \frac{\sigma^2}{B'}.
  \end{align*}
  Let us take $\kappa = \frac{\gamma}{a},$ thus $\gamma + \kappa \left(1 - a\right)^2 \leq \kappa$ and
  \begin{align*}
    &\Exp{f(x^{t + 1})} + \frac{\gamma}{a} \Exp{\norm{g^{t+1} - h^{t+1}}^2} + \eta \Exp{\frac{1}{n}\sum_{i=1}^n\norm{g^{t+1}_i - h^{t+1}_i}^2}\\
    &\quad  + \nu \Exp{\norm{h^{t+1} - \nabla f(x^{t+1})}^2} + \rho \Exp{\frac{1}{n}\sum_{i=1}^n\norm{h^{t+1}_i - \nabla f_i(x^{t+1})}^2}\\
    &\leq \Exp{f(x^t)} - \frac{\gamma}{2}\Exp{\norm{\nabla f(x^t)}^2} \\
    &\quad - \Bigg(\frac{1}{2\gamma} - \frac{L}{2} - \frac{2 \gamma \omega}{a n \probavailable}\left(\frac{L_{\sigma}^2}{B} + \widehat{L}^2\right) - \frac{2\eta \omega}{\probavailable} \left(\frac{L_{\sigma}^2}{B} + \widehat{L}^2\right) \\
    &\qquad\quad - \nu \left(\frac{2 L_{\sigma}^2}{n \probavailable B} + \frac{2\left(\probavailable - \probpairaa\right) \widehat{L}^2}{n \probavailable^2}\right) - \rho \left(\frac{2 L_{\sigma}^2}{\probavailable B} + \frac{2(1 - \probavailable) \widehat{L}^2}{\probavailable}\right)\Bigg) \Exp{\norm{x^{t+1} - x^t}^2} \\
    &\quad + \frac{\gamma}{a} \Exp{\norm{g^{t} - h^{t}}^2} \\
    &\quad + \Bigg( \frac{\gamma\left(\left(2 \omega + 1\right)\probavailable - \probpairaa\right)a}{n \probavailable^2} + \eta \left(\frac{(2\omega + 1 - \probavailable) a^2}{\probavailable} + (1 - a)^2\right)\Bigg)\Exp{\frac{1}{n}\sum_{i=1}^n\norm{g^{t}_i - h^{t}_i}^2} \\
    &\quad + \left(\gamma + \nu \left(1 - b\right)\right) \Exp{\norm{h^{t} - \nabla f(x^{t})}^2} \\
    &\quad + \Bigg(\nu \frac{2\left(\probavailable - \probpairaa\right) b^2}{n \probavailable^2 \probmega} + \rho (1 - b)\Bigg)\Exp{\frac{1}{n}\sum_{i=1}^n\norm{h^{t}_i - \nabla f_i(x^{t})}^2} \\
    &\quad + \left(\frac{2 \nu b^2}{n \probmega \probavailable} + \frac{2 \rho b^2}{\probavailable \probmega}\right) \frac{\sigma^2}{B'}.
  \end{align*}
  Next, since $a = \frac{\probavailable}{2\omega + 1},$ we have $\left(\frac{(2\omega + 1 - \probavailable) a^2}{\probavailable} + (1 - a)^2\right) \leq 1 - a.$ We the choice $\eta = \frac{\gamma\left(\left(2 \omega + 1\right)\probavailable - \probpairaa\right)}{n \probavailable^2},$ we guarantee $\frac{\gamma\left(\left(2 \omega + 1\right)\probavailable - \probpairaa\right)a}{n \probavailable^2} + \eta \left(\frac{(2\omega + 1 - \probavailable) a^2}{\probavailable} + (1 - a)^2\right) \leq \eta$ and 
  \begin{align*}
    &\Exp{f(x^{t + 1})} + \frac{\gamma \left(2 \omega + 1\right)}{\probavailable} \Exp{\norm{g^{t+1} - h^{t+1}}^2} + \frac{\gamma\left(\left(2 \omega + 1\right)\probavailable - \probpairaa\right)}{n \probavailable^2} \Exp{\frac{1}{n}\sum_{i=1}^n\norm{g^{t+1}_i - h^{t+1}_i}^2}\\
    &\quad  + \nu \Exp{\norm{h^{t+1} - \nabla f(x^{t+1})}^2} + \rho \Exp{\frac{1}{n}\sum_{i=1}^n\norm{h^{t+1}_i - \nabla f_i(x^{t+1})}^2}\\
    &\leq \Exp{f(x^t)} - \frac{\gamma}{2}\Exp{\norm{\nabla f(x^t)}^2} \\
    &\quad - \Bigg(\frac{1}{2\gamma} - \frac{L}{2} - \frac{2 \gamma \left(2\omega + 1\right) \omega}{n \probavailable^2}\left(\frac{L_{\sigma}^2}{B} + \widehat{L}^2\right) - \frac{2\gamma\left(\left(2 \omega + 1\right)\probavailable - \probpairaa\right)\omega}{n \probavailable^3} \left(\frac{L_{\sigma}^2}{B} + \widehat{L}^2\right) \\
    &\qquad\quad - \nu \left(\frac{2 L_{\sigma}^2}{n \probavailable B} + \frac{2\left(\probavailable - \probpairaa\right) \widehat{L}^2}{n \probavailable^2}\right) - \rho \left(\frac{2 L_{\sigma}^2}{\probavailable B} + \frac{2(1 - \probavailable) \widehat{L}^2}{\probavailable}\right)\Bigg) \Exp{\norm{x^{t+1} - x^t}^2} \\
    &\quad + \frac{\gamma \left(2 \omega + 1\right)}{\probavailable} \Exp{\norm{g^{t} - h^{t}}^2} + \frac{\gamma\left(\left(2 \omega + 1\right)\probavailable - \probpairaa\right)}{n \probavailable^2}\Exp{\frac{1}{n}\sum_{i=1}^n\norm{g^{t}_i - h^{t}_i}^2} \\
    &\quad + \left(\gamma + \nu \left(1 - b\right)\right) \Exp{\norm{h^{t} - \nabla f(x^{t})}^2} \\
    &\quad + \Bigg(\nu \frac{2\left(\probavailable - \probpairaa\right) b^2}{n \probavailable^2 \probmega} + \rho (1 - b)\Bigg)\Exp{\frac{1}{n}\sum_{i=1}^n\norm{h^{t}_i - \nabla f_i(x^{t})}^2} \\
    &\quad + \left(\frac{2 \nu b^2}{n \probmega \probavailable} + \frac{2 \rho b^2}{\probavailable \probmega}\right) \frac{\sigma^2}{B'} \\
    &\leq \Exp{f(x^t)} - \frac{\gamma}{2}\Exp{\norm{\nabla f(x^t)}^2} \\
    &\quad - \Bigg(\frac{1}{2\gamma} - \frac{L}{2} - \frac{4 \gamma \left(2\omega + 1\right) \omega}{n \probavailable^2}\left(\frac{L_{\sigma}^2}{B} + \widehat{L}^2\right) \\
    &\qquad\quad - \nu \left(\frac{2 L_{\sigma}^2}{n \probavailable B} + \frac{2\left(\probavailable - \probpairaa\right) \widehat{L}^2}{n \probavailable^2}\right) - \rho \left(\frac{2 L_{\sigma}^2}{\probavailable B} + \frac{2(1 - \probavailable) \widehat{L}^2}{\probavailable}\right)\Bigg) \Exp{\norm{x^{t+1} - x^t}^2} \\
    &\quad + \frac{\gamma \left(2 \omega + 1\right)}{\probavailable} \Exp{\norm{g^{t} - h^{t}}^2} + \frac{\gamma\left(\left(2 \omega + 1\right)\probavailable - \probpairaa\right)}{n \probavailable^2}\Exp{\frac{1}{n}\sum_{i=1}^n\norm{g^{t}_i - h^{t}_i}^2} \\
    &\quad + \left(\gamma + \nu \left(1 - b\right)\right) \Exp{\norm{h^{t} - \nabla f(x^{t})}^2} \\
    &\quad + \Bigg(\nu \frac{2\left(\probavailable - \probpairaa\right) b^2}{n \probavailable^2 \probmega} + \rho (1 - b)\Bigg)\Exp{\frac{1}{n}\sum_{i=1}^n\norm{h^{t}_i - \nabla f_i(x^{t})}^2} \\
    &\quad + \left(\frac{2 \nu b^2}{n \probmega \probavailable} + \frac{2 \rho b^2}{\probavailable \probmega}\right) \frac{\sigma^2}{B'},
  \end{align*}
  where simplified the term using $\probpairaa \geq 0.$ Let us take $\nu = \frac{\gamma}{b}$ to obtain
  \begin{align*}
    &\Exp{f(x^{t + 1})} + \frac{\gamma \left(2 \omega + 1\right)}{\probavailable} \Exp{\norm{g^{t+1} - h^{t+1}}^2} + \frac{\gamma\left(\left(2 \omega + 1\right)\probavailable - \probpairaa\right)}{n \probavailable^2} \Exp{\frac{1}{n}\sum_{i=1}^n\norm{g^{t+1}_i - h^{t+1}_i}^2}\\
    &\quad  + \frac{\gamma}{b} \Exp{\norm{h^{t+1} - \nabla f(x^{t+1})}^2} + \rho \Exp{\frac{1}{n}\sum_{i=1}^n\norm{h^{t+1}_i - \nabla f_i(x^{t+1})}^2}\\
    &\leq \Exp{f(x^t)} - \frac{\gamma}{2}\Exp{\norm{\nabla f(x^t)}^2} \\
    &\quad - \Bigg(\frac{1}{2\gamma} - \frac{L}{2} - \frac{4 \gamma \left(2\omega + 1\right) \omega}{n \probavailable^2}\left(\frac{L_{\sigma}^2}{B} + \widehat{L}^2\right) \\
    &\qquad\quad - \left(\frac{2 \gamma L_{\sigma}^2}{b n \probavailable B} + \frac{2\gamma\left(\probavailable - \probpairaa\right) \widehat{L}^2}{b n \probavailable^2}\right) - \rho \left(\frac{2 L_{\sigma}^2}{\probavailable B} + \frac{2(1 - \probavailable) \widehat{L}^2}{\probavailable}\right)\Bigg) \Exp{\norm{x^{t+1} - x^t}^2} \\
    &\quad + \frac{\gamma \left(2 \omega + 1\right)}{\probavailable} \Exp{\norm{g^{t} - h^{t}}^2} + \frac{\gamma\left(\left(2 \omega + 1\right)\probavailable - \probpairaa\right)}{n \probavailable^2}\Exp{\frac{1}{n}\sum_{i=1}^n\norm{g^{t}_i - h^{t}_i}^2} \\
    &\quad + \frac{\gamma}{b} \Exp{\norm{h^{t} - \nabla f(x^{t})}^2} \\
    &\quad + \Bigg(\frac{2 \gamma \left(\probavailable - \probpairaa\right) b}{n \probavailable^2 \probmega} + \rho (1 - b)\Bigg)\Exp{\frac{1}{n}\sum_{i=1}^n\norm{h^{t}_i - \nabla f_i(x^{t})}^2} \\
    &\quad + \left(\frac{2 \gamma b}{n \probmega \probavailable} + \frac{2 \rho b^2}{\probavailable \probmega}\right) \frac{\sigma^2}{B'}.
  \end{align*}
  Next, we take $\rho = \frac{2 \gamma \left(\probavailable - \probpairaa\right)}{n \probavailable^2 \probmega},$ thus
  \begin{align*}
    &\Exp{f(x^{t + 1})} + \frac{\gamma \left(2 \omega + 1\right)}{\probavailable} \Exp{\norm{g^{t+1} - h^{t+1}}^2} + \frac{\gamma\left(\left(2 \omega + 1\right)\probavailable - \probpairaa\right)}{n \probavailable^2} \Exp{\frac{1}{n}\sum_{i=1}^n\norm{g^{t+1}_i - h^{t+1}_i}^2}\\
    &\quad  + \frac{\gamma}{b} \Exp{\norm{h^{t+1} - \nabla f(x^{t+1})}^2} + \frac{2 \gamma \left(\probavailable - \probpairaa\right)}{n \probavailable^2 \probmega} \Exp{\frac{1}{n}\sum_{i=1}^n\norm{h^{t+1}_i - \nabla f_i(x^{t+1})}^2}\\
    &\leq \Exp{f(x^t)} - \frac{\gamma}{2}\Exp{\norm{\nabla f(x^t)}^2} \\
    &\quad - \Bigg(\frac{1}{2\gamma} - \frac{L}{2} - \frac{4 \gamma \left(2\omega + 1\right) \omega}{n \probavailable^2}\left(\frac{L_{\sigma}^2}{B} + \widehat{L}^2\right) \\
    &\qquad\quad - \left(\frac{2 \gamma L_{\sigma}^2}{b n \probavailable B} + \frac{2\gamma\left(\probavailable - \probpairaa\right) \widehat{L}^2}{b n \probavailable^2}\right) - \left(\frac{2 \gamma \left(\probavailable - \probpairaa\right)}{n \probavailable^2 \probmega}\right) \left(\frac{2 L_{\sigma}^2}{\probavailable B} + \frac{2(1 - \probavailable) \widehat{L}^2}{\probavailable}\right)\Bigg) \Exp{\norm{x^{t+1} - x^t}^2} \\
    &\quad + \frac{\gamma \left(2 \omega + 1\right)}{\probavailable} \Exp{\norm{g^{t} - h^{t}}^2} + \frac{\gamma\left(\left(2 \omega + 1\right)\probavailable - \probpairaa\right)}{n \probavailable^2}\Exp{\frac{1}{n}\sum_{i=1}^n\norm{g^{t}_i - h^{t}_i}^2} \\
    &\quad + \frac{\gamma}{b} \Exp{\norm{h^{t} - \nabla f(x^{t})}^2} + \frac{2 \gamma \left(\probavailable - \probpairaa\right)}{n \probavailable^2 \probmega}\Exp{\frac{1}{n}\sum_{i=1}^n\norm{h^{t}_i - \nabla f_i(x^{t})}^2} \\
    &\quad + \left(\frac{2 \gamma b}{n \probmega \probavailable} + \frac{4 \gamma \left(\probavailable - \probpairaa\right) b^2}{n \probavailable^3 \probmega^2}\right) \frac{\sigma^2}{B'}.
  \end{align*}
  Since $\frac{\probmega \probavailable}{2} \leq b \leq \probmega \probavailable$ and $1 - \probavailable \leq 1 - \frac{\probpairaa}{\probavailable}\leq 1,$ we get 
  \begin{align*}
    &\Exp{f(x^{t + 1})} + \frac{\gamma \left(2 \omega + 1\right)}{\probavailable} \Exp{\norm{g^{t+1} - h^{t+1}}^2} + \frac{\gamma\left(\left(2 \omega + 1\right)\probavailable - \probpairaa\right)}{n \probavailable^2} \Exp{\frac{1}{n}\sum_{i=1}^n\norm{g^{t+1}_i - h^{t+1}_i}^2}\\
    &\quad  + \frac{\gamma}{b} \Exp{\norm{h^{t+1} - \nabla f(x^{t+1})}^2} + \frac{2 \gamma \left(\probavailable - \probpairaa\right)}{n \probavailable^2 \probmega} \Exp{\frac{1}{n}\sum_{i=1}^n\norm{h^{t+1}_i - \nabla f_i(x^{t+1})}^2}\\
    &\leq \Exp{f(x^t)} - \frac{\gamma}{2}\Exp{\norm{\nabla f(x^t)}^2} \\
    &\quad - \Bigg(\frac{1}{2\gamma} - \frac{L}{2} - \frac{4 \gamma \left(2\omega + 1\right) \omega}{n \probavailable^2}\left(\frac{L_{\sigma}^2}{B} + \widehat{L}^2\right) \\
    &\qquad\quad - \left(\frac{4 \gamma L_{\sigma}^2}{n \probmega \probavailable^2 B} + \frac{4\gamma\left(\probavailable - \probpairaa\right) \widehat{L}^2}{n \probmega \probavailable^3}\right) - \left(\frac{4 \gamma L_{\sigma}^2}{n \probmega \probavailable^2 B} + \frac{4 \gamma (1 - \probavailable) \widehat{L}^2}{n \probmega \probavailable^2}\right)\Bigg) \Exp{\norm{x^{t+1} - x^t}^2} \\
    &\quad + \frac{\gamma \left(2 \omega + 1\right)}{\probavailable} \Exp{\norm{g^{t} - h^{t}}^2} + \frac{\gamma\left(\left(2 \omega + 1\right)\probavailable - \probpairaa\right)}{n \probavailable^2}\Exp{\frac{1}{n}\sum_{i=1}^n\norm{g^{t}_i - h^{t}_i}^2} \\
    &\quad + \frac{\gamma}{b} \Exp{\norm{h^{t} - \nabla f(x^{t})}^2} + \frac{2 \gamma \left(\probavailable - \probpairaa\right)}{n \probavailable^2 \probmega}\Exp{\frac{1}{n}\sum_{i=1}^n\norm{h^{t}_i - \nabla f_i(x^{t})}^2} \\
    &\quad + \frac{6 \gamma \sigma^2}{n B'} \\
    &\leq \Exp{f(x^t)} - \frac{\gamma}{2}\Exp{\norm{\nabla f(x^t)}^2} \\
    &\quad - \Bigg(\frac{1}{2\gamma} - \frac{L}{2} - \frac{4 \gamma \left(2\omega + 1\right) \omega}{n \probavailable^2}\left(\frac{L_{\sigma}^2}{B} + \widehat{L}^2\right) - \left(\frac{8 \gamma L_{\sigma}^2}{n \probmega \probavailable^2 B} + \frac{8\gamma\left(1 - \frac{\probpairaa}{\probavailable}\right) \widehat{L}^2}{n \probmega \probavailable^2}\right)\Bigg) \Exp{\norm{x^{t+1} - x^t}^2} \\
    &\quad + \frac{\gamma \left(2 \omega + 1\right)}{\probavailable} \Exp{\norm{g^{t} - h^{t}}^2} + \frac{\gamma\left(\left(2 \omega + 1\right)\probavailable - \probpairaa\right)}{n \probavailable^2}\Exp{\frac{1}{n}\sum_{i=1}^n\norm{g^{t}_i - h^{t}_i}^2} \\
    &\quad + \frac{\gamma}{b} \Exp{\norm{h^{t} - \nabla f(x^{t})}^2} + \frac{2 \gamma \left(\probavailable - \probpairaa\right)}{n \probavailable^2 \probmega}\Exp{\frac{1}{n}\sum_{i=1}^n\norm{h^{t}_i - \nabla f_i(x^{t})}^2} \\
    &\quad + \frac{6 \gamma \sigma^2}{n B'}.
  \end{align*}
  Using Lemma~\ref{lemma:gamma} and the assumption about $\gamma,$ we get
  \begin{align*}
    &\Exp{f(x^{t + 1})} + \frac{\gamma \left(2 \omega + 1\right)}{\probavailable} \Exp{\norm{g^{t+1} - h^{t+1}}^2} + \frac{\gamma\left(\left(2 \omega + 1\right)\probavailable - \probpairaa\right)}{n \probavailable^2} \Exp{\frac{1}{n}\sum_{i=1}^n\norm{g^{t+1}_i - h^{t+1}_i}^2}\\
    &\quad  + \frac{\gamma}{b} \Exp{\norm{h^{t+1} - \nabla f(x^{t+1})}^2} + \frac{2 \gamma \left(\probavailable - \probpairaa\right)}{n \probavailable^2 \probmega} \Exp{\frac{1}{n}\sum_{i=1}^n\norm{h^{t+1}_i - \nabla f_i(x^{t+1})}^2}\\
    &\leq \Exp{f(x^t)} - \frac{\gamma}{2}\Exp{\norm{\nabla f(x^t)}^2} \\
    &\quad + \frac{\gamma \left(2 \omega + 1\right)}{\probavailable} \Exp{\norm{g^{t} - h^{t}}^2} + \frac{\gamma\left(\left(2 \omega + 1\right)\probavailable - \probpairaa\right)}{n \probavailable^2}\Exp{\frac{1}{n}\sum_{i=1}^n\norm{g^{t}_i - h^{t}_i}^2} \\
    &\quad + \frac{\gamma}{b} \Exp{\norm{h^{t} - \nabla f(x^{t})}^2} + \frac{2 \gamma \left(\probavailable - \probpairaa\right)}{n \probavailable^2 \probmega}\Exp{\frac{1}{n}\sum_{i=1}^n\norm{h^{t}_i - \nabla f_i(x^{t})}^2} \\
    &\quad + \frac{6 \gamma \sigma^2}{n B'}.
  \end{align*}
  It is left to apply Lemma~\ref{lemma:good_recursion} with 
    \begin{eqnarray*}
      \Psi^t &=& \frac{\left(2 \omega + 1\right)}{\probavailable} \Exp{\norm{g^{t} - h^{t}}^2} + \frac{\left(\left(2 \omega + 1\right)\probavailable - \probpairaa\right)}{n \probavailable^2}\Exp{\frac{1}{n}\sum_{i=1}^n\norm{g^{t}_i - h^{t}_i}^2} \\
      &+& \frac{1}{b} \Exp{\norm{h^{t} - \nabla f(x^{t})}^2} + \frac{2 \left(1 - \frac{\probpairaa}{\probavailable}\right)}{n \probavailable \probmega}\Exp{\frac{1}{n}\sum_{i=1}^n\norm{h^{t}_i - \nabla f_i(x^{t})}^2}
    \end{eqnarray*}
    and $C = \frac{6 \sigma^2}{n B'}$
    to conclude the proof.
\end{proof}

\COROLLARYSYNCSTOCHASTIC*

\begin{proof}
  Due to the choice of $B',$ we have
  \begin{align*}
    \Exp{\norm{\nabla f(\widehat{x}^T)}^2} &\leq \frac{1}{T}\vast[2 \Delta_0\left(L + \sqrt{\frac{8 \left(2\omega + 1\right) \omega}{n \probavailable^2}\left(\widehat{L}^2 + \frac{L_{\sigma}^2}{B}\right) + \frac{16}{n \probmega \probavailable^2} \left(\left(1 - \frac{\probpairaa}{\probavailable}\right) \widehat{L}^2 + \frac{L_{\sigma}^2}{B}\right)}\right) \\
    &\quad + \frac{4}{\probmega \probavailable} \norm{h^{0} - \nabla f(x^{0})}^2 + \frac{4 \left(1 - \frac{\probpairaa}{\probavailable}\right)}{n \probmega \probavailable}\frac{1}{n}\sum_{i=1}^n\norm{h^{0}_i - \nabla f_i(x^{0})}^2 \vast] \\
    &\quad + \frac{2 \varepsilon}{3}.
  \end{align*}
  Using 
  \begin{align*}
    \Exp{\norm{h^{0} - \nabla f(x^{0})}^2} = \Exp{\norm{\frac{1}{n} \sum_{i=1}^n \frac{1}{B_{\textnormal{init}}} \sum_{k = 1}^{B_{\textnormal{init}}} \nabla f_i(x^0; \xi^0_{ik}) - \nabla f(x^{0})}^2} \leq \frac{\sigma^2}{n B_{\textnormal{init}}}
  \end{align*}
  and 
  \begin{align*}
    \frac{1}{n^2} \sum_{i=1}^n\Exp{\norm{h^{0}_i - \nabla f_i(x^{0})}^2} = \frac{1}{n^2} \sum_{i=1}^n \Exp{\norm{\frac{1}{B_{\textnormal{init}}} \sum_{k = 1}^{B_{\textnormal{init}}} \nabla f_i(x^0; \xi^0_{ik}) - \nabla f_i(x^{0})}^2} \leq \frac{\sigma^2}{n B_{\textnormal{init}}},
  \end{align*}
  we have
  \begin{align*}
    \Exp{\norm{\nabla f(\widehat{x}^T)}^2} &\leq \frac{1}{T}\vast[2 \Delta_0\left(L + \sqrt{\frac{8 \left(2\omega + 1\right) \omega}{n \probavailable^2}\left(\widehat{L}^2 + \frac{L_{\sigma}^2}{B}\right) + \frac{16}{n \probmega \probavailable^2} \left(\left(1 - \frac{\probpairaa}{\probavailable}\right) \widehat{L}^2 + \frac{L_{\sigma}^2}{B}\right)}\right) \\
    &\quad + \frac{8 \sigma^2}{n \probmega \probavailable B_{\textnormal{init}}} \vast] \\
    &\quad + \frac{2 \varepsilon}{3}.
  \end{align*}
  Therefore, we can take the following $T$ to get $\varepsilon$--solution.
  \begin{align*}
    T = \cO\left(\frac{1}{\varepsilon}\vast[\Delta_0\left(L + \sqrt{\frac{\omega^2}{n \probavailable^2}\left(\widehat{L}^2 + \frac{L_{\sigma}^2}{B}\right) + \frac{1}{n \probmega \probavailable^2} \left(\widehat{L}^2 + \frac{L_{\sigma}^2}{B}\right)}\right) + \frac{\sigma^2}{n \probmega \probavailable B_{\textnormal{init}}} \vast]\right)
  \end{align*}
  Considering the choice of $\probmega$ and $B_{\textnormal{init}},$ we obtain
  \begin{align*}
    T &=\cO\left(\frac{1}{\varepsilon}\vast[\Delta_0\left(L + \left(\frac{\omega}{\probavailable\sqrt{n}} + \sqrt{\frac{d}{\probavailable^2 \zeta_{\cC} n}}\right) \left(\widehat{L} + \frac{L_{\sigma}}{\sqrt{B}}\right)  + \frac{\sigma}{\probavailable \sqrt{\varepsilon} n} \left(\frac{\widehat{L}}{\sqrt{B}} + \frac{L_{\sigma}}{B}\right)\right) + \frac{\sigma^2}{n \probmega \probavailable B_{\textnormal{init}}} \vast]\right) \\
    &=\cO\left(\frac{\Delta_0}{\varepsilon}\vast[L + \left(\frac{\omega}{\probavailable\sqrt{n}} + \sqrt{\frac{d}{\probavailable^2 \zeta_{\cC} n}}\right) \left(\widehat{L} + \frac{L_{\sigma}}{\sqrt{B}}\right)  + \frac{\sigma}{\probavailable \sqrt{\varepsilon} n} \left(\frac{\widehat{L}}{\sqrt{B}} + \frac{L_{\sigma}}{B}\right)\vast] + \frac{\sigma^2}{\sqrt{\probavailable} n \varepsilon B}\right).
  \end{align*}

  The expected communication complexity equals $\cO\left(d + \probmega d + (1 - \probmega) \zeta_{\cC} \right) = \cO\left(d + \zeta_{\cC} \right)$ and the expected number of stochastic gradient calculations per node equals $\cO\left(B_{\textnormal{init}} + \probmega B' + (1 - \probmega) B\right) = \cO\left(B_{\textnormal{init}} + B\right).$
\end{proof}

\CONVERGENCESYNCPLMVR*

\begin{proof}
  Let us fix constants $\kappa, \eta, \nu, \rho \in [0,\infty)$ that we will define later. As in the proof of Theorem~\ref{theorem:sync_stochastic}, we can get
  \begin{align*}
    &\Exp{f(x^{t + 1})} + \kappa \Exp{\norm{g^{t+1} - h^{t+1}}^2} + \eta \Exp{\frac{1}{n}\sum_{i=1}^n\norm{g^{t+1}_i - h^{t+1}_i}^2}\\
    &\quad  + \nu \Exp{\norm{h^{t+1} - \nabla f(x^{t+1})}^2} + \rho \Exp{\frac{1}{n}\sum_{i=1}^n\norm{h^{t+1}_i - \nabla f_i(x^{t+1})}^2}\\
    &\leq \Exp{f(x^t)} - \frac{\gamma}{2}\Exp{\norm{\nabla f(x^t)}^2} \\
    &\quad - \Bigg(\frac{1}{2\gamma} - \frac{L}{2} - \frac{2 \kappa  \omega}{n \probavailable}\left(\frac{L_{\sigma}^2}{B} + \widehat{L}^2\right) - \frac{2\eta \omega}{\probavailable} \left(\frac{L_{\sigma}^2}{B} + \widehat{L}^2\right) \\
    &\qquad\quad - \nu \left(\frac{2 L_{\sigma}^2}{n \probavailable B} + \frac{2\left(\probavailable - \probpairaa\right) \widehat{L}^2}{n \probavailable^2}\right) - \rho \left(\frac{2 L_{\sigma}^2}{\probavailable B} + \frac{2(1 - \probavailable) \widehat{L}^2}{\probavailable}\right)\Bigg) \Exp{\norm{x^{t+1} - x^t}^2} \\
    &\quad + \left(\gamma + \kappa \left(1 - a\right)^2\right) \Exp{\norm{g^{t} - h^{t}}^2} \\
    &\quad + \Bigg(\kappa \frac{\left(\left(2 \omega + 1\right)\probavailable - \probpairaa\right)a^2}{n \probavailable^2} + \eta \left(\frac{(2\omega + 1 - \probavailable) a^2}{\probavailable} + (1 - a)^2\right)\Bigg)\Exp{\frac{1}{n}\sum_{i=1}^n\norm{g^{t}_i - h^{t}_i}^2} \\
    &\quad + \left(\gamma + \nu \left(1 - b\right)\right) \Exp{\norm{h^{t} - \nabla f(x^{t})}^2} \\
    &\quad + \Bigg(\nu \frac{2\left(\probavailable - \probpairaa\right) b^2}{n \probavailable^2 \probmega} + \rho (1 - b)\Bigg)\Exp{\frac{1}{n}\sum_{i=1}^n\norm{h^{t}_i - \nabla f_i(x^{t})}^2} \\
    &\quad + \left(\frac{2 \nu b^2}{n \probmega \probavailable} + \frac{2 \rho b^2}{\probavailable \probmega}\right) \frac{\sigma^2}{B'}.
  \end{align*}
  Let us take $\kappa = \frac{2\gamma}{a},$ thus $\gamma + \kappa \left(1 - a\right)^2 \leq \left(1 - \frac{a}{2}\right)\kappa$ and
  \begin{align*}
    &\Exp{f(x^{t + 1})} + \frac{2\gamma}{a} \Exp{\norm{g^{t+1} - h^{t+1}}^2} + \eta \Exp{\frac{1}{n}\sum_{i=1}^n\norm{g^{t+1}_i - h^{t+1}_i}^2}\\
    &\quad  + \nu \Exp{\norm{h^{t+1} - \nabla f(x^{t+1})}^2} + \rho \Exp{\frac{1}{n}\sum_{i=1}^n\norm{h^{t+1}_i - \nabla f_i(x^{t+1})}^2}\\
    &\leq \Exp{f(x^t)} - \frac{\gamma}{2}\Exp{\norm{\nabla f(x^t)}^2} \\
    &\quad - \Bigg(\frac{1}{2\gamma} - \frac{L}{2} - \frac{4 \gamma \omega}{a n \probavailable}\left(\frac{L_{\sigma}^2}{B} + \widehat{L}^2\right) - \frac{2\eta \omega}{\probavailable} \left(\frac{L_{\sigma}^2}{B} + \widehat{L}^2\right) \\
    &\qquad\quad - \nu \left(\frac{2 L_{\sigma}^2}{n \probavailable B} + \frac{2\left(\probavailable - \probpairaa\right) \widehat{L}^2}{n \probavailable^2}\right) - \rho \left(\frac{2 L_{\sigma}^2}{\probavailable B} + \frac{2(1 - \probavailable) \widehat{L}^2}{\probavailable}\right)\Bigg) \Exp{\norm{x^{t+1} - x^t}^2} \\
    &\quad + \left(1 - \frac{a}{2}\right)\frac{2\gamma}{a} \Exp{\norm{g^{t} - h^{t}}^2} \\
    &\quad + \Bigg( \frac{2\gamma\left(\left(2 \omega + 1\right)\probavailable - \probpairaa\right)a}{n \probavailable^2} + \eta \left(\frac{(2\omega + 1 - \probavailable) a^2}{\probavailable} + (1 - a)^2\right)\Bigg)\Exp{\frac{1}{n}\sum_{i=1}^n\norm{g^{t}_i - h^{t}_i}^2} \\
    &\quad + \left(\gamma + \nu \left(1 - b\right)\right) \Exp{\norm{h^{t} - \nabla f(x^{t})}^2} \\
    &\quad + \Bigg(\nu \frac{2\left(\probavailable - \probpairaa\right) b^2}{n \probavailable^2 \probmega} + \rho (1 - b)\Bigg)\Exp{\frac{1}{n}\sum_{i=1}^n\norm{h^{t}_i - \nabla f_i(x^{t})}^2} \\
    &\quad + \left(\frac{2 \nu b^2}{n \probmega \probavailable} + \frac{2 \rho b^2}{\probavailable \probmega}\right) \frac{\sigma^2}{B'}.
  \end{align*}
  Next, since $a = \frac{\probavailable}{2\omega + 1},$ we have $\left(\frac{(2\omega + 1 - \probavailable) a^2}{\probavailable} + (1 - a)^2\right) \leq 1 - a.$ We the choice $\eta = \frac{2\gamma\left(\left(2 \omega + 1\right)\probavailable - \probpairaa\right)}{n \probavailable^2},$ we guarantee $\frac{\gamma\left(\left(2 \omega + 1\right)\probavailable - \probpairaa\right)a}{n \probavailable^2} + \eta \left(\frac{(2\omega + 1 - \probavailable) a^2}{\probavailable} + (1 - a)^2\right) \leq \left(1 - \frac{a}{2}\right)\eta$ and 
  \begin{align*}
    &\Exp{f(x^{t + 1})} + \frac{2\gamma(2\omega + 1)}{\probavailable} \Exp{\norm{g^{t+1} - h^{t+1}}^2} + \frac{2\gamma\left(\left(2 \omega + 1\right)\probavailable - \probpairaa\right)}{n \probavailable^2} \Exp{\frac{1}{n}\sum_{i=1}^n\norm{g^{t+1}_i - h^{t+1}_i}^2}\\
    &\quad  + \nu \Exp{\norm{h^{t+1} - \nabla f(x^{t+1})}^2} + \rho \Exp{\frac{1}{n}\sum_{i=1}^n\norm{h^{t+1}_i - \nabla f_i(x^{t+1})}^2}\\
    &\leq \Exp{f(x^t)} - \frac{\gamma}{2}\Exp{\norm{\nabla f(x^t)}^2} \\
    &\quad - \Bigg(\frac{1}{2\gamma} - \frac{L}{2} - \frac{8 \gamma \left(2 \omega + 1\right)\omega}{n \probavailable^2}\left(\frac{L_{\sigma}^2}{B} + \widehat{L}^2\right) \\
    &\qquad\quad - \nu \left(\frac{2 L_{\sigma}^2}{n \probavailable B} + \frac{2\left(\probavailable - \probpairaa\right) \widehat{L}^2}{n \probavailable^2}\right) - \rho \left(\frac{2 L_{\sigma}^2}{\probavailable B} + \frac{2(1 - \probavailable) \widehat{L}^2}{\probavailable}\right)\Bigg) \Exp{\norm{x^{t+1} - x^t}^2} \\
    &\quad + \left(1 - \frac{a}{2}\right)\frac{2\gamma(2\omega + 1)}{\probavailable} \Exp{\norm{g^{t} - h^{t}}^2} \\
    &\quad + \left(1 - \frac{a}{2}\right)\frac{2\gamma\left(\left(2 \omega + 1\right)\probavailable - \probpairaa\right)}{n \probavailable^2}\Exp{\frac{1}{n}\sum_{i=1}^n\norm{g^{t}_i - h^{t}_i}^2} \\
    &\quad + \left(\gamma + \nu \left(1 - b\right)\right) \Exp{\norm{h^{t} - \nabla f(x^{t})}^2} \\
    &\quad + \Bigg(\nu \frac{2\left(\probavailable - \probpairaa\right) b^2}{n \probavailable^2 \probmega} + \rho (1 - b)\Bigg)\Exp{\frac{1}{n}\sum_{i=1}^n\norm{h^{t}_i - \nabla f_i(x^{t})}^2} \\
    &\quad + \left(\frac{2 \nu b^2}{n \probmega \probavailable} + \frac{2 \rho b^2}{\probavailable \probmega}\right) \frac{\sigma^2}{B'},
  \end{align*}
  where simplified the term using $\probpairaa \geq 0.$ Let us take $\nu = \frac{2\gamma}{b}$ to obtain
  \begin{align*}
    &\Exp{f(x^{t + 1})} + \frac{2\gamma(2\omega + 1)}{\probavailable} \Exp{\norm{g^{t+1} - h^{t+1}}^2} + \frac{2\gamma\left(\left(2 \omega + 1\right)\probavailable - \probpairaa\right)}{n \probavailable^2} \Exp{\frac{1}{n}\sum_{i=1}^n\norm{g^{t+1}_i - h^{t+1}_i}^2}\\
    &\quad  + \frac{2\gamma}{b} \Exp{\norm{h^{t+1} - \nabla f(x^{t+1})}^2} + \rho \Exp{\frac{1}{n}\sum_{i=1}^n\norm{h^{t+1}_i - \nabla f_i(x^{t+1})}^2}\\
    &\leq \Exp{f(x^t)} - \frac{\gamma}{2}\Exp{\norm{\nabla f(x^t)}^2} \\
    &\quad - \Bigg(\frac{1}{2\gamma} - \frac{L}{2} - \frac{8 \gamma \left(2 \omega + 1\right)\omega}{n \probavailable^2}\left(\frac{L_{\sigma}^2}{B} + \widehat{L}^2\right) \\
    &\qquad\quad - \left(\frac{4 \gamma L_{\sigma}^2}{b n \probavailable B} + \frac{4 \gamma\left(\probavailable - \probpairaa\right) \widehat{L}^2}{b n \probavailable^2}\right) - \rho \left(\frac{2 L_{\sigma}^2}{\probavailable B} + \frac{2(1 - \probavailable) \widehat{L}^2}{\probavailable}\right)\Bigg) \Exp{\norm{x^{t+1} - x^t}^2} \\
    &\quad + \left(1 - \frac{a}{2}\right)\frac{2\gamma(2\omega + 1)}{\probavailable} \Exp{\norm{g^{t} - h^{t}}^2} + \left(1 - \frac{a}{2}\right)\frac{2\gamma\left(\left(2 \omega + 1\right)\probavailable - \probpairaa\right)}{n \probavailable^2}\Exp{\frac{1}{n}\sum_{i=1}^n\norm{g^{t}_i - h^{t}_i}^2} \\
    &\quad + \left(1 - \frac{b}{2}\right)\frac{2\gamma}{b} \Exp{\norm{h^{t} - \nabla f(x^{t})}^2} \\
    &\quad + \Bigg(\frac{4 \gamma \left(\probavailable - \probpairaa\right) b}{n \probavailable^2 \probmega} + \rho (1 - b)\Bigg)\Exp{\frac{1}{n}\sum_{i=1}^n\norm{h^{t}_i - \nabla f_i(x^{t})}^2} \\
    &\quad + \left(\frac{4 \gamma b}{n \probmega \probavailable} + \frac{2 \rho b^2}{\probavailable \probmega}\right) \frac{\sigma^2}{B'},
  \end{align*}
  Next, we take $\rho = \frac{8 \gamma \left(\probavailable - \probpairaa\right)}{n \probavailable^2 \probmega},$ thus
  \begin{align*}
    &\Exp{f(x^{t + 1})} + \frac{2\gamma(2\omega + 1)}{\probavailable} \Exp{\norm{g^{t+1} - h^{t+1}}^2} + \frac{2\gamma\left(\left(2 \omega + 1\right)\probavailable - \probpairaa\right)}{n \probavailable^2} \Exp{\frac{1}{n}\sum_{i=1}^n\norm{g^{t+1}_i - h^{t+1}_i}^2}\\
    &\quad  + \frac{2\gamma}{b} \Exp{\norm{h^{t+1} - \nabla f(x^{t+1})}^2} + \frac{8 \gamma \left(\probavailable - \probpairaa\right)}{n \probavailable^2 \probmega} \Exp{\frac{1}{n}\sum_{i=1}^n\norm{h^{t+1}_i - \nabla f_i(x^{t+1})}^2}\\
    &\leq \Exp{f(x^t)} - \frac{\gamma}{2}\Exp{\norm{\nabla f(x^t)}^2} \\
    &\quad - \Bigg(\frac{1}{2\gamma} - \frac{L}{2} - \frac{8 \gamma \left(2 \omega + 1\right)\omega}{n \probavailable^2}\left(\frac{L_{\sigma}^2}{B} + \widehat{L}^2\right) \\
    &\qquad\quad - \left(\frac{4 \gamma L_{\sigma}^2}{b n \probavailable B} + \frac{4 \gamma\left(\probavailable - \probpairaa\right) \widehat{L}^2}{b n \probavailable^2}\right) - \left(\frac{8 \gamma \left(\probavailable - \probpairaa\right)}{n \probavailable^2 \probmega}\right) \left(\frac{2 L_{\sigma}^2}{\probavailable B} + \frac{2(1 - \probavailable) \widehat{L}^2}{\probavailable}\right)\Bigg) \Exp{\norm{x^{t+1} - x^t}^2} \\
    &\quad + \left(1 - \frac{a}{2}\right)\frac{2\gamma(2\omega + 1)}{\probavailable} \Exp{\norm{g^{t} - h^{t}}^2} + \left(1 - \frac{a}{2}\right)\frac{2\gamma\left(\left(2 \omega + 1\right)\probavailable - \probpairaa\right)}{n \probavailable^2}\Exp{\frac{1}{n}\sum_{i=1}^n\norm{g^{t}_i - h^{t}_i}^2} \\
    &\quad + \left(1 - \frac{b}{2}\right)\frac{2\gamma}{b} \Exp{\norm{h^{t} - \nabla f(x^{t})}^2} + \left(1 - \frac{b}{2}\right)\frac{8 \gamma \left(\probavailable - \probpairaa\right)}{n \probavailable^2 \probmega}\Exp{\frac{1}{n}\sum_{i=1}^n\norm{h^{t}_i - \nabla f_i(x^{t})}^2} \\
    &\quad + \left(\frac{4 \gamma b}{n \probmega \probavailable} + \frac{16 \gamma \left(\probavailable - \probpairaa\right) b^2}{n \probavailable^3 \probmega^2}\right) \frac{\sigma^2}{B'},
  \end{align*}
  Since $\frac{\probmega \probavailable}{2} \leq b \leq \probmega \probavailable$ and $1 - \probavailable \leq 1 - \frac{\probpairaa}{\probavailable}\leq 1,$ we get 
  \begin{align*}
    &\Exp{f(x^{t + 1})} + \frac{2\gamma(2\omega + 1)}{\probavailable} \Exp{\norm{g^{t+1} - h^{t+1}}^2} + \frac{2\gamma\left(\left(2 \omega + 1\right)\probavailable - \probpairaa\right)}{n \probavailable^2} \Exp{\frac{1}{n}\sum_{i=1}^n\norm{g^{t+1}_i - h^{t+1}_i}^2}\\
    &\quad  + \frac{2\gamma}{b} \Exp{\norm{h^{t+1} - \nabla f(x^{t+1})}^2} + \frac{8 \gamma \left(\probavailable - \probpairaa\right)}{n \probavailable^2 \probmega} \Exp{\frac{1}{n}\sum_{i=1}^n\norm{h^{t+1}_i - \nabla f_i(x^{t+1})}^2}\\
    &\leq \Exp{f(x^t)} - \frac{\gamma}{2}\Exp{\norm{\nabla f(x^t)}^2} \\
    &\quad - \Bigg(\frac{1}{2\gamma} - \frac{L}{2} - \frac{8 \gamma \left(2 \omega + 1\right)\omega}{n \probavailable^2}\left(\frac{L_{\sigma}^2}{B} + \widehat{L}^2\right) \\
    &\qquad\quad - \left(\frac{8 \gamma L_{\sigma}^2}{n \probmega \probavailable^2 B} + \frac{8 \gamma\left(\probavailable - \probpairaa\right) \widehat{L}^2}{n \probmega \probavailable^3 }\right) - \left(\frac{16 \gamma L_{\sigma}^2}{n \probmega \probavailable^2 B} + \frac{16 \gamma (1 - \probavailable) \widehat{L}^2}{n \probmega \probavailable^2}\right)\Bigg) \Exp{\norm{x^{t+1} - x^t}^2} \\
    &\quad + \left(1 - \frac{a}{2}\right)\frac{2\gamma(2\omega + 1)}{\probavailable} \Exp{\norm{g^{t} - h^{t}}^2} + \left(1 - \frac{a}{2}\right)\frac{2\gamma\left(\left(2 \omega + 1\right)\probavailable - \probpairaa\right)}{n \probavailable^2}\Exp{\frac{1}{n}\sum_{i=1}^n\norm{g^{t}_i - h^{t}_i}^2} \\
    &\quad + \left(1 - \frac{b}{2}\right)\frac{2\gamma}{b} \Exp{\norm{h^{t} - \nabla f(x^{t})}^2} + \left(1 - \frac{b}{2}\right)\frac{8 \gamma \left(\probavailable - \probpairaa\right)}{n \probavailable^2 \probmega}\Exp{\frac{1}{n}\sum_{i=1}^n\norm{h^{t}_i - \nabla f_i(x^{t})}^2} \\
    &\quad + \frac{20 \gamma \sigma^2}{n B'} \\
    &\leq \Exp{f(x^t)} - \frac{\gamma}{2}\Exp{\norm{\nabla f(x^t)}^2} \\
    &\quad - \Bigg(\frac{1}{2\gamma} - \frac{L}{2} - \frac{8 \gamma \left(2 \omega + 1\right)\omega}{n \probavailable^2}\left(\frac{L_{\sigma}^2}{B} + \widehat{L}^2\right) - \left(\frac{24 \gamma L_{\sigma}^2}{n \probmega \probavailable^2 B} + \frac{24 \gamma \left(1 - \frac{\probpairaa}{\probavailable}\right) \widehat{L}^2}{n \probmega \probavailable^2}\right)\Bigg) \Exp{\norm{x^{t+1} - x^t}^2} \\
    &\quad + \left(1 - \frac{a}{2}\right)\frac{2\gamma(2\omega + 1)}{\probavailable} \Exp{\norm{g^{t} - h^{t}}^2} + \left(1 - \frac{a}{2}\right)\frac{2\gamma\left(\left(2 \omega + 1\right)\probavailable - \probpairaa\right)}{n \probavailable^2}\Exp{\frac{1}{n}\sum_{i=1}^n\norm{g^{t}_i - h^{t}_i}^2} \\
    &\quad + \left(1 - \frac{b}{2}\right)\frac{2\gamma}{b} \Exp{\norm{h^{t} - \nabla f(x^{t})}^2} + \left(1 - \frac{b}{2}\right)\frac{8 \gamma \left(\probavailable - \probpairaa\right)}{n \probavailable^2 \probmega}\Exp{\frac{1}{n}\sum_{i=1}^n\norm{h^{t}_i - \nabla f_i(x^{t})}^2} \\
    &\quad + \frac{20 \gamma \sigma^2}{n B'}.
  \end{align*}
  Using Lemma~\ref{lemma:gamma} and the assumption about $\gamma,$ we get
  \begin{align*}
    &\Exp{f(x^{t + 1})} + \frac{2\gamma(2\omega + 1)}{\probavailable} \Exp{\norm{g^{t+1} - h^{t+1}}^2} + \frac{2\gamma\left(\left(2 \omega + 1\right)\probavailable - \probpairaa\right)}{n \probavailable^2} \Exp{\frac{1}{n}\sum_{i=1}^n\norm{g^{t+1}_i - h^{t+1}_i}^2}\\
    &\quad  + \frac{2\gamma}{b} \Exp{\norm{h^{t+1} - \nabla f(x^{t+1})}^2} + \frac{8 \gamma \left(\probavailable - \probpairaa\right)}{n \probavailable^2 \probmega} \Exp{\frac{1}{n}\sum_{i=1}^n\norm{h^{t+1}_i - \nabla f_i(x^{t+1})}^2}\\
    &\leq \Exp{f(x^t)} - \frac{\gamma}{2}\Exp{\norm{\nabla f(x^t)}^2} \\
    &\quad + \left(1 - \frac{a}{2}\right)\frac{2\gamma(2\omega + 1)}{\probavailable} \Exp{\norm{g^{t} - h^{t}}^2} + \left(1 - \frac{a}{2}\right)\frac{2\gamma\left(\left(2 \omega + 1\right)\probavailable - \probpairaa\right)}{n \probavailable^2}\Exp{\frac{1}{n}\sum_{i=1}^n\norm{g^{t}_i - h^{t}_i}^2} \\
    &\quad + \left(1 - \frac{b}{2}\right)\frac{2\gamma}{b} \Exp{\norm{h^{t} - \nabla f(x^{t})}^2} + \left(1 - \frac{b}{2}\right)\frac{8 \gamma \left(\probavailable - \probpairaa\right)}{n \probavailable^2 \probmega}\Exp{\frac{1}{n}\sum_{i=1}^n\norm{h^{t}_i - \nabla f_i(x^{t})}^2} \\
    &\quad + \frac{20 \gamma \sigma^2}{n B'}.
  \end{align*}
  Due to $\gamma \leq \frac{a}{2\mu}$ and $\gamma \leq \frac{b}{2\mu},$ we have
  \begin{align*}
    &\Exp{f(x^{t + 1})} + \frac{2\gamma(2\omega + 1)}{\probavailable} \Exp{\norm{g^{t+1} - h^{t+1}}^2} + \frac{2\gamma\left(\left(2 \omega + 1\right)\probavailable - \probpairaa\right)}{n \probavailable^2} \Exp{\frac{1}{n}\sum_{i=1}^n\norm{g^{t+1}_i - h^{t+1}_i}^2}\\
    &\quad  + \frac{2\gamma}{b} \Exp{\norm{h^{t+1} - \nabla f(x^{t+1})}^2} + \frac{8 \gamma \left(\probavailable - \probpairaa\right)}{n \probavailable^2 \probmega} \Exp{\frac{1}{n}\sum_{i=1}^n\norm{h^{t+1}_i - \nabla f_i(x^{t+1})}^2}\\
    &\leq \Exp{f(x^t)} - \frac{\gamma}{2}\Exp{\norm{\nabla f(x^t)}^2} \\
    &\quad + \left(1 - \gamma \mu\right)\frac{2\gamma(2\omega + 1)}{\probavailable} \Exp{\norm{g^{t} - h^{t}}^2} + \left(1 - \gamma \mu\right)\frac{2\gamma\left(\left(2 \omega + 1\right)\probavailable - \probpairaa\right)}{n \probavailable^2}\Exp{\frac{1}{n}\sum_{i=1}^n\norm{g^{t}_i - h^{t}_i}^2} \\
    &\quad + \left(1 - \gamma \mu\right)\frac{2\gamma}{b} \Exp{\norm{h^{t} - \nabla f(x^{t})}^2} + \left(1 - \gamma \mu\right)\frac{8 \gamma \left(\probavailable - \probpairaa\right)}{n \probavailable^2 \probmega}\Exp{\frac{1}{n}\sum_{i=1}^n\norm{h^{t}_i - \nabla f_i(x^{t})}^2} \\
    &\quad + \frac{20 \gamma \sigma^2}{n B'}.
  \end{align*}
  It is left to apply Lemma~\ref{lemma:good_recursion_pl} with 
    \begin{eqnarray*}
      \Psi^t &=& \frac{2(2\omega + 1)}{\probavailable} \Exp{\norm{g^{t} - h^{t}}^2} + \frac{2\left(\left(2 \omega + 1\right)\probavailable - \probpairaa\right)}{n \probavailable^2}\Exp{\frac{1}{n}\sum_{i=1}^n\norm{g^{t}_i - h^{t}_i}^2} \\
      &\quad +& \frac{2}{b} \Exp{\norm{h^{t} - \nabla f(x^{t})}^2} + \frac{8 \left(\probavailable - \probpairaa\right)}{n \probavailable^2 \probmega}\Exp{\frac{1}{n}\sum_{i=1}^n\norm{h^{t}_i - \nabla f_i(x^{t})}^2}
    \end{eqnarray*}
    and $C = \frac{20 \sigma^2}{n B'}$
    to conclude the proof.
\end{proof}

\end{document}